\documentclass{article}

\PassOptionsToPackage{authoryear,round}{natbib}
\ifdefined\neuripscameraready
  \usepackage[main, final]{neurips_2026}
\else
  \usepackage[main, preprint]{neurips_2026}
  \makeatletter\renewcommand{\@noticestring}{\@neuripsordinal\ Conference on Neural Information Processing Systems (NeurIPS \@neuripsyear).}\makeatother
\fi
\usepackage[utf8]{inputenc}
\usepackage[T1]{fontenc}
\usepackage{hyperref}
\usepackage{url}
\usepackage{booktabs}
\usepackage{microtype}
\usepackage{xspace}
\usepackage[table]{xcolor}

\usepackage{graphicx}
\usepackage{etoolbox}
\usepackage{amsmath}
\usepackage{mathtools}
\usepackage{amssymb}
\usepackage{amsthm}
\usepackage{multirow}
\usepackage{longtable}
\usepackage{subcaption}
\usepackage{float}
\usepackage{algorithm}
\usepackage{algpseudocode}
\usepackage{tikz}
\usepackage{enumitem}
\usepackage{silence}
\newif\ifjournal
\journalfalse

\definecolor{darkblue}{rgb}{0, 0, 0.5}
\hypersetup{
    colorlinks=true,
    citecolor=darkblue,
    linkcolor=darkblue,
    urlcolor=darkblue,
    hypertexnames=false,
    pdftitle={Tokens-per-Parameter Coverage Is Critical for Robust LLM Scaling Law Extrapolation},
    pdfauthor={Joshua Shay Kricheli, Alexander Lawrence Reid, Soumajyoti Sarkar, Venkata Gandikota, Paulo Shakarian},
}
\newtheorem{remark}{Remark}
\newtheorem{proposition}{Proposition}
\newtheorem*{proposition*}{Proposition}
\newtheorem{corollary}{Corollary}
\newtheorem*{corollary*}{Corollary}
\newtheorem{theorem}{Theorem}
\newtheorem{lemma}{Lemma}
\newtheorem*{theorem*}{Theorem}
\newtheorem*{lemma*}{Lemma}
\theoremstyle{definition}
\newtheorem{definition}{Definition}

\newcommand{\win}[1]{\cellcolor{black}\textcolor{white}{\textbf{#1}}}
\newcommand{\winrate}{$97.3\%$\xspace}
\newcommand{\winrateNumer}{$1{,}460$\xspace}
\newcommand{\winrateLosses}{$40$\xspace}
\newcommand{\winrateDenom}{$1{,}500$\xspace}
\newcommand{\winrateCI}{$[96.4\%, 98.0\%]$\xspace}
\newcommand{\winratebf}{$98.7\%$\xspace}
\newcommand{\winratebfNumer}{$296$\xspace}
\newcommand{\winratebfCI}{$[96.6\%, 99.5\%]$\xspace}
\newcommand{\winratebfDenom}{$300$\xspace}
\newcommand{\winrateFP}{$99.7\%$\xspace}

\newcommand{\Nsizes}{14}
\newcommand{\COTPPs}{12}
\newcommand{\NCDsizes}{12}
\newcommand{\COtrainN}{120}
\newcommand{\NCtrainN}{120}
\newcommand{\COholdN}{42}
\newcommand{\NCholdN}{42}
\newcommand{\runsperDset}{324}

\newcommand{\RhalfLow}{0.1}
\newcommand{\RhalfHigh}{50.0}

\title{Tokens-per-Parameter Coverage Is Critical for Robust LLM Scaling Law Extrapolation}

\author{%
  Joshua Shay Kricheli\setcounter{footnote}{5}\thanks{Equal contribution.}\hphantom{$^{\|}$}, Alexander Lawrence Reid\footnotemark[6] \\
  Syracuse University \\
  \texttt{\{jkrichel,areid04\}@syr.edu} \\
  \And
  Soumajyoti Sarkar\setcounter{footnote}{0}\thanks{This work does not relate to the author's position at Amazon.} \\
  Amazon AGI Foundations \\
  \AND
  Venkata Gandikota \& Paulo Shakarian \\
  Syracuse University \\
  \texttt{\{vsgandik,pashakar\}@syr.edu} \\
}

\begin{document}

\maketitle
\begin{abstract}
Neural scaling laws approximate a language model's loss as a power-law function of parameter count $N$ and token count $D$. Following Chinchilla-style compute-optimal training, many studies fit scaling laws from runs at one or a few fixed tokens-per-parameter (TPP) ratios $k$, with $D = kN$. We show that this collinear design, combined with the empirically common near-equality of the exponents governing $N$ and $D$, induces an inherent ill-conditioning in the Gauss-Newton least-squares problem: the condition number of the design grows as the inverse square of the gap between the $N$- and $D$-exponents. The scale coefficients become practically unidentifiable, with confidence intervals inflating by orders of magnitude, yielding a ``sloppy'' model whose extrapolations degrade sharply off the training ray. We prove this for four scaling-law formalisms and derive a closed-form TPP-diversity threshold that characterizes well-conditioned estimation to leading order. Empirically, in controlled experiments with 5M-77M-parameter models, non-collinear designs outperform collinear ones on held-out data with a \winrate{} win rate across four laws and five corpora, and \winratebf{} in bfloat16. We further show the degeneracy is rooted in Jacobian geometry and is not an artifact of the loss function: any smooth objective whose Gauss-Newton weights are bounded above and away from zero inherits the same ill-conditioning.
\end{abstract}

\section{Introduction}

Scaling laws have become a cornerstone of large
language model (LLM) research, providing a predictive
framework for estimating model performance as a
function of compute, dataset size, and parameter
count~\citep{hestness2017deep,
rosenfeld2019constructive, kaplan2020scaling, alabdulmohsin2022revisiting,
caballero2023broken, bahri2024explaining}. The seminal Chinchilla study~\citep{hoffmann2022training} 
proposed that model size~$N$ and training tokens~$D$ scale
linearly under a fixed compute budget with $D \approx 20N$. A recent survey of 50+ studies
finds the fitted law shifts substantially under
changes to the included $D/N$ range, with the optimal
ratio a recurring source of
disagreement~\citep{li2025misfitting}.
These works established the fixed tokens-per-parameter (TPP)
heuristic in the compute-optimal setting.\looseness=-1

The same fixed-TPP grids are now widely used to fit scaling laws as well.
Cerebras-GPT~\citep{dey2023cerebrasgptopencomputeoptimallanguage} 
trains 111M-13B models at $D=20N$ and fits only a
loss-versus-compute law; \citet{gadre2024language} parameterize
over-trained scaling along constant-$M$ lines, each slice
collinear, though their pooled grid passes our threshold
with equal weight per ratio (Appendix~\ref{app:published_designs}); the OLMo ladder fits Chinchilla at
fixed Chinchilla multiples~\citep{bhagia2025establishingtaskscalinglaws};
and DataDecide reports that eight scaling-law variants
fitted at $D = 100N$ fail to beat a na\"ive
single-scale
baseline~\citep{magnusson2025datadecidepredictbestpretraining}.
The practical stakes are high:
\citet{choshen2024hitchhiker} fit over $1{,}000$ laws
to $485$ published models and find estimates sensitive
to the choice and coverage of training configurations.
Yet a fixed TPP ratio prescribes how to train one model, not how to
estimate a scaling law. When all training runs lie on a single ray $D = kN$
in the $(N,D)$ plane, the predictor variables become
collinear, and combined with the empirically common
near-equality of the $N$ and $D$ exponents this
renders the scaling-law parameters
\emph{practically unidentifiable}.
We prove this for four scaling-law formalisms -
Chinchilla, unsaturated repeated-data~\citep{muennighoff2023scaling},
Kaplan~\citep{kaplan2020scaling}, and
Droppo-Elibol~\citep{droppo2021scaling} - via
Gauss-Newton analysis.  Because the degeneracy is in
the Jacobian geometry rather than in the loss, any
smooth objective whose Gauss-Newton weights are bounded above and away from zero inherits the same
ill-conditioning (Appendix~\ref{app:proof_full_cond}).  The law becomes a ``sloppy'' model
(Section~\ref{sec:related_work}): a
near-continuum of parameters fits training data
indistinguishably yet diverges off the training ray.
This directly corrupts isoFLOP prediction:
isoFLOP curves cut across the training ray and
thereby probe the sloppy direction that collinear
fitting leaves unconstrained.  Constructively, we give an identified
reduced model recovering the combined effect along a
single ray
(Definition~\ref{def:reduced_chinchilla_model}), and
show empirically that non-collinear designs beat
collinear ones when evaluated on the same held-out models
(Section~\ref{sec:empirical}).

\noindent\textbf{Key contributions.}\quad
(i)~a closed-form TPP-diversity threshold, checkable before training, that
characterizes well-conditioned scale terms to leading order (Proposition~\ref{prop:holdout_r2});
(ii)~a Gauss-Newton analysis of one ray: it is
inherently ill-conditioned when the
exponents governing model size and data size are
close, with the scale-pair condition number growing as the
inverse square of the exponent gap (the full one at least as fast) - established for four scaling-law formalisms, with rates that do not depend on model size and a corresponding
confidence-interval inflation result
(Proposition~\ref{prop:full_cond},
Corollary~\ref{prop:ci_inflation});
(iii)~a holdout-prediction ordering separating
one-ray from non-collinear designs
(Theorem~\ref{thm:holdout_regimes}; Lemma~\ref{lem:regime_a_margin}:
a margin for $K \geq 2$); beyond both, non-collinear (NC) designs also beat our 12-ray collinear (CO) fan
in \winrate{} of comparisons in controlled experiments at 5M-77M parameters;
(iv)~an identified reduced parametrization that
recovers what a single-ray design can
constrain
(Definition~\ref{def:reduced_chinchilla_model});
and
(v)~a fully reproducible experimental
suite - ${\sim}1{,}900$ trained LLMs, all
checkpoints, training metrics, and training and analysis
code - all publicly released.\footnote{Code:
\url{https://github.com/lab-v2/tpp-coverage}.\quad
Checkpoints and metrics:
\url{https://huggingface.co/datasets/leibnitz-lab/colinear_scaling_models}.}

\noindent\textbf{Paper road-map.} The rest of this paper is as follows.
Section~\ref{sec:formal} develops the Gauss-Newton
framework, derives the TPP-diversity threshold, proves
condition-number bounds for four formalisms on one ray and
confidence-interval inflation, and defines induced isoFLOP curves; proofs are in
Appendix~\ref{app:proofs}.
Section~\ref{sec:empirical} validates the analysis
with controlled experiments across five corpora, four
laws, and two floating-point precisions.
Section~\ref{sec:discussion} synthesizes the empirical
patterns, cautions against interpreting separate
data versus model contributions from single-ray fits, and states
the limits of the evidence: model scale, one architecture, English text,
one training recipe, and pre-training loss.

\section{Related work}
\label{sec:related_work}

Predictable power-law scaling of generalization error
was first reported by~\citet{hestness2017deep} and
formalized for autoregressive language models
by~\citet{kaplan2020scaling} using bivariate fits over
grids crossing model sizes with dataset subsets.
\citet{bahri2024explaining} linked the observed
exponents to data-manifold structure via variance- and
resolution-limited regimes. Follow-ups to Chinchilla examined data
availability~\citep{villalobos2022will},
overtraining~\citep{touvron2023llama}, and unbounded
compute with fixed
data~\citep{kim2025pretraining}; we extend our analysis
to repeated-data regimes~\citep{muennighoff2023scaling}
in Appendix~\ref{app:dataconstrained_details}.
\citet{sardana2024beyond} sweep ratios up to $10{,}000$
TPP, producing inherently non-collinear designs that
our diagnosis predicts should be
well-identified (Section~\ref{sec:formal}) - yet
several widely used ladders were fit in the
ill-conditioned regime we formalize.
\citet{besiroglu2024chinchilla} found Chinchilla's
reported parametric estimates inconsistent with its
other two estimation methods, and
\citet{porian2024resolving} attribute the
Kaplan-Chinchilla discrepancy to minor recipe
differences (last-layer cost, warm-up, scale-dependent
tuning); on a single ray, our result explains why
such perturbations can move coefficients far.
\citet{bergsma2025powerlines} show even optimizer
hyperparameters (weight decay, batch size) obey power
laws in the TPP ratio $D/N$, so fixing a single ray
couples recipe with scale.
Concurrently, \citet{volkova2026robust} report that
separate Chinchilla-style fits per optimizer are
ill-conditioned, with correlated pairs $(A, \alpha)$ and
$(B, \beta)$ that they attribute to the law form and the
narrow $N$ and $D$ ranges; their ratios
$\lbrace 30, 50, 100, 200 \rbrace$ also fail
Eq.~\eqref{eqn:tpp_diversity}
(Appendix~\ref{app:published_designs}).
\citet{schaeffer2025pretraining} propose sidestepping
multicollinearity via compute-envelope or
gold-reference parameterizations - both abandon the
$(N,D)$ decomposition.
\citet{zhang2026prescriptive} pursue a related goal at
the downstream-evaluation level, estimating
compute-to-capability boundaries via quantile
regression on post-trained models.
Our work is complementary: we keep the $(N,D)$
decomposition and give designs that \emph{restore} its
identifiability (Proposition~\ref{prop:holdout_r2}). The Relative-Based Scaling Law of \citet{yue2025relative}
is largely a single-axis sweep in model size, which
our analysis flags as insufficient for multivariate
identifiability.  Farseer~\citep{li2025farseer} adds
an $N$-$D$ coupling term and trains
$\sim\!1{,}000$ models on a diverse $(N,D)$
grid, reporting a $433\%$ reduction in
extrapolation error over Chinchilla; even such
interaction terms remain ill-conditioned under
$D = kN$ (Appendix~\ref{app:interaction}).
\citet{shukor2025datamixtures} similarly extend
scaling laws to data-mixture weights, enlarging the
multivariate design space where collinearity matters.
\citet{hu2026neuneu} sidestep parametric fitting
altogether by training a neural extrapolator over
checkpoint trajectories - orthogonal to our diagnosis
of \emph{why} the parametric fit is fragile.
Statistically, the pathology we analyze is the
nonlinear-regression analog of
multicollinearity~\citep{kim2019multicollinearity,
vatcheva2016multicollinearity, dormann2013collinearity,
montgomery2012introduction, obrien2007caution} and of
the ``sloppy'' model
phenomenon~\citep{transtrum2010nonlinear,
transtrum2015perspective}, in which narrow model-manifold
cross-sections yield fragile estimates - here induced
by fixed-TPP sampling. Like work that detects or reduces
the errors of learned models~\citep{kricheli2024error,
leiva2026consistency, taparia2026vlc}, the diagnosis flags
when a fitted law's off-ray predictions cannot be trusted.

Eq.~\eqref{eqn:tpp_diversity} screens a design from its
training ratios and one exponent. Of 206 candidates, seeded
by~\citet{li2025misfitting}, 24 fit a bivariate law on a
recoverable design; at target condition number 100, 11 pass and 13
fail, 3 at one ratio (Table~\ref{tab:published_designs};
Appendix~\ref{app:published_designs}). The test is proved for
collinear designs; all 6 collinear ones fail, and no pass fixes the coefficients. Most failures,
like our CO fan at $\beta_{\mathrm{eff}} = 0.28$ ($\kappa \approx 250$), miss
the target by under fourfold: a weakly determined $N$-$D$ split, not
an unusable fit.

\noindent\textbf{Relative to classical identifiability
analyses.} The pathology is classical; this paper adds:
\begin{enumerate}[label=(\roman*), nosep, leftmargin=*]
\item a rate $\kappa_{A,B} = \Theta(\varepsilon^{-2})$ in a measured
  gap (Proposition~\ref{prop:full_cond}, Eq.~\eqref{eqn:chin_exact_det});
\item a threshold checkable before training from one exponent
  prior that guarantees conditioning to leading order
  (Proposition~\ref{prop:holdout_r2}), where variance inflation factors need a fit
  and locally optimal design all parameters~\citep{atkinson2007optimum};
\item severity by law: Chinchilla's full $\kappa(J^T J)$ is
  ${>}10^{5}\times$ Kaplan's (App.~\ref{app:kaplan_details}), saturated
  repeated data likely escapes it, and an interaction term adds
  another sloppy direction
  (App.~\ref{app:dataconstrained_details},
  \ref{app:interaction});
\item what one ray identifies (Proposition~\ref{prop:reduced_id});
\item a holdout-error ordering (Theorem~\ref{thm:holdout_regimes}).
\end{enumerate}

\begin{table}[t]
\centering
\caption{Six of the 24 published designs (all:
Table~\ref{tab:published_designs_full}). $k_\ell = D/N$: the training ratios;
$K$: their number; $R = k_K/k_1$; $V_K/\tau_K$: margin over the threshold of
Proposition~\ref{prop:holdout_r2} at $\kappa_{\mathrm{target}} = 100$
(${\geq}\,1$ passes). \textsuperscript{*}~No exponent; $\beta_{\mathrm{eff}} = 0.28$ used.}
\label{tab:published_designs}
\footnotesize\renewcommand{\arraystretch}{0.9}
\begin{tabular}{@{}lrlrr@{}}
\toprule
Design & $K$ & $k_\ell$ & $R$ & $V_K/\tau_K$ \\
\midrule
\citet{hoffmann2022training} & 151 & 0.0361--341 & $9.5{\times}10^{3}$ & 5.55 \\
\citet{muennighoff2023scaling} & 95 & 0.0912--$4.26{\times}10^{4}$ & $4.7{\times}10^{5}$ & 5.76 \\
\citet{choshen2024hitchhiker}\textsuperscript{*} & 243 & 0.874--4284 & $4.9{\times}10^{3}$ & 3.74 \\
\citet{hu2024minicpm} & 6 & 10--60 & 6 & 0.30 \\
\citet{bhagia2025establishingtaskscalinglaws} & 4 & 20, 40, 100, 200 & 10 & 0.47 \\
\citet{magnusson2025datadecidepredictbestpretraining}\textsuperscript{*} & 1 & 100 & 1 & 0.00 \\
\bottomrule
\end{tabular}

\end{table}

\section{Formal analysis}
\label{sec:formal}

In plain terms: a collinear design trains at $K$ TPP ratios ($D = k_\ell N$).
On one ray the two power-law terms rise and fall together, so only
their exponent gap $\varepsilon$ (Definition~\ref{def:exponent_gap})
separates them; the spread of the ratios separates them at any gap.
Section~\ref{sec:gnsetup} gives the spread that keeps the scale
coefficients well conditioned; Section~\ref{sec:single_ray} treats one
ray ($K = 1$), where they are individually free but jointly constrained
(``sloppy'') and off-ray predictions inherit that uncertainty.
Table~\ref{tab:notation} collects the notation.

The standard Chinchilla scaling
law~\cite[Eq.~2]{hoffmann2022training} approximates the
loss as a sum of two power-law terms - capturing
limited model capacity and data size - plus an
irreducible loss $E$, a form observed across deep
learning domains~\cite[Sec.~3]{hestness2017deep, kaplan2020scaling}:
\begin{equation}
    L(N, D) = AN^{-\alpha} + BD^{-\beta} + E,
    \label{eqn:chinchilla}
\end{equation}
where $A, B, \alpha, \beta, E \in \mathbb{R}_{>0}$ and
$N, D \in \mathbb{R}_{>0}$.

\begin{definition}[Scaling-law experimental dataset]
\label{def:experimental_dataset}
\label{def:holdout_splits}
Let $\mathcal{A}$ be a training scheme (architecture,
optimizer, algorithm, etc.). Each experiment records
loss~$L_i$ for a model of size~$N_i$ on $D_i$
tokens, yielding
$\mathcal{D} = \{(N_i, D_i, L_i)\}_{i=1}^{m}$.
Let $N_1 < \cdots < N_n \in \mathbb{N}_{>0}$ be the
model sizes and
$k_1 < \cdots < k_K \in \mathbb{R}_{>0}$ the TPP ratios.
A \emph{collinear} design pairs each $(N_i, k_\ell)$
as $(N_i,\, k_\ell N_i,\, L_{i\ell})$, so
$m = nK$; the special case $K = 1$ is fully
collinear.  A parametric law
$\hat{L}(N, D; \boldsymbol{\theta})$ with
$\boldsymbol{\theta} \in \mathbb{R}^p$ is fit to
$\mathcal{D}$ by minimizing the discrepancy
between predicted and observed losses; observations
are split into a training set
$\mathcal{D}_{\mathrm{train}}$ and a holdout
$\mathcal{H}$ of $n_{\mathcal{H}}$ points
(possibly at TPP ratios off the training manifold).
\end{definition}

\begin{definition}[Exponent gap by scaling law]
\label{def:exponent_gap}
Define $\varepsilon$ as a \emph{law-specific exponent
gap} in each law's native notation (see Table~\ref{tab:jacobian_summary}):
\textbf{Chinchilla} (and repeated-data on effective
$(N', D')$): $\varepsilon \coloneqq |\alpha - \beta|$;
\textbf{Kaplan} (exponents $\alpha_N, \alpha_D$):
$\varepsilon \coloneqq |1 - \alpha_N/\alpha_D|$;
\textbf{Droppo-Elibol} (exponents
$\alpha_N, \alpha_D, \alpha$ with
$\gamma_N \coloneqq \alpha_N/\alpha,\;
\gamma_D \coloneqq \alpha_D/\alpha$, $\alpha > 0$):
$\varepsilon \coloneqq |\gamma_N - \gamma_D|$.
For example, Chinchilla reports $\alpha \approx 0.34$
and $\beta \approx 0.28$; the single-variable fits of \citet{kaplan2020scaling} give
$\alpha_N \approx 0.076$ and $\alpha_D \approx 0.095$.
\end{definition}

\begin{table*}[t]
\small
\centering
\setlength{\tabcolsep}{2pt}
\renewcommand{\arraystretch}{1.12}
\caption{Jacobian columns at the $i$-th data point under
$D = kN$ (with $i$ subscripts dropped for brevity).
Rows correspond, in order, to Chinchilla
\citep{hoffmann2022training}, Repeated-data
\citep{muennighoff2023scaling}, Kaplan
\citep{kaplan2020scaling}, and Droppo-Elibol
\citep{droppo2021scaling}.
For each law, $\mathbf{j}_2$ is a scalar multiple of
$\mathbf{j}_1$ times $N^{\pm\varepsilon}$ ($N^{\tilde\varepsilon}$, $\tilde\varepsilon = \alpha - \beta$; $s = \alpha_N/\alpha_D$), exposing their
shared collinearity.}
\label{tab:jacobian_summary}
\begin{tabular}{@{}llllc@{}}
\toprule
\textbf{Law} ($\hat{L}$) &
\textbf{Collinear pair} &
$\mathbf{j}_1$ &
$\mathbf{j}_2$ &
$\varepsilon$ \\
\midrule
$A N^{-\alpha} + B D^{-\beta} + E$ &
$(A, B)$ &
$N^{-\alpha}$ &
$k^{-\beta} N^{-\alpha} N^{\tilde\varepsilon}$ &
$|\alpha - \beta|$ \\[3pt]

$A {N'}^{-\alpha} + B {D'}^{-\beta} + E$ &
$(A, B)$ &
$N'^{-\alpha}$ &
$k^{-\beta}(D'/kN')^{-\beta} N'^{-\alpha} N'^{\tilde\varepsilon}$ &
$|\alpha - \beta|$ \\[3pt]

$\bigl[(N_c/N)^{s} + D_c/D\bigr]^{\alpha_D}$ &
$(N_c, D_c)$ &
$\propto N^{-s}$ &
$\propto N^{-1}$ &
$|1 - s|$ \\[3pt]

$[(N_C/N)^{\gamma_N} + (D_C/D)^{\gamma_D} + E^{1/\alpha}]^{\alpha}$ &
$(N_C, D_C)$ &
$\propto N^{-\gamma_N}$ &
$\propto N^{-\gamma_N} N^{\varepsilon}$ &
$|\gamma_N - \gamma_D|$ \\
\bottomrule
\end{tabular}
\end{table*}

\subsection{Gauss-Newton (GN) framework and multi-ray designs}
\label{sec:gnsetup}
\label{sec:multi_ray}
We fit by nonlinear least squares: minimize
$S(\boldsymbol{\theta}) \coloneqq
\tfrac{1}{2}\|\mathbf{r}(\boldsymbol{\theta})\|^2$ over
$\boldsymbol{\theta} \in \mathbb{R}^p$, where the
residuals
$r_i \coloneqq L_i - \hat{L}(N_i, D_i; \boldsymbol{\theta})$
stack into $\mathbf{r}(\boldsymbol{\theta})$, with
\emph{Jacobian} $J \in \mathbb{R}^{m \times p}$ given by
$J_{ij} = \partial r_i/\partial \theta_j$.  We write
$\kappa(M) \coloneqq
\lambda_{\max}(M)/\lambda_{\min}(M)$ for the condition
number of a symmetric positive-definite matrix $M$.
GN linearizes $\mathbf{r}$, giving the
normal equations
$(J^T J)\,\Delta\boldsymbol{\theta} = -J^T \mathbf{r}$
(Appendix~\ref{app:gauss_newton_derivation}).  For each
law's scale coefficients $(A,B)$, let
$\kappa_{A,B} \coloneqq \kappa\!\bigl((J^T J)_{A,B}\bigr)$
on the $(A,B)$ principal submatrix of $J^T J$.
On collinear designs with $K \geq 2$ rays, the ratio spread can restore
a well-conditioned scale pair:

\begin{proposition}[TPP diversity and scale-pair conditioning]
\label{prop:holdout_r2}
For the Chinchilla and unsaturated repeated-data laws (Kaplan,
Droppo-Elibol: a heuristic, Appendix~\ref{app:proof_prop_holdout_r2}),
let $J$ be the Jacobian at the true parameters (well-specified law),
$\kappa_{\mathrm{target}} \geq 1$, $\beta_{\mathrm{eff}}$ the law-specific
data-size exponent (Table~\ref{tab:law_substitutions}), and a collinear
design (Definition~\ref{def:holdout_splits}) with
ratios $k_1 < \cdots < k_K$ ($K \geq 1$). Then:
\begin{enumerate}[label=(\roman*), nosep, leftmargin=*]
\item to leading order (the limit $\varepsilon \to 0$ at a fixed design),
  $\kappa_{A,B} \leq \kappa_{\mathrm{target}}$ whenever
\begin{equation}
    \underbrace{
        \frac{1}{K}\sum_{\ell=1}^{K}
              k_\ell^{-2\beta_{\mathrm{eff}}}
        \;-\;
        \left(\frac{1}{K}\sum_{\ell=1}^{K}
              k_\ell^{-\beta_{\mathrm{eff}}}
        \right)^{\!2}
    }_{V_K}
    \;\;\geq\;\;
    \underbrace{
        \frac{\displaystyle\Bigl(
                K + \sum_{\ell=1}^{K}
                k_\ell^{-2\beta_{\mathrm{eff}}}
              \Bigr)^{\!2}}
             {K^2\;\kappa_{\mathrm{target}}}
    }_{\tau_K};
    \label{eqn:tpp_diversity}
\end{equation}
\item to the same order, $\kappa_{A,B} \leq \kappa_{\mathrm{target}}$ iff
  $V_K \geq \tau_K/(1 + 1/\kappa_{\mathrm{target}})^2$;
\item for Chinchilla, (i) and (ii) hold at every $\varepsilon$ when $V_K$ and
  $\tau_K$ are replaced by the $N^{-2\alpha}$-weighted $V_{\mathrm{w}}$ and
  $\tau_{\mathrm{w}}$ of Lemma~\ref{lem:weighted_spread}, which tend to $V_K$
  and $\tau_K$ as $\varepsilon \to 0$.
\end{enumerate}
\end{proposition}
\noindent\textit{Proof.} Steps~1-4 of
Appendix~\ref{app:proof_prop_holdout_r2}.
Here $V_K$ is the variance of the
$k_\ell^{-\beta_{\mathrm{eff}}}$; like $\tau_K$, it depends on
the design only through its ratios, which alone decide whether the design
passes~\eqref{eqn:tpp_diversity}:
\begin{corollary}[Number and spread of TPP ratios]
\label{cor:ratio_cases}
In the notation of Proposition~\ref{prop:holdout_r2}:
\begin{enumerate}[label=(\roman*), nosep, leftmargin=*]
\item for $K = 1$, $V_1 = 0 < \tau_1$, so~\eqref{eqn:tpp_diversity}
  never holds;
\item for $K = 2$, \eqref{eqn:tpp_diversity} holds iff
  $R \coloneqq k_2/k_1 \geq R_{\min}(k_1)$
  of~\eqref{eqn:R_min_closed}, which increases with
  $k_1 \geq 1$ and is infinite once
  $\sqrt{\kappa_{\mathrm{target}}}\,k_1^{-\beta_{\mathrm{eff}}}
  \leq 2 + k_1^{-2\beta_{\mathrm{eff}}}$;
\item for $K > 2$, interior rays make $V_K$ smaller than the endpoint
  pair's $V_2$, and no $K$-ray design exceeds the
  $V_K/\tau_K$ of its endpoints by more than the factor
  $(1 + k_1^{-2\beta_{\mathrm{eff}}}/2)^2$
  of~\eqref{eqn:VK_tauK_bound}.
\end{enumerate}
\end{corollary}
\noindent\textit{Proof.} Step~5 of
Appendix~\ref{app:proof_prop_holdout_r2}; Appendix~\ref{app:tpp_recipe} builds a design
recipe on it.

\emph{Before} training, with a literature estimate
of $\beta_{\mathrm{eff}}$ (e.g.\ $\beta \approx 0.28$ for Chinchilla),
\eqref{eqn:tpp_diversity} is a screen: the exact $\kappa_{A,B}$ is $0.06$-$3.4$ times
its leading-order value over $432$ two-ratio designs, so the planned design's exact value
decides (Appendix~\ref{app:proof_prop_holdout_r2}). For $K \geq 2$ the same
block orders holdout errors: NC predicts better than CO in expectation
when $\mathrm{tr}(G_{A,B}^{-1}M)$, with $G_{A,B} \coloneqq (J^T J)_{A,B}$ and $M$ the holdout's second moment of the
pair columns, exceeds NC's predictive variance
(Lemma~\ref{lem:regime_a_margin}).

\subsection{The one-ray limit ($K = 1$)}
\label{sec:single_ray}

On a single ray ($K = 1$) - a common practical
case ($k \approx 20$ for
Chinchilla~\citep{hoffmann2022training, villalobos2022will}) -
substituting $D = kN$ into~\eqref{eqn:chinchilla} gives
$L(N, kN) = AN^{-\alpha} + Bk^{-\beta}N^{-\beta} + E$.
Setting $\beta = \alpha$ collapses the two power-law
terms, yielding a \emph{reduced model}:
\begin{definition}[Reduced Chinchilla model under fixed TPP]
\label{def:reduced_chinchilla_model}
\begin{equation}
    L(N;\,\psi,\alpha,E)
        \;=\; \left( A + Bk^{-\alpha}
        \right) N^{-\alpha} + E
        \;=\; \psi N^{-\alpha} + E,
    \label{eqn:reduced_chinchilla}
\end{equation}
where $\psi \coloneqq A + Bk^{-\alpha}$ combines
the scale pair
of~\eqref{eqn:chinchilla} and $k \in \mathbb{R}_{>0}$
is the ray's fixed TPP ratio.
\end{definition}

With free exponents one ray has a second near-null direction:
numerically, $\kappa(J^T J)$ grows as $\varepsilon^{-6}$ for Chinchilla
(Appendix~\ref{app:proof_ci_inflation}); rates below are for
$\kappa_{A,B}$, bounds for $\kappa(J^T J)$.

\begin{proposition}[Full-matrix conditioning]
\label{prop:full_cond}
Let $J \in \mathbb{R}^{m \times p}$ have columns
$\mathbf{j}_1, \dots, \mathbf{j}_p$ of norm $O(1)$, and suppose
two columns $a \neq b$ satisfy
$\mathbf{j}_b = c\,\mathbf{j}_a + \boldsymbol{\delta}$
with a fixed $c \neq 0$, $\|\mathbf{j}_a\| = \Theta(1)$ and
$\|\boldsymbol{\delta}\| = O(\varepsilon)$, where
$\varepsilon$ is the law-specific exponent gap
(Definition~\ref{def:exponent_gap}).  Let
$\lambda_{\max}(J^T J)$ and $\lambda_{\min}(J^T J)$
denote the largest and smallest eigenvalues.  Then as $\varepsilon \to 0$,
$\lambda_{\max}(J^T J) = \Theta(1)$,
$\lambda_{\min}(J^T J) = O(\varepsilon^2)$, and
\begin{equation}
    \kappa(J^T J) = \Omega(\varepsilon^{-2});
    \label{eqn:rayleigh_bound}
\end{equation}
if the part of $\boldsymbol{\delta}$ orthogonal to
$\mathbf{j}_a$ has norm $\Theta(\varepsilon)$,
$\kappa((J^T J)_{a,b}) = \Theta(\varepsilon^{-2})$.
\end{proposition}
\noindent\textit{Proof.} Appendix~\ref{app:proof_full_cond}.
The lower bound applies to all four laws of
Table~\ref{tab:jacobian_summary} (repeated-data unsaturated): on a
ray $D = kN$ each scale pair is near-proportional
(Appendices~\ref{app:chinchilla_details}-\ref{app:elibol_details}).

On a single ray, the reduced
model~\eqref{eqn:reduced_chinchilla} is identified
(Proposition~\ref{prop:reduced_id}) while the full
model is ill-conditioned
(Proposition~\ref{prop:full_cond}).  The least-squares estimator's
uncertainty comes from inverting $J^T J$
\citep[\S2.8.3, eq.~(2.8.7), p.~118]{bjorck1996numerical}:
when the data cannot distinguish $A$ from $B$, only their
combination $\psi$ is pinned down, leaving each coefficient
individually highly uncertain.  The confidence
interval on $A$ therefore widens by a factor
$\Omega(\varepsilon^{-1})$ relative to the one on $\psi$:

\begin{corollary}[Confidence interval inflation]
    \label{prop:ci_inflation}
    Under i.i.d.\ Gaussian noise of variance $\sigma^2$,
    on one ray $D = kN$ with $n \geq 5$ sizes, the nominal $95\%$ least-squares
    confidence intervals $\mathrm{CI}_{0.95}$ of $A$ in the full Chinchilla
    model~\eqref{eqn:chinchilla} and of the combined
    coefficient $\psi$ in the reduced
    model~\eqref{eqn:reduced_chinchilla} satisfy, as
    $\varepsilon \to 0$,
    \begin{equation}
        \frac{\mathrm{CI}_{0.95}(A)}{\mathrm{CI}_{0.95}(\psi)}
            \;=\; \Omega\!\left(\varepsilon^{-1}\right).
        \label{eqn:ci_ratio}
    \end{equation}
\end{corollary}
\noindent\textit{Proof.} Appendix~\ref{app:proof_ci_inflation}.

The asymptotic factors for three of the
four laws are collected in Table~\ref{tab:severity_comparison}.

\begin{table*}[t]
\small
\centering
\setlength{\tabcolsep}{5pt}
\caption{Statistical identifiability of scale coefficients under
single-TPP designs, using representative published exponent values.
$\varepsilon^{-1}$, $\varepsilon^{-2}$: rates of
$\mathrm{CI}(A)/\mathrm{CI}(\psi) = \Omega(\varepsilon^{-1})$
(Corollary~\ref{prop:ci_inflation}; Kaplan, Droppo-Elibol, without a reduced model: the scale coefficient's own interval, by the Schur floor of Appendix~\ref{app:proof_ci_inflation};
the design constant is about $10^{-2}$, so the proven floor is $0.16$, while the computed five-parameter ratio grows, numerically, as $\varepsilon^{-3}$) and
$\kappa_{A,B} = \Theta(\varepsilon^{-2})$ (Appendix~\ref{app:chinchilla_details}).
Kaplan: the single-variable fits of \citet{kaplan2020scaling}. Droppo-Elibol: gaps of our fits.
Repeated-data inherits Chinchilla's exponents. On one ray at $k = 1$-$20$ and our $\Nsizes$ sizes,
$\kappa_{A,B}$ is $10^{3}$-$10^{4}$ for Chinchilla and, with unit-norm columns,
$1.6$-$1.7{\times}10^{2}$ for Kaplan (Appendices~\ref{app:chinchilla_details} and~\ref{app:kaplan_details}).}
\label{tab:severity_comparison}
\begin{tabular}{lcccc}
\toprule
\textbf{Law} &
\textbf{Exponents} &
$\varepsilon$ &
$\varepsilon^{-1}$ &
$\varepsilon^{-2}$ \\
\midrule
Chinchilla &
$\alpha {\approx}\, 0.34$, $\beta {\approx}\, 0.28$ &
$0.06$ &
$17$ &
$278$ \\[3pt]
Kaplan &
$\alpha_N {\approx}\, 0.076$, $\alpha_D {\approx}\, 0.095$ &
$0.20$ &
$5$ &
$25$ \\[3pt]
Droppo-Elibol &
fitted (Table~\ref{tab:droppo-params}) &
$0.18$-$0.80$ & $1.3$-$5.6$ & $1.6$-$32$ \\
\bottomrule
\end{tabular}
\end{table*}

\noindent\textbf{Scope of the rates.}
None of these bounds assumes that the scale pair is the only weak
direction of $J^T J$. The bound
$\kappa(J^T J) = \Omega(\varepsilon^{-2})$ needs only
near-proportional columns (Appendix~\ref{app:proof_full_cond}),
$\kappa(J^T J) \geq \kappa_{A,B}$ by Cauchy interlacing
(Appendix~\ref{app:proof_prop_holdout_r2}), and the
Schur-complement floor on the interval of $A$ holds for any
full Jacobian (Appendix~\ref{app:proof_ci_inflation}).

\begin{theorem}[Holdout prediction regimes]
\label{thm:holdout_regimes}
For any law of Table~\ref{tab:jacobian_summary} (repeated-data unsaturated) with i.i.d.\ Gaussian residuals of variance $\sigma^2$,
let CO denote a collinear design ($K \geq 1$), NC a
non-collinear design with dense two-dimensional $(N, D)$
coverage, both with full-rank Jacobians (NC's also at
$\varepsilon = 0$), and let $\mathcal{H}$ denote a holdout
with a nonzero fraction of off-ratio points
($k' = D'/N'$ differs from every training ratio).
For $\mathrm{X} \in \{\mathrm{CO},\mathrm{NC}\}$,
$\mathrm{RMSE}_{\mathcal{H}}^{\mathrm{X}}$ is the
root-mean-square error of $\hat{L}$ against $\mathbb{E}[L]$ on
$\mathcal{H}$ under that training design, and
$R^{2\,\mathrm{X}}_{\mathcal{H}}$ the corresponding
coefficient of determination on $\mathcal{H}$. The
regimes below hold at leading order in $\sigma$
(linearizing both fits), then as $\varepsilon \to 0$.

\noindent\textbf{Regime~A} \textup{(CO a single ray, $K = 1$,
so $V_1 = 0 < \tau_1$):}
\begin{equation}
    \mathbb{E}[\mathrm{RMSE}_{\mathcal{H}}^{\mathrm{NC}}]
        < \mathbb{E}[\mathrm{RMSE}_{\mathcal{H}}^{\mathrm{CO}}],
    \label{eqn:regime_a_rmse}
\end{equation}
and $\mathbb{E}[R^{2\,\mathrm{NC}}_{\mathcal{H}}]
        > \mathbb{E}[R^{2\,\mathrm{CO}}_{\mathcal{H}}]$
(under misspecification, for the expected mean squared error at bounded
design-dependent biases: Appendix~\ref{app:holdout_misspec}).

\noindent\textbf{Regime~B} \textup{(several rays, $K \geq 2$,
CO's Jacobian of full rank at $\varepsilon = 0$; for Chinchilla, two rays
with three sizes each suffice, Lemma~\ref{lem:full_rank_rays}):}\;
$\mathbb{E}[\mathrm{RMSE}_{\mathcal{H}}^{\mathrm{CO}}] /
 \mathbb{E}[\mathrm{RMSE}_{\mathcal{H}}^{\mathrm{NC}}]
 = \Theta(1)$ for a fixed CO design, without Regime~A's $\varepsilon$-driven separation, and
Proposition~\ref{prop:holdout_r2} bounds $\kappa_{A,B}$ when $V_K \geq \tau_K$.
\end{theorem}

\noindent\textit{Proof.}
Appendix~\ref{app:proof_holdout_r2}
(Gauss-Newton linearization, a Cauchy-Schwarz
bound, RMSE ordering).
At a fixed gap, Chinchilla's CO variance is at least
$c_{\mathcal{H}}/V_{\mathrm{w}}$, with a holdout constant $c_{\mathcal{H}}$ positive for almost every holdout; collinear designs with $K \geq 2$ keep
$V_{\mathrm{w}} \geq V_K \min_i N_i^{2(\alpha - \beta)}$, rays with unequal
sizes can give $V_{\mathrm{w}} = 0$ (Appendix~\ref{app:proof_holdout_r2}).

\subsection{Applying to respective scaling laws}

Table~\ref{tab:jacobian_summary} gives, for each law, the
collinear pair, its Jacobian columns on a single ray
$D = k_\ell N$ and its exponent gap, so
Proposition~\ref{prop:full_cond} yields
$\kappa(J^T J) = \Omega(\varepsilon^{-2})$ in all four cases
(Appendices~\ref{app:chinchilla_details},
\ref{app:dataconstrained_details}-\ref{app:elibol_details};
statistical interval inflation: Corollary~\ref{prop:ci_inflation},
Table~\ref{tab:severity_comparison}). Kaplan's gap is $0.20$ at the
single-variable exponents of \citet{kaplan2020scaling} ($0.26$ at their joint fit's) and at least $0.105$ in our single-precision (FP32) seed-0 fits, and Kaplan
is the best-conditioned law on our designs (with unit-norm columns, the
smallest $\kappa(J^T J)$ of three laws at each seed-0 fit); its worst off-ray $R^2$ ($0.50$ against $0.78$)
comes from bias (Appendix~\ref{app:kaplan_details}). Chinchilla and
unsaturated repeated-data, with small near-equal exponents,
have a CI-ratio floor $\Omega(\varepsilon^{-1})$: at Hoffmann et al.'s
$\varepsilon = 0.06$ it is $0.16$, and the computed five-parameter ratio is
of order $10^{4}$ (Appendix~\ref{app:proof_ci_inflation}).
Droppo-Elibol is fit-dependent: our fits give
$|\gamma_N - \gamma_D|$ from $0.18$ to $0.80$
(Table~\ref{tab:droppo-params}), and the well-identified outer parameters $(\alpha, E)$
do not relieve the $(N_C, D_C)$ degeneracy. The reduced
model (Definition~\ref{def:reduced_chinchilla_model}) is
therefore most needed for Chinchilla-class fits.

\subsection{IsoFLOP prediction under collinearity}
\label{sec:isoflop}

\emph{IsoFLOP} curves are contours of constant
training compute $C = 6ND$
\citep{hoffmann2022training}.  They cut across
the training ray and probe the sloppy parameter
direction that collinear training leaves
unconstrained, making them an informative
diagnostic.

\begin{definition}[Induced isoFLOP curves plot]
\label{def:induced_isoflop_plot}
Let $\hat{L}(N, D; \boldsymbol{\theta})$ be a
fitted scaling law and
$\mathcal{H} = \{(N_i, D_i, L_i)\}_{i=1}^{n_{\mathcal{H}}}$
a holdout set of size $n_{\mathcal{H}}$ with induced compute
budgets $C_i \coloneqq 6 N_i D_i$.  The
\emph{induced isoFLOP curves} are the family
\begin{equation}
    \{ N \mapsto
        \hat{L}(N, C_i/(6N); \boldsymbol{\theta})
        \,:\, N \in [N_{\min}, N_{\max}]
    \}_{i=1}^{n_{\mathcal{H}}},
    \label{eqn:induced_isoflop_curves}
\end{equation}
with $N_{\min} \coloneqq \min_j N_j$ and
$N_{\max} \coloneqq \max_j N_j$.  The \emph{induced
isoFLOP curves plot} draws this family on the
$(N, L)$-plane with each holdout point
$(N_i, L_i)$ overlaid on its curve.
\end{definition}

\begin{proposition}[RMSE invariance under isoFLOP reparameterization]
\label{prop:rmse_isoflop_invariance}
With $\hat{L}, \mathcal{H}, \{C_i\}$ as in
Definition~\ref{def:induced_isoflop_plot}, the
RMSE of $\hat{L}$ on $\mathcal{H}$ is invariant
under $(N, D) \mapsto (N, C)$: evaluating the
$i$-th induced isoFLOP curve at $N = N_i$ yields
the same error as evaluating $\hat{L}(N_i, D_i)$
directly (Appendix~\ref{app:proof_rmse_invariance}).
\end{proposition}
\section{Empirical evaluation}
\label{sec:empirical}

\noindent\textbf{Empirically evaluating dense TPP coverage.} To validate the formal statements of Section~\ref{sec:formal}, we
fit each scaling law under two training designs - collinear (CO) and non-collinear (NC) - and measure holdout error on a shared CO/NC split across five pre-training corpora: three common and two specialized
(see Appendix~\ref{app:dataset_selection}). All models are LLaMA-style
transformers trained with AdamW on a cosine schedule; full hyperparameters
are in Appendix~\ref{app:experimental_setup}.
Model sizes span $5.04$-$76.5$M parameters
across 14 configurations
($N_{\max}/N_{\min} \approx 15$;
see Table~\ref{tab:model-architectures}).
Tokenization uses cl100k\_base
($V{=}100{,}277$), and $N$ counts all
trainable parameters (embeddings included)
per the Hoffmann-Chinchilla
convention~\citep{hoffmann2022training};
we apply this convention uniformly when
fitting all four laws - including
Kaplan~\citep{kaplan2020scaling}, which
originally defines $N$ on non-embedding
parameters - and no fit
involves a compute budget $C$.
We fit the training loss of each epoch, the mean of its last $100$ batches (Appendix~\ref{app:loss_extraction}).
\noindent\textbf{Experimental designs.}
\label{sec:designs}
Both designs share identical architectures, optimizers, and
schedules; the \emph{only} difference is how $(N,D)$ pairs are
sampled, and total run counts are comparable.
\textbf{Collinear (CO):}
Every run sets $D = kN$ for
$k \in \mathcal{K}_{\mathrm{train}} \coloneqq
\{1,\; 1.5,\; 1.9,\; 2,\; 2.5,\; 2.7,\; 3,\; 3.3,\;
3.5,\; 4,\; 4.5,\; 5\}$.
Although \COTPPs{} ratios are used, each defines a ray through
the origin; the configurations form a \emph{fan of rays} rather than
filling a two-dimensional region, producing the collinearity
diagnosed in Section~\ref{sec:formal}.
\textbf{Non-Collinear (NC):}
The 10 smallest model sizes are crossed with \NCDsizes{}
dataset sizes ($\sim 10$M-$\sim 276$M tokens), yielding a
$10 \times \NCDsizes$ grid spanning a two-dimensional
region.  The implied ratio $k = D/N$
varies freely, decoupling the two axes and eliminating the rank
deficiency.  The grid is visualized in
Appendix~\ref{app:defining_grid}. We fit the four laws with 100 random L-BFGS-B restarts and a differential-evolution fallback under 30 optimizer seeds (Appendix~\ref{app:seed_protocol}); repeated-data takes each epoch's checkpoint as an observation, and the other three fit each epoch (1, 2, 3) separately. A pure-bfloat16 (\textbf{BF16}) round ($k \in \mathcal{K}_{\mathrm{train}}$; Appendix~\ref{app:bf16_results}) tests numerical precision; in a preliminary high-TPP experiment (small-scale,
$k \in \mathcal{K}_{\mathrm{big}} \coloneqq
\{10,\, 11,\, 12,\, 13,\, 14,\, 15\}$,
holdouts up to $20$) with Wikipedia and Chinchilla, NC wins $19$ of $30$ comparisons ($63.3\%$,
CI $[45.5\%, 78.1\%]$, which includes $50\%$; Appendices~\ref{app:bigtpp_prelim}, \ref{app:CI}).

\begin{table*}[!htbp]
\small
\centering
\setlength{\tabcolsep}{3pt}
\caption{$R^2$ and RMSE summary by training design over $30$ optimizer seeds. 95\% CI: optimizer-noise interval (Appendix~\ref{app:CI}); Std: across (corpus, law, epoch) combinations; best per split in \textbf{bold}.}
\label{tab:summary_full}
\begin{tabular}{ll cccc cccc}
\toprule
& & \multicolumn{4}{c}{\textbf{$R^2$}} & \multicolumn{4}{c}{\textbf{RMSE}} \\
\cmidrule(lr){3-6} \cmidrule(lr){7-10}
\textbf{Split} & \textbf{Design} & Mean & 95\% CI & Median & Std & Mean & 95\% CI & Median & Std \\
\midrule
\multirow{2}{*}{Train ($\mathcal{D}$)}
 & CO  & \win{0.9848}  & \win{[0.985, 0.985]} & \win{0.9850} & \win{0.0047} & \win{0.1737} & \win{[0.173, 0.174]} & \win{0.1616} & \win{0.0430} \\
 & NC  & 0.9526  & [0.952, 0.953] & 0.9542 & 0.0172 & 0.2023 & [0.202, 0.203] & 0.1869 & 0.0443 \\
\midrule
\multirow{2}{*}{Holdout ($\mathcal{H}$)}
 & CO  & 0.8370  & [0.832, 0.842] & 0.8811 & 0.1129 & 0.2373 & [0.233, 0.241] & 0.2163 & 0.1229 \\
 & NC  & \win{0.9319}  & \win{[0.930, 0.934]} & \win{0.9392} & \win{0.0269} & \win{0.1561} & \win{[0.154, 0.158]} & \win{0.1242} & \win{0.0599} \\
\bottomrule
\end{tabular}
\end{table*}

\noindent\textbf{Evaluation splits.}
\label{sec:splits}
The collinear holdout $\mathcal{H}_{\mathrm{col}}$ lies just beyond the
training fan: the next model size ($55.5$M) at every training ratio and at
$k \in \mathcal{K}_{\mathrm{test}}
\coloneqq \{6,\, 6.2,\, 6.5,\, 6.7,\, 7\}$, every training size at
$k = 6$, and the three largest sizes ($63.9$-$76.5$M) at
$\mathcal{K}_{\mathrm{test}}$. The non-collinear holdout
$\mathcal{H}_{\mathrm{nc}}$, built the same way in $D$ ($D = 300$M for $k = 6$,
$D \in [300, 401]\,$M for $\mathcal{K}_{\mathrm{test}}$), lies off the fan in the
$(N,D)$ plane. All metrics use the \emph{unified holdout}
$\mathcal{H} = \mathcal{H}_{\mathrm{col}} \cup \mathcal{H}_{\mathrm{nc}}$, so every
comparison is head-to-head; per split (Appendix~\ref{app:per_dataset}), NC wins
$98.7\%$ of the comparisons on $\mathcal{H}_{\mathrm{nc}}$ and $72.1\%$ on
$\mathcal{H}_{\mathrm{col}}$, along CO's own fan, where Wikipedia and RedPajama are
near parity ($48.3\%$ and $50.3\%$).

\begin{table*}[!htbp]
\small
\centering
\setlength{\tabcolsep}{4pt}
\caption{Win rate (NC vs.\ CO) on $\mathcal{H}$, seed-paired; the epoch column omits repeated-data, fit on all epochs at once. Overall: \winrate{} (\winrateNumer/\winrateDenom), 95\% Wilson CI \winrateCI{} over comparisons that share one training run per $(N, D)$, so fitting variability only. NC has the lower mean holdout RMSE in $49$ of the $50$ (corpus, law, epoch) cells.}
\label{tab:winrate}
\begin{tabular}[t]{@{}lr@{}}
\toprule
\textbf{Dataset} & \\
\midrule
C4 & 93.0\% \\
Cosmopedia & 94.7\% \\
peS2o & 99.3\% \\
RedPajama & 100.0\% \\
Wikipedia & 99.7\% \\
\bottomrule
\end{tabular}\hspace{2em}
\begin{tabular}[t]{@{}lr@{}}
\toprule
\textbf{Scaling law} & \\
\midrule
Chinchilla & 97.6\% \\
Droppo-Elibol & 98.0\% \\
Kaplan & 95.6\% \\
Repeated-data & 100.0\% \\
\bottomrule
\end{tabular}\hspace{2em}
\begin{tabular}[t]{@{}lr@{}}
\toprule
\textbf{Epoch} & \\
\midrule
first & 99.8\% \\
second & 97.6\% \\
final & 93.8\% \\
\bottomrule
\end{tabular}
\end{table*}

\noindent\textbf{TPP coverage tracking.}
\phantomsection\label{sec:tpp_convergence}
We widen each design's data axis step by step, adding TPP ratios (CO)
or dataset sizes (NC), and refit each law at each step; the coverage
fraction is the share of levels included ($1.0$: all of
$\mathcal{K}_{\mathrm{train}}$ or all $12$ $D$~sizes). The two curves take
independent paths through the $(N,D)$ plane (setup and plots:
Appendices~\ref{app:tpp_convergence_setup} and~\ref{app:surfaces}).

\begin{table*}[!htbp]
\small
\centering
\setlength{\tabcolsep}{3pt}
\caption{BF16 $R^2$ and RMSE by training design over $30$ optimizer seeds. 95\% CI: optimizer-noise interval (Appendix~\ref{app:CI}); Std: across (law, epoch) combinations; best per split in \textbf{bold}.}
\label{tab:summary_bf16}
\begin{tabular}{ll cccc cccc}
\toprule
& & \multicolumn{4}{c}{\textbf{$R^2$}} & \multicolumn{4}{c}{\textbf{RMSE}} \\
\cmidrule(lr){3-6} \cmidrule(lr){7-10}
\textbf{Split} & \textbf{Design} & Mean & 95\% CI & Median & Std & Mean & 95\% CI & Median & Std \\
\midrule
\multirow{2}{*}{Train ($\mathcal{D}$)}
 & CO  & 0.9896  & [0.990, 0.990] & 0.9916 & 0.0029 & 0.1854 & [0.185, 0.185] & 0.1606 & 0.0353 \\
 & NC  & \win{0.9907}  & \win{[0.991, 0.991]} & \win{0.9915} & \win{0.0020} & \win{0.1457} & \win{[0.145, 0.146]} & \win{0.1428} & \win{0.0171} \\
\midrule
\multirow{2}{*}{Holdout ($\mathcal{H}$)}
 & CO  & 0.9412  & [0.940, 0.942] & 0.9473 & 0.0172 & 0.2546 & [0.252, 0.257] & 0.2402 & 0.0396 \\
 & NC  & \win{0.9657}  & \win{[0.965, 0.967]} & \win{0.9663} & \win{0.0068} & \win{0.1949} & \win{[0.192, 0.198]} & \win{0.1921} & \win{0.0203} \\
\bottomrule
\end{tabular}
\end{table*}

\noindent\textbf{Aggregate metrics.}
\label{sec:metrics}
Table~\ref{tab:summary_full} gives the mean, median and spread of
$R^{2}$ and RMSE over law-dataset combinations per design and split,
and Table~\ref{tab:winrate} tallies seed-paired wins by dataset, law
and epoch. Under pure-BF16 training on Wikipedia
(Table~\ref{tab:summary_bf16}) the effect persists: NC reaches a unified
holdout $R^{2}$ of $0.966$ against CO's $0.941$ and wins \winratebf{}
(CI \winratebfCI; Table~\ref{tab:winrate_bf16}), on a holdout without $14$ of the $42$ collinear holdout runs.
Per-dataset breakdowns are in Appendices~\ref{app:per_dataset}
and~\ref{app:full_all_results}.

\noindent\textbf{Example plots.}
Figure~\ref{fig:select-plots} shows the cell with the largest CO/NC gap
of the $50$ (more in Appendix~\ref{app:surfaces}): the Kaplan CO fit on RedPajama
(first epoch) reaches a holdout $R^{2}$ of only $0.51$ at seed~$0$ ($0.52$ over
$30$ seeds), NC $0.91$ ($0.90$). CO residuals grow toward small~$N$
(Figure~\ref{fig:select-surface}). In Figure~\ref{fig:select-convergence},
\textbf{(i)}~CO sits below NC at every TPP coverage level;
\textbf{(ii)}~CO's seed bands are $2$-$80\times$ wider
than NC's ($7\times$ at full coverage) and narrow with
coverage without closing the gap;
\textbf{(iii)}~CO's mean $R^{2}$ oscillates while
NC's rises smoothly;
\textbf{(iv)}~in Figure~\ref{fig:select-isoflop}, CO's dashed isoFLOP
curves splay apart at high~$N$ while NC's stay tight. For Kaplan this error is bias, not
conditioning (Appendix~\ref{app:kaplan_details}).

\begin{figure*}[!ht]
    \centering
    \begin{subfigure}[b]{0.37\textwidth}
        \centering
        \includegraphics[width=\linewidth]{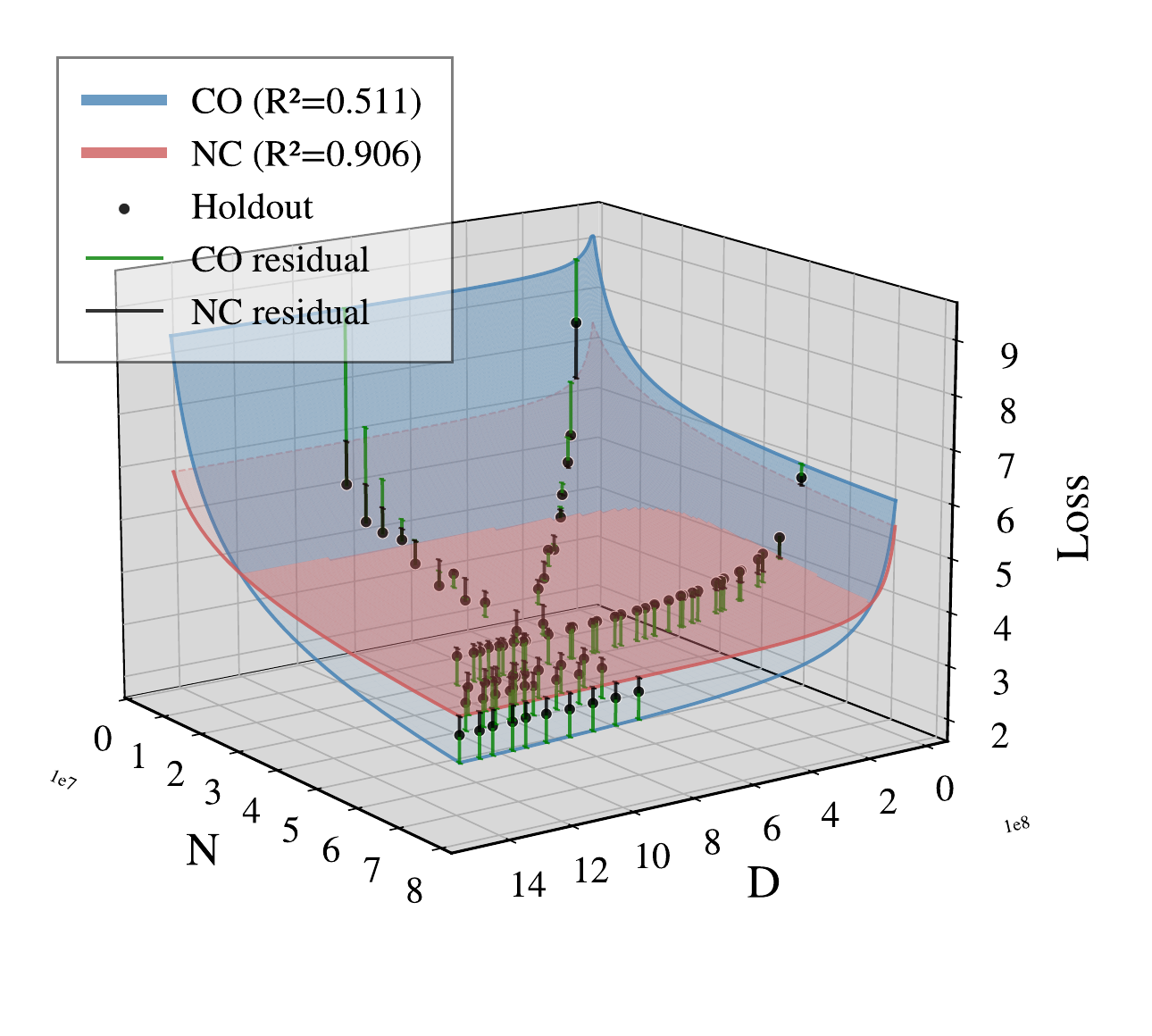}
        \caption{Loss surface.}
        \label{fig:select-surface}
    \end{subfigure}
    \hfill
    \begin{subfigure}[b]{0.30\textwidth}
        \centering
        \includegraphics[width=\linewidth]{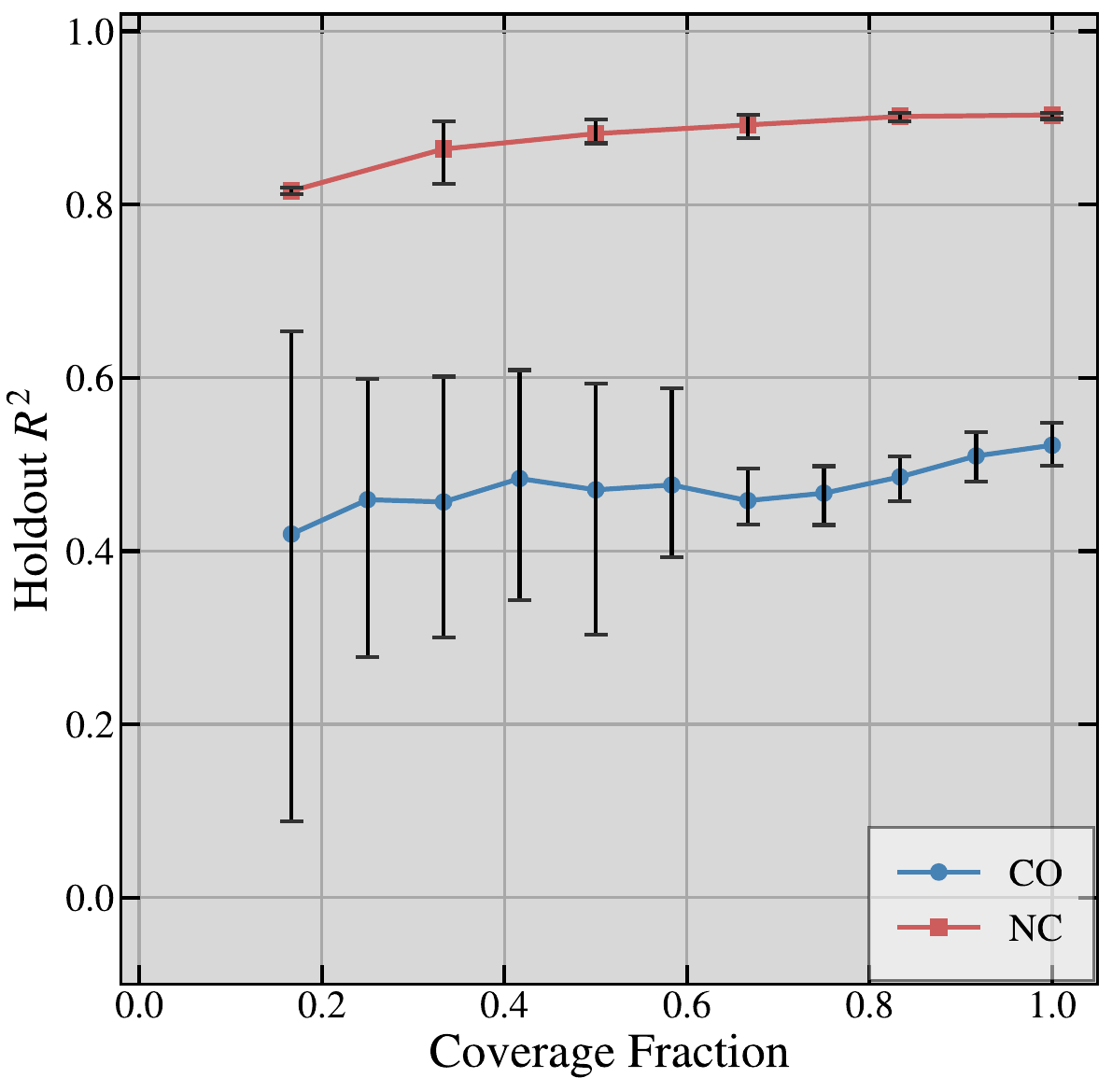}
        \caption{Holdout $R^{2}$ vs.\ TPP coverage.}
        \label{fig:select-convergence}
    \end{subfigure}
    \hfill
    \begin{subfigure}[b]{0.30\textwidth}
        \centering
        \includegraphics[width=\linewidth]{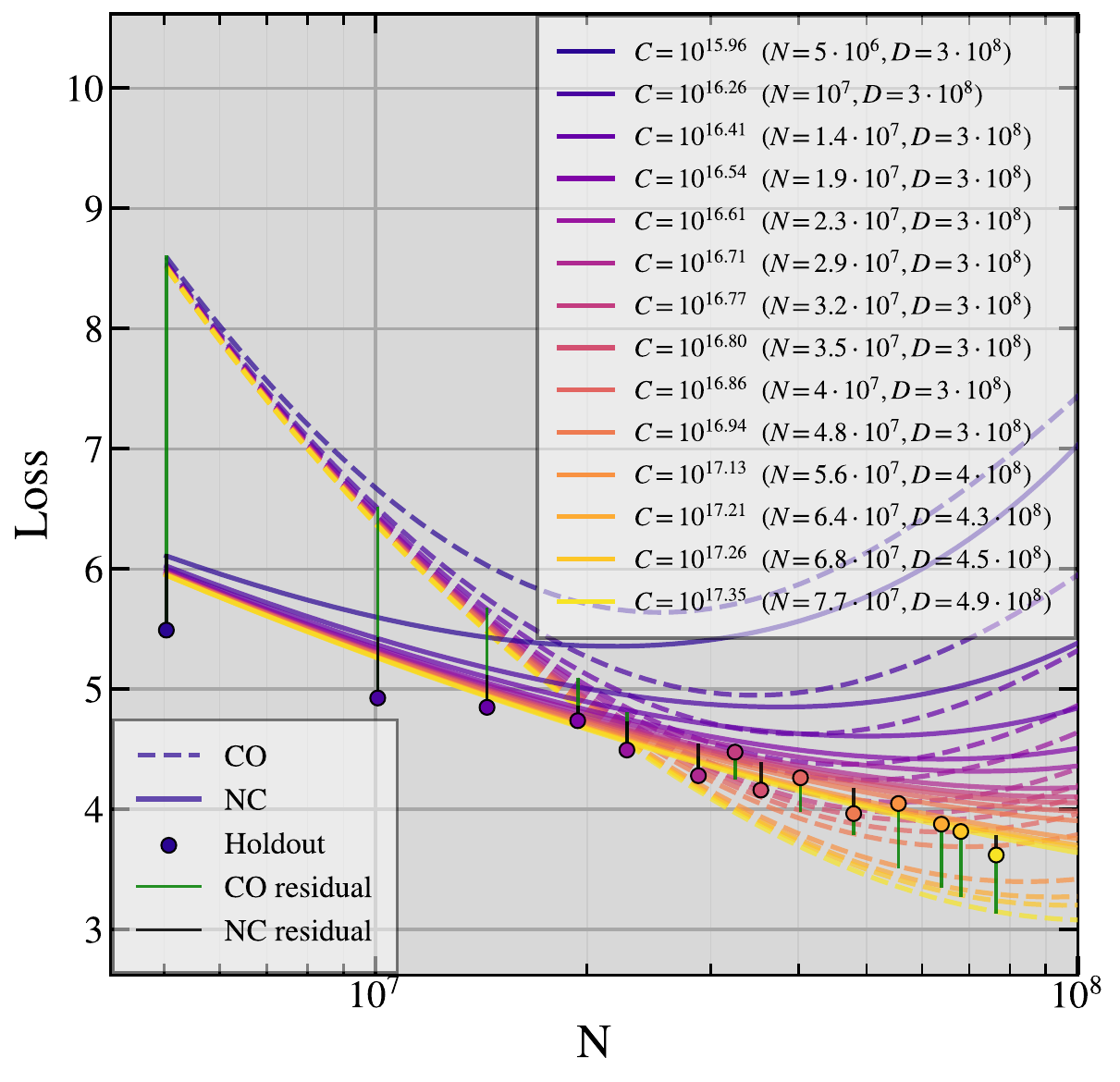}
        \caption{Induced isoFLOP curves.}
        \label{fig:select-isoflop}
    \end{subfigure}
    \caption{Kaplan law on RedPajama, first epoch
        (seed-0 CO holdout $R^{2}\!=\!0.51$, NC $R^{2}\!=\!0.91$;
        CO blue, NC red in (a)-(b), dashed and solid in (c)). (a), (c): seed~$0$,
        whose NC $(N_c, D_c)$ is a start draw (Appendix~\ref{app:optimizer_bounds});
        (b): mean over $30$ seeds with 5th-95th percentile bars, from two rays on, as CO's one-ray mean is $R^2 = -1.35$.
        \textbf{(a)}~Fitted surface $\hat{L}(N,D;\boldsymbol{\theta})$ with
        holdout points and residual bars.
        \textbf{(c)}~Induced isoFLOP curves
        (Definition~\ref{def:induced_isoflop_plot}), one
        per kept holdout point (highest~$D$ per
        unique~$N$), with residual bars at each $N_i$;
        curves and markers share a colormap ordered by
        $C_i = 6 N_i D_i$.}
    \label{fig:select-plots}
\end{figure*}


\noindent\textbf{Budget-matched subset enumeration.}
\phantomsection\label{sec:subset_exhaustion}
Table~\ref{tab:regime_a_winrate} asks empirically whether, below
the threshold ($V_K < \tau_K$), NC also beats a budget-matched CO design
(the same number of runs) of 1 to 12 rays. We enumerate all $2^{12}-1 = 4{,}095$ non-empty subsets
$S \subseteq \{1,\ldots,12\}$ of available TPP ratios.
Each $\mathrm{CO}(S)$ (Definition~\ref{def:co_subset})
is paired with a bounding-box NC design of identical
cardinality (Appendix~\ref{app:regime_a_setup}) and
classified a priori via $V_K < \tau_K(\kappa_{\mathrm{target}})$
using literature $\beta_{\mathrm{eff}}$, a Kaplan subset by its
exact $\kappa_{A,B} > \kappa_{\mathrm{target}}$ at the
parameters of \citet{kaplan2020scaling}
(Appendix~\ref{app:regime_a_setup}).
Because the threshold depends on
$\kappa_{\mathrm{target}}$, we average the NC-win rate
over $2{,}000$ values drawn log-uniformly from
$[1, 10^9]$. Each $\kappa_{\mathrm{target}}$ is
repeated under $22$ seeds that vary the bounding-box
construction and the random restarts of the fit
optimizer ($300$ restarts per fit; random $70\%$
subsample of $\mathcal{H}$). When no subset is below it
for a given $\kappa_{\mathrm{target}}$, that seed
is skipped. NC beats CO on average below the threshold for all three laws: $68.5\%$ (Chinchilla), $57.4\%$ (Droppo-Elibol), and $80.6\%$ (Kaplan) averaged across corpora, all with CIs excluding $50\%$. The effect is strongest on the
web and reference corpora C4, RedPajama, and Wikipedia (NC wins $77$-$88\%$ for Chinchilla); the two specialized corpora (Cosmopedia and peS2o) show win
rates near or below $50\%$ for Chinchilla and Droppo-Elibol and
$71$-$72\%$ for Kaplan, as discussed in
Section~\ref{sec:discussion}.

\noindent\textbf{Public training-loss releases.} Ten public
releases from nine studies give per-run losses on a
two-dimensional $(N, D)$ core, five as $984$ fully trained runs
(Appendix~\ref{app:public_checkpoints}). On a shared off-range
holdout, equal-size NC subsets win $47\%$ of pairs on the
five and $50\%$ on all ten at $\kappa_{\mathrm{target}} = 100$,
where every CO subset is below the threshold at $\beta_{\mathrm{eff}} = 0.28$. Where only NC meets the
threshold (mostly $\kappa_{\mathrm{target}} \geq 10^4$) NC wins
$53$-$76\%$; where both do (Theorem~\ref{thm:holdout_regimes} orders none of
these multi-ratio subsets), CO is
ahead on the five and, on all ten, in windows centred at
1B and above. On the five, CO's relative off-ray error rises
$40\%$ ($[24\%, 52\%]$) per tenfold rise of NC's
$V_K/\tau_K$ over CO's, and $28\%$ ($[-9\%, 41\%]$) in windows
centred at 1B and above. On all ten it is $4\%$
($[-10\%, 21\%]$); its size interaction, $-39\%$
($[-57\%, -18\%]$), is negative, but the band estimates do not
fall monotonically.

\begin{table}[!ht]
\centering\small
\caption{Below-threshold NC-win rate (\%) by law and dataset, first-epoch fits.
Each cell: fraction of paired comparisons where
$\mathrm{RMSE}_{\mathcal{H}}^{\mathrm{NC}} < \mathrm{RMSE}_{\mathcal{H}}^{\mathrm{CO}}$
below the predicted threshold ($V_K < \tau_K$; for Kaplan, the
exact $\kappa_{A,B} > \kappa_{\mathrm{target}}$). \textbf{Dataset Avg.}: mean across corpora with $95\%$ bootstrap CI.}
\label{tab:regime_a_winrate}
\begin{tabular}{l c c c c c c}
\toprule
 & \multicolumn{3}{c}{Web and reference} & \multicolumn{2}{c}{Specialized} & \\
\cmidrule(lr){2-4} \cmidrule(lr){5-6}
Law & C4 & RedPajama & Wikipedia
    & Cosmopedia & peS2o
    & \textbf{Dataset Avg.} \\
\midrule
Chinchilla
    & $77.1$
    & $78.6$
    & $88.0$
    & $51.8$
    & $47.1$
    & $\mathbf{68.5}\ [67.3, 69.7]$ \\
Droppo-Elibol
    & $62.2$
    & $67.2$
    & $65.9$
    & $45.5$
    & $46.2$
    & $\mathbf{57.4}\ [56.1, 58.7]$ \\
Kaplan
    & $92.4$
    & $87.2$
    & $79.8$
    & $72.4$
    & $71.0$
    & $\mathbf{80.6}\ [79.0, 82.0]$ \\
\bottomrule
\end{tabular}
\end{table}
\section{Discussion and conclusions}
\label{sec:discussion}

\textbf{Collinear designs overfit, non-collinear designs
generalize}: on training data CO achieves higher $R^2$ than
NC in nearly every experiment
(Table~\ref{tab:summary_full}: aggregate $0.985$ vs.\
$0.953$), as a fit along a narrow fan of rays
near-interpolates, but on held-out data the ranking reverses
- NC raises unified-holdout $R^2$ from $0.837$ to $0.932$
and cuts RMSE from $0.237$ to $0.156$.
\textbf{Scaling \COTPPs{} TPP ratios does not rescue
collinearity}: our CO design spans
$\mathcal{K}_{\mathrm{train}}$ ($\COTPPs$ ratios) yet NC
still wins \winrate of \winrateDenom seed-paired comparisons
(Table~\ref{tab:winrate}); at $\kappa_{\mathrm{target}} = 100$ its ratios give
$V_K/\tau_K = 0.40$ at Chinchilla's $\beta_{\mathrm{eff}} = 0.28$ ($3.18$, a pass, at
Kaplan's $\beta_{\mathrm{eff}} = 1$), and
Theorem~\ref{thm:holdout_regimes} covers one ray.
\textbf{The effect persists under pure BF16}: on Wikipedia
NC wins \winratebf{} (\winratebfNumer/\winratebfDenom;
Table~\ref{tab:winrate_bf16}), as in FP32 on Wikipedia (\winrateFP{}), with $14$ of the $42$ collinear holdout runs missing.
\textbf{The pattern is consistent across corpora, laws, and
epochs}: on $\mathcal{H}$, NC wins at least $93\%$ on every corpus, law
and epoch (Table~\ref{tab:winrate}).
The combined \textbf{takeaway} is that single-TPP fits do
not support coefficient-level interpretation - along
$D = kN$ rays only the combined effect of $N$ and $D$ is
well constrained, so claims such as ``model size contributes
$X$\% and data $Y$\% of the reducible loss'' require
training to span a two-dimensional region of $(N, D)$; when the design
has one ray, the reduced parameterization
(Definition~\ref{def:reduced_chinchilla_model}) is the
appropriate reporting target. The stakes are practical: on
Wikipedia the seed-0 final-epoch CO and NC Chinchilla fits imply compute-optimal
sizes $N^* \propto C^{0.36}$ and $N^* \propto C^{0.68}$
(Table~\ref{tab:chinchilla-params}; the NC fit's $E$ is at its bound $0$), so an allocation derived
from collinear-design coefficients reflects the design, not
the data.

\noindent\textbf{What each claim rests on.}
\emph{Scale-independent:} for a well-specified law,
Propositions~\ref{prop:holdout_r2}-\ref{prop:reduced_id} (the first
at leading order), Corollaries~\ref{cor:ratio_cases}-\ref{prop:ci_inflation} and
Theorem~\ref{thm:holdout_regimes} have size-independent rates.
\emph{Validated here:} NC wins \winrate{} at 5M-77M
(Table~\ref{tab:winrate}); NC is ahead of budget-matched CO subsets of 1 to 12
rays on web and reference corpora (Table~\ref{tab:regime_a_winrate}); and the
margin effect holds pooled over five public releases
(Appendix~\ref{app:public_checkpoints}).
\emph{Next steps:} the gap at frontier scale, on other architectures and
languages, and in downstream and cross-domain extrapolation; and, on public
releases, at 1B and above with $\kappa_{\mathrm{target}} = 100$.

\noindent\textbf{A design recipe} (Appendix~\ref{app:tpp_recipe}),
with $\beta_{\mathrm{eff}}$ from published exponents
(Tables~\ref{tab:severity_comparison}
and~\ref{tab:law_substitutions}):
\begin{enumerate}[nosep, leftmargin=*]
\item pick the lower ratio $k_1$ and read the minimum spread $R$
  at $\kappa_{\mathrm{target}} = 100$ from
  Table~\ref{tab:R_min_eq5};
\item train on $K = 2$ ratios at the two ends of that spread;
\item split the runs evenly between them, log-uniform in $N$;
\item check the exact $\kappa_{A,B}$ of the planned $(N, D)$ before training.
\end{enumerate}
For example, $20$ runs split evenly between $k \in \lbrace 20, 400 \rbrace$ pass
Eq.~\eqref{eqn:tpp_diversity} at $\beta_{\mathrm{eff}} = 0.28$
($V_2/\tau_2 = 1.22$); $\lbrace 20, 100 \rbrace$ fail ($0.48$), though their exact
$\kappa_{A,B}$ on our ten training sizes is $80$, so the exact check decides.
Appendix~\ref{app:tpp_recipe} works this $20$-run design through step by step.

\noindent\textbf{Limitations.}
\phantomsection\label{sec:limitations}
Theorem~\ref{thm:holdout_regimes} proves the ordering for one ray;
for $K \geq 2$, our 12-ray fan included, we establish it empirically
(Lemma~\ref{lem:regime_a_margin} gives a sufficient margin).
Our models span $5.04$-$76.5$M parameters. The rates and the
leading-order threshold have no size term, but the exact $\kappa_{A,B}$
that decides a design depends on absolute size
(Appendix~\ref{app:proof_prop_holdout_r2}), and so does the measured gap: biases that differ between designs can shrink or reverse the CO/NC error
ratio~\eqref{eqn:mse_ratio_misspec}.
At 1B and above, the public releases we analyze (Section~\ref{sec:empirical})
have subsets that are not single rays and, on the five fully trained releases,
none passes the threshold at $\kappa_{\mathrm{target}} = 100$, so a
budget-matched comparison at that scale is the natural next test.
Corollary~\ref{prop:ci_inflation} assumes i.i.d.\ Gaussian
residuals; correlated noise with covariance $\Sigma$ moves its
ratio by at most the square root of $\kappa(\Sigma)$, not its rate
(Appendix~\ref{app:proof_ci_inflation}).
We test one LLaMA-style architecture, English text, one
recipe applied identically to CO and NC, and pre-training loss,
none of which the theorems assume
(Appendix~\ref{app:experimental_setup}); downstream-task and
cross-domain
extensions~\citep{bhagia2025establishingtaskscalinglaws}
inherit the Jacobian argument; validating them is future work.
On Cosmopedia and peS2o the full-fan comparison holds ($94.7\%$ and $99.3\%$;
Table~\ref{tab:winrate}) and Kaplan keeps $71$-$72\%$ in the budget-matched one,
while Chinchilla and Droppo-Elibol sit near parity there (Table~\ref{tab:regime_a_winrate}).
The recipe follows from Section~\ref{sec:gnsetup}; training an
isoFLOP layout at its spread is the direct next experiment.

\noindent\textbf{Safeguards.}
\phantomsection\label{sec:safeguards}
All corpora (C4, RedPajama, Wikipedia, Cosmopedia, peS2o)
are public and pre-filtered; no new data collection.
Released checkpoints ($5.04$-$76.5$M;
Appendix~\ref{app:experimental_setup}) are diagnostic data
points, not deployable models, and the code targets
scaling-law fitting and small-scale training runs, not
serving production models. The paper's caution against single-TPP coefficient
claims offsets any marginal
acceleration from better-identified scaling laws.

\noindent\textbf{Broader impact.}
\phantomsection\label{sec:broader_impact}
Scaling laws guide multi-billion-dollar pre-training
decisions and inform claims about the relative importance
of capacity vs.\ data volume; under single-TPP designs
those claims are practically unidentifiable (Section~\ref{sec:formal},
Corollary~\ref{prop:ci_inflation}). Our diagnostics
(Proposition~\ref{prop:holdout_r2},
Eq.~\eqref{eqn:tpp_diversity}) flag such fits, and the
framework shows when a replication gap can follow from
ill-conditioning alone: on one Chinchilla ray, $A$'s computed
CI is of order $10^{4}$ times $\psi$'s on our grid
(Appendix~\ref{app:proof_ci_inflation}). Because the
degeneracy is purely Jacobian-geometric, the analysis
transfers to generative vision~\citep{henighan2020scaling,
alabdulmohsin2022revisiting},
time-series~\citep{edwards2024scaling, yao2024scaling},
and scientific machine learning; isoFLOP grids vary the ratio within each budget and can pass the
threshold (Hoffmann et al.'s does; Appendix~\ref{app:published_designs}).

\noindent\textbf{LLM usage.}
\phantomsection\label{sec:llm_usage}
All research questions, results, design, and analyses
originated with the authors. The primary non-trivial
use of LLMs was as a sounding board for candidate
proof steps: the authors drafted the structure of
each derivation, asked an LLM to propose intermediate
algebraic manipulations or identities, and then
independently verified, rewrote, or discarded each
suggestion before incorporation. Auxiliary uses
included exposition, \LaTeX{} typesetting, boilerplate
code, and prose editing, all reviewed by the authors.
LLMs did not generate any claims, theorems, or
experimental conclusions. We used Anthropic's
Claude Opus 4.7, 5 and 5.5 and the Cursor coding agent.
\begin{ack}
This research was funded by ARO grant W911NF-24-1-0007. Separately, this work used DeltaAI at NCSA through allocation CIS230318 from the Advanced Cyberinfrastructure Coordination Ecosystem: Services \& Support (ACCESS) program~\citep{boerner2023access}, which is supported by U.S. National Science Foundation grants \#2138259, \#2138286, \#2138307, \#2137603, and \#2138296.
Three AMD EPYC CPU chips, donated by AMD, were used for the experimental section of this paper.
This research was supported in part through computational resources provided by Syracuse University, including OrangeGrid, which is supported by NSF award ACI-1341006.
The authors declare no competing interests.
\end{ack}
\clearpage

\bibliography{src/neurips2026_conference}
\bibliographystyle{plainnat}
\label{neuripsmainbodyend}%
\clearpage
\appendix
\section{Appendix - notation reference}
\label{app:notation}

Table~\ref{tab:notation} collects all symbols used throughout
the paper and appendices.

\begingroup\small
\setlength{\tabcolsep}{4pt}
\begin{longtable}{@{}p{0.22\linewidth}p{0.58\linewidth}p{0.16\linewidth}@{}}
\caption{Summary of notation.}
\label{tab:notation} \\
\toprule
\textbf{Symbol} & \textbf{Meaning} & \textbf{Defined in} \\
\midrule
\endfirsthead
\multicolumn{3}{l}{\emph{(Table~\ref{tab:notation} continued)}} \\
\toprule
\textbf{Symbol} & \textbf{Meaning} & \textbf{Defined in} \\
\midrule
\endhead
\midrule
\multicolumn{3}{r}{\emph{Continued on next page}} \\
\endfoot
\bottomrule
\endlastfoot
\multicolumn{3}{l}{\emph{Scaling-law variables}} \\
$N$ & Model parameter count & Eq.~\eqref{eqn:chinchilla} \\
$D$ & Training tokens (dataset size) & Eq.~\eqref{eqn:chinchilla} \\
$L(N,D)$ & Observed loss & Eq.~\eqref{eqn:chinchilla} \\
$\hat{L}(N,D;\boldsymbol{\theta})$ & Predicted loss (scaling law) & Def.~\ref{def:experimental_dataset} \\
$k_1 < \cdots < k_K$ & Ordered TPP ratios in a $K$-ray design & Def.~\ref{def:experimental_dataset} \\
$k$ & Tokens-per-parameter (TPP) ratio ($K=1$ case, $D = kN$) & Sec.~\ref{sec:single_ray} \\
$C$ & Training compute, $C = 6ND$ (Chinchilla FLOPs) & Sec.~\ref{sec:isoflop} \\
$C_i$ & Induced compute budget $6 N_i D_i$ for holdout point $i$ & Def.~\ref{def:induced_isoflop_plot} \\
$N_{\min},\,N_{\max}$ & Min/max model sizes used for induced isoFLOP curves & Def.~\ref{def:induced_isoflop_plot} \\
$V$ & Vocabulary size (cl100k\_base; empirical count $100{,}277$) & Sec.~\ref{sec:empirical}; App.~\ref{app:experimental_setup} \\
\midrule
\multicolumn{3}{l}{\emph{Chinchilla parameters}} \\
$A,\, B$ & Scale coefficients (model, data) & Eq.~\eqref{eqn:chinchilla} \\
$\alpha,\, \beta$ & Power-law exponents & Eq.~\eqref{eqn:chinchilla} \\
$E$ & Irreducible loss & Eq.~\eqref{eqn:chinchilla} \\
$\psi$ & Combined coefficient $A + Bk^{-\alpha}$ & Eq.~\eqref{eqn:reduced_chinchilla} \\
$\varepsilon$ & Exponent gap ($|\alpha{-}\beta|$ for Chinchilla) & Def.~\ref{def:exponent_gap} \\
\midrule
\multicolumn{3}{l}{\emph{Observation model and optimization}} \\
$\mathcal{A}$ & Training scheme (architecture, optimizer, algorithm, etc.) & Def.~\ref{def:experimental_dataset} \\
$\mathcal{D}$ & Scaling-law dataset $\{(N_i,D_i,L_i)\}_{i=1}^m$ & Def.~\ref{def:experimental_dataset} \\
$\mathcal{D}_{\mathrm{train}}$ & Training split of $\mathcal{D}$ & Def.~\ref{def:experimental_dataset} \\
$\mathcal{D}_{\mathrm{train}}^{\mathrm{CO}}$,\,
$\mathcal{D}_{\mathrm{train}}^{\mathrm{NC}}$ &
Full collinear and non-collinear training pools before subset enumeration &
Def.~\ref{def:co_subset}, Def.~\ref{def:nc_bbox} \\
$\mathcal{N}$ & Shared ordered model sizes $\{N_1,\ldots,N_n\}$ & Def.~\ref{def:experimental_dataset}; App.~\ref{app:regime_a_setup} \\
$m$ & Number of observations in $\mathcal{D}$ & Def.~\ref{def:experimental_dataset} \\
$\boldsymbol{\theta}$ & Scaling-law parameter vector & Def.~\ref{def:experimental_dataset} \\
$p$ & Number of parameters in $\boldsymbol{\theta}$ & Def.~\ref{def:experimental_dataset} \\
$n$ & Number of distinct model sizes & Def.~\ref{def:experimental_dataset} \\
$N_1 < \cdots < N_n$ & Ordered model sizes & Def.~\ref{def:experimental_dataset} \\
$K$ & Number of TPP ratios in collinear design & Def.~\ref{def:experimental_dataset} \\
$\kappa(\cdot)$ & Condition number ($\lambda_{\max}/\lambda_{\min}$) & Sec.~\ref{sec:gnsetup} \\
$r_i$ & Residual: $L_i - \hat{L}(N_i,D_i;\boldsymbol{\theta})$ & Sec.~\ref{sec:gnsetup} \\
$\mathbf{r}(\boldsymbol{\theta})$ & Stacked residual vector $(r_i)_{i=1}^m$ & Sec.~\ref{sec:gnsetup} \\
$S(\boldsymbol{\theta})$ & GN objective $\tfrac{1}{2}\|\mathbf{r}(\boldsymbol{\theta})\|^2$ & Sec.~\ref{sec:gnsetup} \\
$J$ & Jacobian of residuals & Sec.~\ref{sec:gnsetup} \\
$\Delta\boldsymbol{\theta}$ & Gauss-Newton update step & Sec.~\ref{sec:gnsetup} \\
$\kappa_{A,B}$ & Condition number of $(A,B)$ sub-block of $J^T J$; a lower bound for $\kappa(J^T J)$ (Cauchy interlacing) & Sec.~\ref{sec:gnsetup} \\
$\mathbf{j}_a,\,\mathbf{j}_b$ & Columns of $J$ (near-proportional pair) & Proposition~\ref{prop:full_cond} \\
$c$ & Proportionality scalar, $\mathbf{j}_b = c\,\mathbf{j}_a + \boldsymbol{\delta}$ & Proposition~\ref{prop:full_cond} \\
$\boldsymbol{\delta}$ & Near-proportionality perturbation, $\|\boldsymbol{\delta}\|=O(\varepsilon)$ & Proposition~\ref{prop:full_cond} \\
$\mathbf{j}_1,\,\mathbf{j}_2$ &
First two Jacobian columns in the column-ordering used in Table~\ref{tab:jacobian_summary}
(alias of the illustrative $(a,b)$ pair in Prop.~\ref{prop:full_cond}) &
Tab.~\ref{tab:jacobian_summary} \\
$\hat{\boldsymbol{\theta}}$ & Least-squares estimator & Lemma~\ref{lem:ls_iid_covariance} \\
$\mathbb{E}[\cdot]$ &
Expectation (Regime A/B inequalities and RMSE ratios) &
Thm.~\ref{thm:holdout_regimes} \\
$\sigma^2$ &
Idealized homoscedastic variance of observational noise in least squares (distinct from appendix-table $\sigma$) &
Lemma~\ref{lem:ls_iid_covariance}; Cor.~\ref{prop:ci_inflation} \\
$\sigma$ (appendix tables) &
Seed-to-seed std.\ dev.\ of Monte Carlo summaries across optimizer seeds &
App.~\ref{app:per_dataset}; App.~\ref{app:full_all_results} \\
$\mathrm{CI}_{0.95}(\cdot)$ & Nominal 95\% confidence interval half-width & Corollary~\ref{prop:ci_inflation} \\
\midrule
\multicolumn{3}{l}{\emph{Multi-ray designs and holdout prediction (Prop.~\ref{prop:holdout_r2};
    Theorem~\ref{thm:holdout_regimes})}} \\
$\mathcal{H}$ & General holdout set & Def.~\ref{def:experimental_dataset} \\
$n_{\mathcal{H}}$ & Holdout cardinality & Def.~\ref{def:experimental_dataset} \\
$k_K,\,k_1$ & Largest/smallest ordered TPP ratio in a design & Prop.~\ref{prop:holdout_r2} \\
$R$ & Two-ray TPP spread $k_2/k_1$ ($K=2$) & Cor.~\ref{cor:ratio_cases} \\
$\beta_{\mathrm{eff}}$ & Law-specific data-size exponent (Chinchilla/repeated-data: $\beta$; Kaplan: $1$; Droppo-Elibol: $\gamma_D$) & Prop.~\ref{prop:holdout_r2}; Tab.~\ref{tab:law_substitutions} \\
$V_K$ &
Second central moment (variance) of
$\{k_\ell^{-\beta_{\mathrm{eff}}}\}_{\ell=1}^K$ across TPP rays;
$V_1=0$ &
Prop.~\ref{prop:holdout_r2} \\
$\tau_K$ & Diversity threshold in~\eqref{eqn:tpp_diversity}; at leading order (the limit $\varepsilon \to 0$ at fixed sizes and ratios) $V_K \geq \tau_K$ implies $\kappa_{A,B} \leq \kappa_{\mathrm{target}}$, and $\kappa_{A,B} \leq \kappa_{\mathrm{target}}$ iff $V_K \geq \tau_K/(1+1/\kappa_{\mathrm{target}})^{2}$ & Prop.~\ref{prop:holdout_r2} \\
$\omega_i$, $u_i$ & Weight $N_i^{-2\alpha}/\sum_j N_j^{-2\alpha}$ of training run $i$, the sum over all runs (on one ray, the $w_i$ of App.~\ref{app:submodel}), and its column ratio $D_i^{-\beta} N_i^{\alpha}$ (general columns: $f_i^2/\|\mathbf{f}\|^2$ and $g_i/f_i$) & Lemma~\ref{lem:weighted_spread} \\
$V_{\mathrm{w}}$, $\tau_{\mathrm{w}}$ & Weighted spread $\sum_i \omega_i (u_i - \bar{u})^2$ and threshold $(1 + \sum_i \omega_i u_i^2)^2/\kappa_{\mathrm{target}}$; exact at every gap and on any design, for Chinchilla, with limits $V_K$, $\tau_K$ as $\varepsilon \to 0$ on the collinear design & Lemmas~\ref{lem:weighted_spread}-\ref{lem:ray_spread} \\
$\pi_\ell$, $V_\pi$ & Weight $\sum_{i \text{ on ray } \ell} \omega_i$ of ray $\ell$; the $\pi$-weighted variance of $\{k_\ell^{-\beta}\}$, the limit of $V_{\mathrm{w}}$ as $\varepsilon \to 0$ & Lemma~\ref{lem:ray_spread} \\
$y_\pm$, $\bar{y}$, $\sigma_y^2$ & Largest and smallest $N_i^{\alpha - \beta}$ over the design; weighted mean and variance of $N_i^{\alpha - \beta}$ & Lemma~\ref{lem:ray_spread} \\
$\delta_N$ & $\varepsilon \max_i |\log N_i|$, which bounds how far $V_K/\tau_K$ can be from $V_{\mathrm{w}}/\tau_{\mathrm{w}}$~\eqref{eqn:proxy_bound} & App.~\ref{app:proof_prop_holdout_r2} \\
$\boldsymbol{\theta}^*_{\mathrm{X}}$, $\Delta_{\mathrm{X}}^2$ & Pseudo-true parameters of training design X under misspecification, and its mean squared approximation bias on $\mathcal{H}$ & App.~\ref{app:holdout_misspec} \\
$\Lambda_{\mathrm{w}}(\mathcal{H})$ & Weighted sloppy leverage of holdout $\mathcal{H}$ in the fixed-gap CO bound~\eqref{eqn:regime_b_fixed_gap} & App.~\ref{app:proof_holdout_r2} \\
$\kappa_K$ & Shorthand for $\kappa_{A,B}$ on a collinear $K$-ray design ($V_K \geq \tau_K$ gives $\kappa_K \leq \kappa_{\mathrm{target}}$ at leading order) & App.~\ref{app:tpp_recipe} \\
$\kappa_{\mathrm{target}}$ & Target condition number; $\tau_K$ in~\eqref{eqn:tpp_diversity} depends on it & Prop.~\ref{prop:holdout_r2}; Thm.~\ref{thm:holdout_regimes} \\
$\mathrm{RMSE}_{\mathcal{H}}^{\mathrm{CO/NC}}$ & Holdout RMSE of $\hat{L}$ against $\mathbb{E}[L]$ (superscript: training design) & Thm.~\ref{thm:holdout_regimes} \\
$R^{2\,\mathrm{CO/NC}}_{\mathcal{H}}$ & Holdout $R^2$ under design CO or NC & Thm.~\ref{thm:holdout_regimes} \\
$G_{A,B}$, $M$ & Gram block of the CO scale-coefficient columns; the holdout's second moment of the same two Jacobian entries & Sec.~\ref{sec:gnsetup}; Lemma~\ref{lem:regime_a_margin} \\
\midrule
\multicolumn{3}{l}{\emph{Experimental design (Sec.~\ref{sec:empirical})}} \\
$\mathcal{K}_{\mathrm{train}}$ & Full set of $12$ collinear training TPP ratios used in the CO design & Sec.~\ref{sec:designs} \\
$\mathcal{K}_{\mathrm{test}}$ & Set of $5$ collinear holdout TPP ratios beyond the training fan & Sec.~\ref{sec:splits} \\
$\mathcal{K}_{\mathrm{big}}$ & Set of $6$ collinear training TPP ratios in the preliminary high-TPP CO design & Sec.~\ref{sec:designs}; App.~\ref{app:bigtpp_prelim} \\
$\mathcal{H}_{\mathrm{col}}$ & Collinear holdout (empirical label) & Sec.~\ref{sec:splits} \\
$\mathcal{H}_{\mathrm{nc}}$ & Non-collinear holdout (empirical label) & Sec.~\ref{sec:splits} \\
$R^{2}_{\mathcal{H}}$ &
Unified-holdout $R^{2}$ on $\mathcal{H}$ &
Sec.~\ref{sec:empirical}; Tab.~\ref{tab:summary_full} \\
$R^{2}_{\mathcal{H}_{\mathrm{col}}}$,\,
$R^{2}_{\mathcal{H}_{\mathrm{nc}}}$ &
$R^{2}$ on CO-only vs.\ NC-only holdout splits &
App.~\ref{app:per_dataset} \\
$R^{2}_{\text{train}}$ & Coefficient of determination on the training split & App.~\ref{app:coefficients} \\
$d_{\text{model}},\,d_{\text{ff}},\,d_{\text{head}}$,\,
$n_{\text{layers}},\,n_{\text{heads}}$ &
Transformer shape hyperparameters (Table~\ref{tab:model-architectures}) &
App.~\ref{app:experimental_setup} \\
$\mathrm{RMSE}$ & Root mean squared error on a holdout split & Sec.~\ref{sec:empirical} \\
\midrule
\multicolumn{3}{l}{\emph{Alternative scaling-law notation}} \\
$N',\, D'$ & Effective quantities (repeated-data) & Def.~\ref{def:exponent_gap} \\
$N_c,\, D_c$ & Scale parameters (Kaplan) & Tab.~\ref{tab:jacobian_summary}; App.~\ref{app:kaplan_details} \\
$\alpha_N,\, \alpha_D$ & Separate exponents (Kaplan / Droppo-Elibol) & Def.~\ref{def:exponent_gap} \\
$N_C,\, D_C$ & Scale parameters (Droppo-Elibol) & Tab.~\ref{tab:jacobian_summary}; App.~\ref{app:elibol_details} \\
$\alpha$ (Droppo-Elibol) &
Canonical composite exponent in that law &
Def.~\ref{def:exponent_gap}; Tab.~\ref{tab:jacobian_summary} \\
$\gamma_N,\, \gamma_D$ &
Exponent ratios $\alpha_N/\alpha$, $\alpha_D/\alpha$
(Droppo-Elibol) &
Def.~\ref{def:exponent_gap} \\
$L_\infty$ & Irreducible loss (Droppo-Elibol) & App.~\ref{app:elibol_details} \\
$R_D^*,\, R_N^*$ & Decay constants (repeated-data) & App.~\ref{app:dataconstrained_details} \\
\midrule
\multicolumn{3}{l}{\emph{Appendix-only: design-dependent quantities}} \\
$\Phi_q$ & Power-sum family $\sum_i N_i^{-q\alpha}$ (index $q$) & App.~\ref{app:chinchilla_details}; Def.~\ref{def:power_sums} \\
$\Phi_2^{(2\varepsilon)}$ & $\sum_i N_i^{-2\beta}$ (Chinchilla; $=\sum_i N_i^{-2\alpha+2\tilde\varepsilon}$, $\tilde\varepsilon=\alpha-\beta$) & App.~\ref{app:proof_ci_inflation} \\
$T_q$ & Log-weighted power sum (e.g.\ $T_2 = \sum_i N_i^{-2\alpha}\log N_i$) & App.~\ref{app:proof_ci_inflation} \\
$U_q$ & Log$^2$-weighted power sum (e.g.\ $U_2 = \sum_i N_i^{-2\alpha}(\log N_i)^2$) & App.~\ref{app:proof_ci_inflation} \\
$w_i$ & Power-law weights $N_i^{-2\alpha}/\Phi_2$ & App.~\ref{app:submodel} \\
$\sigma_w^2(\log N)$ & Weighted log-variance of model sizes & App.~\ref{app:submodel} \\
\midrule
\multicolumn{3}{l}{\emph{Appendix-only: experimental design}} \\
$R_{\min}(k_1)$ & Smallest two-ray spread $R = k_2/k_1$ with $V_2 \geq \tau_2$ at lower ratio $k_1$, in closed form~\eqref{eqn:R_min_closed}; $\infty$ when no spread passes~\eqref{eqn:tpp_diversity} & Cor.~\ref{cor:ratio_cases}; Tab.~\ref{tab:R_min_eq5} \\
$\kappa_1$ & Single-ray baseline $\kappa_K$ at $K=1$, $\Theta(\varepsilon^{-2})$ for every $k_1$ & App.~\ref{app:tpp_recipe} \\
$\Lambda(\mathcal{S})$ & Sloppy leverage of test set $\mathcal{S}$ & App.~\ref{app:proof_holdout_r2}, Def.~\ref{def:sloppy_leverage_r2} \\
$V^{\mathrm{X}}_{\mathcal{H}}$ & Mean predictive variance of design X on $\mathcal{H}$, $\mathbb{E}[\mathrm{MSE}^{\mathrm{X}}_{\mathcal{H}}] = \sigma^2 V^{\mathrm{X}}_{\mathcal{H}}$ (well-specified) & Lemma~\ref{lem:regime_a_margin} \\
$\rho_K$, $\rho^\star(\mathcal{H})$ & Conditioning ratio and $K \geq 2$ margin (a ratio above 1 suffices for the MSE ordering, above $\pi/2$ for the RMSE ordering) & Lemma~\ref{lem:regime_a_margin} \\
$G$ &
Training-design scalar $G = \Phi_2 / \sum_i N_i^{-2\alpha}(\ln N_i)^2 > 0$
in the CO $\mathrm{MSE}$ lower bound~\eqref{eqn:mse_co_unified} &
Eq.~\eqref{eqn:mse_co_unified} \\
$\eta$ & Learning rate in optimizer setup & App.~\ref{app:experimental_setup} \\
\midrule
\multicolumn{3}{l}{\emph{Appendix-only: subset enumeration (App.~\ref{app:regime_a_setup})}} \\
$\mathcal{L}$ & Scaling-law family (Chinchilla, Kaplan, Droppo-Elibol) & App.~\ref{app:regime_a_setup} \\
$\mu$ & Epoch/mode index in paired-comparison enumeration & Def.~\ref{def:paired_comparison} \\
$\mathbf{k}^{\mathrm{all}}$ & Full set of $12$ collinear TPP ratios per dataset & Def.~\ref{def:available_tpp} \\
$S$ & Subset of TPP-ratio indices, $S \subseteq \{1,\ldots,12\}$ & Def.~\ref{def:co_subset} \\
$\mathrm{CO}(S)$ & Collinear subset design induced by $S$ & Def.~\ref{def:co_subset} \\
$n_S$ & Cardinality of $\mathrm{CO}(S)$ & Def.~\ref{def:co_subset} \\
$K_S$ & Effective TPP-ray count of subset design, $K_S = |S|$ & Def.~\ref{def:co_subset} \\
$\mathcal{N}_{\mathrm{NC}},\,\mathcal{D}_{\mathrm{NC}}$ & Row/column axes of the NC grid in $(N,D)$ space & Def.~\ref{def:nc_bbox} \\
$M_N,\,M_D$ & NC-grid row/column counts, $|\mathcal{N}_{\mathrm{NC}}|$, $|\mathcal{D}_{\mathrm{NC}}|$ & Def.~\ref{def:nc_bbox} \\
$n^\star,\, \mathbf{k}_S$ & Target NC cardinality and target TPP ratios from $S$ & Def.~\ref{def:nc_bbox} \\
$\mathrm{NC}_{\Box}(n^\star,\mathbf{k}_S)$ & Bounding-box NC design of cardinality $n^\star$ & Def.~\ref{def:nc_bbox} \\
$\mathcal{C}(\mathcal{L},\mathcal{D},\mu,S)$ & Paired comparison tuple & Def.~\ref{def:paired_comparison} \\
$\mathcal{P}$ & Set of all paired comparisons & Def.~\ref{def:paired_comparison} \\
$\kappa^\star$ & Target condition number for empirical regime classification & Def.~\ref{def:win_rate} \\
$\mathcal{P}^A_{\mathcal{L},\mu}(\kappa^\star)$ & Predicted below-threshold set at target $\kappa^\star$ ($V_K < \tau_K$; Kaplan: exact $\kappa_{A,B} > \kappa^\star$) & Def.~\ref{def:win_rate} \\
$\mathrm{WR}(\mathcal{P}')$ & NC win rate on a collection $\mathcal{P}' \subseteq \mathcal{P}$ & Def.~\ref{def:win_rate} \\
$\mathrm{WR}_A(\mathcal{L},\mu;\kappa^\star)$ & Below-threshold NC win rate at target $\kappa^\star$ & Def.~\ref{def:win_rate} \\
$\mathbf{1}\{\cdot\}$ & Indicator used in NC win-rate definition & Eq.~\eqref{eqn:winrate} \\
\end{longtable}
\endgroup

\clearpage
\section{Appendix - proofs and detailed derivations}
\label{app:proofs}

The following proofs and derivations support the results in
Section~\ref{sec:formal}.
For Kaplan, the proofs use the compositional form that the
experiments fit (Appendix~\ref{app:kaplan_details}).

\subsection{Appendix - Derivation of the Gauss-Newton normal equations}\label{app:gauss_newton_derivation}

This subsection provides a self-contained derivation of the
Gauss-Newton normal equations from
the nonlinear least-squares
objective (Section~\ref{sec:gnsetup}), following the
treatment in~\cite[Chap.~10]{nocedal2006numerical}.

\noindent\textbf{Step 1: Objective and gradient.}
Recall the objective:
\begin{equation}
    S(\boldsymbol{\theta})
        = \tfrac{1}{2}\,
          \|\mathbf{r}(\boldsymbol{\theta})\|^2
        = \tfrac{1}{2}\,
          \sum_{i=1}^{m} r_i(\boldsymbol{\theta})^2,
    \label{eqn:gn_objective}
\end{equation}
where $r_i(\boldsymbol{\theta})
= L_i - \hat{L}(N_i, D_i; \boldsymbol{\theta})$.
The gradient with respect to $\boldsymbol{\theta}$ is obtained
by the chain rule:
\begin{equation}
    \frac{\partial S}{\partial \theta_j}
        = \sum_{i=1}^{m} r_i(\boldsymbol{\theta})\;
          \frac{\partial r_i}{\partial \theta_j}
        = -\sum_{i=1}^{m} r_i(\boldsymbol{\theta})\;
          \frac{\partial \hat{L}_i}{\partial \theta_j},
    \label{eqn:gn_grad_component}
\end{equation}
since $\partial r_i / \partial \theta_j
= -\partial \hat{L}_i / \partial \theta_j$.
Define the residual Jacobian:
\begin{equation}
    J \in \mathbb{R}^{m \times p}, \qquad
    J_{ij}
        \coloneqq -\frac{\partial \hat{L}(N_i, D_i;
                   \boldsymbol{\theta})}
                  {\partial \theta_j}
        = \frac{\partial r_i}{\partial \theta_j}.
    \label{eqn:gn_jacobian_def}
\end{equation}
Thus $J_{ij} = \partial r_i / \partial \theta_j$, matching
Section~\ref{sec:gnsetup}.
With this convention, the gradient~\eqref{eqn:gn_grad_component}
can be written in matrix form as:
\begin{equation}
    \nabla_{\boldsymbol{\theta}} S
        = J^T \mathbf{r}(\boldsymbol{\theta}).
    \label{eqn:gn_grad_matrix}
\end{equation}

\noindent\textbf{Step 2: Linearization of the residuals.}
At a current iterate $\boldsymbol{\theta}_k$, expand each
residual to first order in the step
$\Delta\boldsymbol{\theta}
= \boldsymbol{\theta} - \boldsymbol{\theta}_k$:
\begin{equation}
    r_i(\boldsymbol{\theta}_k + \Delta\boldsymbol{\theta})
        \approx r_i(\boldsymbol{\theta}_k)
        + \sum_{j=1}^{p}
          \frac{\partial r_i}{\partial \theta_j}
          \bigg|_{\boldsymbol{\theta}_k}
          \!\!\Delta\theta_j
        = r_{i,k} + (J_k\, \Delta\boldsymbol{\theta})_i,
    \label{eqn:gn_linearise}
\end{equation}
where $r_{i,k} \coloneqq r_i(\boldsymbol{\theta}_k)$ and
$J_k \coloneqq J(\boldsymbol{\theta}_k)$ is the Jacobian
evaluated at $\boldsymbol{\theta}_k$.
In vector notation:
\begin{equation}
    \mathbf{r}(\boldsymbol{\theta}_k
        + \Delta\boldsymbol{\theta})
        \approx \mathbf{r}_k
        + J_k\, \Delta\boldsymbol{\theta}.
    \label{eqn:gn_linearise_vec}
\end{equation}

\noindent\textbf{Step 3: Substitute into the objective.}
Substituting~\eqref{eqn:gn_linearise_vec}
into~\eqref{eqn:gn_objective}:
\begin{equation}
\begin{aligned}
    S(\boldsymbol{\theta}_k + \Delta\boldsymbol{\theta})
        &\approx \tfrac{1}{2}\,
          \|\mathbf{r}_k + J_k\, \Delta\boldsymbol{\theta}\|^2
    \nonumber \\[4pt]
        &= \tfrac{1}{2}\,
          \bigl(\mathbf{r}_k + J_k\, \Delta\boldsymbol{\theta}
          \bigr)^T
          \bigl(\mathbf{r}_k + J_k\, \Delta\boldsymbol{\theta}
          \bigr)
    \nonumber \\[4pt]
        &= \tfrac{1}{2}\, \mathbf{r}_k^T \mathbf{r}_k
          + \mathbf{r}_k^T J_k\, \Delta\boldsymbol{\theta}
          + \tfrac{1}{2}\, \Delta\boldsymbol{\theta}^T
            J_k^T J_k\, \Delta\boldsymbol{\theta}.
    \label{eqn:gn_quadratic}
\end{aligned}
\end{equation}
This is an \emph{exact} quadratic in
$\Delta\boldsymbol{\theta}$ (no higher-order terms remain
because the linearization has already been applied).
Denote this quadratic model as
$m_k(\Delta\boldsymbol{\theta})$.

\noindent\textbf{Step 4: Stationarity condition.}
At the minimizer of $m_k$, the gradient with respect to
$\Delta\boldsymbol{\theta}$ vanishes:
\begin{align}
    \nabla_{\Delta\boldsymbol{\theta}}\, m_k
        &= J_k^T \mathbf{r}_k
          + J_k^T J_k\, \Delta\boldsymbol{\theta}
        = \mathbf{0},
    \label{eqn:gn_stationarity} \\[4pt]
    (J_k^T J_k)\, \Delta\boldsymbol{\theta}
        &= -J_k^T \mathbf{r}_k.
    \label{eqn:gn_normal_derived}
\end{align}
This is the Gauss-Newton normal
equation.  When
$J_k^T J_k$ is non-singular, the unique solution is
$\Delta\boldsymbol{\theta}
= -(J_k^T J_k)^{-1} J_k^T \mathbf{r}_k$, and the
parameter update is
$\boldsymbol{\theta}_{k+1}
= \boldsymbol{\theta}_k + \Delta\boldsymbol{\theta}$.

\begin{remark}[Relationship to the full Hessian]
The exact Hessian of $S$ is:
\begin{equation}
    H(\boldsymbol{\theta})
        = \nabla^2 S
        = J^T J + Q,
    \qquad
    Q \coloneqq \sum_{i=1}^{m} r_i\,
        \nabla^2 r_i.
    \label{eqn:gn_full_hessian}
\end{equation}
The Gauss-Newton method drops the second-order correction~$Q$,
replacing $H$ with the \emph{Gauss-Newton approximation}
$J^T J$.  This is justified when $\|\mathbf{r}\|$ is small
(so $Q \approx 0$) or when the model is approximately linear
in~$\boldsymbol{\theta}$ (so $\nabla^2 r_i \approx 0$).
For the Chinchilla and repeated-data scaling laws,
$\hat{L}$ is \emph{linear} in the scale coefficients
$A$ and $B$, so $\nabla^2_{A,B} r_i = 0$ exactly and
$Q_{A,B} = 0$: the Gauss-Newton and full Newton curvatures
\emph{coincide} in the $(A, B)$ sub-block regardless of
residual magnitudes.
\end{remark}

\begin{remark}[Why $J^T J$ controls identifiability]
When $J^T J$ is well conditioned, the normal
equations~\eqref{eqn:gn_normal_derived} have a unique,
numerically stable solution.  When two columns of $J$ are
nearly proportional, $J^T J$ becomes nearly singular:
\begin{equation}
    \lambda_{\min}(J^T J) \to 0
    \quad \Longrightarrow \quad
    \|(J^T J)^{-1}\| \to \infty,
    \label{eqn:gn_blowup}
\end{equation}
and the step $\Delta\boldsymbol{\theta}$ becomes unbounded
in the direction of the near-null eigenvector.  Under
i.i.d.\ noise with variance $\sigma^2$, one has
\begin{equation}
    \operatorname{Cov}(\hat{\boldsymbol{\theta}})
        = \sigma^2\, (J^T J)^{-1}
    \label{eqn:gn_covariance}
\end{equation}
in the linear least-squares case
(Lemma~\ref{lem:ls_iid_covariance}), and the same
expression approximates the large-sample covariance for
nonlinear least squares
(Remark~\ref{rem:ls_cov_nonlinear}); thus the parameter
variances along the near-null direction
scale as $\sigma^2 / \lambda_{\min}(J^T J)$.  This is the
statistical manifestation of the ill-conditioning: any
parameter combination aligned with the sloppy direction
of $J^T J$ is effectively unconstrained by the data.
Section~\ref{sec:formal} shows that,
on each single collinear ray $D = k_\ell N$,
the scale-coefficient columns of $J$ are
nearly proportional whenever
$\varepsilon \ll 1$ (the law-specific exponent gap of
Definition~\ref{def:exponent_gap}; e.g.,
$|\alpha-\beta|$ for Chinchilla), driving
$\lambda_{\min}(J^T J) = O(\varepsilon^2)$, hence
$\kappa(J^T J) = \Omega(\varepsilon^{-2})$ and
$\kappa_{A,B} = \Theta(\varepsilon^{-2})$.
With $K > 1$ rays, the spread of the ratios separates the pair as
$\varepsilon \to 0$:
$V_K \geq \tau_K$ in~\eqref{eqn:tpp_diversity} bounds $\kappa_{A,B}$
by $\kappa_{\mathrm{target}}$ at leading order
(Proposition~\ref{prop:holdout_r2}). At a fixed gap the separation is the
weighted spread $V_{\mathrm{w}}$ of Lemma~\ref{lem:weighted_spread}, which
the same sizes on every ray keep at least $V_K \min_i N_i^{2(\alpha - \beta)}$
and which other designs on $K \geq 2$ rays can bring to $0$
(Lemma~\ref{lem:ray_spread}).
For $K = 2$, $V_2 \geq \tau_2$ holds iff
$R = k_2/k_1 \geq R_{\min}(k_1)$, which increases with $k_1 \geq 1$
(Eq.~\eqref{eqn:R_min_closed}, Table~\ref{tab:R_min_eq5}), a scale effect
of the $B$ column only.
\end{remark}

\subsection{Appendix - Covariance under homoscedastic i.i.d.\ noise}\label{app:proof_ls_covariance}

\begin{lemma}[Parameter covariance in linear least squares]
\label{lem:ls_iid_covariance}
For any $\mathbf{M} \in \mathbb{R}^{m \times p}$ of full column rank
and $\boldsymbol{\theta}^\star \in \mathbb{R}^p$, suppose
\[
    L_i = (\mathbf{M}\boldsymbol{\theta}^\star)_i + \epsilon_i,
    \qquad i = 1, \dots, m,
\]
where $\{\epsilon_i\}_{i=1}^m$ are i.i.d.\ with mean zero and
variance~$\sigma^2$.
Let $\hat{\boldsymbol{\theta}}$ minimize
$\sum_{i=1}^{m} \bigl(L_i - (\mathbf{M}\boldsymbol{\theta})_i\bigr)^2$,
and let $J \in \mathbb{R}^{m \times p}$ be the Jacobian of the
residuals
$r_i(\boldsymbol{\theta})
\coloneqq L_i - (\mathbf{M}\boldsymbol{\theta})_i$
with respect to~$\boldsymbol{\theta}$ (holding $\mathbf{M}$ fixed).
Then $J = -\mathbf{M}$ and
\begin{equation}
    \operatorname{Cov}(\hat{\boldsymbol{\theta}})
        = \sigma^2\, (J^T J)^{-1}.
    \label{eqn:lem_ls_cov}
\end{equation}
\end{lemma}

\begin{proof}
Differentiating
$r_i(\boldsymbol{\theta})
= L_i - \sum_{j=1}^{p} M_{ij}\,\theta_j$
gives
$\partial r_i / \partial \theta_j = -M_{ij}$, hence
$J = -\mathbf{M}$ and $J^T J = \mathbf{M}^T \mathbf{M}$.
The normal equations yield
$\mathbf{M}^T \mathbf{M}\, \hat{\boldsymbol{\theta}}
= \mathbf{M}^T \mathbf{L}$ with
$\mathbf{L} \coloneqq (L_1, \dots, L_m)^T$, so
\[
    \hat{\boldsymbol{\theta}}
        = (\mathbf{M}^T \mathbf{M})^{-1}
          \mathbf{M}^T \mathbf{L}.
\]
Substituting
$\mathbf{L}
= \mathbf{M}\boldsymbol{\theta}^\star + \boldsymbol{\epsilon}$ with
$\boldsymbol{\epsilon} \coloneqq (\epsilon_1, \dots, \epsilon_m)^T$:
\[
    \hat{\boldsymbol{\theta}} - \boldsymbol{\theta}^\star
        = (\mathbf{M}^T \mathbf{M})^{-1}
          \mathbf{M}^T \boldsymbol{\epsilon}.
\]
Because $\mathbb{E}[\boldsymbol{\epsilon}] = \mathbf{0}$,
also $\mathbb{E}[\hat{\boldsymbol{\theta}}] = \boldsymbol{\theta}^\star$.
Because
$\operatorname{Cov}(\boldsymbol{\epsilon}) = \sigma^2 I_m$,
\begin{align*}
    \operatorname{Cov}(\hat{\boldsymbol{\theta}})
        &= (\mathbf{M}^T \mathbf{M})^{-1} \mathbf{M}^T
           (\sigma^2 I_m)\,
           \mathbf{M}\, (\mathbf{M}^T \mathbf{M})^{-1} \\
        &= \sigma^2\, (\mathbf{M}^T \mathbf{M})^{-1}
         = \sigma^2\, (J^T J)^{-1}.
         \qedhere
\end{align*}
\end{proof}

\begin{remark}[Nonlinear least squares]
\label{rem:ls_cov_nonlinear}
Lemma~\ref{lem:ls_iid_covariance} applies directly when the
predictor $\hat{L}(N_i, D_i; \boldsymbol{\theta})$ is affine
in~$\boldsymbol{\theta}$ on the design (e.g., fitting
$(A,B,E)$ in~\eqref{eqn:chinchilla} with
$(\alpha, \beta)$ held fixed).
For general nonlinear $\hat{L}$, the same matrix
$\sigma^2 (J^T J)^{-1}$ with
$J_{ij} = \partial r_i / \partial \theta_j$ evaluated at the
true parameter is the standard large-sample approximation to
$\operatorname{Cov}(\hat{\boldsymbol{\theta}})$ under
regularity conditions for nonlinear least
squares~\citep{bates1988nonlinear}.
\end{remark}

\subsection{Appendix - Proof of Proposition~\ref{prop:full_cond} (Full-matrix conditioning)}
\label{app:proof_full_cond}

\begin{proposition*}[Proposition~\ref{prop:full_cond}, restated]
    Let $J \in \mathbb{R}^{m \times p}$ have columns
    $\mathbf{j}_1, \dots, \mathbf{j}_p$ of norm $O(1)$, and suppose
    two columns $a \neq b$ satisfy
    $\mathbf{j}_b = c\,\mathbf{j}_a + \boldsymbol{\delta}$
    with a fixed $c \neq 0$, $\|\mathbf{j}_a\| = \Theta(1)$ and
    $\|\boldsymbol{\delta}\| = O(\varepsilon)$, where
    $\varepsilon$ is the law-specific exponent gap
    (Definition~\ref{def:exponent_gap}). Then as $\varepsilon \to 0$,
    $\lambda_{\max}(J^T J) = \Theta(1)$,
    $\lambda_{\min}(J^T J) = O(\varepsilon^2)$, and
    $\kappa(J^T J) = \Omega(\varepsilon^{-2})$~\eqref{eqn:rayleigh_bound};
    if the part of $\boldsymbol{\delta}$ orthogonal to
    $\mathbf{j}_a$ has norm $\Theta(\varepsilon)$,
    $\kappa((J^T J)_{a,b}) = \Theta(\varepsilon^{-2})$, which is
    $\kappa_{A,B}$ when $(a,b)$ is the scale pair.
\end{proposition*}

We proceed in three steps: (i)~show that Proposition~\ref{prop:full_cond}'s assumption $\mathbf{j}_b = c\,\mathbf{j}_a + \boldsymbol{\delta}$ forces
$\lambda_{\min}(J_1^TJ_1) = O(\varepsilon^2)$;
(ii)~use the Rayleigh quotient to transfer this to the full
$p \times p$ Gram matrix; (iii)~conclude the condition-number
bound.

\noindent\textbf{Step 1: $\lambda_{\min}(J_1^TJ_1) = O(\varepsilon^2)$.}

By the near proportionality assumption of Proposition~\ref{prop:full_cond},
$\mathbf{j}_b = c\,\mathbf{j}_a + \boldsymbol{\delta}$ with
$\|\boldsymbol{\delta}\| = O(\varepsilon)$.  Let
\begin{equation}
    J_1 \coloneqq [\,\mathbf{j}_a \mid \mathbf{j}_b\,] \in \mathbb{R}^{m\times 2}
    \label{eqn:J1_def}
\end{equation}
hold the two columns. Its $2 \times 2$ Gram matrix is:
\begin{equation}
    G \coloneqq J_1^T J_1
        = \begin{pmatrix}
            \|\mathbf{j}_a\|^2 &
                \mathbf{j}_a^T \mathbf{j}_b \\[3pt]
            \mathbf{j}_b^T \mathbf{j}_a &
                \|\mathbf{j}_b\|^2
          \end{pmatrix}.
    \label{eqn:G_2x2}
\end{equation}
Substituting $\mathbf{j}_b = c\,\mathbf{j}_a
+ \boldsymbol{\delta}$:
\begin{align}
    \|\mathbf{j}_b\|^2
        &= c^2\,\|\mathbf{j}_a\|^2
          + 2c\,\mathbf{j}_a^T\boldsymbol{\delta}
          + \|\boldsymbol{\delta}\|^2,
    \label{eqn:jb_norm} \\[3pt]
    \mathbf{j}_a^T \mathbf{j}_b
        &= c\,\|\mathbf{j}_a\|^2
          + \mathbf{j}_a^T\boldsymbol{\delta}.
    \label{eqn:ja_jb_inner}
\end{align}
Write $s \coloneqq \|\mathbf{j}_a\|^2 > 0$ and
$\eta \coloneqq \mathbf{j}_a^T\boldsymbol{\delta}$ (so
$|\eta| \leq \|\mathbf{j}_a\|\,\|\boldsymbol{\delta}\|
= O(\varepsilon)$ by Cauchy-Schwarz).  Then:
\[
    G = \begin{pmatrix}
            s & cs + \eta \\[2pt]
            cs + \eta &
                c^2 s + 2c\eta + \|\boldsymbol{\delta}\|^2
        \end{pmatrix}.
\]
The determinant is:
\begin{align}
    \det(G)
        &= s\,(c^2 s + 2c\eta + \|\boldsymbol{\delta}\|^2)
          - (cs + \eta)^2
    \nonumber \\[3pt]
        &= s\, c^2 s + 2sc\eta + s\|\boldsymbol{\delta}\|^2
          - c^2 s^2 - 2cs\eta - \eta^2
    \nonumber \\[3pt]
        &= s\,\|\boldsymbol{\delta}\|^2 - \eta^2.
    \label{eqn:det_G_expanded}
\end{align}
By the Cauchy-Schwarz inequality applied to $\mathbf{j}_a$
and $\boldsymbol{\delta}$:
$\eta^2 = (\mathbf{j}_a^T\boldsymbol{\delta})^2
\leq \|\mathbf{j}_a\|^2\,\|\boldsymbol{\delta}\|^2
= s\,\|\boldsymbol{\delta}\|^2$,
with equality if and only if $\boldsymbol{\delta}$ is
proportional to $\mathbf{j}_a$ (which would make
$\mathbf{j}_b$ exactly proportional to $\mathbf{j}_a$ and
hence $\det(G) = 0$).  In general:
\begin{equation}
    0 \;\leq\; \det(G)
        \;=\; s\,\|\boldsymbol{\delta}\|^2 - \eta^2
        \;\leq\; s\,\|\boldsymbol{\delta}\|^2
        \;=\; O(\varepsilon^2).
    \label{eqn:det_G_bound}
\end{equation}
Since $G$ is a $2 \times 2$ positive semidefinite matrix with
$\operatorname{tr}(G)
= s + c^2 s + O(\varepsilon) = \Theta(1)$,
its eigenvalues $\lambda_+ \geq \lambda_- \geq 0$ satisfy
$\lambda_+ + \lambda_- = \operatorname{tr}(G) = \Theta(1)$
and $\lambda_+ \cdot \lambda_- = \det(G) = O(\varepsilon^2)$.
Therefore:
\begin{equation}
    \lambda_-(G)
        = \frac{\det(G)}{\lambda_+(G)}
        = \frac{O(\varepsilon^2)}{\Theta(1)}
        = O(\varepsilon^2).
    \label{eqn:lam_min_G}
\end{equation}

Let $\mathbf{w} \in \mathbb{R}^2$ be a unit eigenvector of $G$
corresponding to $\lambda_- = \lambda_{\min}(G)$.  Then:
\begin{equation}
    \mathbf{w}^T G\, \mathbf{w}
        = \lambda_{\min}(G)
        = O(\varepsilon^2).
    \label{eqn:w_quadratic}
\end{equation}

\noindent\textbf{Step 2: Rayleigh-quotient transfer to the full system.}

Partition $J = [\,J_1 \mid J_2\,]$ where
$J_1 \in \mathbb{R}^{m \times 2}$ consists of columns $a$ and $b$,
and $J_2 \in \mathbb{R}^{m \times (p-2)}$ contains all remaining
columns.  Construct the padded test vector
$\tilde{\mathbf{w}} \in \mathbb{R}^p$ by placing the entries of
$\mathbf{w}$ in positions $a$ and $b$:
\begin{equation}
    \tilde{w}_a = w_1, \qquad \tilde{w}_b = w_2, \qquad
    \tilde{w}_j = 0 \text{ otherwise},
    \label{eqn:w_tilde_def}
\end{equation}
so that $\|\tilde{\mathbf{w}}\| = \|\mathbf{w}\| = 1$.

The full Gram matrix has block structure:
\begin{equation}
    J^T J
        = \begin{pmatrix}
            J_1^T J_1 & J_1^T J_2 \\[2pt]
            J_2^T J_1 & J_2^T J_2
          \end{pmatrix}.
    \label{eqn:JTJ_block}
\end{equation}
Evaluating the quadratic form at $\tilde{\mathbf{w}}$:
\begin{align}
    \tilde{\mathbf{w}}^T (J^T J)\, \tilde{\mathbf{w}}
        &= \begin{pmatrix} \mathbf{w} \\ \mathbf{0} \end{pmatrix}^{\!T}
           \begin{pmatrix}
               J_1^T J_1 & J_1^T J_2 \\
               J_2^T J_1 & J_2^T J_2
           \end{pmatrix}
           \begin{pmatrix} \mathbf{w} \\ \mathbf{0} \end{pmatrix}
    \nonumber \\[4pt]
        &= \mathbf{w}^T (J_1^T J_1)\, \mathbf{w}
          + \underbrace{\mathbf{0}^T (J_2^T J_1)\, \mathbf{w}}_{=\,0}
          + \underbrace{\mathbf{w}^T (J_1^T J_2)\, \mathbf{0}}_{=\,0}
          + \underbrace{\mathbf{0}^T (J_2^T J_2)\, \mathbf{0}}_{=\,0}
    \nonumber \\[4pt]
        &= \mathbf{w}^T (J_1^T J_1)\, \mathbf{w}
         = \lambda_{\min}(J_1^T J_1)
         = O(\varepsilon^2).
    \label{eqn:rayleigh_expansion}
\end{align}
By the Rayleigh-Ritz principle,
$\lambda_{\min}(M)=\min_{\|\mathbf{v}\|=1}\mathbf{v}^T M\mathbf{v}$,
so every unit~$\mathbf{v}$ satisfies
$\lambda_{\min}(M)\leq \mathbf{v}^T M\mathbf{v}$:
\begin{equation}
    \lambda_{\min}(J^T J)
        \leq \tilde{\mathbf{w}}^T (J^T J)\, \tilde{\mathbf{w}}
        = O(\varepsilon^2).
    \label{eqn:rayleigh_final}
\end{equation}

\noindent\textbf{Step 3: Condition-number bound.}

For $\lambda_{\max}$: it is at least
$\lambda_{\max}(J_1^T J_1) \geq \operatorname{tr}(G)/2
= \Theta(1)$
(since $G$ has two eigenvalues summing to
$\operatorname{tr}(G) = \Theta(1)$, the larger is at least
half the trace), and
$\lambda_{\max}(J^T J) \leq \operatorname{tr}(J^T J)
= \|J\|_F^2 = \Theta(1)$ for any non-degenerate dataset
(all columns of $J$ have bounded norms).  Therefore
$\lambda_{\max}(J^T J) = \Theta(1)$, and:
\begin{equation}
    \kappa(J^T J)
        = \frac{\lambda_{\max}(J^T J)}
               {\lambda_{\min}(J^T J)}
        \geq \frac{\Theta(1)}{O(\varepsilon^2)}
        = \Omega(\varepsilon^{-2}).
    \label{eqn:kappa_final}
\end{equation}
For the pair alone, $\det(J_1^T J_1) = \|\mathbf{j}_a\|^2
\|(I - P_{\mathbf{j}_a})\mathbf{j}_b\|^2 = \Theta(\varepsilon^2)$
under the additional hypothesis, where $P_{\mathbf{j}_a}$ projects
onto $\mathbf{j}_a$, while $\mathrm{tr}(J_1^T J_1) = \Theta(1)$;
so $\lambda_{\min}(J_1^T J_1) = \det/\lambda_{\max} =
\Theta(\varepsilon^2)$ and $\kappa(J_1^T J_1) = \Theta(\varepsilon^{-2})$.
No matching upper bound holds for $\kappa(J^T J)$ in general: with
free exponents on one ray, the exponent columns add a second
near-null direction, and numerically $\kappa(J^T J)$ of the
five-parameter Chinchilla law grows as $\varepsilon^{-6}$
(Appendix~\ref{app:proof_ci_inflation}).\qed

\noindent\textbf{Other smooth objectives.}
The argument uses only the column relation
$\mathbf{j}_b = c\,\mathbf{j}_a + \boldsymbol{\delta}$.
For an objective $\sum_i \rho(r_i)$ with $\rho'' \geq 0$
(weighted least squares, log-space fitting, or a Huber loss
on residuals in its quadratic zone), the Gauss-Newton
curvature is $J^T W J$ with
$W = \operatorname{diag}(\rho''(r_1), \dots, \rho''(r_m))$.
Since $W^{1/2}\mathbf{j}_b = c\,W^{1/2}\mathbf{j}_a
+ W^{1/2}\boldsymbol{\delta}$, the proof above applies to
$W^{1/2} J$ whenever the weights are bounded above and away
from zero: $\kappa(J^T W J) = \Omega(\varepsilon^{-2})$.

\subsection{Appendix - Cauchy-Schwarz gap identity}\label{app:proof_cs_gap}

\begin{lemma}[Cauchy-Schwarz gap identity]
\label{lem:cs_gap}
Let $\mathbf{f}, \mathbf{g} \in \mathbb{R}^m$ with
$g_i = c\, f_i\, h_i$ for a scalar $c \neq 0$ and
arbitrary $h_i > 0$.  Define the power-law weights
$w_i \coloneqq f_i^2 \!\big/ \sum_j f_j^2$ and the
weighted variance
$\sigma_w^2(h) \coloneqq
  \sum_i w_i\, h_i^2 - \bigl(\sum_i w_i\, h_i\bigr)^{\!2}$.
Then:
\begin{equation}
    \det\!\begin{pmatrix}
        \mathbf{f}^T\mathbf{f} \!\!\!\! &
            \mathbf{f}^T\mathbf{g} \\
        \mathbf{g}^T\mathbf{f} \!\!\!\! &
            \mathbf{g}^T\mathbf{g}
    \end{pmatrix}
    = c^2 \Bigl(\textstyle\sum_i f_i^2\Bigr)^{\!2}\;
      \sigma_w^2(h).
    \label{eqn:cs_gap}
\end{equation}
\end{lemma}

\begin{proof}
Expand:
$\mathbf{f}^T\mathbf{g} = c\sum_i f_i^2 h_i$,
$\mathbf{g}^T\mathbf{g} = c^2\sum_i f_i^2 h_i^2$,
so
$\det = (\sum_i f_i^2)(c^2\sum_i f_i^2 h_i^2) - c^2(\sum_i f_i^2 h_i)^2
      = c^2(\sum_i f_i^2)^2[\sum_i w_i h_i^2 - (\sum_i w_i h_i)^2]$.
\end{proof}

\noindent For $h_i = N_i^{\eta}$ with $|\eta|$ small, a Taylor
expansion gives:
\begin{equation}
    \sigma_w^2(N_i^{\eta})
        = \eta^2\,\sigma_w^2(\log N_i)
          + O(|\eta|^3).
    \label{eqn:taylor_var}
\end{equation}
For Chinchilla, take $\eta = \tilde{\varepsilon} \coloneqq \alpha - \beta$ and
$\varepsilon \coloneqq |\alpha-\beta|$
(Definition~\ref{def:exponent_gap}), so
$\eta^2 = \varepsilon^2$.
Thus Lemma~\ref{lem:cs_gap} applies to
$h_i = N_i^{\alpha-\beta}$, giving the exact determinant
\eqref{eqn:chin_exact_det}; since $\kappa_{A,B} \leq \mathrm{tr}^2/\det$
for a $2 \times 2$ block, this yields a non-asymptotic bound on
$\kappa_{A,B}$ depending on $\varepsilon$, the sample size $n$, and the
spread of $\{N_i\}$.

\subsection{Appendix - Chinchilla scaling law: detailed analysis}\label{app:chinchilla_details}

The Chinchilla parameter vector is
$\boldsymbol{\theta} = [A, B, E, \alpha, \beta]^T$ ($p = 5$).
The following analysis holds on a single collinear ray
$D_i = k_\ell N_i$ (i.e., $K = 1$ or a single fixed
ratio $k_\ell$ from the design);
the scale-coefficient columns then satisfy:
\begin{equation}
    \frac{\partial \hat{L}_i}{\partial B}
        = \underbrace{k_\ell^{-\beta}}_{c} \, N_i^{-\alpha}
          \, N_i^{\alpha-\beta},
    \label{eqn:partial_linear_dependence}
\end{equation}
with $\varepsilon \coloneqq |\alpha - \beta|$.
The corresponding residual Jacobian entries are
$J_{iA} = -\partial \hat{L}_i / \partial A$ and
$J_{iB} = -\partial \hat{L}_i / \partial B$, but
$(J^T J)_{A,B}$ equals the Gram matrix of
$(\partial \hat{L}/\partial A,\,\partial \hat{L}/\partial B)$
because both columns pick up the same sign.
The determinant of the $(A,B)$ sub-block of $J^T J$ satisfies:
\begin{equation}
    \det\!\left((J^T J)_{A,B}\right)
        = O(\varepsilon^2),
    \label{eqn:vanishing_determinant}
\end{equation}
and by~\eqref{eqn:chin_exact_det} it is $\Theta(\varepsilon^2)$,
so $\kappa((J^T J)_{A,B}) = \Theta(\varepsilon^{-2})$.
Proposition~\ref{prop:full_cond} gives only
$\kappa(J^T J) = \Omega(\varepsilon^{-2})$ for the full $5\times 5$
system; numerically it grows as $\varepsilon^{-6}$, since the exponent
columns add a second near-null direction (Appendix~\ref{app:proof_ci_inflation}).

\noindent\textbf{Explicit finite-sample bound.}

Applying Lemma~\ref{lem:cs_gap} with
$f_i = N_i^{-\alpha}$, $c = k_\ell^{-\beta}$, and
$h_i = N_i^{\alpha-\beta}$, the sub-block determinant
admits the exact factorization:
\begin{equation}
    \det\!\left((J^T J)_{A,B}\right)
        = k_\ell^{-2\beta}\,\Phi_2^{\,2}\;
          \sigma_w^2\!\bigl(N_i^{\alpha-\beta}\bigr),
    \label{eqn:chin_exact_det}
\end{equation}
where
$\Phi_2 \coloneqq \sum_{i=1}^{n} N_i^{-2\alpha}$
and the weights are
$w_i = N_i^{-2\alpha}/\Phi_2$.
Using the Taylor expansion~\eqref{eqn:taylor_var}:
\begin{equation}
    \det\!\left((J^T J)_{A,B}\right)
        \approx k_\ell^{-2\beta}\,\Phi_2^{\,2}\;
          \varepsilon^2\;\sigma_w^2(\log N),
    \label{eqn:chin_approx_det}
\end{equation}
with $\sigma_w^2(\log N)
= \sum_i w_i (\log N_i)^2
  - \bigl(\sum_i w_i \log N_i\bigr)^{\!2}$.
The condition number of the sub-block is therefore:
\begin{equation}
    \kappa\bigl((J^T J)_{A,B}\bigr)
        \approx
        \frac{\operatorname{tr}\!\bigl((J^T J)_{A,B}\bigr)^2}
             {k_\ell^{-2\beta}\,\Phi_2^{\,2}\;
              \varepsilon^2\;\sigma_w^2(\log N)}\,,
    \label{eqn:chin_kappa_explicit}
\end{equation}
which makes the three sources of ill-conditioning explicit:
(i) small exponent gap $\varepsilon \to 0$;
(ii) narrow model-size range, i.e.\ small
$\sigma_w^2(\log N)$ when the $N_i$ are clustered; and
(iii) finite sample effects through $\Phi_2$ and the weights $w_i$,
which concentrate on the \emph{smallest} models and reduce
the effective sample diversity.
For the empirical gap
$|\alpha - \beta| \approx 0.06$, the exact determinant~\eqref{eqn:chin_exact_det} gives $\kappa_{A,B}$ on
the order of $10^3$-$10^4$ on our grid, rendering the scale pair numerically
ill-conditioned. Equation~\eqref{eqn:chin_kappa_explicit} overstates it
there, since on our 14 sizes $\sigma_w^2(N_i^{\tilde{\varepsilon}})$ is
$7.4$ times its Taylor value $\varepsilon^2 \sigma_w^2(\log N)$ at
$\varepsilon = 0.06$; it gives the rate, not the value.\qed

\subsection{Appendix - Full Jacobian of the Chinchilla law}\ifjournal\else\label{app:proof_tight_bound}\fi

\noindent\textbf{Full Jacobian analysis.}
The preceding argument concerns only the $A$ and $B$ columns.
We now list the five model sensitivities
$\partial\hat{L}_i/\partial\theta_j$ on a single collinear ray
$D_i = k_\ell N_i$; by~\eqref{eqn:gn_jacobian_def}, the residual Jacobian
$J$ is obtained by negating each of these columns, so
$J^T J$ equals the Gram matrix of the sensitivities below.  We claim the
$(E,\alpha,\beta)$ coordinates do not rescue the conditioning.  The complete set of
partial derivatives is:
\begin{align}
    \frac{\partial \hat{L}_i}{\partial A} &= N_i^{-\alpha},
    &\qquad
    \frac{\partial \hat{L}_i}{\partial B} &= k_\ell^{-\beta}\,
        N_i^{-\beta},
    \nonumber \\[4pt]
    \frac{\partial \hat{L}_i}{\partial E} &= 1,
    &\qquad
    \frac{\partial \hat{L}_i}{\partial \alpha}
        &= -A\, N_i^{-\alpha} \log N_i,
    \nonumber \\[4pt]
    \frac{\partial \hat{L}_i}{\partial \beta}
        &= -B\, k_\ell^{-\beta}\, N_i^{-\beta}
           (\log k_\ell + \log N_i).
    \label{eqn:chin_all_cols}
\end{align}
Observe that (writing $\mathbf{j}_u$ for the $m$-vector with entries
$\partial\hat{L}_i/\partial u$)
$\mathbf{j}_\alpha = -A\,\operatorname{diag}(\log N_i)\,\mathbf{j}_A$
and
$\mathbf{j}_\beta = -B\,(\log k_\ell)\,\mathbf{j}_B
    - B\,\operatorname{diag}(\log N_i)\,\mathbf{j}_B$.
Using $\mathbf{j}_B \approx c\,\mathbf{j}_A + O(\varepsilon)$ (from
\eqref{eqn:partial_linear_dependence}), both exponent columns
lie approximately in the two-dimensional subspace
$\operatorname{span}\{\mathbf{j}_A,\,
\operatorname{diag}(\log N_i)\,\mathbf{j}_A\}$.
In particular, if $\mathbf{j}_B = c\,\mathbf{j}_A$ exactly, then
$\mathbf{j}_\beta - (Bc/A)\,\mathbf{j}_\alpha
    = -Bc(\log k_\ell)\,\mathbf{j}_A$,
so $\mathbf{j}_\alpha$ and $\mathbf{j}_\beta$ are \emph{not} scalar
multiples whenever $\log k_\ell \neq 0$; they lie near the span of the scale-coefficient pair and its
log-modulated variants. The exponent columns nevertheless add a further near-null direction
(Appendix~\ref{app:proof_ci_inflation}); the Rayleigh bound below needs
only the pair.
The $E$-column of sensitivities is $\mathbf{1}$, hence the $E$-column
of~$J$ is $-\mathbf{1}$; either way it is linearly
independent of the power-law columns but likewise does not break the
collinearity between the $A$/$B$ pair.

By Proposition~\ref{prop:full_cond}, the near-null eigenvector
$\mathbf{w}$ of $(J^T J)_{A,B}$, extended to
$\tilde{\mathbf{w}} \in \mathbb{R}^5$ with zeros in the
$(E,\alpha,\beta)$ coordinates matching~\eqref{eqn:w_tilde_def},
satisfies
$\tilde{\mathbf{w}}^T (J^T J)\,\tilde{\mathbf{w}} = O(\varepsilon^2)$,
hence
$\kappa(J^T J) = \Omega(\varepsilon^{-2})$ for the full
$5\times 5$ system; the exponent columns make it grow faster
(Appendix~\ref{app:proof_ci_inflation}).\qed

\subsection{Appendix - Repeated-data scaling law: detailed analysis}\label{app:dataconstrained_details}

This subsection analyses the Jacobian of the repeated-data law on
a single collinear ray ($K = 1$, $D_i = k_\ell N_i$).
\citet{muennighoff2023scaling} retain the Chinchilla form but
replace $N$ and $D$ with \emph{effective} quantities $N'$ and $D'$
that decay under repetition:
\begin{equation}
    L(N, D) = \frac{A}{N'^\alpha} + \frac{B}{D'^\beta} + E,
    \label{eqn:dataconstrained}
\end{equation}
with
$D' = U_D + U_D R_D^*(1 - e^{-R_D / R_D^*})$,
$U_D = \min\{D_C, D\}$, $R_D = (D / U_D) - 1$, and analogously
$N' = U_N + U_N R_N^*(1 - e^{-R_N / R_N^*})$,
$R_N = (N/U_N) - 1$, with $U_N$ Chinchilla-optimal for $U_D$.
When $R_D = R_N = 0$ (single epoch, no excess parameters),
$D' = D$ and $N' = N$, recovering Chinchilla exactly. We fit~\eqref{eqn:dataconstrained}
as the released code does, with $U_N = N$, $U_D$ the unique tokens and,
in place of the definitions above, $R_N = R_D = e - 1$ at epoch $e$,
all three epochs in one fit.

Under $D_i = k_\ell N_i$, both $D_i'$ and $N_i'$ become
deterministic functions of $N_i$; write $D_i' = \phi(N_i)$ and
$N_i' = \nu(N_i)$. The scale-coefficient sensitivities are
$\partial\hat{L}_i/\partial A = \nu(N_i)^{-\alpha}$ and
$\partial\hat{L}_i/\partial B = \phi(N_i)^{-\beta}$, so
Lemma~\ref{lem:cs_gap} applies to the scale-coefficient pair with
$f_i = \nu(N_i)^{-\alpha}$, $c = k_\ell^{-\beta}$, and
$h_i = k_\ell^{\beta}\, \phi(N_i)^{-\beta}\, \nu(N_i)^{\alpha}$,
yielding
$\det\!\left((J^T J)_{A,B}\right)
    = c^2\, \Phi_2^{\,2}\, \sigma_w^2(h)$
with $\Phi_2 = \sum_i \nu(N_i)^{-2\alpha}$ and
$w_i = \nu(N_i)^{-2\alpha}/\Phi_2$. The conditioning of the
$(A,B)$ sub-block therefore depends on the variance of $h_i$ across
the ray, which differs sharply between two regimes.

\noindent\textbf{Unsaturated regime ($R_D = R_N = 0$).}
Here $\phi(N) = k_\ell N$ and $\nu(N) = N$, so
$h_i = N_i^{\alpha - \beta}$. The Taylor expansion
$\sigma_w^2(N_i^{\,\alpha-\beta})
    \approx \varepsilon^2\, \sigma_w^2(\log N_i)$
(with $\varepsilon = |\alpha - \beta|$) gives
$\det((J^T J)_{A,B}) = \Theta(\varepsilon^2)$, so
$\kappa_{A,B} = \Theta(\varepsilon^{-2})$ and, by
Proposition~\ref{prop:full_cond}, $\kappa(J^T J) =
\Omega(\varepsilon^{-2})$ for the full $5 \times 5$ system: the law
inherits Chinchilla's ill-conditioning.

\noindent\textbf{Saturated regime ($R_D, R_N \neq 0$).}
Once the ratio $\phi(N_i)/\nu(N_i)$ varies systematically with
$N_i$, $h_i$ ceases to be a pure power of $N_i$: the saturation
curves of $\phi$ and $\nu$ inject an
$\varepsilon$-independent contribution into $\log h_i$.
We therefore expect $\sigma_w^2(h)$ to stay bounded away from zero as
$\varepsilon \to 0$, in which case the Taylor expansion that produced the
$\varepsilon^2$ factor no longer applies, the scale-coefficient columns
are not near-proportional, and $\det((J^T J)_{A,B}) = \Omega(1)$. This is
a heuristic: we do not bound $\sigma_w^2(h)$ away from zero, nor treat the
exponent and saturation columns $R_D^*, R_N^*$, so we claim no rate for
$\kappa(J^T J)$ in this regime.

The repeated-data law therefore inherits the
$\Theta(\varepsilon^{-2})$ collinearity of Chinchilla \emph{only}
in the unsaturated regime. The $(A,B)$ sub-block of $Q$ is
identically zero (as for Chinchilla) because $\hat{L}$ is linear
in $A$ and $B$. Empirically, in our seed-paired comparisons the
repeated-data law tracks the Chinchilla pattern (CO holdout
$R^2$ degrades and NC consistently wins), although its epoch-2 and
epoch-3 observations are in the saturated regime: $R_D = 1, 2$
against fitted $R^*_D$ of $0.80$-$2.35$
(Table~\ref{tab:epoch-params-r2}), so at epoch 3 $D'$ is $58$-$78\%$
of the tokens seen. The formal guarantee does not cover these fits,
and the NC advantage on this law is empirical.\qed

\subsection{Appendix - Kaplan scaling law: detailed analysis}\label{app:kaplan_details}
This subsection expands on Section~\ref{sec:formal}
for a single collinear ray ($K=1$, $D_i = k_\ell N_i$).
We fit the compositional form of \citet{kaplan2020scaling},
\begin{equation}
    \hat{L}(N, D) = u^{\alpha_D},
    \qquad
    u \coloneqq (N_c/N)^{s} + D_c/D,
    \qquad
    s \coloneqq \alpha_N/\alpha_D,
    \label{eqn:kaplan_form}
\end{equation}
with $p=4$ parameters $\boldsymbol{\theta} = [N_c, D_c, \alpha_N, \alpha_D]^T$
(the code's \texttt{kaplan\_loss}; Table~\ref{tab:kaplan-params}).
It is not a reparameterization of the additive form
$(N_c/N)^{\alpha_N} + (D_c/D)^{\alpha_D}$: its mixed derivative
$\partial^2 \hat{L}/\partial N \partial D$ is nonzero, while the additive
form's is zero. Its columns are
\begin{align}
    \mathbf{j}_{N_c} &= \alpha_N\, u^{\alpha_D - 1}\, N_c^{s-1} N^{-s},
    &
    \mathbf{j}_{D_c} &= \alpha_D\, u^{\alpha_D - 1}\, D^{-1},
    \label{eqn:kaplan_scale_cols} \\
    \mathbf{j}_{\alpha_N} &= u^{\alpha_D - 1} (N_c/N)^{s} \log(N_c/N),
    &
    \mathbf{j}_{\alpha_D} &= \hat{L}\log u - s\,\mathbf{j}_{\alpha_N},
    \label{eqn:kaplan_exp_cols}
\end{align}
with products taken element-wise over the observations.

\noindent\textbf{Exponent gap.}
On the ray $D = k_\ell N$ the data term $D_c/(k_\ell N)$ falls as $N^{-1}$,
against $N^{-s}$ for the model term, so
$\mathbf{j}_{D_c} = c\, \mathbf{j}_{N_c} \odot \mathbf{h}$ with
$c = (\alpha_D/\alpha_N)\, N_c^{1-s} k_\ell^{-1}$ and
$h_i = N_i^{\tilde\varepsilon}$, $\tilde\varepsilon \coloneqq s - 1 =
\alpha_N/\alpha_D - 1$. The gap is
$\varepsilon = |1 - \alpha_N/\alpha_D|$
(Definition~\ref{def:exponent_gap}), not $|\alpha_D - \alpha_N|$.
At $\varepsilon = 0$ the law on the ray is
$((N_c + D_c/k_\ell)/N)^{\alpha_D}$, a function of $N_c + D_c/k_\ell$
and $\alpha_D$ alone, so the Jacobian has rank $2$.
Lemma~\ref{lem:cs_gap} with $f_i = (\mathbf{j}_{N_c})_i$ gives
$\det((J^T J)_{N_c,D_c}) \approx c^2 \Phi_2^{\,2}\,\varepsilon^2\,
\sigma_w^2(\log N)$ with $\Phi_2 \coloneqq \sum_i f_i^2$ and weights
$w_i \propto f_i^2$; ill-conditioning worsens when the model sizes are
clustered or $\varepsilon$ is small. The exponent columns
\eqref{eqn:kaplan_exp_cols} add a further near-null direction, as for
Chinchilla (Appendix~\ref{app:proof_ci_inflation}); the Rayleigh bound
below needs only the pair.

\noindent\textbf{Size of the gap.}
At the single-variable exponents of \citet{kaplan2020scaling} ($\alpha_N = 0.076$ from
their $L(N)$, $\alpha_D = 0.095$ from $L(D)$), $\varepsilon = 0.20$ (their joint
$L(N, D)$ fit has $\alpha_D = 0.103$, $\varepsilon = 0.26$); with unit-norm columns
$\kappa_{N_c,D_c}$ is $1.6$-$1.7{\times}10^{2}$ and $\kappa(J^T J)$ is
$4.3$-$4.8{\times}10^{8}$ on one ray of our 14 sizes ($k = 1$-$20$), against
$1.2$-$1.4{\times}10^{14}$ for Chinchilla at Hoffmann et al.'s fit.
The seed-0 fits of Table~\ref{tab:kaplan-params} and its first- and
second-epoch companions have $\varepsilon = 0.105$ to $0.882$ on CO; over
all $450$ CO fits ($30$ seeds), $445$ have $\varepsilon \geq 0.105$, the
smallest is $0.004$, and the BF16 CO fit of
Table~\ref{tab:kaplan-params-bf16} has $0.008$. At every seed-0 fit, on
its own CO or NC training design, Kaplan has the smallest unit-norm
$\kappa(J^T J)$ of the three $(N, D)$ laws. Its poor off-ray prediction
after CO training (mean $R^2$ on $\mathcal{H}_{\mathrm{nc}}$ of $0.50$
against $0.78$ for Chinchilla) is therefore not a conditioning effect:
its Gauss-Newton predictive standard deviation on $\mathcal{H}_{\mathrm{nc}}$ is $0.049$
against $0.048$ for Chinchilla, while its holdout RMSE is $0.366$ against
$0.242$ (medians over the seed-0 CO fits); the error is bias, part of it in the
CO fits whose $(N_c, D_c)$ never left a start draw. In $183$ of the $900$ Kaplan
fits ($37$ CO, $146$ NC; $9$ at seed $0$) the optimizer left $(N_c, D_c)$ within a relative
$10^{-6}$ of one of its start draws in native units and moved only the exponents
(Appendix~\ref{app:optimizer_bounds}); four of them
are rows of Table~\ref{tab:kaplan-params}, marked there
(\texttt{tools/kaplan\_numbers.py} computes every number of this
paragraph).

Applying Proposition~\ref{prop:full_cond} with $p=4$:
the extended near-null vector $\tilde{\mathbf{w}}\in\mathbb{R}^4$
padding $\mathbf{w}$ with zeros in the
$(\alpha_N,\alpha_D)$ coordinates as in~\eqref{eqn:w_tilde_def}
satisfies
$\tilde{\mathbf{w}}^T(J^T J)\tilde{\mathbf{w}} = O(\varepsilon^2)$,
hence $\kappa(J^T J) = \Omega(\varepsilon^{-2})$ for the full
$4 \times 4$ system; with free exponents the full rate is faster,
as for Chinchilla.\qed

\subsection{Appendix - Droppo-Elibol scaling law: detailed analysis}\label{app:elibol_details}
We use the form fit in our experiments and recorded in
Table~\ref{tab:jacobian_summary},
\begin{equation}
    \hat{L}(N, D)
        = \bigl[
            L_\infty^{1/\alpha}
            + (N_C/N)^{\gamma_N}
            + (D_C/D)^{\gamma_D}
        \bigr]^{\alpha},
    \qquad
    \gamma_N \coloneqq \alpha_N/\alpha,\;\;
    \gamma_D \coloneqq \alpha_D/\alpha.
    \label{eqn:elibol_fit_form}
\end{equation}
\citet{droppo2021scaling} motivate this compositional structure;
we fit~\eqref{eqn:elibol_fit_form} with $p = 6$ parameters
$(L_\infty, N_C, D_C, \alpha_N, \alpha_D, \alpha)$, as in
Definition~\ref{def:exponent_gap} and
Table~\ref{tab:droppo-params} (Table~\ref{tab:jacobian_summary} writes
$E$ for $L_\infty$).
Substituting $D = k_\ell N$ (single collinear ray, $K=1$),
\begin{equation}
    \hat{L}(N, k_\ell N)
        = \left[
            L_\infty^{1/\alpha}
            + N^{-\gamma_N}\!\left(
                N_C^{\gamma_N}
                + D_C^{\gamma_D}\, k_\ell^{-\gamma_D}\,
                  N^{\gamma_N - \gamma_D}
              \right)
        \right]^{\alpha}.
    \label{eqn:elibol_merged}
\end{equation}
In particular, when $\gamma_N = \gamma_D$ this reduces to
$[\,L_\infty^{1/\alpha}
  + C_{\mathrm{total}}\, N^{-\gamma_N}\,]^{\alpha}$ with
$C_{\mathrm{total}}
  = N_C^{\gamma_N} + D_C^{\gamma_D}\, k_\ell^{-\gamma_D}$.
When $\varepsilon = |\gamma_N - \gamma_D|$
(Definition~\ref{def:exponent_gap}) is small, Taylor expanding
$N^{\gamma_N - \gamma_D}$ about $\gamma_N = \gamma_D$ recovers the
same $O(\varepsilon)$ near-proportionality as in the Chinchilla case.
The irreducible loss $L_\infty$ remains independently identifiable.

\noindent\textbf{Explicit finite-sample bound.}
With $v \coloneqq L_\infty^{1/\alpha} + (N_C/N)^{\gamma_N} +
(D_C/D)^{\gamma_D}$, the scale columns are
$\partial\hat{L}/\partial N_C = \alpha\, v^{\alpha-1}\, \gamma_N\,
N_C^{\gamma_N - 1} N^{-\gamma_N}$ and
$\partial\hat{L}/\partial D_C = \alpha\, v^{\alpha-1}\, \gamma_D\,
D_C^{\gamma_D - 1} D^{-\gamma_D}$.
After absorbing the shared chain-rule factor
$\alpha\, v_i^{\alpha-1}$ into $f_i$,
Lemma~\ref{lem:cs_gap} applies with
$f_i \propto \alpha\, v_i^{\alpha-1}\, \gamma_N\, N_C^{\gamma_N - 1}\, N_i^{-\gamma_N}$,
the scalar
$\lambda = (\gamma_D / \gamma_N)\,
(D_C^{\gamma_D - 1} / N_C^{\gamma_N - 1})\,
k_\ell^{-\gamma_D}$,
and $h_i = N_i^{\tilde{\varepsilon}}$ with
$\tilde{\varepsilon} \coloneqq \gamma_N - \gamma_D$
and $\varepsilon = |\tilde{\varepsilon}|$, yielding
\begin{equation}
  \det\!\left((J^T J)_{N_C,D_C}\right)
  \approx \lambda^2 \Phi_2^{\,2}\,
  \varepsilon^2\,\sigma_w^2(\log N),
    \label{eqn:elibol_det_approx}
\end{equation}
with weights $w_i \propto f_i^2$ and $\Phi_2 = \sum_i f_i^2$.
As in the Chinchilla and Kaplan cases,
ill-conditioning worsens when the training configurations
are clustered in $\log N$ or when $|\gamma_N - \gamma_D|$ is small.

\noindent\textbf{Full Jacobian analysis.}
Write $\mathbf{j}_u$ for the sensitivity vector
$(\partial\hat{L}_i/\partial u)_i$.
We derive the four remaining columns
$\mathbf{j}_{L_\infty}$,
$\mathbf{j}_{\alpha_N}$, $\mathbf{j}_{\alpha_D}$, and
$\mathbf{j}_\alpha$, and show they do not break the near-singularity.

\noindent\textit{Irreducible loss:}
the chain rule gives
$\mathbf{j}_{L_\infty}
  = \alpha\, v^{\alpha-1}\, \partial v/\partial L_\infty
  = L_\infty^{1/\alpha - 1}\, (v_i^{\alpha-1})_i$, i.e.\
the constant scalar $L_\infty^{1/\alpha - 1}$ times the
observation-dependent vector $(v_i^{\alpha-1})_i$ - so
$\mathbf{j}_{L_\infty}$ is \emph{not} a constant vector,
and only the scalar $\partial v/\partial L_\infty$ is
$i$-independent.
Linear independence with $(\mathbf{j}_{N_C}, \mathbf{j}_{D_C})$
follows because $v_i^{\alpha-1}$ depends on $(N_i, D_i)$ through
the sum $v_i$ raised to $\alpha-1$ and is
therefore not a scalar multiple of $N_i^{-\gamma_N}$ or
$D_i^{-\gamma_D}$, so $\mathbf{j}_{L_\infty}$ does not share
the $(N_C,D_C)$ near-collinearity.

\noindent\textit{Exponent columns:} with $\partial\gamma_N/\partial\alpha_N
= 1/\alpha$ and $\partial\gamma_N/\partial\alpha = -\gamma_N/\alpha$,
\begin{align}
    \frac{\partial \hat{L}}{\partial \alpha_N}
        &= v^{\alpha - 1}\,
           (N_C/N)^{\gamma_N}
           (\log N_C - \log N),
    \label{eqn:elibol_dalphaN} \\[4pt]
    \frac{\partial \hat{L}}{\partial \alpha_D}
        &= v^{\alpha - 1}\,
           (D_C/D)^{\gamma_D}
           (\log D_C - \log D),
    \label{eqn:elibol_dalphaD} \\[4pt]
    \frac{\partial \hat{L}}{\partial \alpha}
        &= v^\alpha \log v
           - v^{\alpha - 1} \Bigl[
             \tfrac{1}{\alpha} L_\infty^{1/\alpha} \log L_\infty
             + \gamma_N (N_C/N)^{\gamma_N} \log(N_C/N)
    \nonumber \\
        &\qquad\qquad\qquad
             + \gamma_D (D_C/D)^{\gamma_D} \log(D_C/D) \Bigr].
    \label{eqn:elibol_dalpha}
\end{align}
The $\alpha_N$-column~\eqref{eqn:elibol_dalphaN} equals
$\mathbf{j}_{N_C}$ scaled element-wise by
$(N_C/\alpha_N)(\log N_C - \log N_i)$;
likewise the $\alpha_D$-column~\eqref{eqn:elibol_dalphaD}
equals $\mathbf{j}_{D_C}$ scaled element-wise by
$(D_C/\alpha_D)(\log D_C - \log D_i)$.
Both therefore lie in the span of the
scale-coefficient columns and their log-modulated variants.
The $\alpha$-column~\eqref{eqn:elibol_dalpha} combines
$\log v$ with the $L_\infty$ term and log-modulated scale terms. The exponent columns add a further near-null direction
(Appendix~\ref{app:proof_ci_inflation}); the Rayleigh bound below needs
only the pair.

Applying Proposition~\ref{prop:full_cond} with $p=6$, extend the
near-null $\mathbf{w}\in\mathbb{R}^2$ for the $(N_C,D_C)$ pair to
$\tilde{\mathbf{w}}\in\mathbb{R}^6$, zero in the other
four coordinates as in~\eqref{eqn:w_tilde_def}; then
$\tilde{\mathbf{w}}^T(J^T J)\tilde{\mathbf{w}} = O(\varepsilon^2)$,
hence $\kappa(J^T J) = \Omega(\varepsilon^{-2})$ for the full
$6 \times 6$ system.  The bound comes from the model-capacity and
data-size terms.\qed

\subsection{Appendix - Identified submodels on a single collinear ray \texorpdfstring{($K=1$)}{(K=1)}}
\label{app:submodel}

This subsection provides the formal statements supporting
Section~\ref{sec:single_ray}.
Throughout, we fix $K=1$, i.e.\ all training points share
a single TPP ratio $k_\ell$, so $D_i = k_\ell N_i$.
The reduced model
$L(N) = \psi\,N^{-\alpha} + E$ with
$\psi \coloneqq A + B\,k_\ell^{-\alpha}$
(equation~\eqref{eqn:reduced_chinchilla}) is defined in the
main text.
We next define the power sums~$\Phi_q$, weights~$w_i$, and
$\sigma_w^2(\log N)$
(Definitions~\ref{def:power_sums}, \ref{def:weights},
and~\ref{def:weighted_log_variance}).
Proposition~\ref{prop:reduced_id} shows that the $(\psi,\alpha)$ Gram
block of the reduced Jacobian has determinant
$\psi^2 \Phi_2^{\,2} \sigma_w^2(\log N)$
(via Lemma~\ref{lem:cs_gap} in Appendix~\ref{app:proof_cs_gap}), positive
for any design with at least two distinct model sizes and independent
of~$\varepsilon$, and Corollary~\ref{cor:kappa_ratio} records an
$\Omega(\varepsilon^{-2})$ condition-number gain relative to the
full model.

Let $N_1 < \cdots < N_n$ be the distinct model sizes (as in the
observation model, Section~\ref{sec:formal}).

\begin{definition}[Power sums]
\label{def:power_sums}
For $q > 0$, the \emph{power sums} are
$\Phi_q \coloneqq \sum_{i=1}^{n} N_i^{-q\alpha}$.
In particular,
$\Phi_2 \coloneqq \sum_{i=1}^{n} N_i^{-2\alpha}$.
\end{definition}

\begin{definition}[Weights]
\label{def:weights}
Given~$\Phi_2$ from Definition~\ref{def:power_sums}, the
\emph{weights} are
$w_i \coloneqq N_i^{-2\alpha}/\Phi_2$ for $i = 1,\ldots,n$.
\end{definition}

\begin{definition}[Weighted log-variance]
\label{def:weighted_log_variance}
Given weights~$(w_i)$ from Definition~\ref{def:weights}, the
\emph{weighted log-variance} is
$\sigma_w^2(\log N) \coloneqq
    \sum_{i=1}^{n} w_i (\log N_i)^2
    - \bigl(\sum_{i=1}^{n} w_i \log N_i\bigr)^{\!2}$.
\end{definition}

\begin{proposition}[Identification of the reduced model]
\label{prop:reduced_id}
The Gram determinant of the $(\psi, \alpha)$ sub-block of
$J_{\mathrm{red}}^T J_{\mathrm{red}}$ is
\begin{equation}
    \det\!\bigl((J_{\mathrm{red}}^T J_{\mathrm{red}})_{\psi,\alpha}\bigr)
        = \psi^2\, \Phi_2^{\,2}\; \sigma_w^2(\log N),
    \label{eqn:reduced_det}
\end{equation}
where $\sigma_w^2(\log N)$ is as in
Definition~\ref{def:weighted_log_variance}.
The determinant is \textbf{positive} for any design with at
least two distinct
model sizes and is \textbf{independent of the exponent gap}
$\varepsilon$.  The full $3 \times 3$ matrix
$J_{\mathrm{red}}^T J_{\mathrm{red}}$, with $J_{\mathrm{red}}$ the
reduced Jacobian of columns $1$, $N^{-\alpha}$ and
$-\psi N^{-\alpha}\ln N$ (eqn.~\eqref{eqn:jred_cols}), is non-singular
iff the design has at least three distinct model sizes.
\end{proposition}

\begin{corollary}[Condition-number improvement]
\label{cor:kappa_ratio}
In the setting of Proposition~\ref{prop:reduced_id}, with at least
three distinct model sizes, and with
$\kappa_{\mathrm{full}} \coloneqq \kappa(J^T J)$ for the full
model~\eqref{eqn:chinchilla} and
$\kappa_{\mathrm{red}} \coloneqq
\kappa(J_{\mathrm{red}}^T J_{\mathrm{red}})$,
$\kappa_{\mathrm{full}} / \kappa_{\mathrm{red}}
    = \Omega(\varepsilon^{-2})$, since
$\kappa_{\mathrm{full}} \geq \kappa_{A,B} = \Theta(\varepsilon^{-2})$
while $\kappa_{\mathrm{red}}$ does not depend on~$\varepsilon$.
\end{corollary}

\subsubsection{Identified submodels: proofs}
\label{app:submodel_proofs}

\begin{proof}[Proof of Proposition~\ref{prop:reduced_id}]
Apply Lemma~\ref{lem:cs_gap} with
$f_i = N_i^{-\alpha}$ (the $\psi$-sensitivity
$\partial\hat{L}_i/\partial\psi$),
$g_i = c\, f_i\, h_i$ where $c = -\psi$ and $h_i = \log N_i$
(so that $g_i = \partial\hat{L}_i/\partial\alpha$).  The Gram determinant
is $c^2\,(\sum f_i^2)^2\,\sigma_w^2(h)
    = \psi^2\, \Phi_2^{\,2}\, \sigma_w^2(\log N)$.
No $\varepsilon$ appears because the reduced model contains a
single power-law term.
For the $3 \times 3$ matrix: with at most two distinct sizes, the
three columns live in a space of dimension at most two, so it is
singular. With three or more, suppose $c_0 + N^{-\alpha}(c_1 - c_2
\psi \ln N)$ vanishes at every size, and let $t = \ln N$ and
$f(t) = c_0 + e^{-\alpha t}(c_1 + c_2' t)$ with $c_2' = -c_2\psi$.
Then $f'(t) = e^{-\alpha t}(c_2' - \alpha c_1 - \alpha c_2' t)$ has at
most one zero unless $c_2' = c_1 = 0$, so by Rolle's theorem $f$ has at
most two zeros unless it is the constant $c_0$; three distinct zeros
therefore force $c_0 = c_1 = c_2 = 0$, and the matrix is non-singular.
\end{proof}

\begin{proof}[Proof of Corollary~\ref{cor:kappa_ratio}]
By Cauchy interlacing, $\kappa_{\mathrm{full}} \geq \kappa_{A,B}$, and
$\kappa_{A,B} = \Theta(\varepsilon^{-2})$ by
Proposition~\ref{prop:full_cond}, whose orthogonal-part condition holds
by eqn.~\eqref{eqn:chin_exact_det}. The reduced Jacobian, with
columns $1$, $N^{-\alpha}$ and $-\psi N^{-\alpha}\ln N$, has a limit as
$\varepsilon \to 0$ that is non-singular with three or more distinct
sizes (Proposition~\ref{prop:reduced_id}), so
$\kappa_{\mathrm{red}} = \Theta(1)$.
Hence $\kappa_{\mathrm{full}}/\kappa_{\mathrm{red}} =
\Omega(\varepsilon^{-2})$.
\end{proof}

\subsection{Appendix - Proof of Corollary~\ref{prop:ci_inflation} (Confidence interval inflation)}
    \label{app:proof_ci_inflation}

    \begin{corollary*}[Confidence interval inflation (full statement)]
        Under i.i.d.\ Gaussian noise with variance~$\sigma^2$ and
        a single collinear ray $D = k_\ell N$ ($K=1$) with $n \geq 5$
        distinct model sizes, the
        least-squares estimator $\hat{A}$ in the full Chinchilla
        model~\eqref{eqn:chinchilla} satisfies
        \begin{equation}
            \operatorname{Var}(\hat{A})
                \;\geq\;
                \frac{\sigma^2\;
                      \Phi_2^{(2\varepsilon)}}{
                      \Phi_2^{\,2}\;\sigma_w^2(N_i^{\tilde{\varepsilon}})},
            \label{eqn:var_A_lower}
        \end{equation}
        where
        $\Phi_2^{(2\varepsilon)} \coloneqq
        \sum_{i=1}^n N_i^{-2\beta}$
        \,(equivalently
        $\sum_i N_i^{-2\alpha + 2\tilde{\varepsilon}}$ with
        $\tilde{\varepsilon} \coloneqq \alpha-\beta$ and
        $|\tilde{\varepsilon}| = \varepsilon$ from
        Definition~\ref{def:exponent_gap}),
        and $\sigma_w^2(N_i^{\tilde{\varepsilon}})$ uses the same
        weights~$(w_i)$ as $\sigma_w^2(\log N)$
        (Definitions~\ref{def:weights}-\ref{def:weighted_log_variance}),
        with equality when $E$, $\alpha$, and $\beta$ are known.
        To leading order in~$\varepsilon$:
        \begin{equation}
            \operatorname{Var}(\hat{A})
                \;\geq\;
                \frac{\sigma^2}{
                      \Phi_2\;\varepsilon^2\;\sigma_w^2(\log N)}
                \;\bigl(1 + O(\varepsilon)\bigr).
            \label{eqn:var_A_leading}
        \end{equation}
        The corresponding variance of $\hat{\psi}$ in the reduced
        model~\eqref{eqn:reduced_chinchilla} satisfies
        \begin{equation}
            \operatorname{Var}(\hat{\psi})
                \;=\;
                \sigma^2\;\frac{G_{\alpha\alpha}^{\mathrm{eff}}}{
                      \bigl(\Phi_2 - \Phi_1^2/n\bigr)\,
                      G_{\alpha\alpha}^{\mathrm{eff}}
                      - \bigl(G_{\psi\alpha}^{\mathrm{eff}}\bigr)^2},
            \label{eqn:var_psi}
        \end{equation}
        where $G^{\mathrm{eff}}$ is the $(\psi, \alpha)$ Schur complement
        after profiling~$E$
        (see below for the explicit entries); all quantities are
        bounded and independent of~$\varepsilon$.
        Combining~\eqref{eqn:var_A_lower}-\eqref{eqn:var_psi} with
        nominal Gaussian interval half-widths (the reduced model is
        misspecified at $\varepsilon > 0$, so the interval on $\psi$
        has no exact $95\%$ coverage) yields the simplified asymptotic
        ratio~\eqref{eqn:ci_ratio} recorded in
        Corollary~\ref{prop:ci_inflation}, a lower bound.
    \end{corollary*}
    
    We derive the variance bounds for $\hat{A}$ (full model) and
    $\hat{\psi}$ (reduced model) under i.i.d.\ Gaussian noise
    $L_i = \hat{L}(N_i, D_i; \boldsymbol{\theta}^*) + \epsilon_i$,
    $\epsilon_i \sim \mathcal{N}(0, \sigma^2)$ with
    $\mathbb{E}[\epsilon_i]=0$.
    
    \noindent\textbf{Variance of $\hat{A}$ in the full model.}
    Under Gaussian noise the linearized covariance of the least-squares
    estimator (Lemma~\ref{lem:ls_iid_covariance},
    Remark~\ref{rem:ls_cov_nonlinear}) is
    $\operatorname{Cov}(\hat{\boldsymbol{\theta}})
    = \sigma^2\,(J^T J)^{-1}$, so
    $\operatorname{Var}(\hat{A}) = \sigma^2\,
    \bigl[(J^T J)^{-1}\bigr]_{AA}$.
    
    Partition $J^T J$ conformably as
    \[
        J^T J =
        \begin{pmatrix}
            G_{AB} & G_{12} \\
            G_{21} & G_{22}
        \end{pmatrix},
    \]
    where $G_{AB} \coloneqq (J^T J)_{A,B} \in \mathbb{R}^{2\times 2}$
    is the $(A, B)$ sub-block and $G_{22}$ contains the remaining
    $(E, \alpha, \beta)$ entries.  By the Schur-complement formula for
    the inverse of a partitioned matrix:
    \begin{equation}
        \bigl[(J^T J)^{-1}\bigr]_{(A,B)}
            = \bigl(G_{AB} - G_{12}\, G_{22}^{-1}\, G_{21}\bigr)^{-1}.
        \label{eqn:schur_full}
    \end{equation}
    Since $G_{12}\, G_{22}^{-1}\, G_{21}$ is positive semidefinite,
    the Schur complement satisfies
    $S_{AB} \coloneqq G_{AB} - G_{12}\, G_{22}^{-1}\, G_{21}
    \preceq G_{AB}$ (L\"owner order), and therefore
    \begin{equation}
        \bigl[(J^T J)^{-1}\bigr]_{(A,B)}
            = S_{AB}^{-1}
            \succeq G_{AB}^{-1}.
        \label{eqn:schur_ineq}
    \end{equation}
    In particular, the $(1,1)$ entry obeys
    $\bigl[(J^T J)^{-1}\bigr]_{AA}
    \geq \bigl[G_{AB}^{-1}\bigr]_{AA}$.
    
    \noindent
    For the $2 \times 2$ matrix
    $G_{AB} = \bigl(\begin{smallmatrix}
        a & b \\ b & d
    \end{smallmatrix}\bigr)$
    with $a = \mathbf{j}_A^T \mathbf{j}_A = \Phi_2$,
    $d = \mathbf{j}_B^T \mathbf{j}_B = k_\ell^{-2\beta}\, \Phi_2^{(2\varepsilon)}$,
    and $\det(G_{AB}) = k_\ell^{-2\beta}\, \Phi_2^{\,2}\,
    \sigma_w^2(N_i^{\tilde{\varepsilon}})$
    (equation~\eqref{eqn:chin_exact_det}):
    \begin{equation}
        \bigl[G_{AB}^{-1}\bigr]_{AA}
            = \frac{d}{\det(G_{AB})}
            = \frac{k_\ell^{-2\beta}\, \Phi_2^{(2\varepsilon)}}{
                  k_\ell^{-2\beta}\, \Phi_2^{\,2}\;
                  \sigma_w^2(N_i^{\tilde{\varepsilon}})}
            = \frac{\Phi_2^{(2\varepsilon)}}{
                  \Phi_2^{\,2}\;
                  \sigma_w^2(N_i^{\tilde{\varepsilon}})}.
        \label{eqn:GAB_inv_AA}
    \end{equation}
    Hence
    \[
        \operatorname{Var}(\hat{A})
            = \sigma^2\, \bigl[(J^T J)^{-1}\bigr]_{AA}
            \geq \sigma^2\, \bigl[G_{AB}^{-1}\bigr]_{AA}
            = \frac{\sigma^2\, \Phi_2^{(2\varepsilon)}}{
                  \Phi_2^{\,2}\; \sigma_w^2(N_i^{\tilde{\varepsilon}})},
    \]
    establishing~\eqref{eqn:var_A_lower}.
    
    \noindent
    To obtain the leading-order form~\eqref{eqn:var_A_leading}, note
    that $\Phi_2^{(2\varepsilon)} = \sum_i N_i^{-2\beta}
    = \Phi_2 + 2\tilde{\varepsilon}\, T_2 + O(\varepsilon^2)$ where
    $T_2 = \sum_i N_i^{-2\alpha} \log N_i$, so
    $\Phi_2^{(2\varepsilon)} / \Phi_2 = 1 + O(\varepsilon)$.  Combined
    with the Taylor expansion
    $\sigma_w^2(N_i^{\tilde{\varepsilon}})
    = \tilde{\varepsilon}^2\, \sigma_w^2(\log N_i)
    + O(|\tilde{\varepsilon}|^3)$
    (equation~\eqref{eqn:taylor_var}):
    \[
        \frac{\Phi_2^{(2\varepsilon)}}{
              \Phi_2^{\,2}\; \sigma_w^2(N_i^{\tilde{\varepsilon}})}
        = \frac{1 + O(\varepsilon)}{
              \Phi_2\; \varepsilon^2\; \sigma_w^2(\log N)}
        = \frac{1}{\Phi_2\; \varepsilon^2\; \sigma_w^2(\log N)}
          \bigl(1 + O(\varepsilon)\bigr).
    \]
    
    \noindent\textbf{Variance of $\hat{\psi}$ in the reduced model.}
    The reduced model~\eqref{eqn:reduced_chinchilla} has parameter
    vector $\boldsymbol{\theta}_{\mathrm{red}} = (\psi, \alpha, E)^T$
    and Jacobian $J_{\mathrm{red}} \in \mathbb{R}^{n \times 3}$
    with columns
    \begin{equation}
        \mathbf{j}_\psi = \bigl(N_i^{-\alpha}\bigr)_{i=1}^n,
        \quad
        \mathbf{j}_\alpha = \bigl(-\psi\, N_i^{-\alpha}\log N_i\bigr)_{i=1}^n,
        \quad
        \mathbf{j}_E = \mathbf{1}_n.
        \label{eqn:jred_cols}
    \end{equation}
    Define
    $\Phi_1 \coloneqq \sum_i N_i^{-\alpha}$,
    $T_1 \coloneqq \sum_i N_i^{-\alpha} \log N_i$,
    $U_2 \coloneqq \sum_i N_i^{-2\alpha} (\log N_i)^2$.
    The $3\times 3$ Gram matrix is
    \[
        G_{\mathrm{red}} =
        \begin{pmatrix}
            \Phi_2          & -\psi\, T_2      & \Phi_1 \\
            -\psi\, T_2     & \psi^2\, U_2     & -\psi\, T_1 \\
            \Phi_1          & -\psi\, T_1      & n
        \end{pmatrix}.
    \]
    Since $\operatorname{Var}(\hat{\psi})
    = \sigma^2\, [G_{\mathrm{red}}^{-1}]_{\psi\psi}$, we compute the
    $(\psi, \alpha)$ Schur complement after profiling~$E$:
    \begin{equation}
    \begin{aligned}
        G_{\psi\psi}^{\mathrm{eff}}
            &= \Phi_2 - \Phi_1^2 / n, \nonumber\\
        G_{\psi\alpha}^{\mathrm{eff}}
            &= -\psi\,(T_2 - \Phi_1\, T_1 / n), \nonumber\\
        G_{\alpha\alpha}^{\mathrm{eff}}
            &= \psi^2\,(U_2 - T_1^2 / n).
    \end{aligned}
    \end{equation}
    The first entry $G_{\psi\psi}^{\mathrm{eff}}$ is the sample variance of
    $\{N_i^{-\alpha}\}$ (times~$n$); the other entries are analogous
    centered moments.  The profiled $2 \times 2$ determinant is
    \[
        \det(G^{\mathrm{eff}})
            = G_{\psi\psi}^{\mathrm{eff}}\;
              G_{\alpha\alpha}^{\mathrm{eff}}
              - \bigl(G_{\psi\alpha}^{\mathrm{eff}}\bigr)^2,
    \]
    and
    \begin{equation}
        \operatorname{Var}(\hat{\psi})
            = \sigma^2\;
              \frac{G_{\alpha\alpha}^{\mathrm{eff}}}{
                    \det(G^{\mathrm{eff}})}.
        \label{eqn:var_psi_exact}
    \end{equation}
    No power of~$\varepsilon$ appears anywhere in
    $G^{\mathrm{eff}}$: the entries depend only on
    $\{N_i\}$, $\alpha$, $\psi$, and~$n$, all of which are
    $O(1)$ with respect to the exponent gap.  Hence
    $\operatorname{Var}(\hat{\psi})$ is bounded and
    $\varepsilon$-independent, confirming~\eqref{eqn:var_psi}.
    
    \noindent\textbf{Confidence-interval ratio.}
    Under Gaussianity, the nominal $95\%$ interval half-width for any
    scalar parameter~$\theta_j$ is
    $1.96\,\sqrt{\operatorname{Var}(\hat{\theta}_j)}$.
    Therefore:
    \begin{equation}
    \begin{aligned}
        \frac{\mathrm{CI}_{0.95}(A)}{\mathrm{CI}_{0.95}(\psi)}
            &\geq
            \sqrt{
              \frac{\operatorname{Var}(\hat{A})}{
                    \operatorname{Var}(\hat{\psi})}}
        \nonumber\\[4pt]
            &\geq
            \sqrt{
              \frac{\Phi_2^{(2\varepsilon)}}{
                    \Phi_2^{\,2}\;
                    \sigma_w^2(N_i^{\tilde{\varepsilon}})}
              \;\cdot\;
              \frac{\det(G^{\mathrm{eff}})}{
                    G_{\alpha\alpha}^{\mathrm{eff}}}}
        \nonumber\\[4pt]
            &=
            \frac{1}{\varepsilon}\;
            \sqrt{
              \frac{\Phi_2^{(2\varepsilon)}}{\Phi_2}}
            \;\cdot\;
            \underbrace{
              \sqrt{
                \frac{\det(G^{\mathrm{eff}})}{
                      \Phi_2\;
                      \sigma_w^2(\log N)\;
                      G_{\alpha\alpha}^{\mathrm{eff}}}}
            }_{\Theta(1)}
            \;\cdot\;
            (1 + O(\varepsilon)),
    \end{aligned}
    \end{equation}
    where in the last step we used
    $\sigma_w^2(N_i^{\tilde{\varepsilon}})
    = \varepsilon^2\, \sigma_w^2(\log N)(1 + O(\varepsilon))$ and
    $\Phi_2^{(2\varepsilon)}/\Phi_2 = 1 + O(\varepsilon)$.
    The $\Theta(1)$ factor depends on the design $\{N_i\}$ but
    \emph{not} on~$\varepsilon$; it can be far below~$1$, so at a
    finite gap the bound may be weak. \qed

\noindent\textbf{Correlated noise.} With a noise covariance $\Sigma$ in
place of $\sigma^2 I$ (checkpoints of one run, a shared corpus), the
linearized covariance of the estimator is
$(J^T J)^{-1} J^T \Sigma J (J^T J)^{-1}$, which lies between
$\lambda_{\min}(\Sigma)\,(J^T J)^{-1}$ and
$\lambda_{\max}(\Sigma)\,(J^T J)^{-1}$ in the L\"owner order, and likewise
in the reduced model. Each variance above is then its value at
$\sigma^2 = 1$ times a factor in
$[\lambda_{\min}(\Sigma), \lambda_{\max}(\Sigma)]$, so a fixed $\Sigma$
moves the ratio~\eqref{eqn:ci_ratio} by at most $\sqrt{\kappa(\Sigma)}$
in either direction and keeps its $\varepsilon^{-1}$ rate.

\noindent\textbf{Numerical rates.} With $E$, $\alpha$ and $\beta$ free,
the exponent columns add a second near-null direction, whose rate we
do not prove. In 80-digit arithmetic at the parameters of
\citet{hoffmann2022training} on our 14 sizes and the ray $k = 20$
(\texttt{tools/single\_ray\_rates.py}), $\kappa(J^T J)$ grows as
$\varepsilon^{-6}$; numerically the ratio of the five-parameter fit grows as $\varepsilon^{-3}$ and
the off-ray predictive variance as $\varepsilon^{-6}$, while the floor is
$0.16$ at $\varepsilon = 0.06$, where the computed ratio is of order
$10^{4}$.

\subsection{Appendix - Proofs of Proposition~\ref{prop:holdout_r2} (TPP diversity and scale-pair conditioning) and Corollary~\ref{cor:ratio_cases}}
\label{app:proof_prop_holdout_r2}

We give a self-contained proof.
Throughout, $\beta_{\mathrm{eff}}$ denotes the
law-specific data-size exponent
(Table~\ref{tab:law_substitutions}).
For any $K \geq 1$ distinct TPP ratios
$k_1 < \cdots < k_K$, let each ratio be trained at
the same $n \geq 1$ model sizes $\{N_{i}\}_{i=1}^{n}$,
as in a collinear design
(Definition~\ref{def:holdout_splits}). The proof uses
this hypothesis: with unequal runs or model sizes per
ratio, the rays enter $G_{A,B}$ with the unequal weights $\pi_\ell$ of
Lemma~\ref{lem:ray_spread} below.
The training Jacobian is evaluated at the collinear
design $D_{i\ell} = k_\ell N_i$. The proposition is a statement about
the limit $\varepsilon \to 0$; the two lemmas below hold at every gap
and on any design, and the steps after them take their limit.

\begin{lemma}[Weighted spread of a column pair]
\label{lem:weighted_spread}
Let $\mathbf{f}, \mathbf{g} \in \mathbb{R}^m$ with $f_i \neq 0$ for every
$i$, and let $G$ be their $2 \times 2$ Gram matrix. Define
the weights $\omega_i \coloneqq f_i^2/\|\mathbf{f}\|^2$, the
column ratios $u_i \coloneqq g_i/f_i$, the weighted moments
$\bar{u} \coloneqq \sum_i \omega_i u_i$ and
$\overline{u^2} \coloneqq \sum_i \omega_i u_i^2$, the \emph{weighted spread}
$V_{\mathrm{w}} \coloneqq \overline{u^2} - \bar{u}^2
= \sum_i \omega_i (u_i - \bar{u})^2$ and, for $\kappa_{\mathrm{target}} \geq 1$,
the \emph{weighted threshold}
$\tau_{\mathrm{w}} \coloneqq (1 + \overline{u^2})^2/\kappa_{\mathrm{target}}$.
Then:
\begin{enumerate}[label=(\roman*), nosep, leftmargin=*]
\item $\min_{c \in \mathbb{R}} \|\mathbf{g} - c\,\mathbf{f}\|^2
  = \|\mathbf{f}\|^2 V_{\mathrm{w}}$, attained at $c = \bar{u}$;
\item $\mathrm{tr}\, G = \|\mathbf{f}\|^2 (1 + \overline{u^2})$ and
  $\det G = \|\mathbf{f}\|^4 V_{\mathrm{w}}$, so $G$ is singular iff
  $V_{\mathrm{w}} = 0$, that is, iff $\mathbf{g}$ is a multiple of $\mathbf{f}$;
\item if $V_{\mathrm{w}} > 0$, then
  $T \coloneqq (1 + \overline{u^2})^2/V_{\mathrm{w}}$ satisfies
  $\kappa(G) + 2 + 1/\kappa(G) = T$, hence
  $T - 3 \leq \kappa(G) < T - 2$, and
  \begin{equation}
      \kappa(G) \leq \kappa_{\mathrm{target}}
      \iff
      V_{\mathrm{w}} \geq \tau_{\mathrm{w}}\,
      \frac{\kappa_{\mathrm{target}}^2}{(1 + \kappa_{\mathrm{target}})^2};
      \label{eqn:weighted_iff}
  \end{equation}
  in particular $V_{\mathrm{w}} \geq \tau_{\mathrm{w}}$ implies
  $\kappa(G) \leq \kappa_{\mathrm{target}}$, and
  $\kappa(G) \leq \kappa_{\mathrm{target}}$ implies
  $V_{\mathrm{w}} \geq \tau_{\mathrm{w}}/(1 + 1/\kappa_{\mathrm{target}})^2$;
  if $V_{\mathrm{w}} = 0$, then $G$ is singular, and with $\kappa(G) \coloneqq \infty$
  both sides of~\eqref{eqn:weighted_iff} fail.
\end{enumerate}
\end{lemma}

\begin{proof}
(i) $\|\mathbf{g} - c\,\mathbf{f}\|^2 = \sum_i f_i^2 (u_i - c)^2
= \|\mathbf{f}\|^2\,[V_{\mathrm{w}} + (\bar{u} - c)^2]$.
(ii) $G = \|\mathbf{f}\|^2 \begin{psmallmatrix} 1 & \bar{u} \\ \bar{u} & \overline{u^2}
\end{psmallmatrix}$, whose trace and determinant are as stated. As every
$\omega_i > 0$, $V_{\mathrm{w}} = 0$ iff all $u_i$ are equal, that is, iff
$\mathbf{g} = c\,\mathbf{f}$ for some $c$.
(iii) By (ii), $G$ is positive definite. With eigenvalues
$\lambda_+ \geq \lambda_- > 0$ and $\kappa(G) = \lambda_+/\lambda_-$,
$\mathrm{tr}^2/\det = (\lambda_+ + \lambda_-)^2/(\lambda_+ \lambda_-)
= \kappa(G) + 2 + 1/\kappa(G)$, and $\mathrm{tr}^2/\det = T$ by (ii). As
$0 < 1/\kappa(G) \leq 1$, $T - 3 \leq \kappa(G) < T - 2$. The map
$x \mapsto x + 2 + 1/x$ is strictly increasing on $[1, \infty)$, so
$\kappa(G) \leq \kappa_{\mathrm{target}}$ iff
$T \leq (1 + \kappa_{\mathrm{target}})^2/\kappa_{\mathrm{target}}$, which
rearranges to~\eqref{eqn:weighted_iff}. The two implications follow from
$\kappa_{\mathrm{target}}^2/(1 + \kappa_{\mathrm{target}})^2
= (1 + 1/\kappa_{\mathrm{target}})^{-2} \leq 1$. If $V_{\mathrm{w}} = 0$, $G$ is singular
by (ii), and the right side of~\eqref{eqn:weighted_iff} fails because
$\tau_{\mathrm{w}} > 0$.
\end{proof}

The Gram block $G_{A,B}$ of the residual Jacobian is the Gram matrix of
$\mathbf{f} = (N_i^{-\alpha})_i$ and $\mathbf{g} = (D_i^{-\beta})_i$, since
both columns change sign together; so for Chinchilla
$\omega_i \propto N_i^{-2\alpha}$ and $u_i = D_i^{-\beta} N_i^{\alpha}$, on any
design. The lemma applies to the native columns of every law of
Table~\ref{tab:jacobian_summary}, with their own $\omega_i$ and $u_i$.

\begin{lemma}[Weighted spread on rays]
\label{lem:ray_spread}
Take the Chinchilla columns $f_i = N_i^{-\alpha}$, $g_i = D_i^{-\beta}$
of Lemma~\ref{lem:weighted_spread} on a design whose $m$ points lie on
$K$ rays, point $i$ on ray $\ell(i)$ with $D_i = k_{\ell(i)} N_i$ and
$k_1 < \cdots < k_K$, with any sizes on any ray. Let
$\tilde{\varepsilon} \coloneqq \alpha - \beta$,
$\varepsilon = |\tilde{\varepsilon}|$, $x_\ell \coloneqq k_\ell^{-\beta}$
and $y_i \coloneqq N_i^{\tilde{\varepsilon}}$, so that
$u_i = x_{\ell(i)}\, y_i$; let $y_- \coloneqq \min_i y_i$ and
$y_+ \coloneqq \max_i y_i$. Let $\pi_\ell \coloneqq \sum_{i : \ell(i) = \ell} \omega_i$
be the weight of ray $\ell$, $\bar{y}_\ell$ and $\sigma_{y,\ell}^2$ the mean and
variance of the $y_i$ on ray $\ell$ under the weights $\omega_i/\pi_\ell$, and
$z_\ell \coloneqq x_\ell\, \bar{y}_\ell$. Then:
\begin{enumerate}[label=(\alph*), nosep, leftmargin=*]
\item $V_{\mathrm{w}} = \sum_\ell \pi_\ell (z_\ell - \bar{u})^2
  + \sum_\ell \pi_\ell\, x_\ell^2\, \sigma_{y,\ell}^2$, with
  $\bar{u} = \sum_\ell \pi_\ell z_\ell$ (spread between rays plus spread
  within rays);
\item for any probability weights $(v_i)$ on the points,
  $y_-^2\, \varepsilon^2\, \sigma_v^2(\log N) \leq \sigma_v^2(y)
  \leq y_+^2\, \varepsilon^2\, \sigma_v^2(\log N)$, where $\sigma_v^2$
  is the $v$-weighted variance;
\item if $K = 1$, then $V_{\mathrm{w}} = x_1^2\, \sigma_w^2(y)$ (on one ray with each size run once, the $\omega_i$ are the weights $w_i$ of Definition~\ref{def:weights}), which by (b)
  lies between $x_1^2 y_-^2 \varepsilon^2 \sigma_w^2(\log N)$ and
  $x_1^2 y_+^2 \varepsilon^2 \sigma_w^2(\log N)$;
\item if $K \geq 2$ and $(k_K/k_1)^{\beta} > y_+/y_-
  = (N_{\max}/N_{\min})^{\varepsilon}$, with $N_{\max}$, $N_{\min}$ the
  largest and smallest size of the design, then
  $V_{\mathrm{w}} \geq \pi_1 \pi_K\, x_K^2\, y_-^2\,
  \bigl((k_K/k_1)^{\beta} - y_+/y_-\bigr)^2 > 0$;
\item if every ray carries the same sizes $N_1, \dots, N_n$ once each
  (Definition~\ref{def:holdout_splits}), let $\bar{y}$ and $\sigma_y^2$ be the
  mean and variance of $(N_i^{\tilde{\varepsilon}})_{i=1}^n$ under the
  weights of Definition~\ref{def:weights}, and $V_K$ as in
  Proposition~\ref{prop:holdout_r2} and $S_2 \coloneqq \sum_{\ell=1}^{K} k_\ell^{-2\beta}$, at $\beta_{\mathrm{eff}} = \beta$; then
  $V_{\mathrm{w}} = \bar{y}^2 V_K + (S_2/K)\, \sigma_y^2 \geq y_-^2 V_K$ and
  $\overline{u^2} = (S_2/K)(\bar{y}^2 + \sigma_y^2)$;
\item for a fixed design and parameters converging with
  $\tilde{\varepsilon} \to 0$, $V_{\mathrm{w}}$ and $\overline{u^2}$
  converge to the $\pi$-weighted variance $V_\pi$ and second moment of
  $(x_\ell)_{\ell=1}^K$ at the limit parameters, and $V_\pi > 0$ iff
  $K \geq 2$; under (e), $\pi_\ell = 1/K$, so the limits are $V_K$ and
  $S_2/K$, and $\tau_{\mathrm{w}} \to \tau_K$;
\item for every $c > 0$ and sizes $N_1 < \cdots < N_m$, the design
  $D_i = c^{-1/\beta} N_i^{\alpha/\beta}$ has $V_{\mathrm{w}} = 0$, and if
  $\alpha \neq \beta$ its $m$ points lie on $m$ distinct rays.
\end{enumerate}
\end{lemma}

\begin{proof}
(a) Group the sum $V_{\mathrm{w}} = \sum_i \omega_i (u_i - \bar{u})^2$ by ray
and write $u_i - \bar{u} = (u_i - z_\ell) + (z_\ell - \bar{u})$ for
$\ell(i) = \ell$. The cross terms vanish, since
$\sum_{\ell(i) = \ell} \omega_i (u_i - z_\ell)
= x_\ell \pi_\ell (\bar{y}_\ell - \bar{y}_\ell) = 0$, and
$\sum_{\ell(i) = \ell} \omega_i (u_i - z_\ell)^2 = \pi_\ell x_\ell^2 \sigma_{y,\ell}^2$;
summing $\omega_i u_i$ by ray gives $\bar{u} = \sum_\ell \pi_\ell z_\ell$.
(b) $\sigma_v^2(y) = \tfrac{1}{2}\sum_{i,j} v_i v_j (y_i - y_j)^2$, and the
same holds for $\log N$. By the mean value theorem applied to
$t \mapsto e^{\tilde{\varepsilon} t}$,
$y_i - y_j = \tilde{\varepsilon}\, e^{\tilde{\varepsilon} \xi}
(\log N_i - \log N_j)$ with $e^{\tilde{\varepsilon} \xi}$ between $y_i$
and $y_j$, hence in $[y_-, y_+]$; square and sum.
(c) With one ray, $u_i = x_1 y_i$.
(d) The first sum in (a) equals
$\tfrac{1}{2}\sum_{\ell, \ell'} \pi_\ell \pi_{\ell'} (z_\ell - z_{\ell'})^2
\geq \pi_1 \pi_K (z_1 - z_K)^2$, and the second is nonnegative. As
$\bar{y}_\ell \in [y_-, y_+]$ and $x_1/x_K = (k_K/k_1)^{\beta}$,
$z_1 - z_K \geq x_1 y_- - x_K y_+
= x_K y_- \bigl((k_K/k_1)^{\beta} - y_+/y_-\bigr) > 0$. Finally
$y_+/y_- = (N_{\max}/N_{\min})^{\varepsilon}$ for either sign of
$\tilde{\varepsilon}$.
(e) Each ray holds the sizes $N_1, \dots, N_n$, so $\pi_\ell = 1/K$, the
weights $\omega_i/\pi_\ell$ on every ray are those of
Definition~\ref{def:weights}, $\bar{y}_\ell = \bar{y}$ and
$\sigma_{y,\ell}^2 = \sigma_y^2$. Then $z_\ell = \bar{y} x_\ell$ and (a) gives
$V_{\mathrm{w}} = \bar{y}^2 V_K + \sigma_y^2 \sum_\ell x_\ell^2/K$, with
$\sum_\ell x_\ell^2 = S_2$, and $\bar{y} \geq y_-$. Likewise
$\overline{u^2} = \sum_\ell x_\ell^2 \sum_i (N_i^{-2\alpha}/(K\Phi_2))\, y_i^2
= (S_2/K)(\bar{y}^2 + \sigma_y^2)$.
(f) At a fixed design, $y_i = N_i^{\tilde{\varepsilon}} \to 1$ for every
$i$, while $\omega_i$ and $x_\ell$ depend continuously on $(\alpha, \beta)$.
So $\bar{y}_\ell \to 1$ and $\sigma_{y,\ell} \to 0$, and (a) gives the limits. As
every $\pi_\ell > 0$ and the $x_\ell$ are distinct, $V_\pi > 0$ iff
$K \geq 2$. Under (e), $\tau_{\mathrm{w}} = (1 + \overline{u^2})^2/\kappa_{\mathrm{target}}
\to (1 + S_2/K)^2/\kappa_{\mathrm{target}} = \tau_K$.
(g) $u_i = D_i^{-\beta} N_i^{\alpha} = c$ for every $i$, so
$V_{\mathrm{w}} = 0$. The ratio $D_i/N_i = c^{-1/\beta} N_i^{\alpha/\beta - 1}$
is strictly monotone in $N_i$ when $\alpha \neq \beta$.
\end{proof}

\noindent\textbf{When $V_K$ is a valid proxy.}
Under Lemma~\ref{lem:ray_spread}(e) with $K \geq 2$, so that $V_K > 0$, let
$\delta_N \coloneqq \varepsilon \max_i |\log N_i|$, so
$y_\pm \in [e^{-\delta_N}, e^{\delta_N}]$. Then
$\overline{u^2} \in [e^{-2\delta_N}, e^{2\delta_N}]\, S_2/K$, and
$e^{-2\delta_N} V_K \leq V_{\mathrm{w}}
\leq e^{2\delta_N}(V_K + (S_2/K)\,\varepsilon^2 \sigma_w^2(\log N))$ by (b),
so
\begin{equation}
    e^{-6\delta_N}
    \;\leq\;
    \frac{V_{\mathrm{w}}/\tau_{\mathrm{w}}}{V_K/\tau_K}
    \;\leq\;
    e^{6\delta_N}\Bigl(1 + \frac{S_2\, \varepsilon^2\, \sigma_w^2(\log N)}{K\, V_K}\Bigr),
    \label{eqn:proxy_bound}
\end{equation}
using $e^{-2\delta_N}(1 + S_2/K) \leq 1 + \overline{u^2}
\leq e^{2\delta_N}(1 + S_2/K)$. The unweighted margin $V_K/\tau_K$ is
therefore a valid proxy for the exact $V_{\mathrm{w}}/\tau_{\mathrm{w}}$
when every ray carries the same sizes and $\delta_N$ is small. At our
sizes $\delta_N$ is near~$1$ (Accuracy of the leading order, below), where
the bound allows a factor of several hundred either way and says nothing
useful; the exact form that follows does. Exactly,
by (e), $V_{\mathrm{w}}$ and $\overline{u^2}$ are $V_K$ and $S_2/K$ at the
rescaled ratios $\bar{y}^{-1/\beta} k_\ell$, each plus the within-ray term
$(S_2/K)\, \sigma_y^2$. With unequal sizes per ray the limit is the
$\pi$-weighted $V_\pi$ of (f), not $V_K$, and at a fixed gap the spread
between rays can cancel, down to $V_{\mathrm{w}} = 0$ in (g).

\noindent\textbf{Step 1: Jacobian columns for the scale-coefficient pair.}
For the Chinchilla and unsaturated repeated-data laws of
Definition~\ref{def:exponent_gap}, the two
scale-coefficient columns of $J$ on observation
$(N_i, k_\ell N_i)$ are (cf.\
Appendix~\ref{app:chinchilla_details},
eqn.~\eqref{eqn:partial_linear_dependence} for
Chinchilla; analogous expressions hold for unsaturated repeated-data
with the substitutions of
Table~\ref{tab:law_substitutions}):
\begin{equation}
    (j_A)_{i\ell} = N_i^{-\alpha},
    \qquad
    (j_B)_{i\ell}
        = k_\ell^{-\beta_{\mathrm{eff}}}\, N_i^{-\alpha}\,
          N_i^{\tilde{\varepsilon}},
    \label{eqn:jAB_Kray_proof}
\end{equation}
where $\tilde{\varepsilon}$ is the \emph{signed} native gap in each
law (e.g.\ $\alpha-\beta$ for Chinchilla and repeated-data,
$\alpha_N/\alpha_D - 1$ for Kaplan,
$\gamma_N-\gamma_D$ for Droppo-Elibol), so
$|\tilde{\varepsilon}| = \varepsilon$ in
Definition~\ref{def:exponent_gap}, and the displayed~$\alpha$ is the
model-side exponent paired with~$\beta_{\mathrm{eff}}$ as in
Table~\ref{tab:law_substitutions}.

\noindent\textbf{Kaplan and Droppo-Elibol.} Their native columns carry
constant prefactors (for Kaplan $\alpha_N N_c^{\alpha_N/\alpha_D - 1}$ and $\alpha_D$,
Appendix~\ref{app:kaplan_details}; for Droppo-Elibol $\gamma_N N_C^{\gamma_N - 1}$ and
$\gamma_D D_C^{\gamma_D - 1}$, Appendix~\ref{app:elibol_details}) and a chain-rule factor that changes from
ray to ray. $\kappa_{A,B}$ changes when a column is rescaled, while $V_K$
and $\tau_K$ do not, so for these two laws the proposition holds for the
term coefficients (for Droppo-Elibol $N_C^{\gamma_N}$, $D_C^{\gamma_D}$) only where the
chain-rule factor is constant, and is otherwise a leading-order
heuristic. Lemma~\ref{lem:weighted_spread} holds for their native
columns, so the exact criterion~\eqref{eqn:weighted_iff} applies to them;
only its reduction to $V_K$ and $\tau_K$ is the heuristic. The Droppo-Elibol law at $\alpha = 1$, which is the Chinchilla
function with $\gamma_N = \alpha_N$, $N_C = A^{1/\alpha_N}$ and $D_C = B^{1/\alpha_D}$, at the
parameters of \citet{hoffmann2022training} on our 10 training sizes
has $\kappa_{A,B} = 27{,}186$ for $\lbrace 20, 400 \rbrace$ in its native pair $(N_C, D_C)$,
against $27$ in $(A, B)$, at the same $V_K/\tau_K$.

\noindent\textbf{Step 2: $(A,B)$ Gram sub-block.}
Write $S_1 \coloneqq \sum_{\ell=1}^{K} k_\ell^{-\beta_{\mathrm{eff}}}$,
$S_2 \coloneqq \sum_{\ell=1}^{K} k_\ell^{-2\beta_{\mathrm{eff}}}$ and
$\Phi_2 \coloneqq \sum_{i=1}^{n} N_i^{-2\alpha}$
(Definition~\ref{def:power_sums}); the variance $V_K$ of
eqn.~\eqref{eqn:tpp_diversity} is
\begin{equation}
    V_K
        = \frac{1}{K}\sum_{\ell=1}^{K}
          k_\ell^{-2\beta_{\mathrm{eff}}}
          - \Bigl(\frac{1}{K}\sum_{\ell=1}^{K}
                k_\ell^{-\beta_{\mathrm{eff}}}\Bigr)^{\!2}
        = \frac{K\,S_2 - S_1^2}{K^2}.
    \label{eqn:VK_identity_proof}
\end{equation}
The columns~\eqref{eqn:jAB_Kray_proof} have
$\|\mathbf{j}_A\|^2 = K\Phi_2$, so Lemma~\ref{lem:weighted_spread}(ii) and
Lemma~\ref{lem:ray_spread}(e), with $\beta_{\mathrm{eff}} = \beta$, give
at every gap
\begin{equation}
    \mathrm{tr}(G_{A,B})
        = K\Phi_2 \Bigl(1 + \frac{S_2}{K}(\bar{y}^2 + \sigma_y^2)\Bigr),
    \qquad
    \det(G_{A,B})
        = K^2\Phi_2^2 \Bigl(\bar{y}^2 V_K + \frac{S_2}{K}\, \sigma_y^2\Bigr),
    \label{eqn:GAB_exact_proof}
\end{equation}
with $\bar{y}$ and $\sigma_y^2$ the weighted mean and variance of the
$N_i^{\tilde{\varepsilon}}$. As $\varepsilon \to 0$ at fixed sizes and
ratios, $\bar{y} \to 1$ and $\sigma_y \to 0$ (Lemma~\ref{lem:ray_spread}(f)), and
$(G_{A,B})_{12} = K\Phi_2\, \bar{u} = \Phi_2\, \bar{y}\, S_1$, so
\begin{equation}
    G_{A,B} \;\to\; G_0 \coloneqq
        \Phi_2
        \begin{pmatrix}
            K   & S_1 \\[2pt]
            S_1 & S_2
        \end{pmatrix},
    \qquad
    \mathrm{tr}\, G_0 = \Phi_2\,(K + S_2),
    \qquad
    \det G_0 = \Phi_2^2\, K^2\, V_K.
    \label{eqn:GAB_matrix_proof}
\end{equation}

\noindent\textbf{Step 3: Condition number.}
Let $K \geq 2$, so $V_K > 0$ and $G_0$ is positive definite. The
condition number is continuous on positive-definite matrices, so
$\kappa_{A,B} \to \kappa_0 \coloneqq \kappa(G_0)$, and the identity
$\mathrm{tr}^2/\det = \kappa + 2 + 1/\kappa$ of
Lemma~\ref{lem:weighted_spread}(iii) gives, with $\tau_K$ as
in~\eqref{eqn:tpp_diversity},
\begin{equation}
    \kappa_0 + 2 + \frac{1}{\kappa_0}
        \;=\;
        \frac{(\mathrm{tr}\, G_0)^2}{\det G_0}
        \;=\;
        \frac{(K + S_2)^2}{K^2\, V_K}
        \;=\;
        \frac{\kappa_{\mathrm{target}}\,\tau_K}{V_K};
    \label{eqn:kappa_AB_proof}
\end{equation}
the factors of $\Phi_2^2$ cancel. This limit $\kappa_0$ is the
leading-order $\kappa_{A,B}$ of the proposition; the leading-order values
printed below are $\mathrm{tr}^2/\det = \kappa_0 + 2 + 1/\kappa_0$, which
exceeds $\kappa_0$ by at most $3$. For $K = 1$,
$\det G_0 = 0 < \mathrm{tr}\, G_0$, so $\lambda_{\min}(G_0) = 0$ and $\kappa_{A,B} \to \infty$ as
$\varepsilon \to 0$; with a single size, $G_{A,B}$ is singular at every $\varepsilon$, and
$\kappa_{A,B} = \infty$ by the convention of Lemma~\ref{lem:weighted_spread}(iii).

\noindent\textbf{Step 4: Items (i)-(iii).}
For $K = 1$, $V_1 = 0 < \tau_1$ and $\kappa_{A,B} \to \infty$ (Step~3), so both sides of the
iff fail and the implication~(i) holds vacuously; let $K \geq 2$.
As $x \mapsto x + 2 + 1/x$ is strictly increasing on $[1, \infty)$,
eqn.~\eqref{eqn:kappa_AB_proof} gives, for
$\kappa_{\mathrm{target}} \geq 1$,
\begin{equation}
    \kappa_0 \leq \kappa_{\mathrm{target}}
    \;\Longleftrightarrow\;
    \frac{\kappa_{\mathrm{target}}\,\tau_K}{V_K}
        \leq \kappa_{\mathrm{target}} + 2 + \frac{1}{\kappa_{\mathrm{target}}}
    \;\Longleftrightarrow\;
    V_K \geq \tau_K\,
        \frac{\kappa_{\mathrm{target}}^2}{(1 + \kappa_{\mathrm{target}})^2}.
    \label{eqn:iff_proof}
\end{equation}
\textbf{(Proposition~\ref{prop:holdout_r2}(i))} If
$V_K \geq \tau_K$, then $V_K \geq \tau_K \kappa_{\mathrm{target}}^2 /
(1+\kappa_{\mathrm{target}})^2$, so $\kappa_0 \leq
\kappa_{\mathrm{target}}$ by~\eqref{eqn:iff_proof}.
\textbf{(Proposition~\ref{prop:holdout_r2}(ii))} As $\kappa_{\mathrm{target}}^2/(1+\kappa_{\mathrm{target}})^2
= (1 + 1/\kappa_{\mathrm{target}})^{-2}$, \eqref{eqn:iff_proof} is the stated iff;
its threshold is below $\tau_K$ by that factor, about $2\%$ at
$\kappa_{\mathrm{target}} = 100$. \textbf{(Proposition~\ref{prop:holdout_r2}(iii))} At every gap,
Lemma~\ref{lem:weighted_spread}(iii) gives the same two directions
exactly, in the form~\eqref{eqn:weighted_iff}, for the weighted
$V_{\mathrm{w}} = \bar{y}^2 V_K + (S_2/K)\, \sigma_y^2$ and
$\tau_{\mathrm{w}} = (1 + (S_2/K)(\bar{y}^2 + \sigma_y^2))^2/\kappa_{\mathrm{target}}$
of~\eqref{eqn:GAB_exact_proof}, which tend to $V_K$ and $\tau_K$ as $\varepsilon \to 0$. This establishes
Proposition~\ref{prop:holdout_r2}.

\noindent\textbf{Step 5: The three cases (Corollary~\ref{cor:ratio_cases}).}

\noindent
\textbf{Case $K = 1$.}
When $K = 1$, there is a single TPP ratio $k_1$, so
$S_1 = k_1^{-\beta_{\mathrm{eff}}}$,
$S_2 = k_1^{-2\beta_{\mathrm{eff}}}$, and
$V_1 = k_1^{-2\beta_{\mathrm{eff}}} - (k_1^{-\beta_{\mathrm{eff}}})^2 = 0$.
The threshold satisfies
$\tau_1 = (1 + k_1^{-2\beta_{\mathrm{eff}}})^2 / \kappa_{\mathrm{target}} > 0$
for every $\kappa_{\mathrm{target}} > 0$.
Therefore $V_1 = 0 < \tau_1$ and the
condition~\eqref{eqn:tpp_diversity} can never hold.
At a fixed gap, Lemma~\ref{lem:ray_spread}(c) gives
$V_{\mathrm{w}} = x_1^2 \sigma_y^2$ between
$x_1^2 y_-^2 \varepsilon^2 \sigma_w^2(\log N)$ and
$x_1^2 y_+^2 \varepsilon^2 \sigma_w^2(\log N)$, with $x_1 = k_1^{-\beta}$,
while $\overline{u^2} \to x_1^2$; so with at least two distinct sizes,
Lemma~\ref{lem:weighted_spread}(iii) gives
$\kappa_{A,B} = \Theta(\varepsilon^{-2})$, as do
Proposition~\ref{prop:full_cond} and eqn.~\eqref{eqn:chin_exact_det}.

\noindent
\textbf{Case $K = 2$.}
Let $k_1 < k_2$ and $R \coloneqq k_2/k_1$.
Then:
\begin{equation}
\begin{aligned}
    S_1 &= k_1^{-\beta_{\mathrm{eff}}}(1 + R^{-\beta_{\mathrm{eff}}}),
    \qquad
    S_2 = k_1^{-2\beta_{\mathrm{eff}}}(1 + R^{-2\beta_{\mathrm{eff}}}),
    \\[3pt]
    V_2 &= \tfrac{1}{2} S_2
          - \bigl(\tfrac{1}{2} S_1\bigr)^2
         = \tfrac{1}{4}\, k_1^{-2\beta_{\mathrm{eff}}}\,
           (1 - R^{-\beta_{\mathrm{eff}}})^2.
\end{aligned}
\label{eqn:V2_proof}
\end{equation}
Let $s \coloneqq k_1^{-2\beta_{\mathrm{eff}}}$,
$r \coloneqq R^{-\beta_{\mathrm{eff}}} \in (0,1)$ and
$a \coloneqq \sqrt{\kappa_{\mathrm{target}}\,s}$.
By~\eqref{eqn:V2_proof}, $V_2 = s(1-r)^2/4$, and
$\tau_2 = (2 + s + s r^2)^2/(4\kappa_{\mathrm{target}})$
by~\eqref{eqn:tpp_diversity}. Hence
\begin{align*}
    V_2 \geq \tau_2
    &\iff \kappa_{\mathrm{target}}\, s\, (1-r)^2
          \geq (2 + s + s r^2)^2
    && \text{multiplying by } 4\kappa_{\mathrm{target}} > 0 \\
    &\iff a\,(1-r) \geq 2 + s + s r^2
    && \text{both sides are positive} \\
    &\iff q(r) \coloneqq s r^2 + a r + 2 + s - a \leq 0 .
\end{align*}
Since $s, a > 0$, $q$ is strictly increasing on $[0,1]$,
and $q(1) = 2 + 2s > 0$. If $q(0) = 2 + s - a \geq 0$, no
$r \in (0,1)$ satisfies $q(r) \leq 0$. Otherwise the
$r \in (0,1)$ with $q(r) \leq 0$ form the interval
$(0, r^\ast]$, where
$r^\ast \coloneqq \bigl(\sqrt{a^2 - 4s(2 + s - a)} - a\bigr)/(2s)$
is the positive root of $q$, and $r^\ast < 1$.
As $R = r^{-1/\beta_{\mathrm{eff}}}$ decreases in $r$, and
$a > 2 + s$ reads
$\sqrt{\kappa_{\mathrm{target}}}\,k_1^{-\beta_{\mathrm{eff}}}
> 2 + k_1^{-2\beta_{\mathrm{eff}}}$, this proves
\begin{equation}
    V_2 \geq \tau_2
    \iff
    R \geq R_{\min}(k_1) \coloneqq
    \begin{cases}
        (r^\ast)^{-1/\beta_{\mathrm{eff}}}
        & \text{if } \sqrt{\kappa_{\mathrm{target}}}\,
          k_1^{-\beta_{\mathrm{eff}}}
          > 2 + k_1^{-2\beta_{\mathrm{eff}}}, \\
        \infty & \text{otherwise},
    \end{cases}
    \label{eqn:R_min_closed}
\end{equation}
so no spread passes~\eqref{eqn:tpp_diversity} when $R_{\min}(k_1) = \infty$.
$R_{\min}$ increases with $k_1 \geq 1$, a scale effect of the $B$
column only: at fixed $R$, $k_1$ multiplies $\mathbf{j}_B$ by a constant. Let
$g(s) \coloneqq s/(2 + (1+r^2)s)^2$. At fixed $R$,
$V_2/\tau_2 = \kappa_{\mathrm{target}}(1-r)^2\,g(s)$, and
$g'(s) = (2 - (1+r^2)s)/(2 + (1+r^2)s)^3 > 0$, because
$(1+r^2)s < 2$ when $s \leq 1$, that is, when $k_1 \geq 1$.
So at fixed $R$, $V_2/\tau_2$ decreases as $k_1$ grows.
Let $1 \leq k_1 < k_1'$ with $R' \coloneqq R_{\min}(k_1') < \infty$.
Then $V_2/\tau_2 = 1$ at $(k_1', R')$, since $q(r^\ast) = 0$,
so $V_2/\tau_2 > 1$ at $(k_1, R')$.
By~\eqref{eqn:R_min_closed} at $k_1$,
$R_{\min}(k_1) \leq R'$, and $R_{\min}(k_1) \neq R'$, since
$V_2/\tau_2 = 1$ at $(k_1, R_{\min}(k_1))$. Hence
$R_{\min}(k_1) < R_{\min}(k_1')$; when
$R_{\min}(k_1') = \infty$,
$R_{\min}(k_1) \leq R_{\min}(k_1')$ holds trivially.
Table~\ref{tab:R_min_eq5}
tabulates~\eqref{eqn:R_min_closed} at
$\kappa_{\mathrm{target}} = 10^{2}$.

\noindent
\textbf{Case $K > 2$.}
For $K > 2$ there is no closed form
like~\eqref{eqn:R_min_closed}; one evaluates
$V_K \geq \tau_K$ in~\eqref{eqn:tpp_diversity} for the
proposed ratios. Let
$x_\ell \coloneqq k_\ell^{-\beta_{\mathrm{eff}}}$, so
$x_K < \cdots < x_1$, and let $c \coloneqq (x_1 + x_K)/2$.
A variance is at most the mean squared deviation from any
point, and $|x_\ell - c| \leq (x_1 - x_K)/2$, so
$V_K \leq \frac{1}{K}\sum_{\ell=1}^{K} (x_\ell - c)^2
\leq (x_1 - x_K)^2/4$, the $V_2$ of the endpoint pair
$\{k_1, k_K\}$. The second inequality is strict when some
$x_\ell$ lies strictly between $x_K$ and $x_1$, so interior
rays lower $V_K$. The threshold depends on the rays only
through $S_2/K$, the mean of the $x_\ell^2$, which lies in
$[x_K^2, x_1^2]$, so
$\tau_K \geq (1 + x_K^2)^2/\kappa_{\mathrm{target}}$. Let
$\tau_2^{\mathrm{end}} \coloneqq (1 + (x_1^2 + x_K^2)/2)^2/\kappa_{\mathrm{target}}$
be the threshold of the endpoint pair. Expanding the right
side gives $1 + (x_1^2 + x_K^2)/2 \leq (1 + x_1^2/2)(1 + x_K^2)$,
and $x_1^2 = k_1^{-2\beta_{\mathrm{eff}}}$, so
\begin{align}
    \frac{V_K}{\tau_K}
    &\leq \frac{(x_1 - x_K)^2/4}{(1 + x_K^2)^2/\kappa_{\mathrm{target}}}
    && \text{both bounds above} \nonumber \\
    &= \Bigl(\frac{1 + (x_1^2 + x_K^2)/2}{1 + x_K^2}\Bigr)^{2}
       \frac{(x_1 - x_K)^2/4}{\tau_2^{\mathrm{end}}}
    && \text{definition of } \tau_2^{\mathrm{end}} \nonumber \\
    &\leq \Bigl(1 + \tfrac{1}{2}\,k_1^{-2\beta_{\mathrm{eff}}}\Bigr)^{2}
       \frac{(x_1 - x_K)^2/4}{\tau_2^{\mathrm{end}}}
    && \text{the inequality above} .
    \label{eqn:VK_tauK_bound}
\end{align}
So a $K$-ray design beats the two-ray design on its own
endpoints by at most the factor
$(1 + k_1^{-2\beta_{\mathrm{eff}}}/2)^2$, which is $1.09$
to $1.20$ at $k_1 = 20$ for
$\beta_{\mathrm{eff}} \in [0.28, 0.40]$; in the worked
example of Appendix~\ref{app:tpp_recipe}, a middle ray
lowers $V_K/\tau_K$ from $1.13$ to $0.785$.
This establishes Corollary~\ref{cor:ratio_cases}.

\noindent\textbf{Step 6: The full matrix.}
By Cauchy interlacing,
$\kappa_{A,B} \leq \kappa(J^T J)$, so
$\kappa(J^T J) \leq \kappa_{\mathrm{target}}$
implies $\kappa_{A,B} \leq \kappa_{\mathrm{target}}$
and hence, at leading order, $V_K \geq \tau_K/(1 + 1/\kappa_{\mathrm{target}})^2$
(Step~4): the threshold is nearly necessary for a well-conditioned full
matrix. It is not sufficient: the offset and
exponent columns enter $J^T J$ at other scales, and passing designs
such as $k \in \lbrace 20, 400 \rbrace$ keep $\kappa(J^T J)$ many
orders of magnitude above $\kappa_{\mathrm{target}}$.
Lemma~\ref{lem:full_rank_rays} (after the proof of
Theorem~\ref{thm:holdout_regimes}) gives design conditions under which
$J^T J$ is non-singular.

\noindent\textbf{Accuracy of the leading order.} Beyond the
factor $(1 + 1/\kappa_{\mathrm{target}})^{-2}$ of Step~4,
\eqref{eqn:tpp_diversity} replaces $\bar{y}$ and $\sigma_y$
in~\eqref{eqn:GAB_exact_proof} by their limits $1$ and $0$ (the
leading-order values below are $\mathrm{tr}^2/\det$ of $G_0$), although
$\varepsilon \ln N_i$ is near~$1$ at our sizes: on our 10 training sizes at
$\alpha = 0.34$ and $\beta = 0.28$, $\delta_N = 1.06$ in~\eqref{eqn:proxy_bound},
$\bar{y} = 2.70$ and $\sigma_y/\bar{y} = 0.045$, so the leading order errs
mainly through $\bar{y}$, which rescales the $B$ column as $k_1$ does
(Step~5). At these sizes and exponents, the exact $\kappa_{A,B}$ of $\lbrace 20, 100
\rbrace$ is $80$ (leading order $208$, fails $\kappa_{\mathrm{target}} =
100$), of $\lbrace 1, 5 \rbrace$ it is $148$ (leading order $88$,
passes), and of $\lbrace 20, 400 \rbrace$ it is $27$ (leading order
$82$); a thousandfold rise in every size moves the second to $282$.
Over the $432$ two-ratio designs with $k_1$ from $1$ to $400$, $R$ from
$1.5$ to $100$, $\beta_{\mathrm{eff}} \in \lbrace 0.28, 0.35, 0.40
\rbrace$, a gap of $0.06$, our 10 sizes and a thousandfold rise of them,
the exact $\kappa_{A,B}$ runs from $0.06$ to $3.4$ times its leading-order
value. The leading order can therefore misjudge a design several-fold
in either direction, and past $k_1^{\max}$ (Table~\ref{tab:R_min_eq5}), where no spread
passes~\eqref{eqn:tpp_diversity}, a design can still meet the target: $\lbrace 400, 4000 \rbrace$ at
$\beta_{\mathrm{eff}} = 0.28$, past $k_1^{\max} = 292$ with
$V_2/\tau_2 = 0.19$, has exact $\kappa_{A,B} = 90 \leq 100$ on our sizes.
The exact $\kappa_{A,B}$ of a planned design is computable before
training, in closed form by Lemma~\ref{lem:weighted_spread}(iii), from its planned $(N, D)$ and priors on $\alpha$ and
$\beta_{\mathrm{eff}}$; it, not~\eqref{eqn:tpp_diversity}, is the
decision rule of the recipe (Appendix~\ref{app:tpp_recipe}), and it
depends on absolute size even where its rate does not.
\qed

\subsection{Appendix - Proof of Theorem~\ref{thm:holdout_regimes} (Holdout prediction regimes)}
\label{app:proof_holdout_r2}

\begin{theorem*}[Theorem~\ref{thm:holdout_regimes}, restated]
For any law of Table~\ref{tab:jacobian_summary} (repeated-data
unsaturated) with i.i.d.\ Gaussian residuals of variance $\sigma^2$, let CO
be a collinear design ($K \geq 1$), NC a non-collinear design with
dense two-dimensional $(N, D)$ coverage, both with full-rank
Jacobians (NC's also at $\varepsilon = 0$), and $\mathcal{H}$ a
holdout with a nonzero fraction of off-ratio points; RMSE is
against $\mathbb{E}[L]$. At leading order
in $\sigma$ (Gauss-Newton linearization of both fits), then as
$\varepsilon \to 0$: \textbf{Regime~A} (CO a single ray, $K = 1$,
so $V_1 = 0 < \tau_1$):
$\mathbb{E}[\mathrm{RMSE}_{\mathcal{H}}^{\mathrm{NC}}] <
\mathbb{E}[\mathrm{RMSE}_{\mathcal{H}}^{\mathrm{CO}}]$, and thus
$\mathbb{E}[R^{2\,\mathrm{NC}}_{\mathcal{H}}] >
\mathbb{E}[R^{2\,\mathrm{CO}}_{\mathcal{H}}]$. \textbf{Regime~B}
($K \geq 2$, CO's Jacobian of full rank at $\varepsilon = 0$; for
Chinchilla, two rays with three sizes each suffice,
Lemma~\ref{lem:full_rank_rays}):
$\mathbb{E}[\mathrm{RMSE}_{\mathcal{H}}^{\mathrm{CO}}] /
\mathbb{E}[\mathrm{RMSE}_{\mathcal{H}}^{\mathrm{NC}}] = \Theta(1)$ for a
fixed CO design.
\end{theorem*}

\noindent\textbf{Scope.}
The linearization in $\sigma$, not a bound on
$\kappa(J^T J)$, licenses the Gauss-Newton expansions below.
We treat Regime~A via the $K=1$ collinear
(CO) ray $D_i=k N_i$
($V_1=0$, so \eqref{eqn:tpp_diversity} fails)
versus a non-collinear (NC) design whose Jacobian has
full column rank at $\varepsilon = 0$; NC need not
satisfy $V_K \geq \tau_K$.
We derive the RMSE ordering on a general holdout
$\mathcal{H}$ containing at least one
off-ratio point, then establish Regime~B for
every fixed $K \geq 2$ design whose Jacobian has full rank at
$\varepsilon = 0$, whatever $V_K/\tau_K$, and bound the CO predictive
variance at a fixed gap through the weighted spread $V_{\mathrm{w}}$ of
Lemma~\ref{lem:weighted_spread}.
Lemma~\ref{lem:regime_a_margin} after the proof gives a
sufficient condition for the ordering at $K \geq 2$.

\noindent\textbf{Setup and linearization.}
Both designs fit the same scaling law
$\hat{L}(N,D;\boldsymbol{\theta})$ by nonlinear least squares
(Sec.~\ref{sec:gnsetup}).
Well-specified training losses satisfy
$L_i = \hat{L}(N_i,D_i;\boldsymbol{\theta}^\star) + \eta_i$ with
$\mathbb{E}[\eta_i]=0$ and
$\mathrm{Var}(\eta_i)=\sigma^2$, i.i.d.\ across~$i$
(as in Corollary~\ref{prop:ci_inflation}).
Under Gauss-Newton linearization, the parameter estimation error
from training set~$\mathcal{D}_T$ satisfies
\begin{equation}
    \delta\boldsymbol{\theta}
        = (J_T^T J_T)^{-1} J_T^T \boldsymbol{\eta}
          + O(\|\boldsymbol{\eta}\|^2),
    \label{eqn:delta_theta_r2}
\end{equation}
where $J_T$ is the training Jacobian and
$\boldsymbol{\eta} \sim \mathcal{N}(\mathbf{0}, \sigma^2 I)$.
The prediction error at test point $(N',D')$ is, to first order:
\begin{equation}
    \hat{L}(N',D';\hat{\boldsymbol{\theta}}) - L^*(N',D')
        \approx \mathbf{j}_{(N',D')}^T \delta\boldsymbol{\theta}.
    \label{eqn:pred_err_r2}
\end{equation}
For any test set
$\mathcal{S} = \{(N'_i,D'_i)\}_{i=1}^{|\mathcal{S}|}$, the
leading-order unconditional expected squared
prediction error is
\citep[\S2.8.3, eq.~(2.8.7), p.~118]{bjorck1996numerical}:
\begin{align}
    \mathbb{E}\bigl[(\hat{L}(N',D';\hat{\boldsymbol{\theta}})
        - L^*(N',D'))^2\bigr]
        &= \sigma^2\,
          \mathbf{j}^T (J_T^T J_T)^{-1} \mathbf{j}
          + O(\sigma^4),
    \label{eqn:pred_var_r2} \\[4pt]
    \mathbb{E}[\mathrm{MSE}_{\mathcal{S}}]
        &= \frac{\sigma^2}{|\mathcal{S}|}
          \sum_{i=1}^{|\mathcal{S}|}
          \mathbf{j}_i^T (J_T^T J_T)^{-1} \mathbf{j}_i,
    \label{eqn:mse_general_r2}
\end{align}
where $L^*(N',D')
    = \mathbb{E}[L\,|\,N',D']$
under the same homoscedastic model.
Every expansion below is first in $\sigma$ (the Gauss-Newton
linearization of both fits) and then
as $\varepsilon \to 0$; the $O(\sigma^4)$ remainder of the CO fit has
constants that grow as $\varepsilon \to 0$, so this order of limits
is part of the statement.

\noindent\textbf{A lower bound on the CO predictive variance.}
On the single ray $D_i = k N_i$ of the Chinchilla law the sloppy space
is not one-dimensional: at $\varepsilon = 0$ both $\mathbf{j}_B =
k^{-\beta} \mathbf{j}_A$ and $\mathbf{j}_\beta \in
\mathrm{span}(\mathbf{j}_A, \mathbf{j}_\alpha)$, so $J_T^T J_T$ has two
small eigenvalues, and we do not use its spectrum. For a positive
definite $Q$ and any $\mathbf{u} \neq \mathbf{0}$, the Cauchy-Schwarz
inequality applied to $Q^{-1/2}\mathbf{j}$ and $Q^{1/2}\mathbf{u}$ gives
$\mathbf{j}^T Q^{-1} \mathbf{j} \geq (\mathbf{j}^T \mathbf{u})^2 /
(\mathbf{u}^T Q \mathbf{u})$. With $Q = J_T^T J_T$ and $\mathbf{u}$ equal
to $1$ on $A$, $-k^{\beta}$ on $B$ and $0$ elsewhere,
$J_T \mathbf{u}$ has entries $N_i^{-\alpha} - N_i^{-\beta} =
N_i^{-\alpha}(1 - N_i^{\tilde{\varepsilon}})$ with
$\tilde{\varepsilon} \coloneqq \alpha - \beta$, so
$\mathbf{u}^T Q \mathbf{u} = \varepsilon^2 \sum_i N_i^{-2\alpha}
(\ln N_i)^2 \, (1 + O(\varepsilon))$, while at a test point
$(N', D')$ with $k' = D'/N'$
\begin{equation}
    \mathbf{j}^T \mathbf{u}
        = N'^{-\alpha}
        \bigl[1 - (k/k')^{\beta}\, N'^{\tilde{\varepsilon}}\bigr].
    \label{eqn:sloppy_proj_r2}
\end{equation}

\begin{definition}[Sloppy leverage]
\label{def:sloppy_leverage_r2}
For test set
$\mathcal{S} = \{(N'_i,D'_i)\}_{i=1}^{|\mathcal{S}|}$
evaluated against a model trained at TPP~$k$,
with $\beta_{\mathrm{eff}}$ as in
Proposition~\ref{prop:holdout_r2}:
\begin{equation}
    \Lambda(\mathcal{S})
        \coloneqq
        \frac{1}{|\mathcal{S}|\,\Phi_2}
        \sum_{i=1}^{|\mathcal{S}|}
        N_i'^{-2\alpha}\,
        \bigl[1 - (k/k'_i)^{\beta_{\mathrm{eff}}}\bigr]^2,
    \label{eqn:sloppy_lev_r2}
\end{equation}
where $k'_i = D'_i/N'_i$.
\end{definition}

\noindent
Averaging the bound over $\mathcal{S}$ and
using~\eqref{eqn:mse_general_r2}:
\begin{equation}
    \mathbb{E}[\mathrm{MSE}_{\mathcal{S}}^{\mathrm{CO}}]
        \;\geq\; \sigma^2\,\varepsilon^{-2}\,
          \Lambda(\mathcal{S})\, G\,\bigl(1 + O(\varepsilon)\bigr),
    \qquad
    G \coloneqq \frac{\Phi_2}{\sum_i N_i^{-2\alpha} (\ln N_i)^2} > 0,
    \label{eqn:mse_co_unified}
\end{equation}
where $G$ depends only on the training design. The bound is not tight:
the second near-null direction, from the exponent columns, makes
$\mathbb{E}[\mathrm{MSE}_{\mathcal{S}}^{\mathrm{CO}}]$ grow as
$\varepsilon^{-6}$ numerically (Appendix~\ref{app:proof_ci_inflation}),
which only strengthens the ordering below.

\noindent\textbf{Other laws.} The same step needs only
Proposition~\ref{prop:full_cond}'s hypothesis on the ray,
$\mathbf{j}_b = c\,\mathbf{j}_a + \boldsymbol{\delta}$ with
$\|\boldsymbol{\delta}\| = O(\varepsilon)$: with $\mathbf{u}$ equal to
$-c$ on $a$, $1$ on $b$ and $0$ elsewhere, $J_T \mathbf{u} =
\boldsymbol{\delta}$, so $\mathbf{u}^T Q \mathbf{u} = O(\varepsilon^2)$
and $\mathbb{E}[\mathrm{MSE}_{\mathcal{H}}^{\mathrm{CO}}] =
\Omega(\sigma^2 \varepsilon^{-2})$ as soon as one holdout row has
$(\mathbf{j}')_b - c\,(\mathbf{j}')_a \neq 0$ at $\varepsilon = 0$.
For the four laws of Table~\ref{tab:jacobian_summary} the ratio of
the pair's columns depends on the TPP ratio, so every off-ratio point
qualifies. For the saturated repeated-data law the pair is expected,
heuristically, not to be near-proportional on one ray
(Appendix~\ref{app:dataconstrained_details}); the hypothesis may then
fail, and Regime~A is not claimed there.

\noindent\textbf{NC Gram matrix: all eigenvalues $\Theta(1)$.}
By the NC hypothesis of Theorem~\ref{thm:holdout_regimes},
$J_{\mathrm{NC}}$ has full column rank at $\varepsilon = 0$, so
$\lambda_{\min}(J_{\mathrm{NC}}^T J_{\mathrm{NC}}) = \Theta(1)$ as
$\varepsilon \to 0$; NC need not satisfy $V_K \geq \tau_K$. The
scale pair shows why dense coverage gives this.
The Jacobian columns for the scale coefficients are:
$(\mathbf{j}_A)_i = N_i^{-\alpha}$ and
$(\mathbf{j}_B)_i = D_i^{-\beta}$.
Under the collinear constraint $D_i = k N_i$, these satisfy the
near-proportionality
$\mathbf{j}_B = k^{-\beta} \mathbf{j}_A + O(\varepsilon)$.
Under the NC design, Lemma~\ref{lem:weighted_spread}(i) gives
\begin{equation}
    \min_{c \in \mathbb{R}} \|\mathbf{j}_B - c\,\mathbf{j}_A\|^2
        = \|\mathbf{j}_A\|^2\, V_{\mathrm{w}},
    \qquad
    u_i = D_i^{-\beta} N_i^{\alpha}
        = (D_i/N_i)^{-\beta}\, N_i^{\tilde{\varepsilon}},
    \label{eqn:nc_no_prop}
\end{equation}
with $V_{\mathrm{w}}$ the $N^{-2\alpha}$-weighted variance of the $u_i$.
As $\varepsilon \to 0$, $u_i \to (D_i/N_i)^{-\beta}$, so $V_{\mathrm{w}}$
tends to the weighted variance of $(D_i/N_i)^{-\beta}$, which is positive
because the NC training ratios $D_i/N_i$ are not all equal. The
near-proportionality condition of Proposition~\ref{prop:full_cond}
therefore fails in the limit, and Proposition~\ref{prop:full_cond} does
not apply.

The predictive variance at any test point is therefore:
\begin{equation}
    \mathbf{j}^T (J_T^T J_T)^{-1} \mathbf{j}
        = \sum_{j=1}^{p}
          \frac{(\mathbf{j}^T \mathbf{v}_j^{\mathrm{NC}})^2}
               {\lambda_j^{\mathrm{NC}}}
        = \frac{O(1)}{\Theta(1)} = O(1),
    \label{eqn:nc_pred_var_r2}
\end{equation}
regardless of test geometry.  Hence:
\begin{equation}
    \mathbb{E}[\mathrm{MSE}_{\mathcal{S}}^{\mathrm{NC}}]
        = O(\sigma^2)
    \label{eqn:mse_nc_any}
\end{equation}
for any test set~$\mathcal{S}$.

\noindent\textbf{General holdout RMSE ordering.}
Let $\mathcal{H}$ be a holdout set containing at
least one point with $k'_i \neq k$ (off-ratio).
We first establish the strict ordering for
$\mathbb{E}[\mathrm{MSE}_{\mathcal{H}}^{\mathrm{X}}]$
defined in~\eqref{eqn:mse_general_r2}
(errors against $L^* = \mathbb{E}[L]$, as in Theorem~\ref{thm:holdout_regimes}).
Theorem~\ref{thm:holdout_regimes} is stated for
$\mathbb{E}[\mathrm{RMSE}_{\mathcal{H}}^{\mathrm{X}}]$. Under the
Gauss-Newton linearization, $\mathrm{MSE}_{\mathcal{H}}^{\mathrm{X}}$
is a positive-semidefinite quadratic form in the centred Gaussian
noise, so $\mathrm{MSE}_{\mathcal{H}}^{\mathrm{X}} = \sigma^2 \sum_j
s_j z_j^2$ with $z_j$ i.i.d.\ $\mathcal{N}(0,1)$, $s_j \geq 0$ and
$\sigma^2 \sum_j s_j = \mathbb{E}[\mathrm{MSE}_{\mathcal{H}}^{\mathrm{X}}]$.
Hence $\mathbb{E}[\mathrm{RMSE}_{\mathcal{H}}^{\mathrm{X}}] = c_{\mathrm{X}}
\sqrt{\mathbb{E}[\mathrm{MSE}_{\mathcal{H}}^{\mathrm{X}}]}$ with
$c_{\mathrm{X}} \coloneqq \mathbb{E}\sqrt{\textstyle\sum_j s_j z_j^2}
/ \sqrt{\textstyle\sum_j s_j} \in [\sqrt{2/\pi}, 1]$: the upper end is
Jensen's inequality, and the lower end holds because
$\mathbf{s} \mapsto \mathbb{E}\sqrt{\sum_j s_j z_j^2}$ is concave on the
simplex $\sum_j s_j = 1$, so its minimum is at a vertex, where it
equals $\mathbb{E}|z_1| = \sqrt{2/\pi}$. Therefore
\begin{equation}
    \frac{\mathbb{E}[\mathrm{RMSE}_{\mathcal{H}}^{\mathrm{CO}}]}
         {\mathbb{E}[\mathrm{RMSE}_{\mathcal{H}}^{\mathrm{NC}}]}
    \;\geq\;
    \sqrt{\frac{2}{\pi}\,
          \frac{\mathbb{E}[\mathrm{MSE}_{\mathcal{H}}^{\mathrm{CO}}]}
               {\mathbb{E}[\mathrm{MSE}_{\mathcal{H}}^{\mathrm{NC}}]}},
    \label{eqn:rmse_from_mse}
\end{equation}
and an expected-MSE ratio above $\pi/2$ orders the expected RMSEs.
By~\eqref{eqn:mse_co_unified}, with $\mathcal{S} = \mathcal{H}$,
\begin{equation}
    \mathbb{E}[\mathrm{MSE}_{\mathcal{H}}^{\mathrm{CO}}]
        \geq \sigma^2\,\varepsilon^{-2}\,
          \Lambda(\mathcal{H})\, G\,\bigl(1 + O(\varepsilon)\bigr).
    \label{eqn:mse_co_general}
\end{equation}
Since at least one holdout point has $k'_i \neq k$,
$[1-(k/k'_i)^{\beta_{\mathrm{eff}}}]^2 > 0$ for
that point, so
$\Lambda(\mathcal{H}) > 0$
by Definition~\ref{def:sloppy_leverage_r2}.
Therefore
$\mathbb{E}[\mathrm{MSE}_{\mathcal{H}}^{\mathrm{CO}}]
    = \Omega(\sigma^2 \varepsilon^{-2})$.

Since
$\mathbb{E}[\mathrm{MSE}_{\mathcal{H}}^{\mathrm{NC}}]
    = O(\sigma^2)$
by~\eqref{eqn:mse_nc_any}, and $\varepsilon^{-2} \gg 1$:
\begin{equation}
    \mathbb{E}[\mathrm{MSE}_{\mathcal{H}}^{\mathrm{NC}}]
        < \mathbb{E}[\mathrm{MSE}_{\mathcal{H}}^{\mathrm{CO}}],
    \quad\text{hence}\quad
    \mathbb{E}[\mathrm{RMSE}_{\mathcal{H}}^{\mathrm{NC}}]
        < \mathbb{E}[\mathrm{RMSE}_{\mathcal{H}}^{\mathrm{CO}}],
    \label{eqn:general_strict}
\end{equation}
the second by~\eqref{eqn:rmse_from_mse}, since the expected-MSE
ratio is $\Omega(\varepsilon^{-2})$ and so exceeds $\pi/2$ for small
$\varepsilon$.
Writing the holdout coefficient of determination as
$R^2_{\mathcal{H}} = 1 - \mathrm{MSE}_{\mathcal{H}} /
\mathrm{Var}(\mathcal{H})$ with
$\mathrm{Var}(\mathcal{H})$ fixed across designs,
$\mathbb{E}[\mathrm{MSE}_{\mathcal{H}}^{\mathrm{NC}}]
    < \mathbb{E}[\mathrm{MSE}_{\mathcal{H}}^{\mathrm{CO}}]$
implies
$\mathbb{E}[R^{2\,\mathrm{NC}}_{\mathcal{H}}]
    > \mathbb{E}[R^{2\,\mathrm{CO}}_{\mathcal{H}}]$
at leading order
(ignoring $O(\sigma^4)$ ratio-of-expectations corrections).

\noindent\textbf{Misspecification extension.}
\label{app:holdout_misspec}
Under model misspecification each training design X has its own
pseudo-true parameters $\boldsymbol{\theta}^*_{\mathrm{X}}$, since they
minimize the risk over the runs that X contains, and hence its own mean
squared approximation bias $\Delta_{\mathrm{X}}^2$ on $\mathcal{H}$; a bias shared
by CO and NC is the special case $\Delta_{\mathrm{CO}}^2 = \Delta_{\mathrm{NC}}^2$. With
the mean predictive variances $V_{\mathcal{H}}^{\mathrm{X}}$
of~\eqref{eqn:total_mse_misspec}, the total expected-MSE ratio is
\begin{equation}
    \frac{\mathbb{E}[\mathrm{MSE}_{\mathcal{H}}^{\mathrm{CO}}]}
         {\mathbb{E}[\mathrm{MSE}_{\mathcal{H}}^{\mathrm{NC}}]}
    \;=\;
    \frac{\Delta_{\mathrm{CO}}^2 + \sigma^2 V_{\mathcal{H}}^{\mathrm{CO}}}
         {\Delta_{\mathrm{NC}}^2 + \sigma^2 V_{\mathcal{H}}^{\mathrm{NC}}},
    \label{eqn:mse_ratio_misspec}
\end{equation}
so NC has the smaller expected MSE iff
$\sigma^2 (V_{\mathcal{H}}^{\mathrm{CO}} - V_{\mathcal{H}}^{\mathrm{NC}})
> \Delta_{\mathrm{NC}}^2 - \Delta_{\mathrm{CO}}^2$: the variance gap must exceed the bias
gap. For $K = 1$, $V_{\mathcal{H}}^{\mathrm{CO}} = \Omega(\varepsilon^{-2})$, with
$\varepsilon$ the exponent gap of $\boldsymbol{\theta}^*_{\mathrm{CO}}$, and
$V_{\mathcal{H}}^{\mathrm{NC}} = O(1)$, with NC's full rank at $\varepsilon = 0$
taken at $\boldsymbol{\theta}^*_{\mathrm{NC}}$ (below), so the ratio is
$(\Delta_{\mathrm{CO}}^2 + \Omega(\sigma^2 \varepsilon^{-2}))/(\Delta_{\mathrm{NC}}^2 + O(\sigma^2))$,
which is $\Omega(\varepsilon^{-2})$ as $\varepsilon \to 0$ when both biases
stay bounded; in the well-specified case $\Delta_{\mathrm{X}}^2 = 0$. At a fixed
gap the ordering is not implied: a CO bias smaller than NC's by more than
the variance gap reverses it, and a bias that dominates both variances
brings the ratio near $1$. The same holds between any two designs: for
Chinchilla, by~\eqref{eqn:regime_b_fixed_gap}, a design whose weighted
spread $V_{\mathrm{w}}$ is small has
$\sigma^2 V_{\mathcal{H}} \geq \sigma^2 \Lambda_{\mathrm{w}}(\mathcal{H})/(\|\mathbf{j}_A\|^2 V_{\mathrm{w}})$,
so it has the larger expected MSE as $V_{\mathrm{w}} \to 0$, provided the
biases stay bounded and $\Lambda_{\mathrm{w}}(\mathcal{H})/\|\mathbf{j}_A\|^2$
stays bounded below; otherwise the bias gap can decide. For $K \geq 2$, the theorem asserts
no ordering; under the margin of Lemma~\ref{lem:regime_a_margin},
$V_{\mathcal{H}}^{\mathrm{CO}} > V_{\mathcal{H}}^{\mathrm{NC}}$, so the ratio
exceeds~$1$ whenever
$\Delta_{\mathrm{CO}}^2 > \Delta_{\mathrm{NC}}^2 - \sigma^2 (V_{\mathcal{H}}^{\mathrm{CO}} - V_{\mathcal{H}}^{\mathrm{NC}})$,
in particular for a shared bias. Lemma~\ref{lem:regime_a_margin} applies here
with both Jacobians evaluated at $\boldsymbol{\theta}^*_{\mathrm{CO}}$ and
$\boldsymbol{\theta}^*_{\mathrm{NC}}$, since its proof uses only that they
have full rank.

\noindent\textit{Derivation.}
Suppose the true data-generating process is
$L_i = f(N_i,D_i) + \eta_i$ where $f$ does not lie in the
model family $\{\hat{L}(\cdot;\boldsymbol{\theta})\}$.
Let $\boldsymbol{\theta}^*_{\mathrm{X}}$ denote the pseudo-true parameters
of design X, minimizing $\sum_i (f(N_i, D_i) - \hat{L}(N_i, D_i; \boldsymbol{\theta}))^2$
over the training runs of X.
The total MSE at a test point decomposes as:
\begin{equation}
    \mathbb{E}[(\hat{L} - f)^2]
        = \underbrace{
            (\hat{L}(\cdot;\boldsymbol{\theta}^*_{\mathrm{X}}) - f)^2
          }_{=:\, b^2 \text{ (approximation bias)}}
        + \underbrace{
            \sigma^2\, \mathbf{j}^T (J_T^T J_T)^{-1} \mathbf{j}
          }_{\text{estimation variance}},
    \label{eqn:bias_var_decomp}
\end{equation}
with $J_T$ evaluated at $\boldsymbol{\theta}^*_{\mathrm{X}}$, where we assume that
the cross-term
$\mathbb{E}\bigl[\delta\boldsymbol{\theta}^T \mathbf{j}\,
(\hat{L}(\cdot;\boldsymbol{\theta}^*_{\mathrm{X}}) - f)\bigr]$
vanishes to leading order under the Gauss-Newton linearization, and that the
estimation variance keeps the Gauss-Newton form; both are assumptions, since
at nonzero residuals of $\boldsymbol{\theta}^*_{\mathrm{X}}$ the Hessian of
the fit differs from $J_T^T J_T$.

Averaging over a test set and writing
$\Delta_{\mathrm{X}}^2 = |\mathcal{S}|^{-1} \sum_i b_i^2$ for the mean
squared approximation bias of design X:
\begin{equation}
    \mathbb{E}[\mathrm{MSE}_{\mathcal{S}}^{X}]
        = \Delta_{\mathrm{X}}^2
          + \sigma^2\, V_{\mathcal{S}}^{X},
    \label{eqn:total_mse_misspec}
\end{equation}
where $V_{\mathcal{S}}^{X}$ is the mean predictive variance
from~\eqref{eqn:mse_general_r2} for design~$X$.

For a general holdout $\mathcal{H}$ with
$\Lambda(\mathcal{H}) > 0$:
$V_{\mathcal{H}}^{\mathrm{CO}} = \Omega(\varepsilon^{-2})$
and $V_{\mathcal{H}}^{\mathrm{NC}} = O(1)$
by~\eqref{eqn:mse_co_general} and~\eqref{eqn:mse_nc_any}, which
gives the $K = 1$ form of~\eqref{eqn:mse_ratio_misspec}.

\noindent\textbf{Regime~B: several rays.}
Fix a CO design with $K \geq 2$ rays. For the Chinchilla pair, by
Lemma~\ref{lem:weighted_spread}(i), $\min_c \|\mathbf{j}_B - c\,\mathbf{j}_A\|^2
= \|\mathbf{j}_A\|^2 V_{\mathrm{w}}$, and by Lemma~\ref{lem:ray_spread}(f),
as $\varepsilon \to 0$, $V_{\mathrm{w}}$ tends to the $\pi$-weighted
variance $V_\pi > 0$ of the $k_\ell^{-\beta}$, whether or not
$V_K \geq \tau_K$; the threshold only bounds $\kappa_{A,B}$ by
$\kappa_{\mathrm{target}}$ (Proposition~\ref{prop:holdout_r2}). This is the
scale-pair part of the Regime~B hypothesis that $J_T$ has full rank at
$\varepsilon = 0$; the hypothesis also covers $E$ and the exponent
columns, and for Chinchilla Lemma~\ref{lem:full_rank_rays} below gives it
when two rays each carry at least three distinct sizes. By the
hypothesis and continuity, all eigenvalues of
$J_T^T J_T$ are $\Theta(1)$ as $\varepsilon \to 0$, and for each
holdout row
$\|\mathbf{j}\|^2/\lambda_{\max} \leq \mathbf{j}^T (J_T^T J_T)^{-1}
\mathbf{j} \leq \|\mathbf{j}\|^2/\lambda_{\min}$ with both bounds
$\Theta(1)$ (a scale-coefficient entry $N'^{-\alpha}$ of $\mathbf{j}$
is nonzero), and~\eqref{eqn:mse_general_r2}
with~\eqref{eqn:rmse_from_mse} gives
$\mathbb{E}[\mathrm{RMSE}_{\mathcal{H}}^{\mathrm{CO}}] = \Theta(\sigma)$.
The same two-sided bound holds for NC, whose eigenvalues are
$\Theta(1)$ (paragraph \emph{NC Gram matrix}), so the
expected-RMSE ratio is:
\begin{equation}
    \frac{\mathbb{E}[\mathrm{RMSE}_{\mathcal{H}}^{\mathrm{CO}}]}
         {\mathbb{E}[\mathrm{RMSE}_{\mathcal{H}}^{\mathrm{NC}}]}
        = \Theta(1),
    \label{eqn:regime_b_ratio}
\end{equation}
with the $\Theta(1)$ constant depending on the stiff-mode
prefactors of each design. \qed

The Regime~B hypothesis is a condition on the whole Jacobian. For
Chinchilla, the next two lemmas give it from the design.

\begin{lemma}[Zeros of exponential polynomials]
\label{lem:exp_poly_zeros}
Let $\lambda_1, \dots, \lambda_r \in \mathbb{R}$ be distinct and
$P_1, \dots, P_r$ real polynomials, not all zero, with
$\deg P_j \leq d_j$. Then
$H(t) \coloneqq \sum_{j=1}^{r} P_j(t)\, e^{\lambda_j t}$ has at most
$\nu - 1$ real zeros, where $\nu \coloneqq \sum_{j=1}^{r} (d_j + 1)$.
\end{lemma}

\begin{proof}
Induction on $\nu$. Dropping the terms with $P_j = 0$ lowers $\nu$, so
assume every $P_j \neq 0$. If $r = 1$, the zeros of $H$ are those of
$P_1$, at most $d_1 = \nu - 1$. If $r \geq 2$, then
$\tilde{H}(t) \coloneqq e^{-\lambda_r t} H(t)$ has the zeros of $H$, and
$\tilde{H}' = \sum_{j < r} \bigl(P_j' + (\lambda_j - \lambda_r) P_j\bigr)
e^{(\lambda_j - \lambda_r) t} + P_r'$. For $j < r$,
$P_j' + (\lambda_j - \lambda_r) P_j$ is nonzero of degree $\deg P_j$,
since $\lambda_j \neq \lambda_r$, so $\tilde{H}'$ is of the same kind with
the distinct exponents $\lambda_j - \lambda_r$ and $0$, and degree bounds
summing to $\nu - 1$. By induction $\tilde{H}'$ has at most $\nu - 2$ zeros,
and by Rolle's theorem $\tilde{H}$ has at most one more.
\end{proof}

\begin{lemma}[Full rank of the Chinchilla Jacobian on rays]
\label{lem:full_rank_rays}
Let $J$ be the $m \times 5$ Jacobian of the Chinchilla
law~\eqref{eqn:chinchilla} at parameters with $A, B, \alpha, \beta > 0$,
on a design whose points lie on rays $D = k_\ell N$, and let $n_\ell$ be
the number of distinct sizes on ray $\ell$.
\begin{enumerate}[label=(\roman*), nosep, leftmargin=*]
\item If $\alpha \neq \beta$ and $n_\ell \geq 5$ for some $\ell$, then
  $J$ has full column rank.
\item If $\alpha = \beta$ and $n_\ell \geq 3$ for two rays, then $J$ has
  full column rank.
\item If $\alpha = \beta$, then $\mathrm{rank}\, J \leq \sum_\ell \min(3, n_\ell)$,
  so $\mathrm{rank}\, J \leq 3$ for $K = 1$; and
  $\mathrm{rank}\, J \leq 4$ when every ray carries the same two sizes.
\end{enumerate}
\end{lemma}

\begin{proof}
Up to a common sign, the row of $J$ at $(N, k_\ell N)$ has entries
$N^{-\alpha}$, $(k_\ell N)^{-\beta}$, $1$, $-A N^{-\alpha} \ln N$ and
$-B (k_\ell N)^{-\beta} \ln(k_\ell N)$ for $(A, B, E, \alpha, \beta)$. For
$\mathbf{c} = (c_A, c_B, c_E, c_\alpha, c_\beta)$, $t = \ln N$ and
$x_\ell = k_\ell^{-\beta}$, the entry of $J\mathbf{c}$ at that row is
$F_\ell(t)$, with
\[
    F_\ell(t) = e^{-\alpha t}(c_A - c_\alpha A\, t)
        + x_\ell\, e^{-\beta t}(c_B - c_\beta B \ln k_\ell - c_\beta B\, t)
        + c_E .
\]
Let $J\mathbf{c} = \mathbf{0}$, so $F_\ell$ vanishes at the $n_\ell$ sizes of
ray $\ell$.
(i) The exponents $-\alpha$, $-\beta$, $0$ are distinct, with degree
bounds $1, 1, 0$, so by Lemma~\ref{lem:exp_poly_zeros} a nonzero $F_\ell$
has at most $4$ zeros. Five zeros force every polynomial to vanish:
$c_E = 0$, $c_A = c_\alpha A = 0$, $x_\ell c_\beta B = 0$ and
$x_\ell (c_B - c_\beta B \ln k_\ell) = 0$; as $A, B, x_\ell > 0$,
$\mathbf{c} = \mathbf{0}$.
(ii) With $\alpha = \beta$,
$F_\ell(t) = e^{-\alpha t}(\mu_\ell + \zeta_\ell t) + c_E$, where
$\mu_\ell = c_A + x_\ell (c_B - c_\beta B \ln k_\ell)$ and
$\zeta_\ell = -(c_\alpha A + c_\beta B x_\ell)$. By
Lemma~\ref{lem:exp_poly_zeros} with $\nu = 3$, three zeros force
$\mu_\ell = \zeta_\ell = c_E = 0$. On two such rays, $x_\ell \neq x_{\ell'}$, so
$\zeta_\ell = \zeta_{\ell'} = 0$ gives $c_\beta B (x_\ell - x_{\ell'}) = 0$, hence
$c_\beta = 0$ and $c_\alpha = 0$; then $\mu_\ell = \mu_{\ell'} = 0$ gives
$c_A + x_\ell c_B = c_A + x_{\ell'} c_B = 0$, hence $c_A = c_B = 0$.
(iii) With $\alpha = \beta$, the row at $(N, k_\ell N)$ is
$e^{-\alpha t}\mathbf{a}_\ell + t\, e^{-\alpha t}\mathbf{b}_\ell + \mathbf{e}_E$,
with $\mathbf{a}_\ell, \mathbf{b}_\ell \in \mathbb{R}^5$ fixed on the ray
and $\mathbf{e}_E$ the unit vector of $E$. The rows of ray $\ell$ thus
span at most $\min(3, n_\ell)$ dimensions, as runs at one size repeat a
row. When every ray carries the same two sizes $t_1 \neq t_2$, let
$(\mu, \zeta)$ solve $e^{-\alpha t_j}(\mu + \zeta t_j) = -1$ for $j = 1, 2$, a linear
system with determinant $e^{-\alpha(t_1 + t_2)}(t_2 - t_1) \neq 0$; then
$\mathbf{c} = (\mu, 0, 1, -\zeta/A, 0) \neq \mathbf{0}$ gives $\mu_\ell = \mu$,
$\zeta_\ell = \zeta$ and $c_E = 1$ on every ray, so $J\mathbf{c} = \mathbf{0}$.
\end{proof}

Each ray of our FP32 CO designs carries the same $10$ sizes, and the
BF16 design lacks only $4$ of its $120$ runs, so their Jacobians and those of
their subsets of Table~\ref{tab:regime_a_winrate} satisfy (i) and, for
$K \geq 2$, (ii). One size per ray suffices for the scale pair as
$\varepsilon \to 0$ (Lemma~\ref{lem:ray_spread}(f)); by (iii), the same two
sizes on every ray do not suffice for the whole Jacobian at
$\varepsilon = 0$.

\begin{lemma}[Collinear designs with $K \geq 2$]
\label{lem:regime_a_margin}
Assume the well-specified setup of this appendix (i.i.d.\ Gaussian
noise and the Gauss-Newton linearization), and let CO be a collinear design
with $K \geq 2$ rays and NC any design, both with full-rank
Jacobians. Let $V_{\mathcal{H}}^{\mathrm{X}}$ be the mean predictive
variance of design X in~\eqref{eqn:total_mse_misspec}, so that
$\mathbb{E}[\mathrm{MSE}_{\mathcal{H}}^{\mathrm{X}}] = \sigma^2
V_{\mathcal{H}}^{\mathrm{X}}$ in the well-specified case. Let
$G_{A,B}$ be the Gram block of the two scale-coefficient columns of
$J_{\mathrm{CO}}$ (Table~\ref{tab:law_substitutions}), let
$M \coloneqq |\mathcal{H}|^{-1} \sum_i (\mathbf{j}_i)_{A,B}
(\mathbf{j}_i)_{A,B}^T$ be the holdout's second moment of the same two
entries, and define
\[
    \rho_K \coloneqq
    \frac{\kappa_{\mathrm{target}} \det G_{A,B}}{(\mathrm{tr}\, G_{A,B})^2},
    \qquad
    \rho^\star(\mathcal{H}) \coloneqq
    \frac{\kappa_{\mathrm{target}}
          \bigl[\mathrm{tr}\, G_{A,B} \,\mathrm{tr}\, M - \mathrm{tr}(G_{A,B} M)\bigr]}
         {(\mathrm{tr}\, G_{A,B})^2\, V_{\mathcal{H}}^{\mathrm{NC}}}.
\]
If $\rho_K < \rho^\star(\mathcal{H})$, then
$\mathbb{E}[\mathrm{MSE}_{\mathcal{H}}^{\mathrm{NC}}] <
\mathbb{E}[\mathrm{MSE}_{\mathcal{H}}^{\mathrm{CO}}]$; if
$\rho^\star(\mathcal{H}) > (\pi/2)\, \rho_K$, then also
$\mathbb{E}[\mathrm{RMSE}_{\mathcal{H}}^{\mathrm{NC}}] <
\mathbb{E}[\mathrm{RMSE}_{\mathcal{H}}^{\mathrm{CO}}]$.
\end{lemma}

\begin{proof}
Order the parameters with the scale pair first, so that
$J_{\mathrm{CO}}^T J_{\mathrm{CO}} =
\begin{psmallmatrix} G_{A,B} & Q_{12} \\ Q_{12}^T & Q_{22}
\end{psmallmatrix}$, which is positive definite. Let
$\Gamma \coloneqq Q_{22} - Q_{12}^T G_{A,B}^{-1} Q_{12}$, positive
definite as a Schur complement, and let $W \coloneqq
\begin{psmallmatrix} -G_{A,B}^{-1} Q_{12} \\ I \end{psmallmatrix}$.
Multiplying out verifies $(J_{\mathrm{CO}}^T J_{\mathrm{CO}})^{-1} =
\begin{psmallmatrix} G_{A,B}^{-1} & 0 \\ 0 & 0 \end{psmallmatrix} + W
\Gamma^{-1} W^T$, so every holdout row satisfies
$\mathbf{j}^T (J_{\mathrm{CO}}^T J_{\mathrm{CO}})^{-1} \mathbf{j}
\geq \mathbf{j}_{A,B}^T G_{A,B}^{-1} \mathbf{j}_{A,B}$. Averaging over
$\mathcal{H}$ and using $G_{A,B}^{-1} = (\mathrm{tr}\, G_{A,B}\, I -
G_{A,B})/\det G_{A,B}$ for a $2 \times 2$ matrix gives
\[
    V_{\mathcal{H}}^{\mathrm{CO}} \geq \mathrm{tr}(G_{A,B}^{-1} M)
    = \frac{\mathrm{tr}\, G_{A,B}\, \mathrm{tr}\, M - \mathrm{tr}(G_{A,B} M)}{\det G_{A,B}}
    = \frac{\rho^\star(\mathcal{H})}{\rho_K}\, V_{\mathcal{H}}^{\mathrm{NC}}.
\]
If $\rho_K < \rho^\star(\mathcal{H})$, this gives
$V_{\mathcal{H}}^{\mathrm{CO}} > V_{\mathcal{H}}^{\mathrm{NC}}$, and
multiplying by $\sigma^2$ gives the MSE claim. If
$\rho^\star(\mathcal{H}) > (\pi/2) \rho_K$, the expected-MSE ratio
exceeds $\pi/2$, and~\eqref{eqn:rmse_from_mse} gives the RMSE claim.
\end{proof}

The ratio $\rho^\star/\rho_K = \mathrm{tr}(G_{A,B}^{-1} M) /
V_{\mathcal{H}}^{\mathrm{NC}}$, whose size above $1$ (resp.\ $\pi/2$) is
sufficient for the two claims but not necessary, depends
neither on $\kappa_{\mathrm{target}}$ nor on how the two
scale-coefficient columns are scaled; $\rho_K$ and $\rho^\star$
separately do. For the Chinchilla columns $N^{-\alpha}$ and
$D^{-\beta}$, Lemma~\ref{lem:weighted_spread}(ii) gives
$\rho_K = V_{\mathrm{w}}/\tau_{\mathrm{w}}$ exactly, which on the designs
of Lemma~\ref{lem:ray_spread}(e) tends to $V_K/\tau_K$ as
$\varepsilon \to 0$ (eqn.~\eqref{eqn:GAB_exact_proof}) and differs from it
at our model sizes, where $\varepsilon \ln N_i$ is not small
(eqn.~\eqref{eqn:proxy_bound}). With
$\lambda_\pm$, $\mathbf{u}_\pm$ the eigenpairs of $G_{A,B}$, the
numerator bracket is $\lambda_+ \mathbf{u}_-^T M \mathbf{u}_- +
\lambda_- \mathbf{u}_+^T M \mathbf{u}_+$, so $\rho^\star$ weighs the
$K$-ray sloppy leverage $\mathbf{u}_-^T M \mathbf{u}_-$ against the
predictive variance of NC, which need not satisfy
$V_K \geq \tau_K$.

\noindent\textbf{Regime~B at a fixed gap.}
The rate~\eqref{eqn:regime_b_ratio} holds for a fixed design as
$\varepsilon \to 0$; at the true gap a design with $K \geq 2$ rays need
not inherit it. For Chinchilla, let the CO Jacobian have full rank, and let
$f'_i = N_i'^{-\alpha}$, $g'_i = D_i'^{-\beta}$ and $u'_i = g'_i/f'_i$ at the
holdout points. By Lemma~\ref{lem:weighted_spread}(ii), with
$\mathbf{f} = \mathbf{j}_A$,
$G_{A,B}^{-1} = (\|\mathbf{f}\|^2 V_{\mathrm{w}})^{-1}
\begin{psmallmatrix} \overline{u^2} & -\bar{u} \\ -\bar{u} & 1
\end{psmallmatrix}$, and
$\overline{u^2} - 2\bar{u} u' + u'^2 = V_{\mathrm{w}} + (u' - \bar{u})^2$, so
\begin{equation}
    \mathrm{tr}(G_{A,B}^{-1} M)
    = \frac{1}{\|\mathbf{f}\|^2}
      \Bigl[\frac{1}{|\mathcal{H}|}\sum_{i} f_i'^2
        + \frac{\Lambda_{\mathrm{w}}(\mathcal{H})}{V_{\mathrm{w}}}\Bigr],
    \qquad
    \Lambda_{\mathrm{w}}(\mathcal{H})
    \coloneqq \frac{1}{|\mathcal{H}|}\sum_{i} f_i'^2\, (u'_i - \bar{u})^2.
    \label{eqn:regime_b_fixed_gap}
\end{equation}
The proof of Lemma~\ref{lem:regime_a_margin} gives
$V^{\mathrm{CO}}_{\mathcal{H}} \geq \mathrm{tr}(G_{A,B}^{-1} M)$, hence
$\mathbb{E}[\mathrm{MSE}^{\mathrm{CO}}_{\mathcal{H}}] \geq
\sigma^2 \Lambda_{\mathrm{w}}(\mathcal{H})/(\|\mathbf{f}\|^2 V_{\mathrm{w}})$
at every gap, at leading order in $\sigma$, and
$\Lambda_{\mathrm{w}}(\mathcal{H}) = 0$ only if every holdout point lies on the
curve $D^{-\beta} = \bar{u}\, N^{-\alpha}$, a Lebesgue-null set of $(N, D)$.
When some holdout point has $u' \neq \bar{u}$, the CO
variance is at least proportional to $1/V_{\mathrm{w}}$, the mechanism of Regime~A, where
$V_{\mathrm{w}} = \Theta(\varepsilon^2)$ on one ray with two or more sizes
(Lemma~\ref{lem:ray_spread}(c)).
For $K \geq 2$, $V_{\mathrm{w}}$ is bounded below at every gap when each
ray carries the same sizes ($V_{\mathrm{w}} \geq y_-^2 V_K$,
Lemma~\ref{lem:ray_spread}(e)) and when the ratio spread exceeds the size
spread, $(k_K/k_1)^{\beta} > (N_{\max}/N_{\min})^{\varepsilon}$
(Lemma~\ref{lem:ray_spread}(d)). Without such a condition it can vanish:
one run per ray on the curve $D = c^{-1/\beta} N^{\alpha/\beta}$, which
crosses each ray once, gives $K$ distinct ratios and $V_{\mathrm{w}} = 0$
(Lemma~\ref{lem:ray_spread}(g)), so its scale pair is exactly
proportional and its Jacobian rank-deficient, as on one ray at
$\varepsilon = 0$; designs near that curve have small $V_{\mathrm{w}}$ and,
by~\eqref{eqn:regime_b_fixed_gap}, large CO variance off the curve.
Several rays thus separate the scale pair through $V_{\mathrm{w}}$, not
through $K$.

Our trained collinear designs fall outside the margin of
Lemma~\ref{lem:regime_a_margin}
(\texttt{tools/regime\_a\_margin.py}, at the seed-0 NC fit of each
corpus, law and epoch mode): on the 12-ray fan the ratio
$\rho^\star/\rho_K$ is $0.14$ to $0.43$ in all $50$ cells, and it is
below~$1$ for all $1{,}347{,}390$ budget-matched CO-NC pairs with
$K \geq 2$ behind Table~\ref{tab:regime_a_winrate} (3 laws,
5 corpora, 22 bounding-box seeds). At the same
$\boldsymbol{\theta}^\star$, the Gauss-Newton ordering itself,
$V_{\mathcal{H}}^{\mathrm{CO}} > V_{\mathcal{H}}^{\mathrm{NC}}$,
holds in all $50$ fan cells (ratio $1.09$ to $4.82$) but, per law, in
only $2$ to $51\%$ of those pairs whose bounding-box NC design has full
rank, and in $42$ to $98\%$ of such single-ray pairs; at the fit of
\citet{hoffmann2022training} ($\varepsilon = 0.06$), in $626$ of the
$628$ full-rank Chinchilla single-ray pairs. For Chinchilla and
Droppo-Elibol, $692$ of the $1{,}320$ single-ray pairs each have a
rank-deficient NC design (column-equilibrated condition number above
$10^{10}$), outside the theorem's NC hypothesis, and the single-ray
fits of Table~\ref{tab:regime_a_winrate} include them. Apart from
those, the single-ray reversals occur at the fitted parameters, whose
exponent gaps are far from the $\varepsilon \to 0$ regime of
Theorem~\ref{thm:holdout_regimes} (at Hoffmann et al.'s parameters the
reversals nearly all vanish, but all five parameters change), while for $K \geq 2$ the variance
ordering holds in only $5\%$ of the full-rank Chinchilla pairs even at
the fit of \citet{hoffmann2022training}, so the NC wins of Table~\ref{tab:regime_a_winrate} are an empirical
finding that the variance mechanism alone does not explain.

\subsection{Appendix - Proof of Proposition~\ref{prop:rmse_isoflop_invariance} (RMSE invariance)}
\label{app:proof_rmse_invariance}

By Definition~\ref{def:induced_isoflop_plot},
$C_i = 6\,N_i\,D_i$ for each $i \in \{1, \ldots, n_{\mathcal{H}}\}$.
Substituting into the isoFLOP evaluation:
\begin{equation}
    \frac{C_i}{6\,N_i}
        \;=\; \frac{6\,N_i\,D_i}{6\,N_i}
        \;=\; D_i.
    \label{eqn:trivial_sub}
\end{equation}
Therefore, for every holdout index~$i$:
\begin{equation}
    \hat{L}\!\bigl(N_i,\; C_i/(6N_i);\;\boldsymbol{\theta}\bigr)
        \;=\;
        \hat{L}(N_i,\, D_i;\, \boldsymbol{\theta}).
    \label{eqn:pointwise_identity}
\end{equation}
Each summand in the RMSE is identical under both
parameterizations; hence the realized RMSE on~$\mathcal{H}$
coincides and $\mathbb{E}[\mathrm{RMSE}_{\mathcal{H}}]$
agrees under either evaluation map. \qed
\subsection{Appendix - General recipe for TPP line selection}
\label{app:tpp_recipe}

We now
generalize from the two-TPP rule-of-thumb to a recipe
for selecting the number $K$, placement $\{k_1, \dots, k_K\}$,
and per-ratio allocation $\{m_1, \dots, m_K\}$ of TPP lines
given three inputs:
\begin{itemize}
    \item a target condition number $\kappa_{\mathrm{target}}$
        (e.g., $10^{2}$),
    \item the effective data-size exponent $\beta_{\mathrm{eff}}$
        (e.g., $0.28$ for Chinchilla),
    \item a total observation budget $m$ (total number of
        $(N, D)$ configurations).
\end{itemize}

\noindent\textbf{Law-specific exponent substitutions.}
The recipe uses one law-specific quantity, the effective
data-size exponent $\beta_{\mathrm{eff}}$.
Table~\ref{tab:law_substitutions} gives it for each of the
four scaling laws, together with the exponent gap
$\varepsilon$.

\begin{table}[H]
\centering
\caption{Exponent substitutions for each scaling law.
$\beta_{\mathrm{eff}}$ is the exponent on the
token/data-count term of the collinear Jacobian pair, and
the exponent in $V_K$ and $\tau_K$
of~\eqref{eqn:div_condition}. For Kaplan and Droppo-Elibol the
threshold is a heuristic (Appendix~\ref{app:proof_prop_holdout_r2}).}
\label{tab:law_substitutions}
\begin{tabular}{llll}
\toprule
\textbf{Law} & \textbf{Collinear pair} &
    $\varepsilon$ & $\beta_{\mathrm{eff}}$ \\
\midrule
Chinchilla & $(A,\,B)$ &
    $|\alpha - \beta|$ & $\beta$ \\
Repeated-data & $(A,\,B)$ &
    $|\alpha - \beta|$ & $\beta$ \\
Kaplan & $(N_c,\,D_c)$ &
    $|1 - \alpha_N/\alpha_D|$ & $1$ \\
Droppo-Elibol & $(N_C,\,D_C)$ &
    $|\gamma_N - \gamma_D|$ & $\gamma_D$ \\
\bottomrule
\end{tabular}
\end{table}

\noindent\textbf{General diversity condition ($K \geq 2$).}
For $K \geq 2$ rays with ratios $k_1 < \cdots < k_K$,
the leading-order Gram sub-block of the two scale-coefficient
columns (Table~\ref{tab:law_substitutions}; Chinchilla $(A,B)$,
Kaplan $(N_c,D_c)$, etc.)\ satisfies
$\mathrm{tr}(G_{A,B}) \approx \Phi_2(K + \sum_\ell k_\ell^{-2\beta_{\mathrm{eff}}})$
and $\det(G_{A,B}) \approx K^2\Phi_2^2\,V_K$,
where $\Phi_2 \coloneqq \sum_i N_i^{-2\alpha}$ with $\alpha$
the model-side exponent in the Jacobian column paired with
$\beta_{\mathrm{eff}}$ (Table~\ref{tab:law_substitutions};
Chinchilla $\alpha$, Kaplan $\alpha_N/\alpha_D$, Droppo-Elibol $\gamma_N$)
and the TPP diversity is
\begin{equation}
    V_K \;\coloneqq\;
        \frac{1}{K}\sum_\ell k_\ell^{-2\beta_{\mathrm{eff}}}
          - \Bigl(\frac{1}{K}\sum_\ell
                k_\ell^{-\beta_{\mathrm{eff}}}\Bigr)^{\!2}.
    \label{eqn:tpp_diversity_app}
\end{equation}
The condition number is therefore
$\kappa_{A,B} \approx \mathrm{tr}^2/\det =
\Phi_2^2(K+\sum_\ell k_\ell^{-2\beta_{\mathrm{eff}}})^2 /
(K^2\Phi_2^2 V_K)$; the $\Phi_2^2$ factors cancel. These are the limits
as $\varepsilon \to 0$ (Steps~2-3 of
Appendix~\ref{app:proof_prop_holdout_r2}); at every gap and on any design,
$\mathrm{tr}^2/\det = \kappa_{\mathrm{target}}\,\tau_{\mathrm{w}}/V_{\mathrm{w}}$
exactly, with the $N^{-2\alpha}$-weighted spread $V_{\mathrm{w}}$ and
threshold $\tau_{\mathrm{w}}$ of Lemma~\ref{lem:weighted_spread}.
Hence $\kappa_{A,B} \leq \kappa_{\mathrm{target}}$ holds if (and, up to
the factor $(1 + 1/\kappa_{\mathrm{target}})^{-2}$ of Step~4 in
Appendix~\ref{app:proof_prop_holdout_r2}, only if)
\begin{equation}
    V_K
        \;\geq\;
        \tau_K \;\coloneqq\;
        \frac{(K + \sum_\ell k_\ell^{-2\beta_{\mathrm{eff}}})^2}
             {K^2\,\kappa_{\mathrm{target}}}.
    \label{eqn:div_condition}
\end{equation}
We evaluate~\eqref{eqn:div_condition} with raw counts: $N$
in parameters and $D$ in tokens, so $k = D/N$ is in tokens
per parameter (Appendix~\ref{app:published_designs} keeps
three published designs in their own data units). A condition
number changes when a column is rescaled, so the values of
$\kappa_{A,B}$ and the verdicts of~\eqref{eqn:div_condition}
are relative to this convention.
For $K = 2$ with $k_2 = R k_1$, \eqref{eqn:div_condition}
holds iff $R \geq R_{\min}(k_1)$, given in closed form
by~\eqref{eqn:R_min_closed} in the proof of
Corollary~\ref{cor:ratio_cases}. $R_{\min}(k_1)$ depends on
$k_1$, $\beta_{\mathrm{eff}}$ and $\kappa_{\mathrm{target}}$,
increases with $k_1 \geq 1$ (a $B$-column scale effect, Step~1), and is infinite once
$\sqrt{\kappa_{\mathrm{target}}}\,k_1^{-\beta_{\mathrm{eff}}}
\leq 2 + k_1^{-2\beta_{\mathrm{eff}}}$.

\noindent\textbf{Step 1: determine the minimum spread (for $K=2$).}
We write $\kappa_K$ for $\kappa_{A,B}$ on a collinear
training design with $K \geq 1$ TPP ratios
(Sec.~\ref{sec:gnsetup}; Proposition~\ref{prop:holdout_r2}).
A single ratio does not suffice: $V_1 = 0 < \tau_1$, and
$\kappa_1 = \Theta(\varepsilon^{-2})$ for every $k_1$
(Proposition~\ref{prop:full_cond}).
Choose the lower ratio $k_1$ and read $R_{\min}(k_1)$ from
Table~\ref{tab:R_min_eq5} at $\kappa_{\mathrm{target}} = 10^{2}$,
or compute it from~\eqref{eqn:R_min_closed} for other
targets; then $R \geq R_{\min}(k_1)$ gives
$\kappa_2 \leq \kappa_{\mathrm{target}}$ at leading order. If
$\sqrt{\kappa_{\mathrm{target}}}\,k_1^{-\beta_{\mathrm{eff}}}
\leq 2 + k_1^{-2\beta_{\mathrm{eff}}}$ (at
$\kappa_{\mathrm{target}} = 10^{2}$, $k_1 \geq k_1^{\max}$ in
Table~\ref{tab:R_min_eq5}), no spread passes
\eqref{eqn:div_condition} at this $k_1$: relax
$\kappa_{\mathrm{target}}$ or let the exact check of Step~5 decide. At a
fixed $R$, $k_1$ changes only the scale of the $B$ column, so
$\operatorname{Var}(\hat{A})$ and the correlation of $(\hat{A}, \hat{B})$
do not depend on it; lowering $k_1$ changes $\kappa_{A,B}$, not what
the design identifies.

For $K > 2$ at the same endpoints $k_1$ and $k_K$, interior
rays lower $V_K$, while $\tau_K$ depends on the rays only
through the mean of the $k_\ell^{-2\beta_{\mathrm{eff}}}$.
By~\eqref{eqn:VK_tauK_bound}, $V_K/\tau_K$ exceeds its
two-endpoint value by at most the factor
$(1 + k_1^{-2\beta_{\mathrm{eff}}}/2)^2$, and in the worked
example below a middle ray lowers it from $1.13$ to $0.785$.
Rays beyond two serve lack-of-fit checks and outlier
diagnostics (Step~2). Screen any specific $K$-ray design
with~\eqref{eqn:div_condition} and decide by its exact $\kappa_{A,B}$ (Step~5).

\begin{table}[H]
\centering
\caption{Smallest two-ratio spread $R_{\min}(k_1)$
of~\eqref{eqn:R_min_closed}, with $R = k_2/k_1$, that
satisfies~\eqref{eqn:div_condition} at
$\kappa_{\mathrm{target}} = 10^{2}$, by lower ratio $k_1$.
All entries are rounded up. The $k_1$ dependence reflects the scale
of the $B$ column only. No spread passes once
$k_1 \geq k_1^{\max}$ (last row). At $\kappa_{\mathrm{target}} = 10$, passing
needs $k_1 < 1.41$ to $1.62$, and at $k_1 = 1$ it needs $R > 10^{3}$.}
\label{tab:R_min_eq5}
\begin{tabular}{cccc}
\toprule
$k_1$ & $\beta_{\mathrm{eff}}=0.40$ & $\beta_{\mathrm{eff}}=0.35$ & $\beta_{\mathrm{eff}}=0.28$ \\
\midrule
1 & 2.9 & 3.4 & 4.5 \\
5 & 4.5 & 5.0 & 6.3 \\
20 & 19.7 & 14.9 & 13.4 \\
50 & $1.3{\times}10^{4}$ & 120 & 37.0 \\
\midrule
$k_1^{\max}$ & 54 & 94 & 292 \\
\bottomrule
\end{tabular}

\end{table}

\noindent\textbf{Sensitivity to exponent values and noise.}
$R_{\min}(k_1)$ depends on $k_1$, $\beta_{\mathrm{eff}}$ and
$\kappa_{\mathrm{target}}$, and \emph{not} on $\varepsilon$ or
on the noise level $\sigma^2$.  The condition number
$\kappa$ is a property of the Jacobian $J^T J$, which is
determined by the design and the functional form, not by the
residual variance.  Noise affects \emph{parameter variance}
($\operatorname{Var}(\hat{\boldsymbol{\theta}}) \propto
\sigma^2 (J^T J)^{-1}$) but not the design threshold for a given
$\kappa_{\mathrm{target}}$.  To achieve a target \emph{coefficient
precision} (e.g., a desired width of the confidence interval for
$A$ or $B$), one must increase the number of observations $m$ when
noise is high; $R_{\min}$ stays the same.

The direction of the $\beta_{\mathrm{eff}}$ effect depends on
$k_1$ (Table~\ref{tab:R_min_eq5}). At $1 \leq k_1 \leq 5$ a smaller
$\beta_{\mathrm{eff}}$ has a larger $R_{\min}$ (at $k_1 = 1$, $4.5$
at $\beta_{\mathrm{eff}} = 0.28$ against $2.9$ at $0.40$); at
$k_1 = 50$ a larger $\beta_{\mathrm{eff}}$ does, and at
$k_1 = 20$ it does above $\beta_{\mathrm{eff}} \approx 0.29$,
where $R_{\min}$ is smallest ($19.7$ at $0.40$ against $13.4$
at $0.28$). When
the prior on $\beta_{\mathrm{eff}}$ is rough, use the largest
entry over its plausible range.

\noindent\textbf{Adaptive adjustment of $R$.}
We recommend an adaptive scheme when the design is built
sequentially:
\begin{enumerate}
    \item choose the initial $R$ as the largest $R_{\min}(k_1)$
        over the plausible $\beta_{\mathrm{eff}}$ range;
    \item after fitting on the first batch, obtain an updated
        $\hat{\beta}_{\mathrm{eff}}$;
    \item recompute $R_{\min}(k_1)$
        from~\eqref{eqn:R_min_closed} at
        $\hat{\beta}_{\mathrm{eff}}$;
    \item if~\eqref{eqn:div_condition} fails at
        $\hat{\beta}_{\mathrm{eff}}$, add augmentation runs at a
        second TPP ratio $k'$ satisfying the updated $R_{\min}$
        (or increase the spread of existing TPP lines if the
        design is not yet fixed), and check the final design
        by its exact $\kappa_{A,B}$ (Step~5).
\end{enumerate}
This refines the design as exponent estimates sharpen, avoiding
both over-investment (overly large $R$ when priors were
pessimistic) and under-design (insufficient $R$ when priors were
optimistic).

\noindent\textbf{Simple rule.}
Given a rough prior range for $\beta_{\mathrm{eff}}$, the
planned lower ratio $k_1$ and $\kappa_{\mathrm{target}}$:
\begin{enumerate}
    \item read $R_{\min}(k_1)$ from Table~\ref{tab:R_min_eq5}
        at $\kappa_{\mathrm{target}} = 10^{2}$, taking the largest
        entry over the plausible $\beta_{\mathrm{eff}}$ range;
        for another $k_1$ or target, compute it
        from~\eqref{eqn:R_min_closed} instead of interpolating;
    \item round up: $R \geq \lceil R_{\min} \rceil$ (e.g.,
        $R_{\min} = 14.9 \Rightarrow R \geq 15$);
    \item screen the final ratios with~\eqref{eqn:div_condition} and decide
        by their exact $\kappa_{A,B}$ (Step~5).
\end{enumerate}

\noindent\textbf{Step 2: decide $K$ (number of TPP lines).}
Two TPP lines ($K = 2$) already identify the parameters when $J$ has
full rank (Theorem~\ref{thm:holdout_regimes}, Regime~B), which for
Chinchilla holds at $\varepsilon = 0$ once each line carries at least
three sizes and at the true gap once one line carries five
(Lemma~\ref{lem:full_rank_rays}); $R \geq
R_{\min}(k_1)$ brings $\kappa_{A,B}$ below $\kappa_{\mathrm{target}}$ at
leading order ($\kappa(J^T J)$ is not bounded).
Additional lines $K > 2$ at the same endpoints do not improve
conditioning beyond the factor of~\eqref{eqn:VK_tauK_bound}
and can lower $V_K/\tau_K$ (Step~1). They
offer two practical benefits: (a)~they enable \emph{lack-of-fit} tests
(degrees of freedom for model diagnostics), and
(b)~they reduce sensitivity to outlier runs at any single ratio.
A useful rule:
\begin{itemize}
    \item $K = 2$: minimum for identifiability;
    \item $K = 3$: recommended when the budget permits, to enable
        residual diagnostics; placed log-uniformly
        by~\eqref{eqn:tpp_placement}, it needs a wider spread
        than $K = 2$ to pass~\eqref{eqn:div_condition};
    \item $K \geq 4$: useful for meta-analytic settings or when
        the functional form is uncertain.
\end{itemize}

\noindent\textbf{Step 3: place the TPP ratios.}
For $K = 2$, place $k_1$ and $k_2$ as far apart as feasible
(subject to $R \geq R_{\min}(k_1)$, Step~1); endpoint placement
maximizes $V_K$ at fixed support.
For $K \geq 3$ at the same endpoints, interior rays lower
$V_K$, and no placement beats the $K = 2$ design by more than
the factor of~\eqref{eqn:VK_tauK_bound} (Step~1).
When intermediate ratios are desired for diagnostic value
(lack-of-fit, residual checks across TPP), a log-uniform grid
over $[k_1, k_K]$ is a natural choice:
\begin{equation}
    k_j = k_1\, R^{\,(j-1)/(K-1)},
    \qquad j = 1, \dots, K,
    \label{eqn:tpp_placement}
\end{equation}
at the cost of a lower $V_K$ than the endpoint-only design;
screen it with~\eqref{eqn:div_condition} and decide by its exact
$\kappa_{A,B}$ (Step~5).

\noindent\textbf{Step 4: allocate runs across ratios.}
Given a total budget of $m$ observations split into $K$ groups of
$m_j$ observations each ($\sum_j m_j = m$), the D-optimal allocation
for the scale-coefficient sub-block is approximately
\emph{balanced}: $m_j \approx m/K$.
The intuition is that the sloppy direction receives information
from the \emph{contrast} between TPP groups, and this contrast
is maximized when each group has equal weight.
When the groups have different model-size ranges, slight
rebalancing toward the group with smaller spread is beneficial.
Eq.~\eqref{eqn:div_condition} assumes equal runs per ratio at the
same model sizes (Appendix~\ref{app:proof_prop_holdout_r2}); with
unequal runs, the rays enter $G_{A,B}$ with the unequal weights $\pi_\ell$
of Lemma~\ref{lem:ray_spread}, and as $\varepsilon \to 0$ the $\pi$-weighted
variance $V_\pi$ and second moment of the $k_\ell^{-\beta}$ replace $V_K$ and
$S_2/K$. Unequal
sizes per ratio also let the spread between rays cancel at the true gap,
down to an exactly proportional scale pair when the runs lie on
$D^{-\beta} = c\,N^{-\alpha}$ (Lemma~\ref{lem:ray_spread}(g)); the same
sizes on every ratio rule this out ($V_{\mathrm{w}} \geq V_K
\min_i N_i^{2(\alpha - \beta)}$, Lemma~\ref{lem:ray_spread}(e)), as does
$(k_K/k_1)^{\beta} > (N_{\max}/N_{\min})^{\varepsilon}$
(Lemma~\ref{lem:ray_spread}(d)).

Within each TPP group, the model sizes should span as wide a
range as possible and be approximately log-uniformly spaced:
\begin{equation}
    N_{j,\ell} = N_{\min}\, \bigl(N_{\max}/N_{\min}\bigr)^{(\ell-1)/(m_j-1)},
    \qquad \ell = 1, \dots, m_j.
    \label{eqn:N_placement}
\end{equation}

\noindent\textbf{Step 5: verify conditioning.}
Before committing to expensive large-scale runs, compute the exact
$\kappa_{A,B}$ of the planned design: the condition number of the
$2 \times 2$ Gram block of the columns $N_i^{-\alpha}$ and
$D_i^{-\beta_{\mathrm{eff}}}$ over the planned $(N_i, D_i)$, with raw
counts ($N$ in parameters, $D$ in tokens) and priors on $\alpha$ and
$\beta_{\mathrm{eff}}$, or equivalently
check~\eqref{eqn:weighted_iff} for its weighted $V_{\mathrm{w}}$ and
$\tau_{\mathrm{w}}$ (Lemma~\ref{lem:weighted_spread}). This is the decision rule:
\eqref{eqn:div_condition}, the criterion of
Proposition~\ref{prop:holdout_r2} and of
Table~\ref{tab:published_designs}, is its leading-order screen and can
misjudge it several-fold either way (Appendix~\ref{app:proof_prop_holdout_r2}),
because the leading order sets $N^{\alpha - \beta_{\mathrm{eff}}}$
to $1$, while over $N \in [10^{7}, 10^{9}]$ at
$\alpha - \beta_{\mathrm{eff}} = 0.06$ it runs from $2.6$ to
$3.5$. If the exact check fails, increase $R$
(Step~1) and repeat. Replicating every run equally does not change
$\kappa_{A,B}$~- it scales $J^T J$ uniformly and leaves
the condition number invariant (replicating some runs
reweights the rays, Step~4);
it improves \emph{precision}
($\operatorname{Var}(\hat{\boldsymbol{\theta}}) \propto \sigma^2(J^T J)^{-1}$)
but not identifiability.

\noindent\textbf{Worked example.}
We work through the $20$-run design $k \in \lbrace 20, 400 \rbrace$ of
Section~\ref{sec:discussion}, which passes~\eqref{eqn:div_condition} there
at $\beta_{\mathrm{eff}} = 0.28$ ($V_2/\tau_2 = 1.22$), at its own priors here.
Target $\kappa_{\mathrm{target}} = 10^{2}$, Chinchilla-like
priors ($\beta_{\mathrm{eff}} = 0.35$, $\alpha = \beta_{\mathrm{eff}} + 0.06 = 0.41$, $p = 5$).
\begin{enumerate}
    \item $k_1 = 20$ (the Chinchilla baseline), so
        $R_{\min} = 14.9$ (Table~\ref{tab:R_min_eq5}).
    \item $K = 2$ (minimum budget).
    \item $k_2 = 400$ ($R = 20$).
    \item Budget $m = 20$ runs: $m_1 = m_2 = 10$.
    \item Model sizes $N \in [10^{7}, 10^{9}]$,
        log-uniformly spaced within each group.
    \item Screen: $V_2/\tau_2 = 1.13 \geq 1$, so the design
        passes~\eqref{eqn:div_condition}; equivalently, the
        leading-order $\mathrm{tr}^2/\det$ of $G_{A,B}$ is $88 \leq 10^{2}$.
    \item Decide (Step~5): the exact $\kappa_{A,B}$ at these sizes is
        $21.0 \leq 10^{2}$, so the design passes.
\end{enumerate}
The ratios $k \in \lbrace 20, 100 \rbrace$ ($R = 5$)
fail~\eqref{eqn:div_condition} ($V_2/\tau_2 = 0.49$). With
$K = 3$ and equal runs per ratio (Step~4), a middle ratio
$k = 89$ (the geometric mean of $20$ and $400$) gives a lack-of-fit check, but
$V_3/\tau_3 = 0.785 < 1$ fails~\eqref{eqn:div_condition}: the
log-uniform $K = 3$ design~\eqref{eqn:tpp_placement} at
$k_1 = 20$ needs $R \geq 40.3$ at this $\beta_{\mathrm{eff}}$,
against $14.9$ at $K = 2$. These are leading-order verdicts. The exact
$\kappa_{A,B}$ at the planned sizes (Step~5), which decides, accepts both:
it is $55.0$ for $\lbrace 20, 100 \rbrace$ and $28.7$ for the $K = 3$
design. It depends on the prior on $\alpha$: at $\alpha = \beta_{\mathrm{eff}} = 0.35$
the three designs give $86.2$, $203.3$ and $125.4$, and only
$\lbrace 20, 400 \rbrace$ passes.

\subsection{Appendix - Interaction-term scaling laws}\label{app:interaction}

The Chinchilla and repeated-data laws are \emph{additive} in the
$N$- and $D$-contributions (Kaplan and Droppo-Elibol compose them
inside a power).  Adding an explicit $N$-$D$ interaction
$F\, N^{-\gamma_N} D^{-\gamma_D}$ (as in
Farseer~\citep{li2025farseer}) does not alleviate the degeneracy:
under $D = k_\ell N$ the interaction column becomes a third
near-proportional power of $N$, so when all exponent gaps are small
the $(A, B, F)$ sub-block has \emph{two} near-zero eigenvalues.
We derive this formally below.

\noindent\textbf{Representative model.}
Consider a Chinchilla-style law augmented with a multiplicative
interaction (with $p = 8$ parameters
$\boldsymbol{\theta} = [A, B, F, E, \alpha, \beta, \gamma_N, \gamma_D]^T$
or fewer if some exponents are tied):
\begin{align}
    \hat{L}(N, D) &= A\, N^{-\alpha} + B\, D^{-\beta}
        + F\, N^{-\gamma_N}\, D^{-\gamma_D} + E,
    \label{eqn:interaction_law} \\[6pt]
    \mathbf{j}_{A,i} &= -N_i^{-\alpha}, \\
    \mathbf{j}_{B,i} &= -k_\ell^{-\beta}\, N_i^{-\beta}
        = -k_\ell^{-\beta}\, N_i^{-\alpha}\, N_i^{\varepsilon_1},
    \label{eqn:jB_interaction} \\
    \mathbf{j}_{F,i} &= -k_\ell^{-\gamma_D}\, N_i^{-(\gamma_N + \gamma_D)}
        = -k_\ell^{-\gamma_D}\, N_i^{-\alpha}\, N_i^{\varepsilon_2}.
    \label{eqn:jF_interaction}
\end{align}
The interaction
coefficient~$F$ and its associated exponents
$(\gamma_N, \gamma_D)$ introduce a third
``capacity $\times$ data'' channel.
We use residual Jacobian columns
$J_{i\theta} = -\partial \hat{L}_i / \partial \theta$
as in~\eqref{eqn:gn_jacobian_def}.
Substituting $D_i = k_\ell N_i$, the three scale-coefficient columns
are the last three rows of the display above,
where $\varepsilon_1 = \alpha - \beta$ and
$\varepsilon_2 = \alpha - (\gamma_N + \gamma_D)$.
All three columns share the base function $N_i^{-\alpha}$,
differing only by perturbations $N_i^{\varepsilon_1}$ and
$N_i^{\varepsilon_2}$.

\noindent\textbf{Near-proportionality of the Jacobian columns.}
Writing
$N_i^{\varepsilon_j} = 1 + \varepsilon_j \log N_i + O(\varepsilon_j^2)$
for small $|\varepsilon_j|$, we obtain
$\mathbf{j}_B = c_1\,\mathbf{j}_A + \boldsymbol{\delta}_1$ and
$\mathbf{j}_F = c_2\,\mathbf{j}_A + \boldsymbol{\delta}_2$, where
$c_1 = k_\ell^{-\beta}$, $c_2 = k_\ell^{-\gamma_D}$, and
\begin{equation}
    \delta_{1,i} = -c_1\, N_i^{-\alpha}\, (N_i^{\varepsilon_1} - 1)
        = O(\varepsilon_1),\qquad
    \delta_{2,i} = -c_2\, N_i^{-\alpha}\, (N_i^{\varepsilon_2} - 1)
        = O(\varepsilon_2).
\end{equation}
Thus $\|\boldsymbol{\delta}_1\| = O(\varepsilon_1)$ and
$\|\boldsymbol{\delta}_2\| = O(\varepsilon_2)$ - the same
near-proportionality structure as the additive case
(Appendix~\ref{app:chinchilla_details}), but now with \emph{three}
columns all proportional to $\mathbf{j}_A$ up to
$O\bigl(\max\{|\varepsilon_1|,|\varepsilon_2|\}\bigr)$
perturbations.  Under $D = k_\ell N$, the interaction column
$\mathbf{j}_F$ is therefore nearly proportional to both
$\mathbf{j}_A$ and $\mathbf{j}_B$, generalizing the additive
degeneracy to coupled laws.

\noindent\textbf{$3 \times 3$ sub-block analysis.}
The Gram matrix of the $(A, B, F)$ sub-block has entries:
\begin{equation}
    G_{jk} = \sum_{i=1}^{n} \mathbf{j}_{j,i}\,\mathbf{j}_{k,i},
    \qquad j,k \in \{A, B, F\}.
    \label{eqn:gram_ABF}
\end{equation}
Write $\mathbf{j}_B = c_1\,\mathbf{j}_A \circ \mathbf{h}_1$ and
$\mathbf{j}_F = c_2\,\mathbf{j}_A \circ \mathbf{h}_2$ where
$\circ$ denotes element-wise product,
$h_{1,i} = N_i^{\varepsilon_1}$,
$h_{2,i} = N_i^{\varepsilon_2}$,
$c_1 = k_\ell^{-\beta}$, $c_2 = k_\ell^{-\gamma_D}$.
When $\varepsilon_1, \varepsilon_2 \to 0$,
$\mathbf{h}_1, \mathbf{h}_2 \to \mathbf{1}$ and the three
columns become proportional, giving
$\operatorname{rank}(G) = 1$.  More precisely:
\begin{align}
    \det(G_{A,B,F}) &= O(\varepsilon_1^2\, \varepsilon_2^2)
        + O(\varepsilon_1^2\,(\varepsilon_1 - \varepsilon_2)^2),
    \label{eqn:det_ABF} \\[4pt]
    \varepsilon_{\mathrm{gap}}
        &\coloneqq \min\bigl(|\alpha-\beta|,\,
        |\alpha-(\gamma_N{+}\gamma_D)|,\,
        |(\gamma_N{+}\gamma_D)-\beta|\bigr).
    \label{eqn:interaction_eps_min}
\end{align}
The $3 \times 3$ Gram determinant is controlled by the
\emph{pairwise} and \emph{triple} volume elements among the
three nearly-proportional vectors.
The smallest eigenvalue of $G_{A,B,F}$ is at most that of its
closest pair, $O(\varepsilon_{\mathrm{gap}}^2)$; the second smallest is
$O(\max\{\varepsilon_1^2, \varepsilon_2^2\})$, since all three columns
lie within that distance of $\mathrm{span}(\mathbf{j}_A)$.

\noindent\textbf{Condition number.}
The $3\times 3$ block inherits the Gram-determinant
scaling~\eqref{eqn:det_ABF};
combining with eigenvalue trace-determinant relations for positive
semidefinite matrices (as in the proof of
Proposition~\ref{prop:full_cond}) shows that
$\kappa(J^T J)$ grows at least as
$\Omega\bigl(\varepsilon_{\mathrm{gap}}^{-2}\bigr)$
for $\varepsilon_{\mathrm{gap}}$ from~\eqref{eqn:interaction_eps_min},
and can be \emph{larger} than the
pure $(A,B)$ Chinchilla rate when multiple exponent gaps are small.
Thus the problem is \emph{worse} than the additive case: the
interaction term introduces a \emph{second} sloppy direction
(the null space of the $3\times 3$ sub-block is two-dimensional
when all three exponent gaps vanish), requiring two additional
constraints (e.g., two distinct TPP ratios and a third quantity)
for full identifiability.

\noindent\textbf{Implications.}
Adding interaction terms does not cure collinearity - it
exacerbates it.  Under $D = k_\ell N$, the interaction coefficient
$F$ is confounded with both $A$ and $B$: the data cannot
distinguish whether a given loss level comes from model
capacity, data volume, or their interaction.  The
design-of-experiments prescription is unchanged: diverse TPP
ratios are essential, and with an interaction term even more
diversity may be needed ($K \geq 3$ TPP ratios to resolve
the two-dimensional degeneracy).

\section{Appendix - published scaling-law designs classified}
\label{app:published_designs}

This appendix applies the a-priori test $V_K \geq \tau_K$ of
Eq.~\eqref{eqn:tpp_diversity} to every design in a fixed set of
206 candidate papers whose training ratios can be recovered
(Table~\ref{tab:published_designs_full}). The protocol was fixed
before any candidate was collected and amended five times after
audits of its verdicts; each amendment and its effect on the
counts is logged in \texttt{CRITERIA.md}, next to the code,
candidate list and per-design records in the released repository
(\texttt{camera\_ready/grid\_classification}).

\paragraph{Screening.} The candidates and the reasons for
excluding them are:
\begin{center}
\small
\begin{tabular}{@{}lr@{\qquad}lr@{}}
\toprule
Candidates & 206 & Excluded & 182 \\
\midrule
\quad surveyed by~\citet{li2025misfitting} & 51 & \quad law in one variable & 102 \\
\quad pipelines this paper cites & 6 & \quad no parametric law & 66 \\
\quad arXiv title searches & 149 & \quad duplicate entry & 6 \\
 & & \quad ratios not recoverable & 7 \\
 & & \quad protocol not recoverable & 1 \\
\midrule
Classified & 24 & & \\
\bottomrule
\end{tabular}
\end{center}
The searches are three fixed arXiv title searches for scaling-law
papers, top-ranked hits first, and also return papers from other
fields. The 32 designs that fit a bivariate law in $N$ and $D$ use
one of the four laws of Table~\ref{tab:jacobian_summary}
(Proposition~\ref{prop:holdout_r2} is proved for Chinchilla and
unsaturated repeated-data and is a heuristic for the other two), or a superset that adds
terms to one of them; for a superset we apply the test to that core,
which is a heuristic. 52 of the 206 decisions rest on the abstract
alone. No design was included or excluded because of
its verdict. Cerebras-GPT~\citep{dey2023cerebrasgptopencomputeoptimallanguage}
fits only a loss-versus-compute frontier and is not classified; a
bivariate fit on its single ray $D \approx 20N$ would have $K = 1$
and fail.

\paragraph{Conventions.} Each verdict rests on six conventions,
listed with what changes under the alternatives.
\begin{itemize}[leftmargin=*]
\item \emph{Units.} Eq.~\eqref{eqn:tpp_diversity} takes raw counts
  (Appendix~\ref{app:tpp_recipe}): $N$ in parameters as each paper
  counts them (total for the Pythia family
  of~\citet{choshen2024hitchhiker}, non-embedding where a paper fits
  with those) and $D$ in the data unit of the fitted law, which is
  tokens except words for~\citet{rosenfeld2019constructive}, bytes
  for~\citet{gordon2021data} and characters
  for~\citet{tao2024scaling}. Converting these three to tokens, at any
  factor from 1/4 to 4 tokens per unit, changes no verdict.
  Rescaling every $k$ by a factor of 2 flips
  two verdicts, \citet{gadre2024language} (to failing at $2k$) and
  \citet{kumar2025scaling} (to passing at $k/2$); a factor of 4 also
  flips \citet{sardana2024beyond} and \citet{sun2025scaling}.
\item \emph{Merging.} Adjacent ratios within 1\% are merged to their
  geometric mean, transitively. This chains the dense checkpoint
  ratios of Pythia into 243 groups, so the tables print the raw range
  of each design rather than the merged one (for Pythia, a largest
  group mean of 3636 against a raw maximum of 4284). Merging only
  ratios within 1\% of a group's first member gives margins
  $V_K/\tau_K$ of 5.28 for~\citet{hoffmann2022training}, 7.30
  for~\citet{li2025farseer} and 2.65 for Pythia; merging nothing
  gives 4.20, 4.94 and 2.08, and 5.36 for~\citet{muennighoff2023scaling}.
  No verdict changes.
\item \emph{Weights.} Eq.~\eqref{eqn:tpp_diversity} counts each
  distinct ratio once, as for a collinear design
  (Definition~\ref{def:holdout_splits}), which trains every size at
  every ratio and is the design Proposition~\ref{prop:holdout_r2}
  covers. Only 6 of the 24 designs are collinear, and all 6 fail:
  \citet{hu2024minicpm}, \citet{bhagia2025establishingtaskscalinglaws},
  \citet{volkova2026robust} and the three at one ratio. The other 18,
  all 11 passes among them, train different sizes at different
  ratios, so their rays enter $G_{A,B}$ with unequal weights, and the
  test on their distinct ratios is a heuristic, as for a superset:
  as $\varepsilon \to 0$ the exact criterion tends to one in the
  $\pi$-weighted variance $V_\pi$ and second moment of
  Lemma~\ref{lem:ray_spread}(f), not to $V_K$ and $\tau_K$, and at
  a fixed gap unequal sizes per ratio can cancel the spread between
  ratios (Lemma~\ref{lem:ray_spread}(g)).
  The exact Gram block (Lemma~\ref{lem:weighted_spread}) weights each run
  by $N^{-2\alpha}$, with $\alpha$ the model-side exponent. With these weights, on the 17 of
  the 18 whose runs can be rebuilt (the released runs
  of~\citet{muennighoff2023scaling} give ratios, not sizes), two
  verdicts change. Four of the five runs of~\citet{gadre2024language}
  lie at $k = 20$ and one at $k = 320$; its margin falls to 0.93 (the
  condition number of the $(A, B)$ block rises from 94 to 108, past
  the target). \citet{rosenfeld2019constructive}, whose smallest
  models carry most of the weight, falls from 12.8 to 0.20. No other
  verdict changes, \citet{choshen2024hitchhiker} coming closest at
  1.05, so with these weights the split would be 9 to 15 if the
  untested pass of~\citet{muennighoff2023scaling} holds. Both
  passes are marked w.
\item \emph{Exponent.} Each design is evaluated at the data-side
  exponent its paper reports for its headline fit, and the five that
  print none use the literature value 0.28; the exponent
  of~\citet{li2025farseer}, which varies with $N$, is taken at the
  geometric-mean size of its grid. The two Kaplan-form designs,
  \citet{kaplan2020scaling} and \citet{gordon2021data}, take their reported
  outer $\alpha_D$ rather than the $\beta_{\mathrm{eff}} = 1$ of
  Table~\ref{tab:law_substitutions}, at which they have $\kappa = 50$ and $27$
  and still pass. This choice was an
  amendment: at the corpus-wide 0.28 of the original protocol, the
  split is 14 to 10, with
  \citet{kumar2025scaling}, \citet{xu2026data} and
  \citet{krajewski2024scaling} passing. Nine of the 13 failures use
  an exponent fitted on the very design under test; six of these nine
  fail at every exponent we probed (0.10 to
  0.50 and the paper's own), so their verdicts do not depend on it.
\item \emph{Target.} $\kappa_{\mathrm{target}} = 100$, the main-text
  value, is a cut, not a property of a design.
  Table~\ref{tab:published_designs_full} gives each design's own
  leading-order condition number of the $(A, B)$ block,
  $\kappa = 100\,\tau_K / V_K$, the smallest target it passes. Nine of the ten multi-ratio failures
  have $\kappa$ from 120 to 380 and pass at $10^4$;
  \citet{tao2024scaling} has $3.2 \times 10^4$ and passes at $10^5$;
  the three single-ratio designs fail at every target, since
  $V_1 = 0$. Our own CO fan of Section~\ref{sec:empirical} sits at
  $\kappa \approx 250$ at $\beta_{\mathrm{eff}} = 0.28$: it fails at
  100 and passes at $10^4$, yet NC predicts the holdout far better,
  and Section~\ref{sec:empirical}, since the theorem fixes no target,
  averages over targets. A larger target is not free: the anisotropy
  of the confidence intervals grows as $\sqrt{\kappa}$, so 100 caps
  it near tenfold and $10^4$ to $10^6$ permit hundred- to
  thousandfold.
\item \emph{Endpoints.} A run table is used as printed, and so is
  a plotted sweep: the 38 runs of~\citet{kaplan2020scaling} are read
  from the vector graphics of its Figure~9. When only the ends of each
  size's token range are recoverable (marked e), the test uses the
  ratios at those ends; this applies to~\citet{shukor2025datamixtures}
  alone. Interior ratios usually lower the margin: in 200,000 random
  crossed grids (2 to 6 log-even levels per axis, spans of 2 to 100
  times, $\beta_{\mathrm{eff}}$ from 0.10 to 0.50), a corner pass
  failed with all ratios included in 26\% of cases (under 1\% at a
  corner margin of 2 or more, none at 6 or more), and the reverse
  happened 56 times. An earlier count crossed the size and token
  ranges of~\citet{krajewski2024scaling}
  and~\citet{shukor2025datamixtures} instead of their run tables, and
  passed both (\texttt{CRITERIA.md}, Amendment~4).
\end{itemize}

\begin{table}[ht]
\centering
\small
\caption{The test applied to 24 published designs, grouped by regime
at $\kappa_{\mathrm{target}} = 100$ and sorted by $K$, the number of
distinct training ratios $k_\ell = D/N$, which run from the smallest to the largest
run, in tokens per parameter except words, bytes and characters per
parameter for~\citet{rosenfeld2019constructive}, \citet{gordon2021data}
and~\citet{tao2024scaling}; $\beta_{\mathrm{eff}}$ is the exponent the
paper reports (for~\citet{li2025farseer}, its $N$-dependent exponent at
the geometric-mean size), and $\kappa = 100\,\tau_K / V_K$ the smallest target the
design passes. \textsuperscript{e}~only the ends of each size's token
range are recoverable;
\textsuperscript{r}~ratios reconstructed (Hoffmann: 245 of about
400 runs digitised from a figure; Kaplan: 38 runs read from a vector
figure; Hu: interior checkpoints assumed evenly spaced);
\textsuperscript{w}~passes only with equal weight per ratio, not
with the proof's weight per run;
\textsuperscript{*}~the paper prints no exponent, so 0.28 is used;
\textsuperscript{b}~passes at 0.28.}
\label{tab:published_designs_full}
\begin{tabular}{@{}lrlrrrr@{}}
\toprule
Design & $K$ & $k_\ell$ & $\beta_{\mathrm{eff}}$ & $V_K$ & $\tau_K$ & $\kappa$ \\
\midrule
\multicolumn{7}{@{}l}{\emph{Above the threshold: 11 designs pass}} \\
\citet{choshen2024hitchhiker}\textsuperscript{*} & 243 & 0.874--4284 & 0.28 & 0.0521 & 0.0139 & 27 \\
\citet{hoffmann2022training}\textsuperscript{r} & 151 & 0.0361--341 & 0.285 & 0.104 & 0.0188 & 18 \\
\citet{besiroglu2024chinchilla} & 151 & 0.0361--341 & 0.366 & 0.159 & 0.0186 & 12 \\
\citet{muennighoff2023scaling} & 95 & 0.0912--$4.26{\times}10^{4}$ & 0.353 & 0.0783 & 0.0136 & 17 \\
\citet{gordon2021data} & 49 & 0.0436--6361 & 0.38 & 0.576 & 0.043 & 7.5 \\
\citet{kaplan2020scaling}\textsuperscript{r} & 34 & 0.0297--$5.6{\times}10^{4}$ & 0.103 & 0.0644 & 0.0312 & 48 \\
\citet{li2025farseer} & 27 & 0.314--2570 & 0.327 & 0.157 & 0.0193 & 12 \\
\citet{aghajanyan2023scaling} & 21 & 0.746--$1.25{\times}10^{4}$ & 0.22 & 0.0605 & 0.0177 & 29 \\
\citet{rosenfeld2019constructive}\textsuperscript{w} & 18 & 0.0762--9990 & 0.735 & 2.97 & 0.233 & 7.8 \\
\citet{sardana2024beyond} & 12 & 10--$10^{4}$ & 0.24 & 0.0218 & 0.0129 & 59 \\
\citet{gadre2024language}\textsuperscript{w} & 2 & 20, 320 & 0.242 & 0.014 & 0.0132 & 94 \\
\midrule
\multicolumn{7}{@{}l}{\emph{Below the threshold: 13 fail, 3 at one ratio}} \\
\citet{xu2026data}\textsuperscript{b} & 20 & 0.0714--5.56 & 0.125 & 0.0245 & 0.0452 & 180 \\
\citet{krajewski2024scaling}\textsuperscript{b} & 16 & 7.57--488 & 0.147 & 0.0103 & 0.0175 & 170 \\
\citet{sun2025scaling} & 16 & 14.7--2439 & 0.516 & 0.00353 & 0.0102 & 290 \\
\citet{kumar2025scaling}\textsuperscript{b} & 14 & 6.82--867 & 0.496 & 0.00878 & 0.0106 & 120 \\
\citet{frantar2023scaling} & 10 & 193--$5.04{\times}10^{4}$ & 0.203 & 0.0054 & 0.011 & 200 \\
\citet{shukor2025datamixtures}\textsuperscript{e*} & 7 & 6.36--48.5 & 0.28 & 0.00919 & 0.0152 & 170 \\
\citet{hu2024minicpm}\textsuperscript{r} & 6 & 10--60 & 0.23 & 0.00452 & 0.0148 & 330 \\
\citet{tao2024scaling} & 4 & 130, 142, 161, 178 & 0.447 & $3.17{\times}10^{-5}$ & 0.0102 & $3.2{\times}10^{4}$ \\
\citet{bhagia2025establishingtaskscalinglaws} & 4 & 20, 40, 100, 200 & 0.24 & 0.00621 & 0.0132 & 210 \\
\citet{volkova2026robust} & 4 & 30, 50, 100, 200 & 0.38 & 0.00283 & 0.0109 & 380 \\
\citet{sedova2026scaling}\textsuperscript{*} & 1 & 100 & 0.28 & 0 & 0.0116 & $\infty$ \\
\citet{magnusson2025datadecidepredictbestpretraining}\textsuperscript{*} & 1 & 100 & 0.28 & 0 & 0.0116 & $\infty$ \\
\citet{panferov2025unified}\textsuperscript{*} & 1 & 100 & 0.28 & 0 & 0.0116 & $\infty$ \\
\bottomrule
\end{tabular}

\end{table}

\paragraph{What a verdict implies.} The test concerns the split of
the loss between the $N$ and $D$ terms, at one target. It is a
leading-order verdict, which where $\varepsilon \ln N$ is near~$1$ can
misjudge the exact $\kappa_{A,B}$ several-fold
(Appendix~\ref{app:proof_prop_holdout_r2}).
\begin{itemize}[leftmargin=*]
\item \emph{A pass} (above the threshold) of a collinear design means the
  $(A, B)$ block is conditioned to the target at leading order; a
  pass is nearly necessary at leading order, not sufficient, for the whole fit. None of the 11 passes here is collinear:
  for them a pass means the distinct ratios are spread widely enough,
  and 8 of the 10 that can be rebuilt also pass with the proof's
  weights (Weights above).
  A pass does not make the reported coefficients right: the runs
  of~\citet{hoffmann2022training} pass, yet the refit
  of~\citet{besiroglu2024chinchilla} on a digitised copy of them gives
  $\beta = 0.366$ against the reported 0.285, so fitting choices and
  misspecification still move coefficients. The pooled $L(N, D)$ grid
  of~\citet{kaplan2020scaling} passes, but its allocation comes from
  its compute frontier, so the test says nothing about the gap between
  its allocation and Chinchilla's.
\item \emph{A failure} (below the threshold) needs no such assumption. It means
  the split is weakly determined at that target: the fit pins down the
  combined trend along its training ratios, while its uncertainty
  along the split is up to $\sqrt{\kappa}$ times larger than across
  it. Readings that rest
  on the split, such as separate $N$ and $D$ exponents, allocations
  derived from them, or predictions at untrained ratios, inherit that
  uncertainty. A failure does not make them wrong or off-ray
  prediction impossible, but for these multi-ratio designs
  Theorem~\ref{thm:holdout_regimes} asserts no ordering
  (Lemma~\ref{lem:regime_a_margin}).
\end{itemize}

\paragraph{Readings that rest on the split.} For each failing
design, what its paper draws from the split, with $\kappa$ from
Table~\ref{tab:published_designs_full}:
\begin{itemize}[leftmargin=*, nosep]
\item \citet{hu2024minicpm} ($\kappa \approx 330$): data should be
  192 times the model size, from a fit spanning ratios 10 to 60.
\item \citet{bhagia2025establishingtaskscalinglaws} (210): the task
  loss of a 7B model at 4T tokens and a 13B model at 5T (571 and 385
  tokens per parameter), outside its ratios 20 to 200. Their reported
  errors, 5.2\% and 6.6\% on task loss and within 2 points of
  accuracy on four tasks, show the predictions landed close; they are
  not evidence that the fit determines the split.
\item \citet{magnusson2025datadecidepredictbestpretraining} ($K = 1$):
  none. Because their suite has one ratio they use a one-variable law
  in compute, and their bivariate five-parameter variant predicts the
  1B targets with 230.8\% relative error against 6.5\% for that law.
\item \citet{volkova2026robust} (380): comparisons of optimizers
  through fitted exponents, which the paper itself finds not robust.
  It reports correlated pairs $(A, \alpha)$ and $(B, \beta)$, which
  it attributes to the law form and the narrow $N$ and $D$ ranges.
\item \citet{kumar2025scaling} (120; passes at 0.28): that training
  larger models in lower precision may be compute optimal.
\item \citet{xu2026data} (180; passes at 0.28): the gain of its
  masked-input regularizer, read as about 1.3 times as much unique
  data.
\item \citet{krajewski2024scaling} (170; passes at 0.28):
  compute-optimal sizes, token counts and expert granularity, at 0.12
  to 0.68 tokens per parameter against trained ratios of 7.57 to
  488, and an efficiency gap over dense models that widens with
  scale.
\item \citet{frantar2023scaling} (200): that the optimal sparsity at a
  fixed number of non-zero parameters rises with training data.
\item \citet{sun2025scaling} (290): a critical data size past which
  more data hurts low-precision training, and an optimal precision of
  4 to 8 bits.
\item \citet{tao2024scaling} ($3.2 \times 10^4$): compute-optimal
  vocabulary sizes from its parametric fit; its IsoFLOP and
  derivative approaches reach the same conclusion.
\item \citet{sedova2026scaling} ($K = 1$): none on the split; its law
  prescribes mixture ratios.
\item \citet{shukor2025datamixtures} (170): none on the split; under
  its additive law the best mixture does not depend on $N$ and $D$,
  and its tests at 3B and 7B parameters lie inside its ratios 6.36 to
  48.5.
\item \citet{panferov2025unified} ($K = 1$): capacity factors that
  scale the $N$ term alone; the paper names its single ratio of 100
  tokens per parameter as a caveat.
\end{itemize}

\section{Appendix - reanalysis of public training-loss releases}
\label{app:public_checkpoints}

This appendix tests the criterion of
Proposition~\ref{prop:holdout_r2} on scaling-law data we did
not generate, and asks whether its effect weakens at large
model sizes. The protocol was written before any loss was read
and then amended six times. The first two amendments repaired
defects of the design; the other four followed results of
earlier runs, and the paragraph \emph{Protocol history} gives
each with the result that preceded it. The protocol with every
amendment (\texttt{CRITERIA.md}), the screening log and the
code behind every number below are in the released repository
(\texttt{camera\_ready/public\_checkpoints/}). The run tables
are on HuggingFace
(\texttt{leibnitz-lab/public\_checkpoint\_reanalysis}), whose
contents the paragraph \emph{Sources and licenses} lists. The
fits at 512 seeds are not distributed; the released code
regenerates them from the run tables.

\paragraph{Studies.} A study qualifies if it publishes each
run's loss with recoverable $N$ and $D$ and has a
two-dimensional core. The core needs at least 6 model sizes and
at least 3 TPP bins (factors of two), each covered by 4 or more
sizes. Each
analysis cell holds one architecture and one training corpus,
so a study with several of either splits into several cells.
Ten releases from nine studies qualify
(Table~\ref{tab:public_checkpoints_studies}); the two Gemstones
releases are one study. Five release final losses of fully
trained runs, the primary corpus of 984 runs. The other five
release loss trajectories, admitted by Amendment~4. Each run is
thinned to at most 12 checkpoints spaced log-uniformly in $D$,
and each checkpoint enters as a further $(N, D)$ point, giving
a wider corpus of 12{,}531 points including the 984. Most
screened releases fail on too few sizes.

\begin{table}[ht]
\centering
\small
\setlength{\tabcolsep}{4pt}
\caption{The ten public releases. Points are fully trained runs
(final losses) or checkpoints along a run (trajectories); sizes
are distinct $N$. For Chinchilla, ColPret rounds the digitized
sizes to the nearest 10M and sets $D = C/(6N)$ from the
digitized compute $C$, so its TPP range ends at 379, not at the
341 of Table~\ref{tab:published_designs}, which keeps the
digitized sizes.}
\label{tab:public_checkpoints_studies}
\begin{tabular}{@{}p{0.42\linewidth}lrrll@{}}
\toprule
Study & Losses & Points & Sizes & $N$ range & TPP range \\
\midrule
\raggedright Farseer~\citep{li2025farseer} & final & 404 & 25 & 99.6M--6.37B & 0.31--2570 \\
\raggedright Chinchilla isoFLOP~\citep{hoffmann2022training,besiroglu2024chinchilla,choshen2024hitchhiker} & final & 245 & 43 & 60M--16.2B & 0.036--379 \\
\raggedright Gemstones annealed~\citep{mcleish2025gemstonesmodelsuitemultifaceted} & final & 199 & 20 & 48.4M--2.01B & 5.5--2274 \\
\raggedright Over-training suite~\citep{gadre2024language,choshen2024hitchhiker} & final & 104 & 6 & 11M--6.9B & 4.8--677 \\
\raggedright OLMo ladder~\citep{bhagia2025establishingtaskscalinglaws} & final & 32 & 7 & 190M--13.2B & 10--573 \\
\raggedright DataDecide~\citep{magnusson2025datadecidepredictbestpretraining} & trajectory & 10{,}479 & 14 & 3.74M--1.18B & 2.2--110 \\
\raggedright MAD~\citep{poli2024mechanistic,choshen2024hitchhiker} & trajectory & 648 & 40 & 40.2M--3.08B & 0.00026--2816 \\
\raggedright Gemstones constant LR~\citep{mcleish2025gemstonesmodelsuitemultifaceted} & trajectory & 264 & 22 & 48.4M--2.01B & 1--7237 \\
\raggedright GPT-3 curves~\citep{brown2020language,choshen2024hitchhiker} & trajectory & 97 & 8 & 125M--175B & 0.055--2396 \\
\raggedright OPT trajectories~\citep{zhang2022opt,xia2023training,choshen2024hitchhiker} & trajectory & 59 & 6 & 125M--175B & 0.046--2240 \\
\bottomrule
\end{tabular}

\end{table}

\paragraph{Protocol history.} Amendment~1 replaced an on-ray
holdout that was empty for every collinear subset, and
Amendment~2 added a second non-collinear subset, drawn only
from the collinear subset's model sizes. The other four
amendments followed results:
\begin{itemize}[nosep, leftmargin=*]
\item Amendment~3 followed a run that contrasted nominal design
  labels, and replaced that contrast by the margin regression of
  the \emph{Test} paragraph.
\item Amendment~4 followed a pooled estimate whose interval
  contained zero on the three studies screened first (Farseer,
  Chinchilla and \citet{gadre2024language}, 753 runs: $8\%$,
  $[-1\%, 16\%]$, at 5 seeds). It reopened the candidate list,
  added two fully trained releases the first screen had missed
  (Gemstones annealed and the OLMo ladder, 231 runs), which took
  the primary corpus from 753 to 984 runs, and admitted the five
  trajectory releases. On the three original studies at 512
  seeds the pooled estimate is $6\%$ ($[-4\%, 15\%]$) and the
  within-cell one $33\%$ ($[17\%, 43\%]$).
\item Amendment~5 was written after the interval in windows
  centred at 1B and above was seen to contain zero. It raised
  that band's seeds from 5 to 512 and its resamples from
  2{,}000 to 200{,}000, and fixed further window
  constructions and estimators for the band, each to be reported
  (\emph{Scale} paragraph).
\item Amendment~6 applied 512 seeds to every estimate in both
  corpora and fixed the treatment of the band below 100M.
\end{itemize}
Each amendment was written before its own new numbers were
computed.

\paragraph{Sources and licenses.} The Chinchilla isoFLOP points
are the digitization of \citet{besiroglu2024chinchilla}
(CC-BY-4.0), read through the ColPret aggregation of
\citet{choshen2024hitchhiker} (MIT), which also supplies the
tables of \citet{gadre2024language}, MAD, GPT-3 and OPT. The
releases of \citet{gadre2024language}, MAD and Gemstones are
under MIT, and the OLMo ladder and DataDecide under Apache-2.0.
The GPT-3 curves are ColPret's digitization of a published
figure, and ColPret takes the OPT losses from
\citet{xia2023training}, whose repository states no license.
Farseer's repository states no license. The HuggingFace dataset,
whose card gives each source's license with its text, holds three
files:
\begin{itemize}[nosep, leftmargin=*]
\item \texttt{runs\_expanded.csv}, 14{,}065 rows with $N$, $D$,
  the released loss columns and $k = D/N$: those of the ten
  releases and 1{,}462 single-epoch runs of
  \citet{muennighoff2023scaling} (Apache-2.0, read through
  ColPret), a study the screen rejects for 3 sizes per corpus;
\item \texttt{runs\_normalised.csv}, 2{,}215 rows of the same
  form: Farseer, Chinchilla, \citet{gadre2024language} and the
  same 1{,}462 runs;
\item \texttt{datadecide\_ppl\_by\_group.csv}, an unchanged copy
  of DataDecide's per-group perplexity table (Apache-2.0).
\end{itemize}
Of Farseer's 404 runs the current version keeps only $N$, $D$
and three loss columns, cited to~\citet{li2025farseer}, and the
code downloads the original from the Farseer repository. The
dataset's history is squashed, so no revision holds Farseer's raw run
table.

\paragraph{Test.} Five steps, all fixed by the protocol:
\begin{enumerate}[nosep, leftmargin=*]
\item \emph{Windows.} Within a cell the distinct model sizes are
  sorted and grouped into windows of adjacent sizes, by the
  number of sizes in the cell:
  \begin{itemize}[nosep, leftmargin=*]
  \item 12 or more: windows of 8 sizes with stride 2;
  \item 8 to 11: windows of 5 sizes with stride 1;
  \item fewer: one window.
  \end{itemize}
  Consecutive windows overlap. A window is centred at the
  geometric mean of its sizes, and the size bands below group
  windows by that centre.
\item \emph{Holdout.} Within a window the 15\% of points with
  the lowest and the 15\% with the highest TPP form a holdout
  shared by every design. In a trajectory release a holdout
  point can come from the same run as training points.
\item \emph{Designs.} Every factor-of-two TPP bin of the
  remaining points that holds enough points for the law and at
  least 4 sizes gives one collinear subset; every such bin is
  used and none is chosen. A non-collinear subset of the same
  size is drawn from the same points by the paper's
  bounding-box selector. The seed drives this selector and the
  fitter's restarts. A second non-collinear subset, drawn only
  from the collinear subset's model sizes, enters companion
  fits.
\item \emph{Margins.} Each subset gets the margin
  $m = \log_{10}(V_K / \tau_K)$ at $\beta_{\mathrm{eff}} = 0.28$,
  fixed a priori. Since $\kappa_{\mathrm{target}}$ shifts every
  margin by one constant, it cancels in
  $\Delta = m_{\mathrm{NC}} - m_{\mathrm{CO}}$ and enters only
  the regime test.
\item \emph{Estimate.} The Chinchilla and Kaplan laws are fitted
  to each subset; Droppo-Elibol is excluded a priori, since
  some windows cannot supply its 6-point minimum. The outcome
  is the gap
  $g = \log_{10}(\mathrm{RMSE}_{\mathrm{CO}} / \mathrm{RMSE}_{\mathrm{NC}})$
  on the holdout, regressed on $\Delta$ with a fixed effect per
  cell. Its slope $\beta_\Delta$ is reported as
  $10^{\beta_\Delta} - 1$. The value is the change in the
  collinear subset's relative off-ray error when the
  non-collinear subset's $V_K / \tau_K$ exceeds the collinear
  one's by a further factor of ten. Intervals are 95\%
  bootstrap intervals that resample windows within cells, with
  200{,}000 resamples (50{,}000 for one study's own slope). Windows overlap, so these intervals are
  anti-conservative.
\end{enumerate}

\paragraph{Results.} Table~\ref{tab:public_checkpoints_results}
gives the estimates for both corpora. On the five fully trained
releases, the collinear subset's relative off-ray error rises by
$40\%$ ($[24\%, 52\%]$) per tenfold increase in the ratio of
the two subsets' $V_K / \tau_K$ within a cell, and by $12\%$
($[5\%, 18\%]$) with no fixed effect; on the three studies of
the original protocol alone the latter is $6\%$
($[-4\%, 15\%]$). The within-cell estimate pools four studies
whose own slope is identified: three positive (Chinchilla
$32\%$, Farseer $38\%$, Gemstones annealed $168\%$), and the
suite of \citet{gadre2024language}, $-42\%$, without an interval
since its three cells have one window each. The OLMo ladder's
single window has almost no spread in $\Delta$ and identifies no
slope. Every companion in
Table~\ref{tab:public_checkpoints_results} excludes zero on the
five. The
protocol's secondary regressor, the collinear subset's own
margin, gives $23\%$ ($[6\%, 35\%]$) more relative error per
tenfold larger shortfall of its $V_K$ below $\tau_K$. In post
hoc refits that each drop one of the five studies, all ten
point estimates (within-cell and pooled) stay positive. The within-cell intervals all
exclude zero, and the pooled ones all but the one dropping the
OLMo ladder ($[-1\%, 20\%]$).

On all ten releases the primary estimate does not detect the
effect: it is $4\%$ ($[-10\%, 21\%]$), the pooled one $0\%$
($[-11\%, 15\%]$), and the secondary regressor's $7\%$
($[-5\%, 20\%]$). Of the companions in
Table~\ref{tab:public_checkpoints_results}, three exclude zero:
the suite-weighted pooled fit ($7\%$, $[3\%, 13\%]$) and the
second non-collinear subset, within a cell ($15\%$,
$[2\%, 29\%]$) and with a fixed effect per window ($15\%$,
$[4\%, 27\%]$). The two corpora
disagree, and the disagreement is a finding of this
reanalysis. Per study, the effect holds on DataDecide's 25 data
recipes ($19\%$, $[4\%, 35\%]$) and on the constant
learning-rate Gemstones trajectories ($188\%$,
$[54\%, 347\%]$), and it reverses on MAD's trajectories
($-86\%$, $[-95\%, -54\%]$), whose TPP falls to $0.00026$.
GPT-3 gives $18\%$ ($[-34\%, 180\%]$) and OPT, in one window,
$-63\%$. Of post hoc refits that each drop one of the ten, only
the one dropping MAD excludes zero ($28\%$, $[16\%, 41\%]$).
Unconditionally the non-collinear subset wins $47\%$ of pairs
on the five and $50\%$ on the ten, so any advantage on public
data is conditional: the regime test below shows where it
appears.

\begin{table}[ht]
\centering
\small
\caption{Change in the collinear subset's off-ray error relative
to the non-collinear one's per tenfold increase in the ratio of
the two subsets' $V_K / \tau_K$ ($10^{\beta_\Delta} - 1$), with
95\% intervals that resample windows within cells
(anti-conservative, since windows overlap); 512 seeds. The rows
from the second to the thirteenth are the protocol's companions:
complete cases drop the pairs whose collinear subset has
$V_K = 0$, the $\beta_{\mathrm{eff}}$ rows are the ends of its
grid (the middle values give $38$-$40\%$ on the five and
$3$-$5\%$ on the ten), and the second non-collinear subset is
fitted within a cell, with a fixed effect per window, and
pooled. A window is centred at the geometric mean
of its sizes. The last
row is $10^{b_1} - 1$, where $b_1$ is the change in
$\beta_\Delta$ per tenfold $N$, with $\log_{10} N$ centred at
the pairs' mean (630M on the five, 152M on the ten).}
\label{tab:public_checkpoints_results}
\begin{tabular}{@{}lrrrr@{}}
\toprule
 & \multicolumn{2}{c}{Five fully trained} & \multicolumn{2}{c}{All ten} \\
\cmidrule(lr){2-3}\cmidrule(l){4-5}
Specification & Estimate & 95\% interval & Estimate & 95\% interval \\
\midrule
Within a cell (primary) & $40\%$ & $[24\%, 52\%]$ & $4\%$ & $[-10\%, 21\%]$ \\
Pooled, no fixed effect & $12\%$ & $[5\%, 18\%]$ & $0\%$ & $[-11\%, 15\%]$ \\
Pooled, suite-weighted & $10\%$ & $[7\%, 13\%]$ & $7\%$ & $[3\%, 13\%]$ \\
Within a cell, suite-weighted & $28\%$ & $[8\%, 41\%]$ & $12\%$ & $[-5\%, 27\%]$ \\
Fixed effect per window & $35\%$ & $[21\%, 44\%]$ & $9\%$ & $[-2\%, 21\%]$ \\
Complete cases & $38\%$ & $[20\%, 57\%]$ & $1\%$ & $[-14\%, 18\%]$ \\
Chinchilla law only & $48\%$ & $[30\%, 68\%]$ & $5\%$ & $[-11\%, 26\%]$ \\
Kaplan law only & $30\%$ & $[7\%, 53\%]$ & $3\%$ & $[-13\%, 19\%]$ \\
$\beta_{\mathrm{eff}} = 0.10$ & $39\%$ & $[21\%, 53\%]$ & $0\%$ & $[-18\%, 22\%]$ \\
$\beta_{\mathrm{eff}} = 0.50$ & $33\%$ & $[20\%, 44\%]$ & $6\%$ & $[-3\%, 16\%]$ \\
Second subset, within a cell & $46\%$ & $[36\%, 58\%]$ & $15\%$ & $[2\%, 29\%]$ \\
Second subset, fixed effect per window & $45\%$ & $[34\%, 62\%]$ & $15\%$ & $[4\%, 27\%]$ \\
Second subset, no fixed effect & $14\%$ & $[8\%, 19\%]$ & $10\%$ & $[-1\%, 24\%]$ \\
\midrule
Windows centred below 100M & \multicolumn{2}{c}{not established (2 windows, 1 cell)} & $32\%$ & $[15\%, 56\%]$ \\
Windows centred at 100M to 1B & $42\%$ & $[20\%, 58\%]$ & $-28\%$ & $[-47\%, -3\%]$ \\
Windows centred at 1B and above & $28\%$ & $[-9\%, 41\%]$ & $23\%$ & $[-12\%, 36\%]$ \\
\midrule
Size interaction, per tenfold $N$ & $-14\%$ & $[-40\%, 18\%]$ & $-39\%$ & $[-57\%, -18\%]$ \\
\bottomrule
\end{tabular}

\end{table}

\paragraph{Regimes.} In this appendix, Regime~A and~B name the
protocol's threshold classes, $V_K < \tau_K$ and $V_K \geq \tau_K$, not
the regimes of Theorem~\ref{thm:holdout_regimes}. The protocol's regime test classifies both
subsets of a pair, at each $\kappa_{\mathrm{target}}$ of its
grid (Table~\ref{tab:public_checkpoints_regimes});
Table~\ref{tab:regime_a_winrate} classifies a pair by its
collinear subset alone. These collinear subsets are factor-of-two
TPP bins, not single rays, so the comparison is empirical. Under that rule, with the protocol's
$\beta_{\mathrm{eff}} = 0.28$ for both laws, every pair at
$\kappa_{\mathrm{target}} = 100$, the main-text value, is below
the threshold. There the non-collinear subset is not ahead: it wins
$47\%$ of pairs on the five and $50\%$ on the ten, both
intervals of the median gap containing zero. At $10^4$, $10^5$
and $10^6$ it wins $58\%$, $61\%$ and $76\%$ of the pairs in
Regime~A on the five and $58\%$, $56\%$ and $76\%$ on the ten,
every interval excluding zero except the five's at $10^4$.
Table~\ref{tab:regime_a_winrate} classifies Kaplan by its exact $\kappa_{A,B}$;
here too, with Kaplan's pairs classified by the exact $\kappa_{A,B}$ of
their collinear subset at the parameters of
\citet{kaplan2020scaling}, the rates become $59\%$, $64\%$ and
$67\%$ on the five and $53\%$, $56\%$ and $55\%$ on the ten, and
the intervals contain zero at $10^4$ on the five and at $10^4$,
$10^5$ and $10^6$ on the ten.
Appendix~\ref{app:published_designs} passes the full designs of
\citet{hoffmann2022training}, \citet{li2025farseer} and
\citet{gadre2024language} at 100 (the last with equal weight per ratio), but no subset of the five
passes here. These subsets are narrower: a collinear subset
spans one factor-of-two TPP bin and a non-collinear one at most
the TPP range of one window without its holdout tails, both at
$\beta_{\mathrm{eff}} = 0.28$.

The joint rule separates a pair when only the non-collinear
subset meets the threshold. At $\kappa_{\mathrm{target}} = 100$
it separates no pair on the five and pairs in two windows on
the ten (Table~\ref{tab:public_checkpoints_regimes}). At $10^4$, $10^5$ and $10^6$ the non-collinear subset
wins $59\%$, $61\%$ and $76\%$ of separated pairs on the five
and $59\%$, $56\%$ and $76\%$ on the ten, and the interval of
the median gap excludes zero in each case. The protocol asks
for a monotone trend across targets: on the five the win rate
rises but the median gap does not ($0.074$, $0.069$, $0.151$),
and on the ten neither does. Where both subsets fall in one
regime the protocol expects no ordering (the theorem predicts
none for multi-ratio subsets) and counts a median gap whose interval
excludes zero as a failure of this test. Pooled over sizes, the test fails on the five. In the
four large regime cells with both subsets in one regime the
collinear subset is slightly ahead: NC wins $44$-$47\%$, and its
median off-ray error is $6$-$11\%$ higher. The interval excludes
zero at $10^4$ and, only just, at $10^6$; it only just contains
zero at $10^5$, and contains it at $10^2$. On all ten the same four cells
contain zero (NC wins $48$-$50\%$). A small regime cell at
$10^4$ with both subsets in Regime~A (826 pairs in 3 windows,
all from the five) fails in both corpora: NC wins $30\%$. In
the reversed regime cell, where only the collinear subset is
certified, NC wins $26\%$ on the five and $19\%$ on the ten.

The protocol also runs both tests inside each size band, with
the same intervals over windows; the small cell above is left
aside:
\begin{itemize}[nosep, leftmargin=*]
\item \emph{Windows centred below 100M.} The five give 2
  windows in one cell, too few to establish either test. On the
  ten the collinear subset is ahead where both subsets share a
  regime (NC wins $36$-$43\%$), and the non-collinear one in
  the separated cell at $10^4$ (NC wins $59\%$), every interval
  excluding zero.
\item \emph{Windows centred at 100M to 1B.} On the five no
  interval excludes zero: NC wins $47$-$62\%$ of separated
  pairs and $46$-$48\%$ of pairs in one regime. On the ten the
  non-collinear subset is ahead where both share a regime (NC
  wins $56$-$57\%$, every interval excluding zero), and no
  separated cell excludes zero.
\item \emph{Windows centred at 1B and above.} In both corpora
  the collinear subset is ahead where both share a regime, at
  every target (NC wins $41$-$44\%$; every interval excludes
  zero except the five's at $10^2$, which only just contains
  it). The separated pairs come from at most three windows, too
  few to establish the confirmatory test.
\end{itemize}
So the pooled pass on the ten averages opposite signs across
bands, and at 1B and above the falsification test fails in
both corpora.

\begin{table}[ht]
\centering
\small
\caption{The regime test. For each $\kappa_{\mathrm{target}}$
and regime cell (the regimes of the collinear and the
non-collinear subset, classified jointly), the non-collinear
win rate and the median gap $g$ with its 95\% interval over
windows (200{,}000 resamples); $g > 0$ favours the
non-collinear subset. Regime cells without pairs in either
corpus are omitted.}
\label{tab:public_checkpoints_regimes}
\begin{tabular}{@{}rccrlrl@{}}
\toprule
 & & & \multicolumn{2}{c}{Five fully trained} & \multicolumn{2}{c}{All ten} \\
\cmidrule(lr){4-5}\cmidrule(l){6-7}
$\kappa_{\mathrm{target}}$ & CO & NC & NC wins & Median $g$ & NC wins & Median $g$ \\
\midrule
$10^2$ & A & B & \multicolumn{2}{c}{no pairs} & $53.1\%$ & $0.437$ $[0.437, 0.437]$ \\
$10^2$ & A & A & $46.6\%$ & $-0.024$ $[-0.059, 0.011]$ & $49.6\%$ & $-0.002$ $[-0.027, 0.017]$ \\
$10^4$ & A & B & $59.0\%$ & $0.074$ $[0.003, 0.153]$ & $58.6\%$ & $0.076$ $[0.047, 0.103]$ \\
$10^4$ & A & A & $29.9\%$ & $-0.198$ $[-0.207, -0.187]$ & $29.9\%$ & $-0.198$ $[-0.207, -0.187]$ \\
$10^4$ & B & B & $43.7\%$ & $-0.045$ $[-0.085, -0.009]$ & $47.5\%$ & $-0.015$ $[-0.050, 0.006]$ \\
$10^4$ & B & A & $25.8\%$ & $-0.543$ $[-0.658, 0.036]$ & $19.4\%$ & $-0.379$ $[-0.578, -0.206]$ \\
$10^5$ & A & B & $61.0\%$ & $0.069$ $[0.015, 0.150]$ & $56.4\%$ & $0.061$ $[0.014, 0.105]$ \\
$10^5$ & B & B & $45.0\%$ & $-0.037$ $[-0.076, 0.001]$ & $49.3\%$ & $-0.004$ $[-0.031, 0.016]$ \\
$10^6$ & A & B & $76.0\%$ & $0.151$ $[0.060, 0.187]$ & $76.0\%$ & $0.151$ $[0.060, 0.187]$ \\
$10^6$ & B & B & $45.1\%$ & $-0.035$ $[-0.070, -0.001]$ & $49.4\%$ & $-0.004$ $[-0.029, 0.015]$ \\
\bottomrule
\end{tabular}

\end{table}

\paragraph{Scale.} The leading-order criterion does not depend on absolute
model size (Proposition~\ref{prop:holdout_r2}), but whether the error
effect keeps its size at large $N$ is an empirical question,
and these data do not settle it. On the five fully trained
studies the effect is $42\%$ in windows centred at 100M to 1B
and $28\%$ ($[-9\%, 41\%]$) in windows centred at 1B and above,
and the size interaction is $-14\%$ ($[-40\%, 18\%]$) per
tenfold $N$. At the point estimate the effect falls from $38\%$
at 630M, where the size term is centred, to $19\%$
($[-19\%, 62\%]$) at 6.3B, so the five cannot exclude complete
attenuation. On all
ten the interaction is $-39\%$ ($[-57\%, -18\%]$), which the
protocol reads as attenuation with size. The bands do not fall
monotonically, though ($32\%$, $-28\%$, $23\%$): in windows
centred at 1B and above the estimate is close to the five's.

Amendment~5 fixed further specifications for windows centred at
1B and above and requires each to be reported:
Table~\ref{tab:public_checkpoints_1b} gives the estimators and
Table~\ref{tab:public_checkpoints_windows} the window
constructions. On the five, the specifications that exclude zero
each add an assumption: the second non-collinear subset ($36\%$,
$[2\%, 54\%]$), windows of 5 sizes with stride 1 under the
overlap-aware interval ($26\%$, $[2\%, 47\%]$), and a slope
smooth in $\log N$ ($29\%$, $[5\%, 50\%]$ at the band's
centroid, 1.8B, with a one-sided 95\% floor of $9\%$). Pair
resampling is not admissible, and of the constructions with
disjoint windows two contain zero and the one with 4 windows
gives a degenerate interval. Under
the pre-registered estimator the data do not exclude the loss
of the whole effect below 1B (the attenuation's 95th
percentile is $112\%$ of it); under the smooth slope they
exclude a loss of more than $77\%$. On the ten the second
subset gives $40\%$ ($[17\%, 61\%]$), but the smooth slope
reverses ($-41\%$, $[-64\%, -10\%]$ at the band's centroid,
1.9B), and the second subset's own size interaction is $-25\%$
($[-43\%, -8\%]$).

\begin{table}[ht]
\centering
\small
\caption{The specifications Amendment~5 fixed for windows
centred at 1B and above ($10^{\beta_\Delta} - 1$, 512 seeds,
200{,}000 resamples; pair resampling 20{,}000). Pair resampling
treats pairs that share runs and a holdout as independent and
is not admissible as an interval. The smooth rows fit one slope
linear (or quadratic) in $\log N$ over all windows and read it
at the band's pair-weighted centroid or at a given $N$. The
attenuation bounds are the 95th percentiles of
$1 - \beta_{\ge 1\mathrm{B}} / \beta_{< 1\mathrm{B}}$, the share
of the effect below 1B that is lost; on the ten, whose effect
below 1B is near zero, the share is undefined.}
\label{tab:public_checkpoints_1b}
\begin{tabular}{@{}lrrrr@{}}
\toprule
 & \multicolumn{2}{c}{Five fully trained} & \multicolumn{2}{c}{All ten} \\
\cmidrule(lr){2-3}\cmidrule(l){4-5}
Specification & Estimate & 95\% interval & Estimate & 95\% interval \\
\midrule
Pre-registered (Table~\ref{tab:public_checkpoints_results}) & $28\%$ & $[-9\%, 41\%]$ & $23\%$ & $[-12\%, 36\%]$ \\
Moving-block interval, $L = 4$ & $28\%$ & $[-6\%, 41\%]$ & $23\%$ & $[-10\%, 36\%]$ \\
Clusters of cell, window and ray & $28\%$ & $[-18\%, 46\%]$ & $23\%$ & $[-20\%, 42\%]$ \\
Pairs resampled (anti-conservative) & $28\%$ & $[27\%, 30\%]$ & $23\%$ & $[22\%, 25\%]$ \\
Fixed effect per window & $32\%$ & $[-0.3\%, 39\%]$ & $27\%$ & $[-6\%, 35\%]$ \\
Second subset & $36\%$ & $[2\%, 54\%]$ & $40\%$ & $[17\%, 61\%]$ \\
Second subset, clusters of cell, window and ray & $36\%$ & $[-2\%, 57\%]$ & $40\%$ & $[6\%, 63\%]$ \\
\midrule
Smooth in $\log N$, at the band's centroid & $29\%$ & $[5\%, 50\%]$ & $-41\%$ & $[-64\%, -10\%]$ \\
Same, quadratic in $\log N$ & $36\%$ & $[5\%, 58\%]$ & $-26\%$ & $[-64\%, 43\%]$ \\
Smooth in $\log N$, at $N = 1$B & $34\%$ & $[16\%, 49\%]$ & $-32\%$ & $[-55\%, -5\%]$ \\
Smooth in $\log N$, at $N = 3.2$B & $25\%$ & $[-6\%, 55\%]$ & $-47\%$ & $[-70\%, -14\%]$ \\
Smooth in $\log N$, at $N = 10$B & $16\%$ & $[-26\%, 66\%]$ & $-59\%$ & $[-80\%, -23\%]$ \\
Second subset, smooth, at the centroid & $39\%$ & $[17\%, 63\%]$ & $-18\%$ & $[-44\%, 9\%]$ \\
\midrule
One-sided 95\% lower bound & \multicolumn{2}{c}{$-4\%$} & \multicolumn{2}{c}{$-8\%$} \\
Same, smooth in $\log N$ & \multicolumn{2}{c}{$9\%$} & \multicolumn{2}{c}{$-61\%$} \\
Attenuation bound, share of the effect below 1B & \multicolumn{2}{c}{$112\%$} & \multicolumn{2}{c}{-} \\
Same, smooth in $\log N$ & \multicolumn{2}{c}{$77\%$} & \multicolumn{2}{c}{-} \\
\bottomrule
\end{tabular}

\end{table}

\begin{table}[ht]
\centering
\small
\caption{Window constructions of Amendment~5 in windows centred
at 1B and above, five fully trained releases, 128 seeds (the
first row is the protocol's construction). The moving-block
interval resamples blocks of $L$ consecutive windows of a cell,
so windows that share runs move together; for disjoint windows
($L = 1$) it equals the window interval.}
\label{tab:public_checkpoints_windows}
\begin{tabular}{@{}rrrrll@{}}
\toprule
Width & Stride & Windows & Estimate & Window interval & Moving-block interval \\
\midrule
8 & 2 & 14 & $27\%$ & $[-9\%, 38\%]$ & $[-7\%, 39\%]$ ($L = 4$) \\
8 & 1 & 26 & $18\%$ & $[-4\%, 29\%]$ & $[-2\%, 32\%]$ ($L = 8$) \\
8 & 8 & 4 & $2\%$ & $[2\%, 47\%]$ & same \\
5 & 5 & 6 & $22\%$ & $[-23\%, 60\%]$ & same \\
5 & 1 & 31 & $26\%$ & $[8\%, 48\%]$ & $[2\%, 47\%]$ ($L = 5$) \\
4 & 4 & 8 & $52\%$ & $[-27\%, 90\%]$ & same \\
\bottomrule
\end{tabular}

\end{table}

The ratio
$(V_K / \tau_K)_{\mathrm{NC}} / (V_K / \tau_K)_{\mathrm{CO}}$,
that is $10^{\Delta}$, has median $5.8$ ($[5.4, 6.9]$) in
windows centred below 1B and $3.7$ ($[2.1, 5.6]$) in windows
centred at 1B and above, whose largest model is 12.6B
(intervals over windows, 300{,}000 resamples). At
fixed $\beta_{\mathrm{eff}}$ this ratio depends on the TPP
ratios alone, so it describes the public grids and does not
test a size effect. The non-collinear subsets these grids offer
stay better identified at large sizes, though less so at the
point estimate (the ratio of the two medians is $0.64$,
$[0.2, 1.3]$). The ratio is below 1 in three of the 13
windows at 1B and above. It uses the five fully trained
studies. The one window whose collinear subsets all have
$V_K = 0$, the OLMo ladder's (190M to 13.2B), is left out, and
Chinchilla's 16.2B run, its largest model, enters no window. A
conditioning test of scale would need the exponents fitted in
each window; we do not run one.

\section{Appendix - experimental setup}
\label{app:experimental_setup}

\subsection{Appendix - Transformer architecture}

All models use the HuggingFace Transformers~\citep{wolf2020transformers}
\texttt{LlamaForCausalLM} implementation of the LLaMA
architecture~\citep{touvron2023llama} with RoPE positional embeddings, RMSNorm, and
SwiGLU activations.

\subsection{Appendix - Architecture configurations}

\begin{table}[H]
\centering
\caption{Model architecture configurations. All models use the LLaMA architecture
with rotary position embeddings (RoPE), SwiGLU activation, RMSNorm, no bias terms,
a context length of 256, and vocabulary size of 100{,}277 (cl100k\_base).}
\label{tab:model-architectures}
\begin{tabular}{r r r r r r}
\toprule
\textbf{Params} & $\boldsymbol{n_{\text{layers}}}$ & $\boldsymbol{d_{\text{model}}}$ & $\boldsymbol{d_{\text{ff}}}$ & $\boldsymbol{n_{\text{heads}}}$ & $\boldsymbol{d_{\text{head}}}$ \\
\midrule
5{,}035{,}656   & 24 & 24  & 96   & 12 & 2   \\
10{,}070{,}160  & 12 & 48  & 192  & 1  & 48  \\
14{,}411{,}456  & 24 & 64  & 256  & 32 & 2   \\
19{,}400{,}928  & 1  & 96  & 384  & 1  & 96  \\
22{,}796{,}832  & 24 & 96  & 384  & 24 & 4   \\
28{,}819{,}840  & 12 & 128 & 512  & 4  & 32  \\
32{,}498{,}720  & 1  & 160 & 640  & 16 & 10  \\
35{,}368{,}160  & 8  & 160 & 640  & 4  & 40  \\
40{,}277{,}184  & 3  & 192 & 768  & 12 & 16  \\
47{,}949{,}888  & 16 & 192 & 768  & 1  & 192 \\
55{,}538{,}432  & 4  & 256 & 1024 & 32 & 8   \\
63{,}931{,}136  & 12 & 256 & 1024 & 32 & 8   \\
68{,}127{,}488  & 16 & 256 & 1024 & 32 & 8   \\
76{,}520{,}192  & 24 & 256 & 1024 & 1  & 256 \\
\bottomrule
\end{tabular}
\end{table}

Model sizes were specified as target parameter counts and mapped to
concrete architectures via nearest-neighbor matching on a grid of
transformer hyperparameters (layer count, hidden dimension, head count),
with $d_{\text{ff}} = 4 \times d_{\text{model}}$. Because the grid is
discrete, realized parameter counts differ slightly from targets;
the $\Nsizes$ configurations span
$N \in [5.04,\, 76.5]\,$M
($N_{\max}/N_{\min} \approx 15.2$).
Table~\ref{tab:model-architectures} reports the exact parameter counts
as verified from trained checkpoints. Following
\citet{hoffmann2022training}, $N$ includes all trainable parameters
(embedding and unembedding layers included). An earlier version of the architecture-search procedure produced
configurations with odd per-head dimension $d_{\text{head}} = d_{\text{model}}
/ n_{\text{heads}}$, which crashed at training time because rotary positional
encoding's \texttt{rotate\_half} operation requires $d_{\text{head}}$ to be
even. The search was patched to filter to head counts producing even
$d_{\text{head}}$, which is why $n_{\text{heads}}$ varies across architectures
in Table~\ref{tab:model-architectures}.

\noindent\textbf{Architecture heterogeneity, embeddings, and the role of $C \approx 6ND$.}
The grid spans $n_{\text{layers}} \in \{1, 3, 4, 8, 12, 16, 24\}$,
deliberately wider than a constant-shape sweep so that the fits are
not informed only by smooth depth-vs-width co-variation; the four
scaling laws under study are functions of $(N, D)$ only and are
agnostic to the internal architectural decomposition.  With the cl100k\_base vocabulary ($V = 100{,}277$), the input and
output embeddings $2 V d_{\text{model}}$ are $67$-$99\%$ of~$N$ in every
configuration of Table~\ref{tab:model-architectures} (e.g.\
$2 \cdot 100{,}277 \cdot 96 \approx 19.3$M of the single-layer $19.4$M
model), so the Kaplan FLOPs approximation $C \approx 6ND$, calibrated for
embedding-light transformers, is loose throughout.  Two consequences,
neither of which compromises the formal results:
\begin{itemize}
    \item \emph{Conditioning and CI inflation}
        (Proposition~\ref{prop:full_cond},
        Corollary~\ref{prop:ci_inflation},
        Theorem~\ref{thm:holdout_regimes}) are statements about the
        Jacobian geometry of the regression problem
        $\hat L(N, D; \boldsymbol{\theta})$ as a function of the
        empirical pairs $(N_i, D_i, L_i)$.  They do not assume that
        the underlying loss is a smooth power law, that depth is
        held fixed, or that any specific FLOPs identity holds; the
        Hoffmann/Chinchilla convention of taking $N$ as
        \emph{all} trainable parameters is preserved end-to-end.
        Corollary~\ref{prop:ci_inflation} and
        Theorem~\ref{thm:holdout_regimes} additionally assume a
        well-specified law with Gaussian noise (misspecification:
        Appendix~\ref{app:holdout_misspec}).
        Architectural heterogeneity and any resulting departures of
        the empirical $L_i$ from a clean power-law surface are
        absorbed by the misspecification extension
        (Eq.~\eqref{eqn:mse_ratio_misspec}), whose biases depend on
        the training design, so it preserves the CO/NC ordering
        only while the bias gap is smaller than the variance gap.
    \item \emph{IsoFLOP visualization} uses $C_i \coloneqq 6 N_i D_i$
        (Definition~\ref{def:induced_isoflop_plot}) purely as a
        relabelling: Proposition~\ref{prop:rmse_isoflop_invariance}
        and its proof in Appendix~\ref{app:proof_rmse_invariance}
        show that
        $\hat L(N_i, C_i/(6N_i); \boldsymbol{\theta})
        = \hat L(N_i, D_i; \boldsymbol{\theta})$ as an algebraic
        identity - any constant or model-dependent rescaling of the
        FLOPs/token coefficient leaves the holdout RMSE unchanged.
        Embedding dominance therefore shifts where shallow points
        sit on the $C$-axis but does not affect any reported metric.
\end{itemize}
The shallow rows survive in the grid because dropping them would
collapse the $(N, D)$ design at the small-$N$ end and weaken the
TPP-coverage contrast that the formal analysis targets; treating
them as standard $(N, D, L)$ observations under the
Hoffmann convention is consistent with how prior fits handle the
same regime. Figure~\ref{fig:nlayers_dist} shows the empirical
$n_{\text{layers}}$ distribution across all $1{,}933$ trained
models. Since $C$ enters only the isoFLOP visualization, the loose
$C \approx 6ND$ changes no reported metric (second item above).

\begin{figure}[H]
    \centering
    \includegraphics[width=0.5\linewidth]{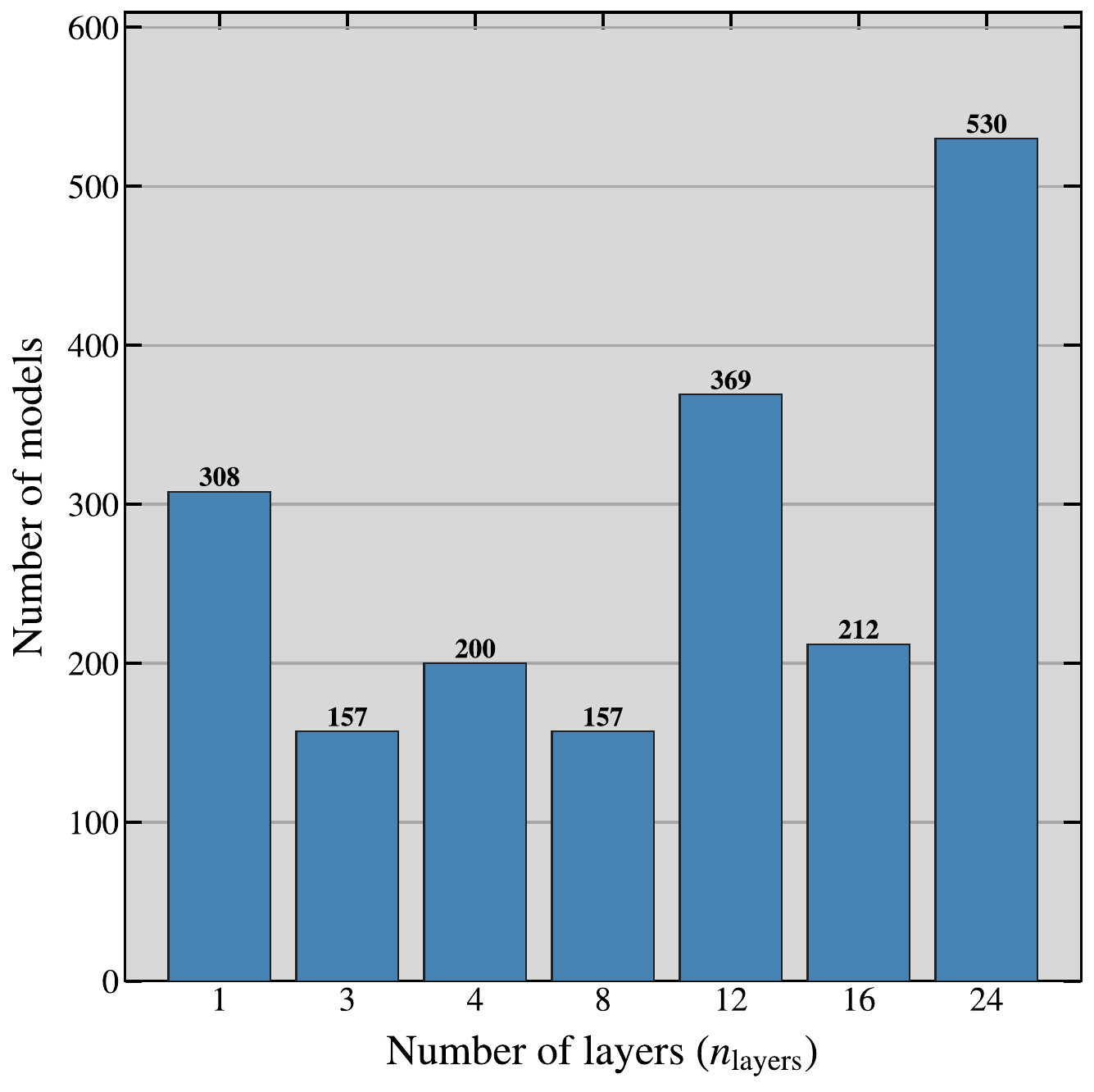}
    \caption{$n_{\text{layers}}$ distribution across all $1{,}933$ trained models.}
    \label{fig:nlayers_dist}
\end{figure}

\subsection{Appendix - Hyperparameter selection}
Table~\ref{tab:hyperparams} summarizes the training hyperparameters
shared across all experimental runs. 

\begin{table}[H]
\centering
\caption{Training hyperparameters shared across all runs.}
\label{tab:hyperparams}
\begin{tabular}{l l}
\toprule
\textbf{Hyperparameter} & \textbf{Value} \\
\midrule
Optimizer              & AdamW ($\beta_1{=}0.9$, $\beta_2{=}0.999$, $\epsilon{=}10^{-8}$) \\
Gradient clipping      & 1.0 \\
Weight decay           & 0.1 \\
Batch size             & 96 per GPU\\
Max sequence length    & 256 \\
Epochs                 & 3 \\
Warmup                 & Linear, first 3\% of steps (start factor 0.01) \\
LR schedule            & Cosine annealing to 10\% of peak LR \\
Seed & 42 \\
\bottomrule
\end{tabular}
\end{table}

\subsubsection{Learning rate}
We selected $\eta = 3 \times 10^{-4}$ for Wikipedia and
RedPajama, $\eta = 1 \times 10^{-3}$ for peS2o, Cosmopedia, and C4.

\subsection{Appendix - Loss extraction}\label{app:loss_extraction}

We record the converged training loss as the mean of the final 100 batch losses per epoch, to account for batch-to-batch noise.

\subsection{Appendix - Dataset selection}\label{app:dataset_selection}

We evaluate scaling law fits across five datasets spanning distinct
data domains: Cosmopedia~\citep{cosmopedia}
(synthetic educational text), Wikipedia~\citep{wikidump} (encyclopedic articles),
peS2o~\citep{pes2o} (academic papers and abstracts),
RedPajama-V2~\citep{redpajama} (filtered web text), and C4~\citep{raffel2020exploring} (web text).
Each is used under the license or terms its dataset card states (read 2026-11-03):
\begin{itemize}[nosep, leftmargin=*]
\item Cosmopedia (\texttt{HuggingFaceTB/cosmopedia}, \texttt{web\_samples\_v2}): Apache-2.0;
\item Wikipedia (\texttt{wikimedia/wikipedia}, \texttt{20231101.en}): CC BY-SA 3.0 and GFDL;
\item peS2o (\texttt{allenai/peS2o}, v2): ODC-By;
\item RedPajama-V2 (\texttt{togethercomputer/RedPajama-Data-V2}): no license of its own for the data, which falls under the Common Crawl Foundation Terms of Use;
\item C4 (\texttt{allenai/c4}, \texttt{en}): ODC-By, and the Common Crawl Terms of Use for its content.
\end{itemize}

\subsubsection{Train/evaluation splits}

Natural language corpora often exhibit systematic ordering
biases; Wikipedia sorts by page creation date, peS2o groups
abstracts before full papers, and RedPajama clusters by WARC
file origin. A naive sequential (left/right) split therefore
produces train and evaluation sets with different underlying
distributions. \emph{All five text corpora used in this paper
(C4, Wikipedia, peS2o, RedPajama, and Cosmopedia) are split
by the same procedure}: we randomly permute all document
indices with a fixed seed and take contiguous slices,
guaranteeing zero overlap while eliminating ordering
artifacts. No corpus uses a sequential / left-right /
chronological split, and no corpus is treated as a special
case. Table~\ref{tab:ngram_kl} confirms this approach.

\begin{table}[H]
\centering
\caption{Symmetric KL divergence and cosine similarity between train and validation $n$-gram distributions ($n=1,\ldots,3$). Both train and validation distributions have 600M tokens, drawn from the random-permutation split described above; the same split procedure is applied to every corpus, including Cosmopedia.}
\label{tab:ngram_kl}
\small
\begin{tabular}{llrrr}
\toprule
\textbf{Dataset} & \textbf{Metric} & $n=1$ & $n=2$ & $n=3$ \\
\midrule
\multirow{2}{*}{Wikipedia} & Sym.~KL & 0.0006 & 0.4446 & 3.3681 \\
 & Cosine & 0.999998 & 0.999952 & 0.999543 \\
\midrule
\multirow{2}{*}{RedPajama} & Sym.~KL & 0.0016 & 0.5904 & 4.1139 \\
 & Cosine & 0.999965 & 0.997919 & 0.947550 \\
\midrule
\multirow{2}{*}{Cosmopedia} & Sym.~KL & 0.0002 & 0.2852 & 2.1356 \\
 & Cosine & 1.000000 & 0.999994 & 0.999933 \\
\midrule
\multirow{2}{*}{peS2o} & Sym.~KL & 0.0007 & 0.4211 & 3.2188 \\
 & Cosine & 0.999998 & 0.999977 & 0.999717 \\
\midrule
\multirow{2}{*}{C4} & Sym.~KL & 0.0013 & 0.8733 & 5.2464 \\
 & Cosine & 0.999998 & 0.999953 & 0.998718 \\
\bottomrule
\end{tabular}
\end{table}

\subsubsection{Pre-processing}
Our training data was pre-tokenized with cl100k\_base (tiktoken~\citep{tiktoken}), truncated
to 256 tokens, or padded to 256 if an article was shorter. The training and
validation losses we report average the cross-entropy over all 256
positions, the padded ones included, as the released code does. A
model learns early to predict the padding token after an article
ends, so the padded positions add little loss, and each corpus pads a
roughly fixed share $p$ of its positions; the reported loss is then close
to $(1-p)$ times the loss on real tokens. That factor shrinks the
fitted $E$, $A$ and $B$ of a corpus by a common amount and leaves its
exponents nearly unchanged, and since every comparison in the paper
fits the collinear and non-collinear designs to the same corpus, it
cancels from the win rates and the conditioning comparisons, and the
theorems do not depend on it. Absolute losses are not real-token losses and
differ between corpora by their padding share, so they should not be
compared across corpora. Our repository's trainer masks padding by
default; its \texttt{--score-padding} flag reproduces our runs.

\subsection{Appendix - Training conditions and hardware}\label{hardware}
All pre-training used the NCSA DeltaAI GH200 partition (ACCESS
allocation CIS230318) with four GH200 GPUs per node via PyTorch
Distributed Data Parallel, totaling approximately 3,624 GPU-hours
(151 GPU-days) across five primary corpora, BF16, and large-TPP
settings. Scaling-law fitting, seed-variance analysis
(Tables~\ref{tab:summary_full}-\ref{tab:winrate}, Appendix~\ref{app:full_all_results}-\ref{app:coefficients}), and visualization used two AMD EPYC 9755 nodes, three CPUs in all: a dual-socket node
(256 cores, 512 threads, 1.5\,TB RAM) and a single-socket node (128 cores, 256 threads,
0.75\,TB RAM), with 230
parallel workers (${\sim}5$ wall-clock hours). The exhaustive
subset enumeration (Table~\ref{tab:regime_a_winrate}, 22
independent seeds) used 480 workers on the same hardware
(${\sim}24$ wall-clock hours, ${\sim}65$ min per seed). Syracuse University's OrangeGrid
high-throughput cluster, both its CPU and its GPU nodes, ran supporting
tasks around these stages. Total
storage footprint is approximately 2.2\,TB (1.5\,TB model
checkpoints and training metrics, 0.5\,TB pre-tokenized datasets,
0.2\,TB analysis artifacts and redundancy). Full-project compute
was higher than the final reported runs because we also executed
pilot and failed sweeps, which these totals do not include.
All training uses
worldsize=4 (PyTorch DDP); the LR schedule length depends on
worldsize through the DistributedSampler batch count, so
reproducing identical checkpoints requires exactly 4 GPUs.
Training is bitwise reproducible within a fixed CUDA driver
version (verified by repeated runs on NVIDIA driver 590.48).
Slight cross-driver numerical drift of $O(10^{-3})$ in the
training loss over three epochs has been observed following a
cluster driver update; we flag it as a minor reproducibility
issue. The 30-seed CIs in Tables~\ref{tab:summary_full}
and~\ref{tab:winrate} are computed over L-BFGS-B optimizer
restarts on \emph{fixed} training data
(Appendix~\ref{app:CI}), so they quantify curve-fitting
variance rather than absorbing a systematic shift in the
data-generating process. A common driver-induced shift in the
training losses, however, does not compromise our conclusions
for two reasons. \textbf{(i)}~All reported metrics are
\emph{paired} CO-vs-NC comparisons on the \emph{same} pre-trained
loss table per seed and dataset (Tables~\ref{tab:summary_full},
\ref{tab:winrate}; Appendix~\ref{app:CI}); a uniform shift
$\Delta L$ in the underlying losses adds the same offset to the
fitted irreducible loss $E$ in both CO and NC fits of the two additive
laws (Chinchilla, repeated-data), and is absorbed only approximately
by Droppo-Elibol, whose $L_\infty$ sits inside the outer power, and by
Kaplan, which has none; so the head-to-head difference
$\mathrm{RMSE}^{\mathrm{NC}}_\mathcal{H}
- \mathrm{RMSE}^{\mathrm{CO}}_\mathcal{H}$ and the
full-precision and BF16 win-rate summaries
(\winrate{} and \winratebf, respectively) are first-order
invariant to $\Delta L$ for the additive laws. \textbf{(ii)}~The drift magnitude
$O(10^{-3})$ is about two orders of magnitude smaller than the
CO-vs-NC RMSE gap of Table~\ref{tab:summary_full} ($0.237$ against
$0.156$), so a shift that propagated unevenly between designs
would still contaminate the win rate far less than the gap it estimates.

\subsection{Appendix - L-BFGS-B optimizer bounds}\label{app:optimizer_bounds}

All scaling laws are fit via multi-start L-BFGS-B with a differential-evolution global fallback. Parameter bounds are identical across all designs, datasets, and random seeds. For the Kaplan law, $N_c, D_c \in [10^3,\,10^{14}]$ and $\alpha_N, \alpha_D \in [0.01,\,2.0]$. For the Chinchilla law, $E \in [0,\,10]$, $A, B \in [10^{-2},\,10^{10}]$, and $\alpha, \beta \in [0.01,\,2.0]$. For the Droppo-Elibol law, $L_\infty \in [10^{-6},\,0.99\cdot\min(L)]$, $N_C, D_C \in [10^3,\,10^{14}]$, and $\alpha_N, \alpha_D, \alpha \in [0.01,\,2.0]$. For the repeated-data law, $A, B \in [10^{-2},\,10^{12}]$, $\alpha, \beta \in [0.01,\,2.0]$, $E \in [0,\,10]$, and the half-life parameters $R^*_D, R^*_N \in [\RhalfLow,\,\RhalfHigh]$; the repeated-data law is fit by Huber loss with $\delta = 0.5$, while the other three laws use squared-error loss.

\noindent\textbf{Scale pairs left at a start draw.} L-BFGS-B works in native units,
where the bounds of a scale coefficient span eleven or twelve orders of
magnitude. In $183$ of the $900$ Kaplan fits ($37$ CO, $146$ NC) and $510$ of the
$900$ Droppo-Elibol fits ($365$ CO, $145$ NC), and in $1$ Kaplan and $14$ Droppo-Elibol
fits of the $180$ BF16 fits of each law, the optimizer left the scale pair within a
relative $10^{-6}$ of one of its start draws (the counts are the same at $10^{-2}$) and
moved only the other parameters. Their mean $R^2_{\mathcal{H}}$ differs from the other
fits' by at most $0.07$ for both laws and designs: it is higher for Droppo-Elibol and lower
for Kaplan ($0.69$ against $0.75$ on CO). At seed~$0$ they include the marked rows of
Tables~\ref{tab:kaplan-params} and~\ref{tab:droppo-params} and the NC fit of
Figure~\ref{fig:select-plots}. Chinchilla's $(A, B)$ moves: $107$ of its $900$ fits stay
within $10^{-6}$ of a start draw, and only $3$ match one exactly.

\noindent\textbf{Active lower bound on $R^*_N$.} In the
repeated-data fit (Tables~\ref{tab:epoch-params-r2}
and~\ref{tab:epoch-params-bf16-r2}), the parameter-side
half-life $R^*_N$ frequently terminates at exactly
$\RhalfLow$, the lower bound declared above. This is
expected and not a fitting failure: our experimental grid
fixes $N$ per run and only $D$ is repeated (epochs $1,2,3$),
so the data carries strong signal about the data-side
half-life $R^*_D$ but is essentially uninformative about
$R^*_N$, whose objective contribution becomes flat at small
values. The L-BFGS-B optimizer therefore drifts to the
lower bound; the bound itself is the documented value
$\RhalfLow$ and applies uniformly across all designs,
datasets, and seeds. The reported $R^2$ on holdout is
unaffected because the loss surface is flat in $R^*_N$ at
that point. Reporting the boundary-hitting fits as-is
(rather than censoring or post-processing them) is what
allows the diagnostic to be visible in the tables.

\subsection{Appendix - Confidence intervals and variability bands}\label{app:CI}

Throughout the paper we characterize optimizer-induced variability by
re-fitting each scaling law $N_{\text{seeds}}=30$ times per (dataset, law,
epoch mode, design, coverage step) configuration, varying only the optimizer
seed that controls the L-BFGS-B random-restart sequence and the
differential-evolution fallback. The training data and holdout splits are held
fixed across seeds, so the resulting $R^2$ distribution isolates
optimizer-induced variance from data-induced variance. 

\noindent\textbf{The importance of seed-to-seed variance.} This choice probes a different variance source from Corollary~\ref{prop:ci_inflation},
which concerns sampling noise in the losses, while the seeds vary only the optimizer
with the data fixed. Both stem from the same flat direction: the
corollary bounds parameter CIs below by a factor of order
$\varepsilon^{-1}$, and a flat direction also lets optimizer restarts land
far apart; that parameter-space spread propagates to
prediction uncertainty on holdout, manifesting as wide seed-to-seed spread
in holdout~$R^2$. A single CO seed can fit a particular holdout well by
luck; the Gauss-Newton step along the sloppy $(A,B)$ direction happens to
land near the truth on that draw. What the formal analysis predicts,
and the seed-to-seed CI measures, is that this spread is systematically
larger for CO than for NC. The CI gap between designs in
Tables~\ref{tab:summary_full},~\ref{tab:summary_bf16} is the predicted
parameter-CI inflation made empirical, observed one layer downstream in
prediction error rather than in $A$ directly.

\noindent\textbf{Coverage-fraction plots.} The central marker is the \emph{mean}
$R^2$ across the $30$ seeds; error bars span the $5$th-$95$th percentile of
that distribution. These are empirical percentile bands of the per-seed
$R^2$ values, not Gaussian confidence intervals on the mean.

\noindent\textbf{Summary tables (e.g.\ Tables~\ref{tab:summary_full}, \ref{tab:summary_bf16}).} Reported $R^2$ \emph{Mean} is the grand mean of the
per-combo seed means; reported $R^2$ \emph{Median} is the median of those
combo means; reported $R^2$ \emph{Std} is their across-combo standard
deviation. The $95\%$ CI is an optimizer-noise interval centered on $R^2$
Mean with half-width
\[
  \Delta = 1.96 \cdot \frac{\overline{\sigma_{\text{seed}}}}{\sqrt{N_{\text{seeds}}}},
\]
where $\overline{\sigma_{\text{seed}}}$ is the average within-combo
seed-to-seed standard deviation. RMSE columns follow the same convention. The
CI quantifies how tightly seed-to-seed optimizer noise pins down a typical
combo's mean and is independent of the Median column.

\noindent\textbf{Subset analysis (Table~\ref{tab:regime_a_winrate}).} For the
budget-matched subset enumeration we use $22$ seeds rather than
$30$ due to compute restrictions; each varies the bounding-box
construction and the optimizer restarts. The \textbf{Dataset Avg.}
intervals are $95\%$ bootstrap intervals over the seeds ($5{,}000$
resamples of the per-seed mean across corpora). The preliminary
high-TPP win rate (Table~\ref{tab:winrate_bigtpp_prelim}) carries a Wilson
$95\%$ binomial CI.

\section{Appendix - detailed results}
\label{app:per_dataset}

\subsection{Appendix - Collinear and non-collinear experimental grid definitions}\label{app:defining_grid}

Figures~\ref{grid:ncgrid} and~\ref{grid:tpp} visualize the non-collinear and collinear experimental designs as grids. Red cells are holdout models; blue cells are training models. Each dataset comprises $\runsperDset$ runs ($\COtrainN$ collinear training $+$ $\COholdN$ collinear holdout $+$ $\NCtrainN$ non-collinear training $+$ $\NCholdN$ non-collinear holdout). Four of the five corpora completed all runs; C4 has $321$ ($3$ non-collinear holdout jobs crashed).

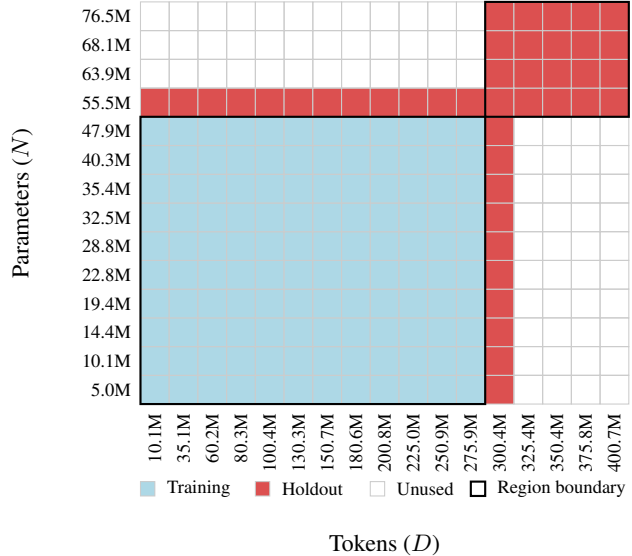
\begin{figure}[H]
\centering
\begin{tikzpicture}[scale=0.38]

  \definecolor{trainblue}{RGB}{173,216,230}
  \definecolor{holdoutred}{RGB}{220,80,80}

  \foreach \y in {0,...,13}{
    \foreach \x in {0,...,16}{
      \pgfmathtruncatemacro{\isholdout}{%
        (\x == 12) || (\y == 10) || (\x > 12 && \y > 10) ? 1 : 0}
      \pgfmathtruncatemacro{\istrain}{%
        (\x < 12) && (\y < 10) ? 1 : 0}
      \ifnum\isholdout=1
        \fill[holdoutred] (\x,\y) rectangle ++(1,1);
      \else
        \ifnum\istrain=1
          \fill[trainblue] (\x,\y) rectangle ++(1,1);
        \fi
      \fi
      \draw[gray!40, thin] (\x,\y) rectangle ++(1,1);
    }
  }

  \draw[black, thick] (0,0) rectangle (12,10);

  \draw[black, thick] (12,10) rectangle (17,14);

  \foreach \label [count=\x from 0] in {
    10.1,35.1,60.2,80.3,100.4,130.3,150.7,180.6,200.8,225.0,250.9,275.9,300.4,325.4,350.4,375.8,400.7%
  }{
    \node[below, rotate=90, anchor=east, font=\scriptsize] at (\x+0.5, 0) {\label M};
  }

  \foreach \label [count=\y from 0] in {
    5.0,10.1,14.4,19.4,22.8,28.8,32.5,35.4,40.3,47.9,55.5,63.9,68.1,76.5%
  }{
    \node[left, font=\scriptsize] at (0, \y+0.5) {\label M};
  }

  \node[below=1.6cm, font=\small] at (8.5, 0) {Tokens ($D$)};
  \node[left=1.3cm, rotate=90, anchor=south, font=\small] at (0, 7) {Parameters ($N$)};

  \fill[trainblue] (0,-3.2) rectangle ++(0.5,0.5);
  \draw[gray!40, thin] (0,-3.2) rectangle ++(0.5,0.5);
  \node[right, font=\scriptsize] at (0.6,-2.95) {Training};

  \fill[holdoutred] (4,-3.2) rectangle ++(0.5,0.5);
  \draw[gray!40, thin] (4,-3.2) rectangle ++(0.5,0.5);
  \node[right, font=\scriptsize] at (4.6,-2.95) {Holdout};

  \fill[white] (8,-3.2) rectangle ++(0.5,0.5);
  \draw[gray!40, thin] (8,-3.2) rectangle ++(0.5,0.5);
  \node[right, font=\scriptsize] at (8.6,-2.95) {Unused};

  \draw[black, thick] (11.5,-3.2) rectangle ++(0.5,0.5);
  \node[right, font=\scriptsize] at (12.1,-2.95) {Region boundary};

\end{tikzpicture}
\caption{Grid of the non-collinear (NC) experimental design. Each cell represents a unique $(N, D)$ training configuration. The holdout set (red) forms an L-shaped boundary at the next parameter count above the training region ($N = 55.5$M) and at $D = 300.4$M tokens, plus the $3 \times 4$ extrapolation block at the largest token and parameter counts. White cells denote configurations not included in either set. The bordered lower-left region contains the $\NCtrainN$ in-distribution training configurations.}
\label{grid:ncgrid}
\end{figure}

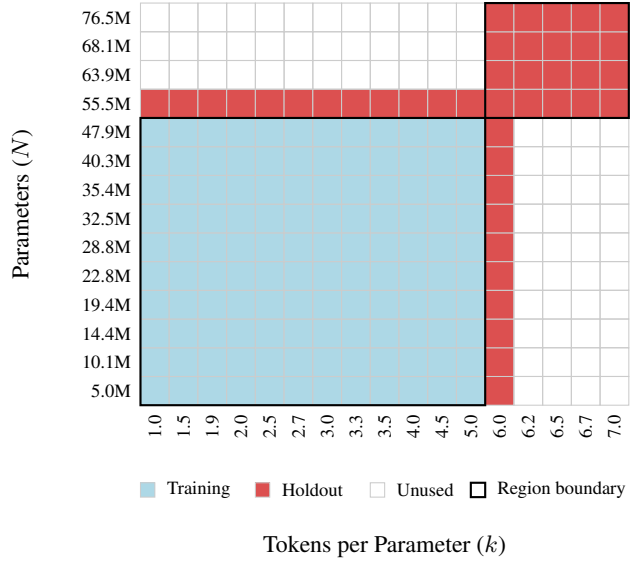
\begin{figure}[H]
\centering
\begin{tikzpicture}[scale=0.38]

  \definecolor{trainblue}{RGB}{173,216,230}
  \definecolor{holdoutred}{RGB}{220,80,80}

  \foreach \y in {0,...,13}{
    \foreach \x in {0,...,16}{
      \pgfmathtruncatemacro{\isholdout}{%
        (\x == 12) || (\y == 10) || (\x > 12 && \y > 10) ? 1 : 0}
      \pgfmathtruncatemacro{\istrain}{%
        (\x < 12) && (\y < 10) ? 1 : 0}
      \ifnum\isholdout=1
        \fill[holdoutred] (\x,\y) rectangle ++(1,1);
      \else
        \ifnum\istrain=1
          \fill[trainblue] (\x,\y) rectangle ++(1,1);
        \fi
      \fi
      \draw[gray!40, thin] (\x,\y) rectangle ++(1,1);
    }
  }

  \draw[black, thick] (0,0) rectangle (12,10);

  \draw[black, thick] (12,10) rectangle (17,14);

  \foreach \label [count=\x from 0] in {
    1.0,1.5,1.9,2.0,2.5,2.7,3.0,3.3,3.5,4.0,4.5,5.0,6.0,6.2,6.5,6.7,7.0%
  }{
    \node[below, rotate=90, anchor=east, font=\scriptsize] at (\x+0.5, 0) {\label};
  }

  \foreach \label [count=\y from 0] in {
    5.0,10.1,14.4,19.4,22.8,28.8,32.5,35.4,40.3,47.9,55.5,63.9,68.1,76.5%
  }{
    \node[left, font=\scriptsize] at (0, \y+0.5) {\label M};
  }

  \node[below=1.6cm, font=\small] at (8.5, 0) {Tokens per Parameter ($k$)};
  \node[left=1.3cm, rotate=90, anchor=south, font=\small] at (0, 7) {Parameters ($N$)};

  \fill[trainblue] (0,-3.2) rectangle ++(0.5,0.5);
  \draw[gray!40, thin] (0,-3.2) rectangle ++(0.5,0.5);
  \node[right, font=\scriptsize] at (0.6,-2.95) {Training};

  \fill[holdoutred] (4,-3.2) rectangle ++(0.5,0.5);
  \draw[gray!40, thin] (4,-3.2) rectangle ++(0.5,0.5);
  \node[right, font=\scriptsize] at (4.6,-2.95) {Holdout};

  \fill[white] (8,-3.2) rectangle ++(0.5,0.5);
  \draw[gray!40, thin] (8,-3.2) rectangle ++(0.5,0.5);
  \node[right, font=\scriptsize] at (8.6,-2.95) {Unused};

  \draw[black, thick] (11.5,-3.2) rectangle ++(0.5,0.5);
  \node[right, font=\scriptsize] at (12.1,-2.95) {Region boundary};

\end{tikzpicture}
\caption{Tokens-per-parameter (TPP) grid of the collinear (CO) experimental design. Each cell represents a unique $(N, k)$ configuration where the token count is $D = k \cdot N$. The holdout set (red) forms an L-shaped boundary at the next parameter count above the training region ($N = 55.5$M) and at $k = 6$, plus the $3 \times 4$ extrapolation block. White cells denote configurations not included in either set. The bordered lower-left region contains the $\COtrainN$ in-distribution training configurations.}
\label{grid:tpp}
\end{figure}

\subsection{Appendix - Table~\ref{tab:summary_full} holdout breakdown}\label{app:full_metric_breakdown}

Tables~\ref{tab:summary} and~\ref{tab:winrate_appendix} give the
main-experiment results by holdout split, the win rates also by corpus.
\begin{table}[H]
\small
\centering
\setlength{\tabcolsep}{3pt}
\caption{$R^2$ and RMSE summary by training design over $30$ optimizer seeds. 95\% CI: optimizer-noise interval (Appendix~\ref{app:CI}); Std: across (corpus, law, epoch) combinations; best per split in \textbf{bold}.}
\label{tab:summary}
\begin{tabular}{ll cccc cccc}
\toprule
& & \multicolumn{4}{c}{\textbf{$R^2$}} & \multicolumn{4}{c}{\textbf{RMSE}} \\
\cmidrule(lr){3-6} \cmidrule(lr){7-10}
\textbf{Split} & \textbf{Design} & Mean & 95\% CI & Median & Std & Mean & 95\% CI & Median & Std \\
\midrule
\multirow{2}{*}{Train ($\mathcal{D}$)}
 & CO  & \win{0.9848}  & \win{[0.985, 0.985]} & \win{0.9850} & \win{0.0047} & \win{0.1737} & \win{[0.173, 0.174]} & \win{0.1616} & \win{0.0430} \\
 & NC  & 0.9526  & [0.952, 0.953] & 0.9542 & 0.0172 & 0.2023 & [0.202, 0.203] & 0.1869 & 0.0443 \\
\midrule
\multirow{2}{*}{Holdout ($\mathcal{H}$)}
 & CO  & 0.8370  & [0.832, 0.842] & 0.8811 & 0.1129 & 0.2373 & [0.233, 0.241] & 0.2163 & 0.1229 \\
 & NC  & \win{0.9319}  & \win{[0.930, 0.934]} & \win{0.9392} & \win{0.0269} & \win{0.1561} & \win{[0.154, 0.158]} & \win{0.1242} & \win{0.0599} \\
\midrule
\multirow{2}{*}{Holdout ($\mathcal{H}_{\mathrm{col}}$)}
 & CO  & 0.9113  & [0.907, 0.915] & 0.9336 & 0.0514 & 0.1885 & [0.185, 0.192] & 0.1645 & 0.0903 \\
 & NC  & \win{0.9437}  & \win{[0.942, 0.945]} & \win{0.9439} & \win{0.0182} & \win{0.1560} & \win{[0.154, 0.158]} & \win{0.1182} & \win{0.0631} \\
\midrule
\multirow{2}{*}{Holdout ($\mathcal{H}_{\mathrm{nc}}$)}
 & CO  & 0.6984  & [0.688, 0.708] & 0.7490 & 0.2730 & 0.2757 & [0.271, 0.281] & 0.2445 & 0.1516 \\
 & NC  & \win{0.9092}  & \win{[0.907, 0.911]} & \win{0.9297} & \win{0.0492} & \win{0.1560} & \win{[0.154, 0.158]} & \win{0.1275} & \win{0.0569} \\
\bottomrule
\end{tabular}
\end{table}
\begin{table}[H]
\centering
\caption{Win rate (NC vs.\ CO, seed-paired) by holdout split and corpus.}
\label{tab:winrate_appendix}
\begin{tabular}{lcccc}
\toprule
\textbf{Holdout Split} & \textbf{NC Wins} & \textbf{CO Wins} & \textbf{Total} & \textbf{NC \%} \\
\midrule
$\mathcal{H}$ & \winrateNumer & \winrateLosses & \winrateDenom & \winrate \\
$\mathcal{H}_{\mathrm{col}}$ & 1082 & 418 & \winrateDenom & 72.1\% \\
$\mathcal{H}_{\mathrm{nc}}$ & 1480 & 20 & \winrateDenom & 98.7\% \\
\midrule
  $\mathcal{H}$ / C4 & 279 & 21 & 300 & 93.0\% \\
  $\mathcal{H}$ / Cosmopedia & 284 & 16 & 300 & 94.7\% \\
  $\mathcal{H}$ / peS2o & 298 & 2 & 300 & 99.3\% \\
  $\mathcal{H}$ / RedPajama & 300 & 0 & 300 & 100.0\% \\
  $\mathcal{H}$ / Wikipedia & 299 & 1 & 300 & 99.7\% \\
  $\mathcal{H}_{\mathrm{col}}$ / C4 & 252 & 48 & 300 & 84.0\% \\
  $\mathcal{H}_{\mathrm{col}}$ / Cosmopedia & 239 & 61 & 300 & 79.7\% \\
  $\mathcal{H}_{\mathrm{col}}$ / peS2o & 295 & 5 & 300 & 98.3\% \\
  $\mathcal{H}_{\mathrm{col}}$ / RedPajama & 151 & 149 & 300 & 50.3\% \\
  $\mathcal{H}_{\mathrm{col}}$ / Wikipedia & 145 & 155 & 300 & 48.3\% \\
  $\mathcal{H}_{\mathrm{nc}}$ / C4 & 292 & 8 & 300 & 97.3\% \\
  $\mathcal{H}_{\mathrm{nc}}$ / Cosmopedia & 289 & 11 & 300 & 96.3\% \\
  $\mathcal{H}_{\mathrm{nc}}$ / peS2o & 299 & 1 & 300 & 99.7\% \\
  $\mathcal{H}_{\mathrm{nc}}$ / RedPajama & 300 & 0 & 300 & 100.0\% \\
  $\mathcal{H}_{\mathrm{nc}}$ / Wikipedia & 300 & 0 & 300 & 100.0\% \\
\bottomrule
\end{tabular}
\end{table}


\subsection{Appendix - Experimental setup for the budget-matched subset comparison}\label{app:regime_a_setup}

We formalize the empirical protocol that asks, at matched
training cardinality, whether NC also beats CO designs with
$K \geq 2$ rays, for which Theorem~\ref{thm:holdout_regimes} gives no
ordering (its Regime~B; Lemma~\ref{lem:regime_a_margin} and the text
after it). The protocol is instantiated independently for
each scaling law $\mathcal{L} \in \{\text{Chinchilla},\, \text{Kaplan},\,
\text{Droppo-Elibol}\}$ (Definition~\ref{def:exponent_gap}), each
training corpus, and each training checkpoint mode
$\mu \in \{\text{first},\, \text{second},\, \text{final}\}$. Throughout this
section we reuse the symbols $(N, D, L, \mathcal{D}, K, k_1, \ldots, k_K,
\mathcal{H}, \boldsymbol{\theta})$ from
Definition~\ref{def:experimental_dataset} without redefinition; the
quantities introduced below are the new objects specific to subset
enumeration.

\begin{definition}[Available TPP ratios]
\label{def:available_tpp}
For each corpus and mode $\mu$, let
$\mathbf{k}^{\text{all}} = (k_1 < \cdots < k_{12})$ denote the ordered
set of $12$ TPP ratios at which the collinear training pool
$\mathcal{D}_{\mathrm{train}}^{\mathrm{CO}}$
(Definition~\ref{def:experimental_dataset}) was collected (so $K = 12$
for the full pool). Each ratio $k_\ell$ indexes a ray
$\{(N, k_\ell N) : N \in \mathcal{N}\}$ in the $(N, D)$ plane, where
$\mathcal{N} = \{N_1, \ldots, N_n\}$ is the shared set of model sizes
of Definition~\ref{def:experimental_dataset}.
\end{definition}

\begin{definition}[Collinear (CO) subset design]
\label{def:co_subset}
Given any non-empty index set $S \subseteq \{1, \ldots, \COTPPs\}$, the
induced \emph{collinear subset design} is the union of training runs on
the selected rays:
\begin{equation}
    \mathrm{CO}(S)
    \;\coloneqq\;
    \bigl\{(N_i, D_i, L_i) \in \mathcal{D}_{\mathrm{train}}^{\mathrm{CO}}
      \;:\; D_i / N_i = k_\ell \text{ for some } \ell \in S \bigr\},
\end{equation}
with cardinality $n_S \coloneqq |\mathrm{CO}(S)|$ and effective ray
count $K_S \coloneqq |S|$ (the design's $K$ in the sense of
Definition~\ref{def:experimental_dataset}). Enumeration over all
non-empty $S$ yields $2^{12} - 1 = 4{,}095$ CO subset designs per
(corpus, mode) pair.
\end{definition}

\begin{definition}[Bounding-box non-collinear (NC) design]
\label{def:nc_bbox}
Let $\mathcal{D}_{\mathrm{train}}^{\mathrm{NC}}$ be the non-collinear
training pool covering the rectangular grid
$\mathcal{N}_{\mathrm{NC}} \times \mathcal{D}_{\mathrm{NC}}$, with
$M_N \coloneqq |\mathcal{N}_{\mathrm{NC}}|$ rows and
$M_D \coloneqq |\mathcal{D}_{\mathrm{NC}}|$ columns. Index the grid
cells $(i, j) \in \{0, \ldots, M_N - 1\} \times \{0, \ldots, M_D - 1\}$.
Given a target cardinality $n^\star \in \mathbb{N}$ and target TPP
ratios $\mathbf{k}_S = \{k_\ell : \ell \in S\}$ from a CO subset $S$,
the \emph{bounding-box NC design}
$\mathrm{NC}_{\Box}(n^\star, \mathbf{k}_S)$ is constructed by the
following deterministic algorithm
(seeded by $\mathrm{hash}(\text{corpus}, \mu, S)$ and the global seed;
Appendix~\ref{app:seed_protocol}):

\begin{enumerate}
    \item \textbf{Center.} Initialize a $2 \times 2$ bounding box
        $[i_{\mathrm{lo}}, i_{\mathrm{hi}}] \times [j_{\mathrm{lo}},
        j_{\mathrm{hi}}]$ centered on $(\lfloor M_N/2 \rfloor,
        \lfloor M_D/2 \rfloor)$ (with a random $\pm 1$ offset when
        $M_N$ or $M_D$ is even).
    \item \textbf{Initial absorption.} Include all training runs
        whose $(N, D)$ falls inside the initial bounding box. If
        $|\mathrm{NC}_{\Box}| > n^\star$, subsample to exactly
        $n^\star$ using the priority rule below.
    \item \textbf{Adaptive growth.} While
        $|\mathrm{NC}_{\Box}| < n^\star$ and the box has not spanned
        the full grid:
        \begin{enumerate}
            \item Enumerate the two candidate expansions: expand the
                $N$-dimension by one row on both edges, or expand the
                $D$-dimension by one column on both edges.
            \item Score each candidate by $(s_1, s_2, \xi)$, where
                $s_1$ counts newly-covered CO-target TPP bins, $s_2$
                counts newly-covered non-target TPP bins (induced by
                the $(N, D)$ combinations on the boundary), and
                $\xi \sim \mathrm{Uniform}[0, 1]$.
            \item Select the candidate maximizing $(s_1, s_2, \xi)$
                lexicographically. If the ring fits within
                $n^\star - |\mathrm{NC}_{\Box}|$ slots, absorb it
                entirely; otherwise absorb a priority subsample and
                stop.
        \end{enumerate}
    \item \textbf{Priority subsample rule.} When a ring must be
        trimmed to $r$ training points, partition its runs by TPP bin
        into three priority buckets: (a) runs at a target TPP not yet
        covered, (b) runs at a non-target TPP not yet covered, (c)
        runs at an already-covered TPP. Shuffle each bucket with the
        seeded RNG and take the first $r$ from their concatenation.
\end{enumerate}

By construction, $|\mathrm{NC}_{\Box}(n^\star, \mathbf{k}_S)| =
n^\star$ whenever the NC grid has enough runs, and
$\mathrm{NC}_{\Box}$ is a connected rectangular region of
$\mathcal{N}_{\mathrm{NC}} \times \mathcal{D}_{\mathrm{NC}}$.
\end{definition}

\begin{definition}[Paired comparison]
\label{def:paired_comparison}
For any scaling law $\mathcal{L}$, corpus $\mathcal{D}$, mode $\mu$, and
CO subset $S$, let $n^\star = n_S$ (Definition~\ref{def:co_subset}).
Define:
\begin{align*}
    \mathrm{CO}(S) &\text{ - training set from the selected rays.} \\
    \mathrm{NC}_{\Box}(n^\star, \mathbf{k}_S)
        &\text{ - budget-matched bounding-box NC training set.} \\
    \hat{\boldsymbol{\theta}}^{\mathrm{CO}}, \hat{\boldsymbol{\theta}}^{\mathrm{NC}}
        &\text{ - least-squares fits of } \mathcal{L} \text{ on each design.}
\end{align*}
Both fits use the same optimization protocol and evaluate on the same
fixed $70\%$ subsample of the unified holdout $\mathcal{H}$
(Definition~\ref{def:experimental_dataset}; Appendix~\ref{app:seed_protocol}).
The \emph{paired comparison} is the tuple
\begin{equation}
    \mathcal{C}(\mathcal{L}, \mathcal{D}, \mu, S)
    \;\coloneqq\;
    \bigl(
        \mathrm{RMSE}_{\mathcal{H}}^{\mathrm{CO}},\;
        \mathrm{RMSE}_{\mathcal{H}}^{\mathrm{NC}},\;
        \hat{\boldsymbol{\theta}}^{\mathrm{CO}},\;
        \hat{\boldsymbol{\theta}}^{\mathrm{NC}}
    \bigr),
\end{equation}
with $\mathrm{RMSE}_{\mathcal{H}}^{\mathrm{CO}}$ and
$\mathrm{RMSE}_{\mathcal{H}}^{\mathrm{NC}}$ as in
Theorem~\ref{thm:holdout_regimes}. The set of all paired comparisons is
\begin{equation}
    \begin{aligned}
        \mathcal{P}
        \;=\;
        \bigl\{
            \mathcal{C}(\mathcal{L}, \mathcal{D}, \mu, S)
            \;:\;&\;
            \mathcal{L} \in \{\text{Ch.},\, \text{Ka.},\, \text{DE}\}, \\
            &\mathcal{D} \in \{\text{C4}, \ldots, \text{Wiki}\}, \\
            &\mu \in \{\text{first, second, final}\}, \\
            &S \subseteq \{1, \ldots, 12\}
        \bigr\},
    \end{aligned}
\end{equation}
with $|\mathcal{P}| = 3 \cdot 5 \cdot 3 \cdot 4095 = 184{,}275$.
\end{definition}

\noindent\textbf{Note on the budget match.}
$\mathrm{NC}_{\Box}(n^\star, \mathbf{k}_S)$ contains exactly $n^\star$
training points \emph{by construction}, matching $\mathrm{CO}(S)$
precisely. Any observed RMSE advantage of $\mathrm{NC}_{\Box}$ over
$\mathrm{CO}$ therefore cannot arise from a training-data size
difference.

\begin{definition}[NC win rate and below-threshold NC win rate]
\label{def:win_rate}
For any collection $\mathcal{P}' \subseteq \mathcal{P}$, the
\emph{NC win rate} on $\mathcal{P}'$ is
\begin{equation}
    \mathrm{WR}(\mathcal{P}')
    \;\coloneqq\;
    \frac{1}{|\mathcal{P}'|}
    \sum_{\mathcal{C} \in \mathcal{P}'}
        \mathbf{1}\!\left\{
            \mathrm{RMSE}_{\mathcal{H}}^{\mathrm{NC}}(\mathcal{C})
            <
            \mathrm{RMSE}_{\mathcal{H}}^{\mathrm{CO}}(\mathcal{C})
        \right\},
    \label{eqn:winrate}
\end{equation}
where $\mathbf{1}\{\cdot\}$ is the standard indicator (1 if the
inequality holds, 0 otherwise) and ties are broken in favor of CO
(negligible at our sample sizes).

Given a scaling law $\mathcal{L}$, mode $\mu$, and \textbf{target
condition number} $\kappa^\star > 0$ (in the sense of
$\kappa_{\mathrm{target}}$ from Proposition~\ref{prop:holdout_r2}), the
\emph{predicted below-threshold set at $\kappa^\star$} is
\begin{equation}
    \mathcal{P}^{A}_{\mathcal{L}, \mu}(\kappa^\star)
    \;\coloneqq\;
    \bigl\{
        \mathcal{C}(\mathcal{L}, \mathcal{D}, \mu, S) \in \mathcal{P}
        \;:\;
        V_{K_S}(S; \beta_{\mathrm{eff}}^{\mathcal{L}})
        \;<\;
        \tau_{K_S}(S; \beta_{\mathrm{eff}}^{\mathcal{L}}, \kappa^\star)
    \bigr\},
    \label{eqn:regime_a_set}
\end{equation}
where $V_{K}(\cdot;\beta)$, $\tau_{K}(\cdot;\beta,\kappa^\star)$, and
$\beta_{\mathrm{eff}}^{\mathcal{L}}$ are exactly the quantities of
Proposition~\ref{prop:holdout_r2} and Definition~\ref{def:exponent_gap},
evaluated on the rays selected by $S$ (so $K_S = |S|$). We use the
\textbf{literature value} of $\beta_{\mathrm{eff}}^{\mathcal{L}}$
(Table~\ref{tab:law_substitutions}; $\beta = 0.28$ for Chinchilla and
$\gamma_D = 0.08$ for Droppo-Elibol) so that regime classification can
be performed \emph{a priori} without fitting. For Chinchilla and Droppo-Elibol the classification is the
leading-order one, which at our sizes can misjudge the exact
$\kappa_{A,B}$ several-fold, and for Droppo-Elibol it is a
heuristic (Appendix~\ref{app:proof_prop_holdout_r2}). For Kaplan
we use the exact rule instead: $S$ is below the threshold iff
$\kappa_{A,B}$ of the Gram block of the columns
$\partial\hat{L}/\partial N_c$ and $\partial\hat{L}/\partial
D_c$ of the fitted law, on the runs of $\mathrm{CO}(S)$, at the
parameters of \citet{kaplan2020scaling} ($N_c = 8.8 \times
10^{13}$, $D_c = 5.4 \times 10^{13}$, $\alpha_N = 0.076$,
$\alpha_D = 0.095$, from their single-variable fits), exceeds $\kappa^\star$. Droppo-Elibol keeps
the leading-order rule because its literature $\gamma_D = 0.08$
comes without the other parameters its $\kappa_{A,B}$ needs. The
\emph{below-threshold NC win rate} at target $\kappa^\star$ is
\begin{equation}
    \mathrm{WR}_A(\mathcal{L}, \mu; \kappa^\star)
    \;\coloneqq\;
    \mathrm{WR}\!\bigl(
        \mathcal{P}^{A}_{\mathcal{L}, \mu}(\kappa^\star)
    \bigr).
\end{equation}
\end{definition}

\subsection{Appendix - BF16 results}\label{app:bf16_results}

Tables~\ref{tab:summary_bf16_app},~\ref{tab:winrate_bf16} and~\ref{tab:winrate_appendix_bf16} report
summary statistics (Table~\ref{tab:summary_bf16} split by holdout) and pairwise win rates for the BF16
experiments described in
Section~\ref{sec:designs}. These runs use the same $\Nsizes$ model sizes $N$, the same $\NCDsizes$ NC dataset sizes $D$, and the same $\COTPPs$ collinear training TPP ratios $\mathcal{K}_{\mathrm{train}}$ as the FP32 main experiment, but are restricted to a single corpus (Wikipedia) due to compute. They cast the whole model to BF16 (\texttt{model.bfloat16()}): weights, gradients and AdamW states are BF16, and only the logits are cast to FP32 for the loss. Of their $324$ runs, $302$ completed: $116$ of the $120$ in each training design, $28$ of the $42$ collinear holdout runs and all $42$ non-collinear holdout runs. They are not to be confused with the preliminary high-TPP (BigTPP) experiment of Appendix~\ref{app:bigtpp_prelim}, which uses the disjoint set $\mathcal{K}_{\mathrm{big}} = \{10, 11, 12, 13, 14, 15\}$ in FP32.

\begin{table}[H]
\small
\centering
\setlength{\tabcolsep}{3pt}
\caption{BF16 runs on Wikipedia: $R^2$ and RMSE by training design (CO: collinear, NC: non-collinear) over $30$ optimizer seeds, on the training set $\mathcal{D}$, the unified holdout $\mathcal{H}$ and its collinear and non-collinear parts $\mathcal{H}_{\mathrm{col}}$ and $\mathcal{H}_{\mathrm{nc}}$ (Table~\ref{tab:summary_bf16} with the holdout split). 95\% CI: optimizer-noise interval (Appendix~\ref{app:CI}); Std: standard deviation across the $10$ (law, epoch) combinations; the design with the higher mean $R^2$ per split is highlighted.}
\label{tab:summary_bf16_app}
\begin{tabular}{ll cccc cccc}
\toprule
& & \multicolumn{4}{c}{\textbf{$R^2$}} & \multicolumn{4}{c}{\textbf{RMSE}} \\
\cmidrule(lr){3-6} \cmidrule(lr){7-10}
\textbf{Split} & \textbf{Design} & Mean & 95\% CI & Median & Std & Mean & 95\% CI & Median & Std \\
\midrule
\multirow{2}{*}{Train ($\mathcal{D}$)}
 & CO  & 0.9896  & [0.990, 0.990] & 0.9916 & 0.0029 & 0.1854 & [0.185, 0.185] & 0.1606 & 0.0353 \\
 & NC  & \win{0.9907}  & \win{[0.991, 0.991]} & \win{0.9915} & \win{0.0020} & \win{0.1457} & \win{[0.145, 0.146]} & \win{0.1428} & \win{0.0171} \\
\midrule
\multirow{2}{*}{Holdout ($\mathcal{H}$)}
 & CO  & 0.9412  & [0.940, 0.942] & 0.9473 & 0.0172 & 0.2546 & [0.252, 0.257] & 0.2402 & 0.0396 \\
 & NC  & \win{0.9657}  & \win{[0.965, 0.967]} & \win{0.9663} & \win{0.0068} & \win{0.1949} & \win{[0.192, 0.198]} & \win{0.1921} & \win{0.0203} \\
\midrule
\multirow{2}{*}{Holdout ($\mathcal{H}_{\mathrm{col}}$)}
 & CO  & 0.9650  & [0.964, 0.966] & 0.9693 & 0.0089 & 0.2194 & [0.218, 0.221] & 0.2063 & 0.0279 \\
 & NC  & \win{0.9734}  & \win{[0.973, 0.974]} & \win{0.9741} & \win{0.0037} & \win{0.1919} & \win{[0.190, 0.194]} & \win{0.1897} & \win{0.0136} \\
\midrule
\multirow{2}{*}{Holdout ($\mathcal{H}_{\mathrm{nc}}$)}
 & CO  & 0.9153  & [0.913, 0.917] & 0.9190 & 0.0254 & 0.2753 & [0.272, 0.279] & 0.2673 & 0.0477 \\
 & NC  & \win{0.9568}  & \win{[0.956, 0.958]} & \win{0.9582} & \win{0.0106} & \win{0.1964} & \win{[0.193, 0.199]} & \win{0.1906} & \win{0.0266} \\
\bottomrule
\end{tabular}
\end{table}
\begin{table}[H]
\small
\centering
\setlength{\tabcolsep}{4pt}
\caption{BF16 runs on Wikipedia: NC win rate (NC vs.\ CO) on $\mathcal{H}$ by law and by epoch, seed-paired over $30$ optimizer seeds; the epoch column omits repeated-data, fit on all epochs at once. Overall NC win rate: \winratebf (\winratebfNumer/\winratebfDenom); $95\%$ Wilson CI \winratebfCI{} over comparisons that share one training run per $(N, D)$. NC has the lower mean holdout RMSE in all $10$ (law, epoch) cells.}
\label{tab:winrate_bf16}
\begin{tabular}{lr@{\hspace{1.5em}}lr}
\toprule
\multicolumn{2}{c}{\textbf{Law}} & \multicolumn{2}{c}{\textbf{Epoch}} \\
\cmidrule(lr){1-2} \cmidrule(lr){3-4}
Chinchilla & 97.8\% & first &  100.0\%\\
Droppo-Elibol & 97.8\% & second & 96.7\% \\
Kaplan & 100.0\% & final & 98.9\% \\
Repeated-data & 100.0\% &  &  \\
\bottomrule
\end{tabular}
\end{table}
\begin{table}[H]
\centering
\caption{BF16 runs on Wikipedia: seed-paired NC vs.\ CO wins on the unified holdout $\mathcal{H}$ and on its collinear and non-collinear parts $\mathcal{H}_{\mathrm{col}}$ and $\mathcal{H}_{\mathrm{nc}}$, over $10$ (law, epoch) cells $\times$ $30$ optimizer seeds.}
\label{tab:winrate_appendix_bf16}
\begin{tabular}{lcccc}
\toprule
\textbf{Holdout Split} & \textbf{NC Wins} & \textbf{CO Wins} & \textbf{Total} & \textbf{NC \%} \\
\midrule
$\mathcal{H}$ & \winratebfNumer & 4 & \winratebfDenom & \winratebf \\
$\mathcal{H}_{\mathrm{col}}$ & 222 & 78 & 300 & 74.0\% \\
$\mathcal{H}_{\mathrm{nc}}$ & 299 & 1 & 300 & 99.7\% \\
\bottomrule
\end{tabular}
\end{table}

They cover Chinchilla, repeated-data~\citep{muennighoff2023scaling}, Kaplan and
Droppo-Elibol, and Appendix~\ref{app:full_all_results} breaks
Table~\ref{tab:summary_bf16_app} down by law, epoch mode and holdout split.

\subsection{Appendix - TPP convergence fraction detailed setup}\label{app:tpp_convergence_setup}

We construct an incremental coverage analysis to measure how scaling law fit quality evolves as the experimental design space is progressively expanded. This analysis reuses the runs from our main experiment rather than conducting independent trials, so the results should be interpreted as a post-hoc decomposition of existing data rather than a controlled ablation. For the collinear (CO) design, we sort the $\COTPPs$ designed TPP levels in ascending order and add them one at a time: step~1 trains on runs from only the lowest TPP, step~2 adds the second-lowest, and so on until all $\COTPPs$ levels are included. Each TPP level contributes $\COtrainN/\COTPPs = 10$ training-eligible runs (one per training model size after the L-shaped holdout removal), growing the training set from $10$ to $\COtrainN$ runs across $\COTPPs$ steps. For the non-collinear (NC) design, we expand coverage by pinching inward from the extremes of the $D$ range: step~1 includes runs at both the smallest and largest $D$, step~2 adds the second-smallest and second-largest, and so on, adding two $D$ levels per step. This pairing ensures that even the earliest step spans the full $D$ range. At each step, all four scaling laws are fit on the current subset and evaluated against fixed holdout sets, so changes in $R^2$ reflect only the effect of broader coverage. The TPP bins are scaled to match the epoch mode under evaluation ($\times 1$ for first-epoch, $\times 2$ for second-epoch, $\times 3$ for final-epoch), aligning the groupings with the effective tokens-per-parameter that each fitter sees.

A coverage fraction of 1.0 means all $\COTPPs$ TPP levels (for CO) or all $\NCDsizes$ $D$ levels (for NC) are included in the training set - this is the same configuration used in the main results tables, and the $R^2$ values at full coverage correspond exactly to those reported there. The CO and NC coverage fractions are not directly comparable: they traverse different axes of the design space using non-corresponding sets of values. Furthermore, the observed $R^2$ trajectory may depend not only on the number of levels included but also on their spacing; our designed TPP values $(1.0,\; 1.5,\; 1.9,\; 2.0,\; 2.5,\; 2.7,\; 3.0,\; 3.3,\; 3.5,\; 4.0,\; 4.5,\; 5.0)$ are unevenly distributed, with denser coverage at the low end, so a different spacing (e.g., uniformly from 1 to 5) could yield a different convergence profile. The primary takeaway is qualitative - that collinear designs require substantially more coverage to approach their asymptotic $R^2$, while non-collinear designs converge rapidly with minimal coverage - rather than a precise quantitative comparison between the two trajectories.



\subsection{Appendix - Optimizer seed protocol}\label{app:seed_protocol}

All scaling-law parameters are estimated via nonlinear least
squares (L-BFGS-B with differential evolution polish) using
multiple random restarts seeded by a deterministic sequence.
The \texttt{seed} column in released parquet files and CSVs
stores the seed \emph{index}; the actual RNG seed passed to
the optimizer is computed from this index as described below.
All model pre-training uses a fixed seed of $42$ for weight
initialization and data shuffling, with one run per $(N, D)$; the seeds documented here
govern only the scaling-law curve fitting.

\noindent\textbf{Tables~\ref{tab:summary_full}-\ref{tab:winrate} and convergence-fraction plots.}
Both the full-coverage analysis (Tables~\ref{tab:summary_full}-\ref{tab:winrate})
and the seed-variance convergence fraction plots are produced by identical seed logic: $30$
optimizer seeds at $100$ random restarts each, base seed
$b = 0$, stride $1$:
\[
  s_i = i, \qquad i \in \{0,\, 1,\, 2,\, \ldots,\, 29\}.
\]
Each seed index $i$ is passed directly as the RNG seed to the
optimizer.  The full-coverage stage fits, for both designs at each seed,
five datasets $\times$ ten (law, epoch) cells (Chinchilla, Kaplan and
Droppo-Elibol per epoch mode, repeated-data once on all epochs),
producing \winrateDenom paired comparisons
(Table~\ref{tab:winrate}).  The seed-variance stage fits the
same configurations at each incremental coverage step,
recording the distribution of holdout $R^2$ across the $30$
seeds at every step.

\noindent\textbf{Tables~\ref{tab:summary_bf16} and~\ref{tab:winrate_bf16} (BF16 replication).}
The BF16 analysis uses $30$ optimizer seeds at
$100$ random restarts each, base seed $b = 0$, stride $137$:
\[
  s_i = 137 \cdot i, \qquad i \in \{0,\, 1,\, 2,\, \ldots,\, 29\},
\]
yielding the sequence
$\{0,\; 137,\; 274,\; 411,\; \ldots,\; 3{,}973\}$.
The wider stride avoids accidental overlap with the main
analysis seeds; the results are statistically equivalent to
any other choice of $30$ independent seeds.

\noindent\textbf{Table~\ref{tab:regime_a_winrate} (subset enumeration).}
The budget-matched subset enumeration uses $22$ independent
global seeds $g \in \{0,\, 1,\, \ldots,\, 21\}$, each with
$300$ random restarts per fit.  The actual optimizer seed for
global seed $g$ is:
\[
  s_g = 42 + 100{,}003 \cdot g.
\]
Table~\ref{tab:subset_seeds} lists all $22$ values.  The NC
bounding-box construction uses a separate deterministic seed
derived via BLAKE2b hash of the dataset name, epoch mode, and
subset mask, XORed with $g \cdot \mathtt{0x9E3779B1}$ (a
32-bit golden-ratio constant ensuring uniform dispersion).  The
holdout subsample seed is intentionally \emph{not} varied across
$g$: it is fixed at
$7{,}777 + \mathrm{hash}(\text{dataset},\, \text{mode})
\bmod 10^5$
so that the holdout composition remains constant across all
$22$ runs.

\noindent\textbf{Preliminary Large TPP Seeds.}
Our preliminary large TPP results use the same stride as in our BF16 results, using a stride of 137.

\begin{table}[H]
\centering\small
\caption{Fit seeds for each of the $22$ subset-enumeration runs
($s_g = 42 + 100{,}003 \cdot g$).}
\label{tab:subset_seeds}
\begin{tabular}{rc@{\hspace{2em}}rc@{\hspace{2em}}rc}
\toprule
$g$ & $s_g$ & $g$ & $s_g$ & $g$ & $s_g$ \\
\midrule
 0 &            42 &  8 &    800{,}066 & 16 & 1{,}600{,}090 \\
 1 &    100{,}045 &  9 &    900{,}069 & 17 & 1{,}700{,}093 \\
 2 &    200{,}048 & 10 & 1{,}000{,}072 & 18 & 1{,}800{,}096 \\
 3 &    300{,}051 & 11 & 1{,}100{,}075 & 19 & 1{,}900{,}099 \\
 4 &    400{,}054 & 12 & 1{,}200{,}078 & 20 & 2{,}000{,}102 \\
 5 &    500{,}057 & 13 & 1{,}300{,}081 & 21 & 2{,}100{,}105 \\
 6 &    600{,}060 & 14 & 1{,}400{,}084 &    &               \\
 7 &    700{,}063 & 15 & 1{,}500{,}087 &    &               \\
\bottomrule
\end{tabular}
\end{table}

\begin{table}[H]
\centering\small
\caption{Seed protocol summary.  Count: number of seeds; Restarts:
random restarts per fit.  Index $i$ ranges over optimizer
seeds within a stage; $g$ ranges over independent subset-enumeration
runs.}
\label{tab:seed_summary}
\begin{tabular}{l c c c l}
\toprule
\textbf{Stage} & \textbf{Count} & \textbf{Restarts} & \textbf{Tables/Figures} & \textbf{RNG seed} \\
\midrule
Full coverage       & 30 & 100 & \ref{tab:summary_full}, \ref{tab:winrate}            & $i$ \\
Seed variance       & 30 & 100 & Convergence plots                                     & $i$ \\
BF16 replication    & 30 & 100 & \ref{tab:summary_bf16}, \ref{tab:winrate_bf16}        & $137i$ \\
Subset enumeration  & 22 & 300 & \ref{tab:regime_a_winrate}                             & $42 + 100{,}003g$ \\
BigTPP prelim     & 30 & 100 & \ref{tab:winrate_bigtpp_prelim}                       & $137i$ \\
\bottomrule
\end{tabular}
\end{table}

\subsection{Appendix - Preliminary large TPP experiment}\label{app:bigtpp_prelim}

To check whether the
NC advantage persists at higher TPP, we ran a single-corpus
preliminary experiment on Wikipedia. Due to compute restrictions, the grid is smaller and we fit only Chinchilla.

\noindent\textbf{Training grid.}
Both designs share ten model sizes
\[
N \in \{5.04, 10.07, 14.41, 19.40, 22.80,\,
28.82, 32.50, 35.37, 40.28, 47.95\}\text{M}.
\]
CO sweeps $k \in \mathcal{K}_{\mathrm{big}}
\coloneqq \{10,\, 11,\, 12,\, 13,\, 14,\, 15\}$.
NC sweeps $D \in \{100, 225, 400, 550, 700, 825\}$\,M.

\noindent\textbf{Holdout.}
Four out-of-grid sizes
$N_{\mathrm{holdout}} \in \{55.54, 63.93, 68.13, 76.52\}$\,M
(Table~\ref{tab:model-architectures}), all
above the training maximum. For holdout model size $55$M, the holdout
model used TPP values $10,11,12,13,14,15,16,18,20$. For non-collinear,
$55$M models used
$D \in \{100, 225, 400, 550, 700, 825, 900, 1100, 1400\}$. At the three largest
holdout sizes, $\mathcal{H}_{\mathrm{col}}$ only includes
TPP\,$\in\{16, 18, 20\}$ and $\mathcal{H}_{\mathrm{nc}}$ additionally
includes $D \in \{900, 1100, 1400\}$\,M.

\noindent\textbf{Fitting.}
Chinchilla, $30$ seeds (stride $137$, matching
Appendix~\ref{app:bf16_results}), $100$ restarts with differential
evolution polish. We report seed-paired NC win rate
(Table~\ref{tab:winrate_bigtpp_prelim}) with Wilson 95\% CI.
A full BigTPP replication is left to future work.

\begin{table}[H]
\small
\centering
\caption{Preliminary BigTPP NC win rate on Wikipedia (Chinchilla, first epoch, CO at $k \in \mathcal{K}_{\mathrm{big}}$): seed-paired NC vs.\ CO comparisons over $30$ optimizer seeds, with a 95\% Wilson CI. A compute-limited preliminary experiment, not a full result.}
\label{tab:winrate_bigtpp_prelim}
\begin{tabular}{lccc}
\toprule
\textbf{Law} & \textbf{NC Wins} & \textbf{Total} & \textbf{NC \% [95\% CI]} \\
\midrule
Chinchilla & 19 & 30 & 63.3\% [45.5\%, 78.1\%] \\
\bottomrule
\end{tabular}
\end{table}

\subsection{Appendix - Three-panel plots of the fitted laws}
\label{app:surfaces}

Figures~\ref{fig:dump-01}-\ref{fig:dump-bf16-05} continue the
three-panel plots of Figure~\ref{fig:select-plots}, with the same
panels. The BF16 figures (Figures~\ref{fig:dump-bf16-01}-\ref{fig:dump-bf16-05})
have no coverage-fraction panel, since the BF16 runs were not part of the
coverage tracking of Section~\ref{sec:tpp_convergence}.

\begin{figure*}[!ht]
    \centering
    \begin{subfigure}[b]{0.37\textwidth}
        \centering
        \includegraphics[width=\linewidth]{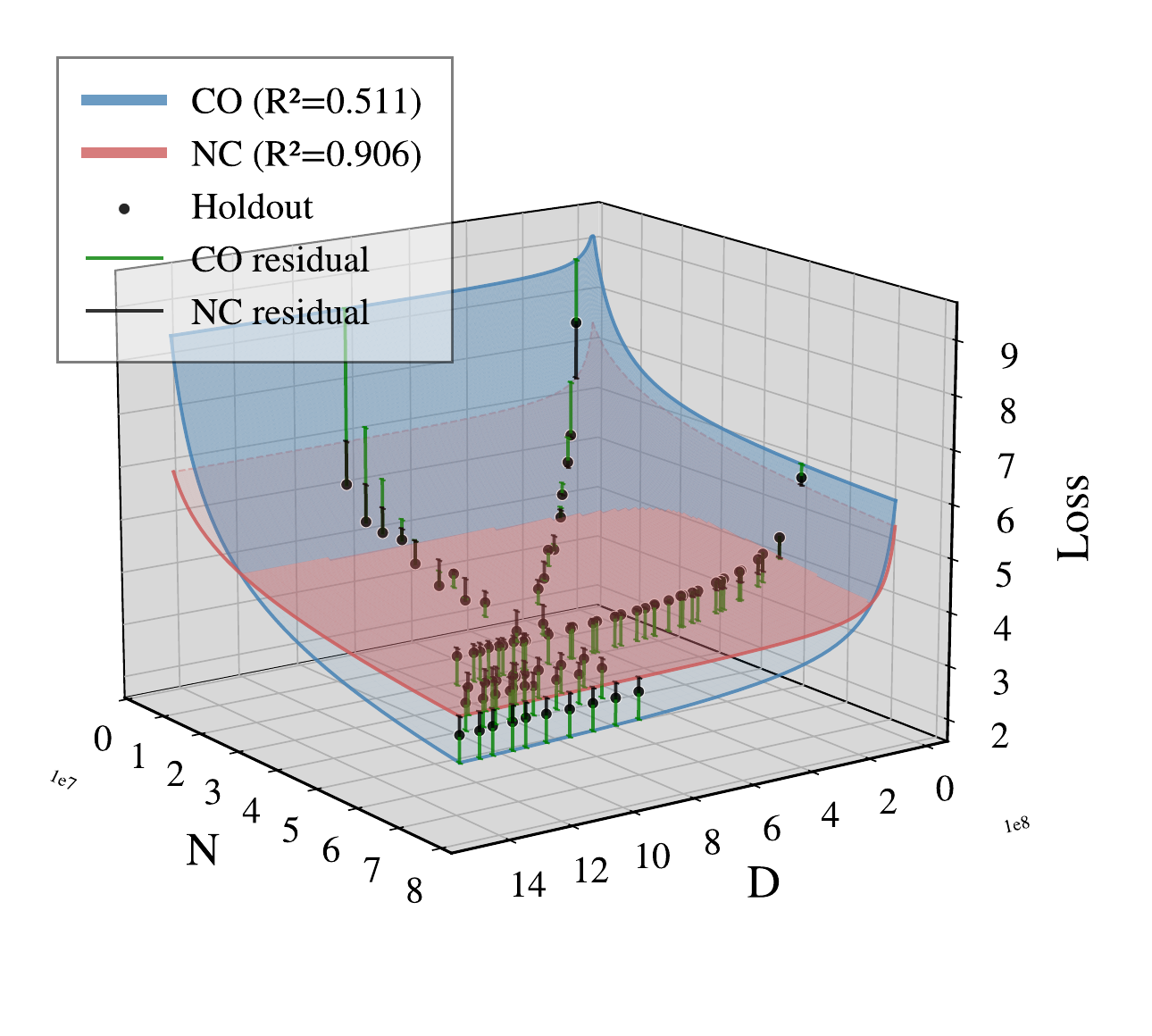}
        \caption{Loss surface.}
        \label{fig:01-surface}
    \end{subfigure}%
    \hfill
    \begin{subfigure}[b]{0.30\textwidth}
        \centering
        \includegraphics[width=\linewidth]{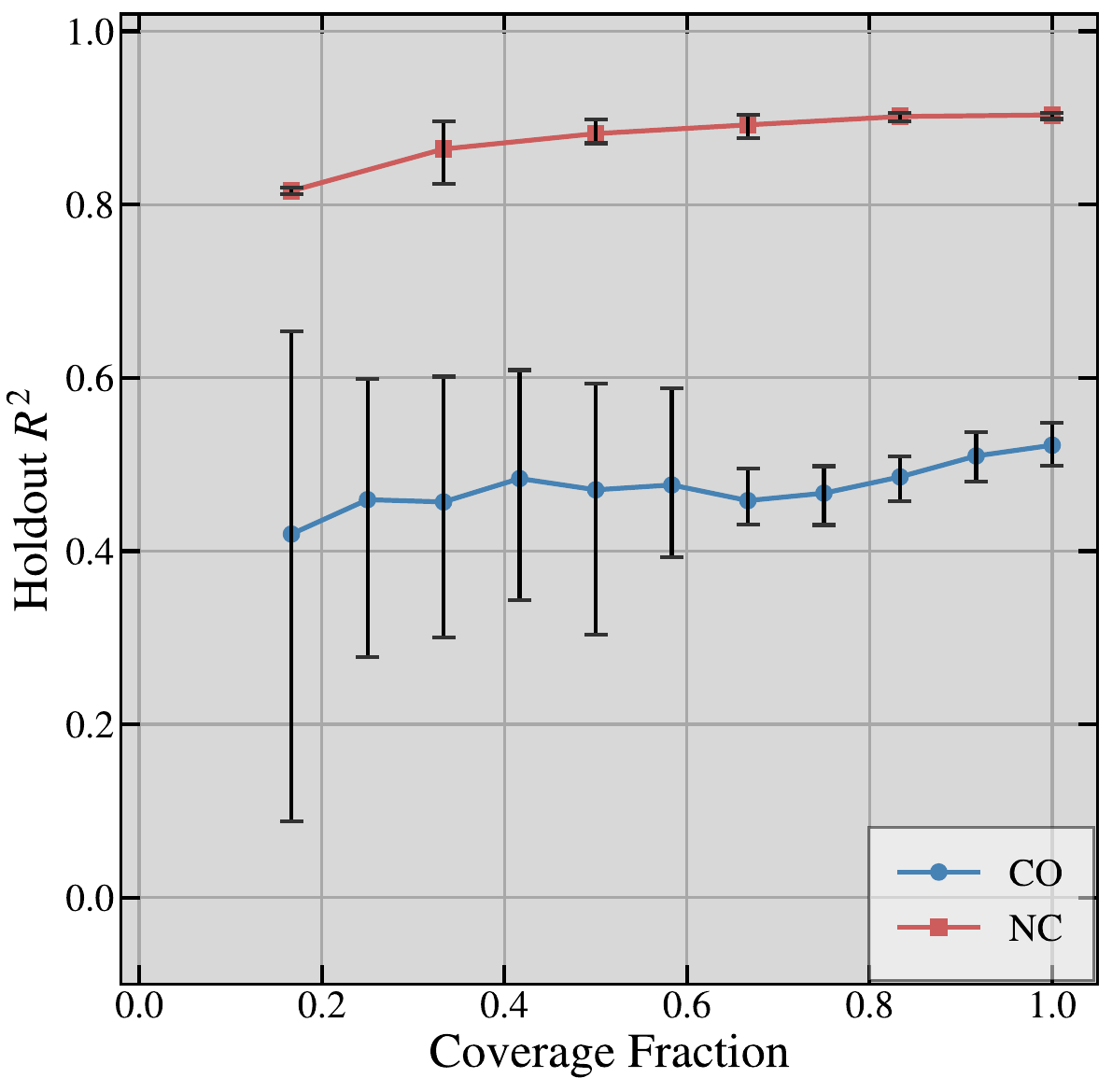}
        \caption{Holdout $R^{2}$ vs.\ TPP coverage.}
        \label{fig:01-convergence}
    \end{subfigure}%
    \hfill
    \begin{subfigure}[b]{0.30\textwidth}
        \centering
        \includegraphics[width=\linewidth]{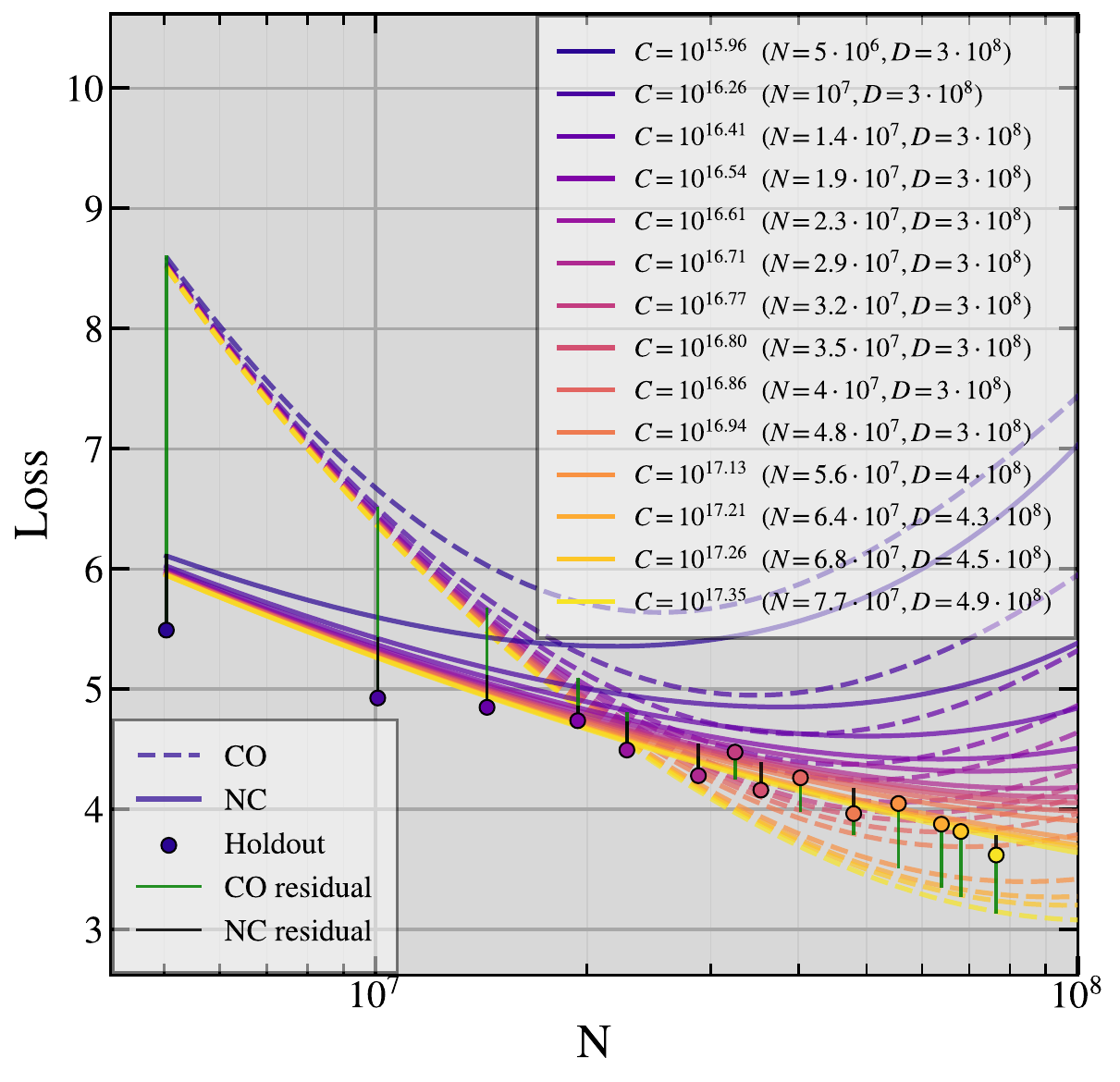}
        \caption{Induced isoFLOP curves.}
        \label{fig:01-isoflop}
    \end{subfigure}
    \caption{Kaplan law on RedPajama, first epoch: CO and NC fits, with the panels of Figure~\ref{fig:select-plots}.}
    \label{fig:dump-01}
\end{figure*}

\begin{figure*}[!ht]
    \centering
    \begin{subfigure}[b]{0.37\textwidth}
        \centering
        \includegraphics[width=\linewidth]{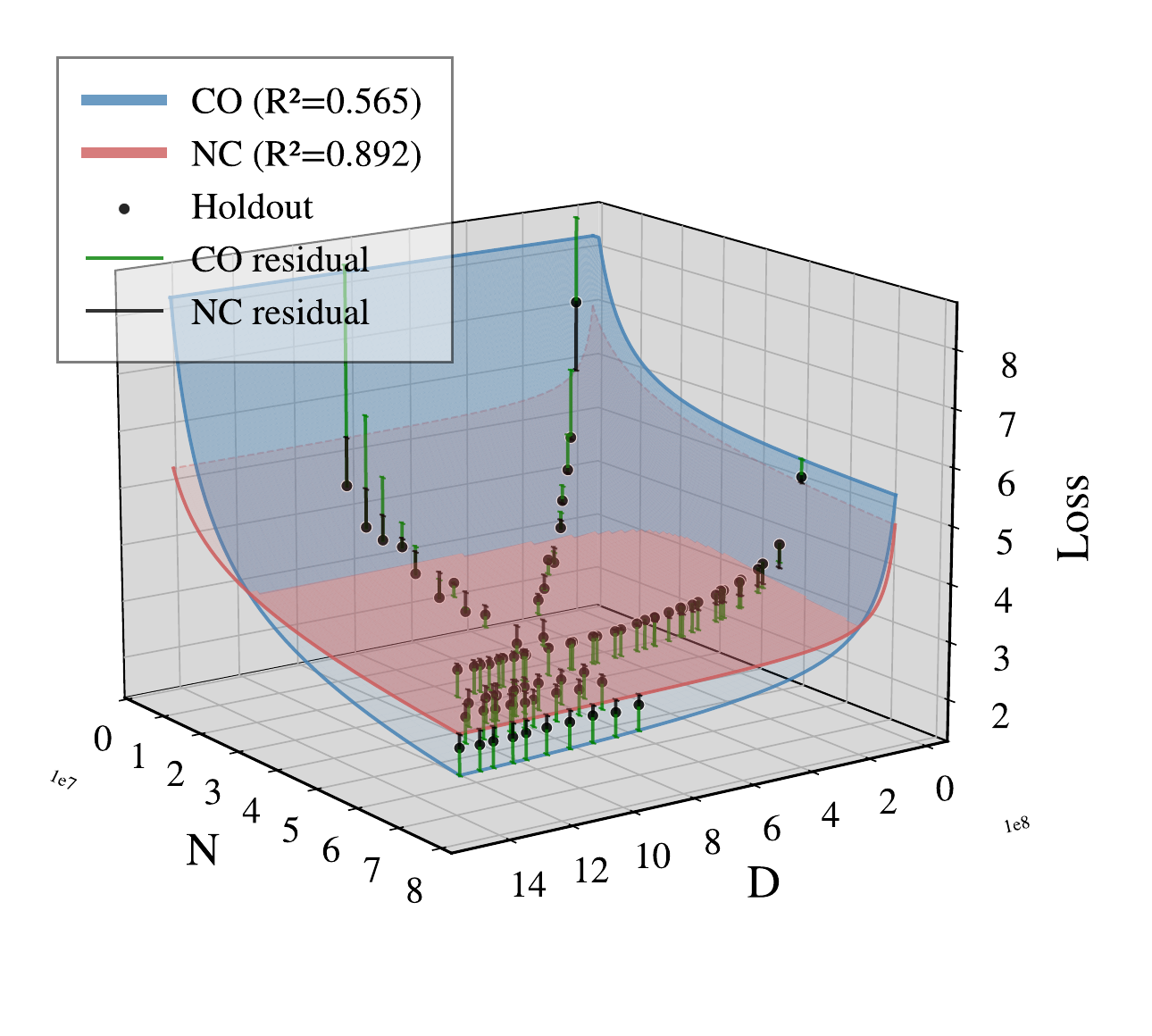}
        \caption{Loss surface.}
        \label{fig:02-surface}
    \end{subfigure}%
    \hfill
    \begin{subfigure}[b]{0.30\textwidth}
        \centering
        \includegraphics[width=\linewidth]{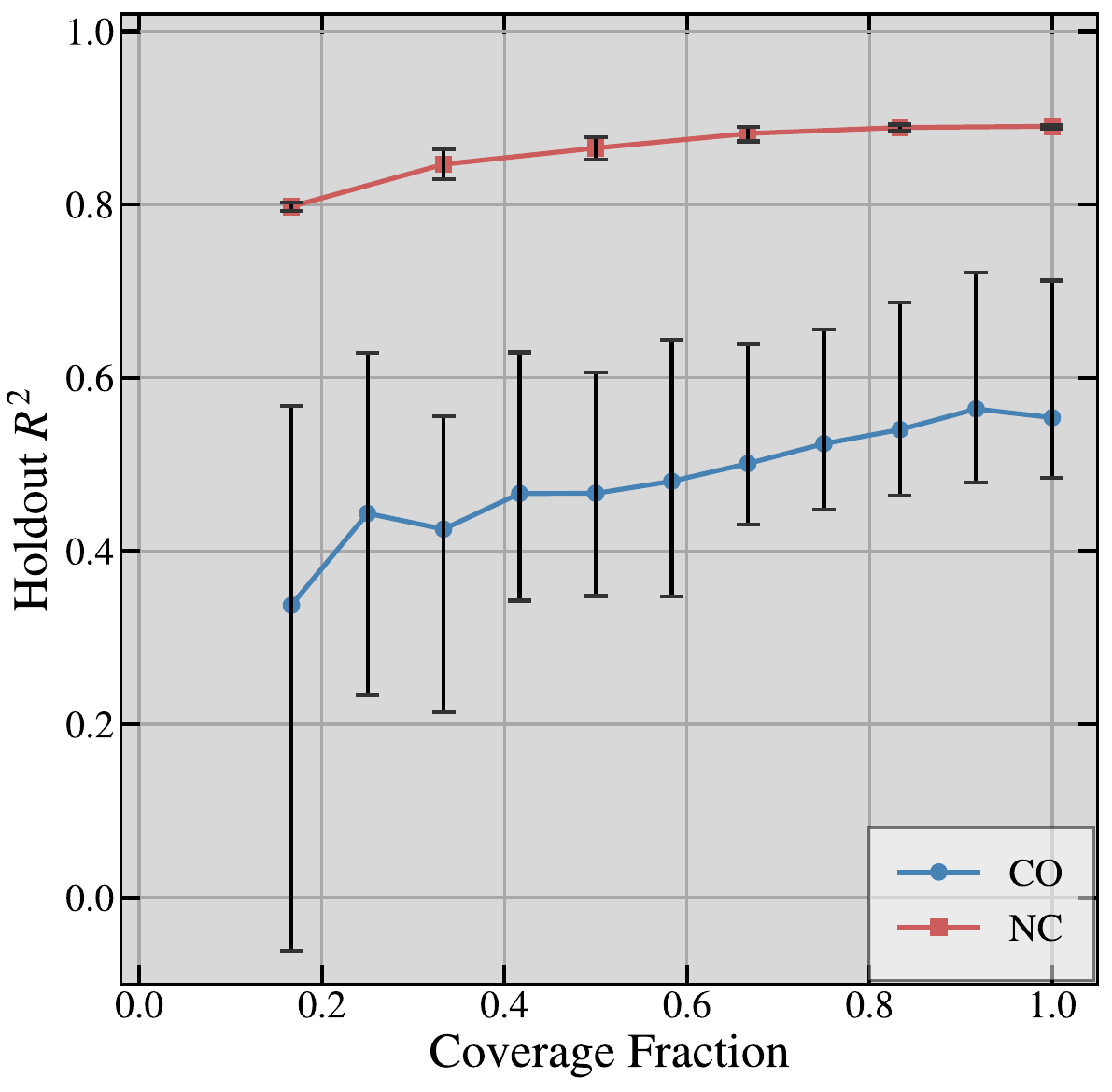}
        \caption{Holdout $R^{2}$ vs.\ TPP coverage.}
        \label{fig:02-convergence}
    \end{subfigure}%
    \hfill
    \begin{subfigure}[b]{0.30\textwidth}
        \centering
        \includegraphics[width=\linewidth]{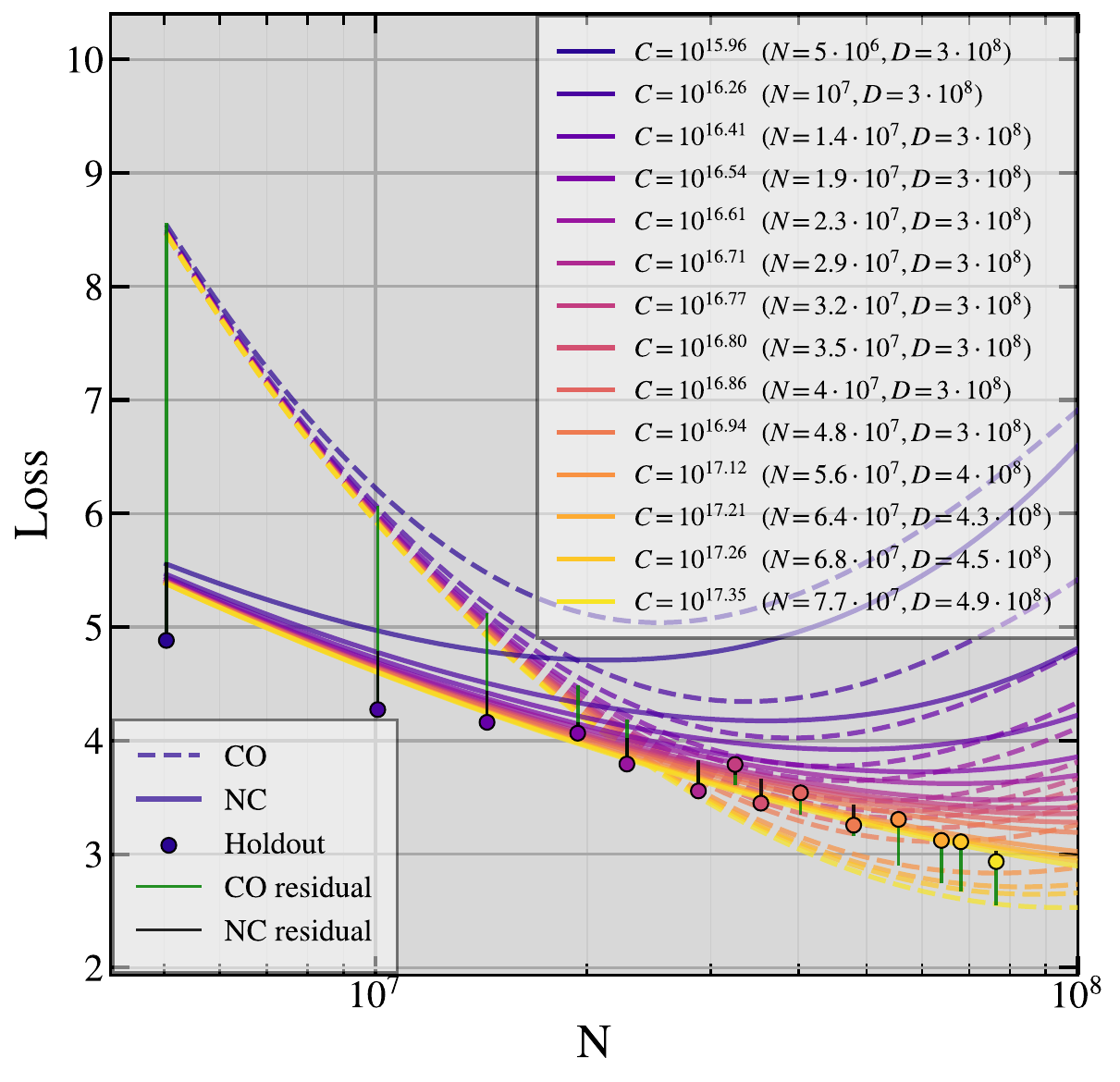}
        \caption{Induced isoFLOP curves.}
        \label{fig:02-isoflop}
    \end{subfigure}
    \caption{Kaplan law on Wikipedia, first epoch: CO and NC fits, with the panels of Figure~\ref{fig:select-plots}.}
    \label{fig:dump-02}
\end{figure*}

\begin{figure*}[!ht]
    \centering
    \begin{subfigure}[b]{0.37\textwidth}
        \centering
        \includegraphics[width=\linewidth]{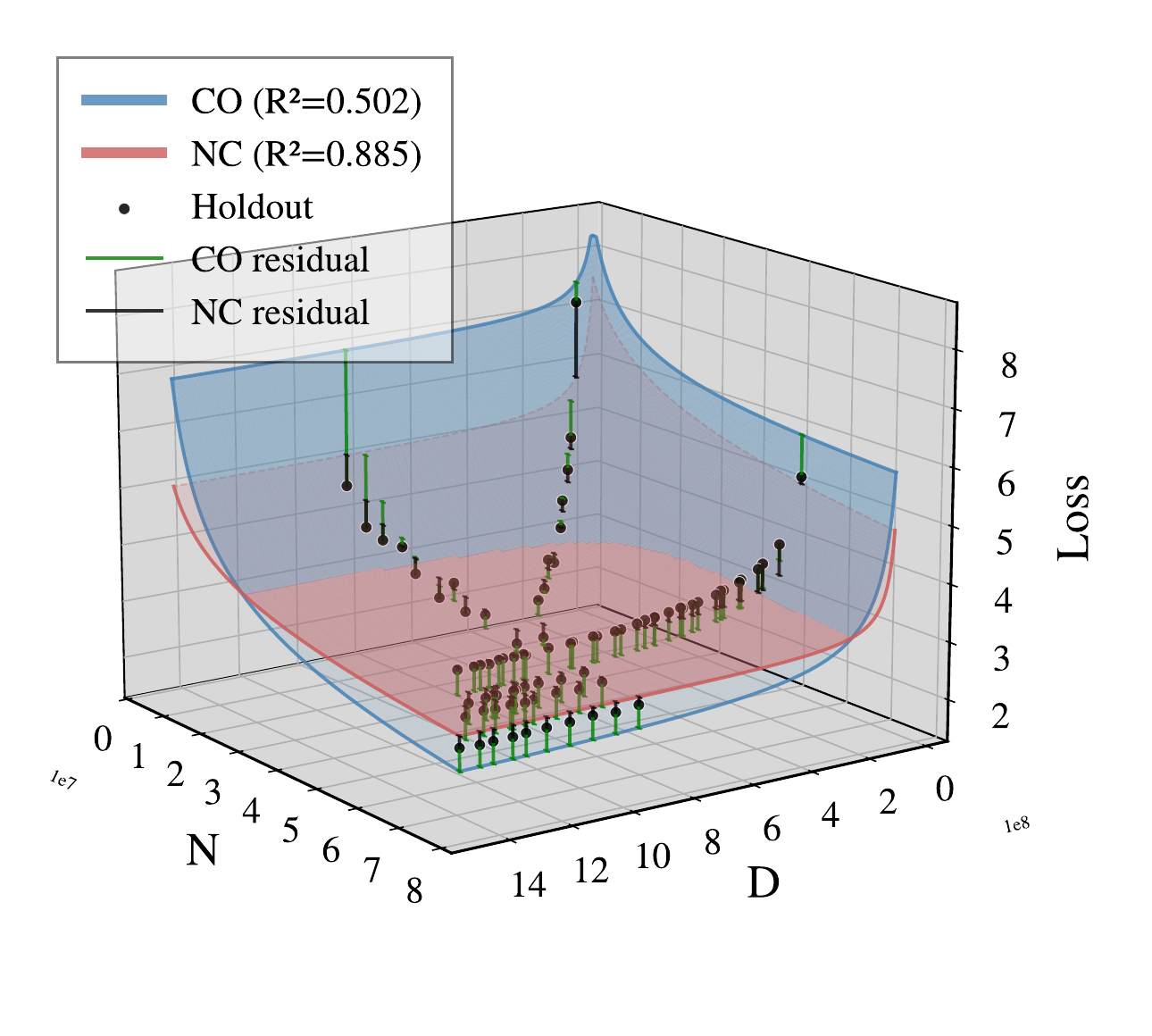}
        \caption{Loss surface.}
        \label{fig:03-surface}
    \end{subfigure}%
    \hfill
    \begin{subfigure}[b]{0.30\textwidth}
        \centering
        \includegraphics[width=\linewidth]{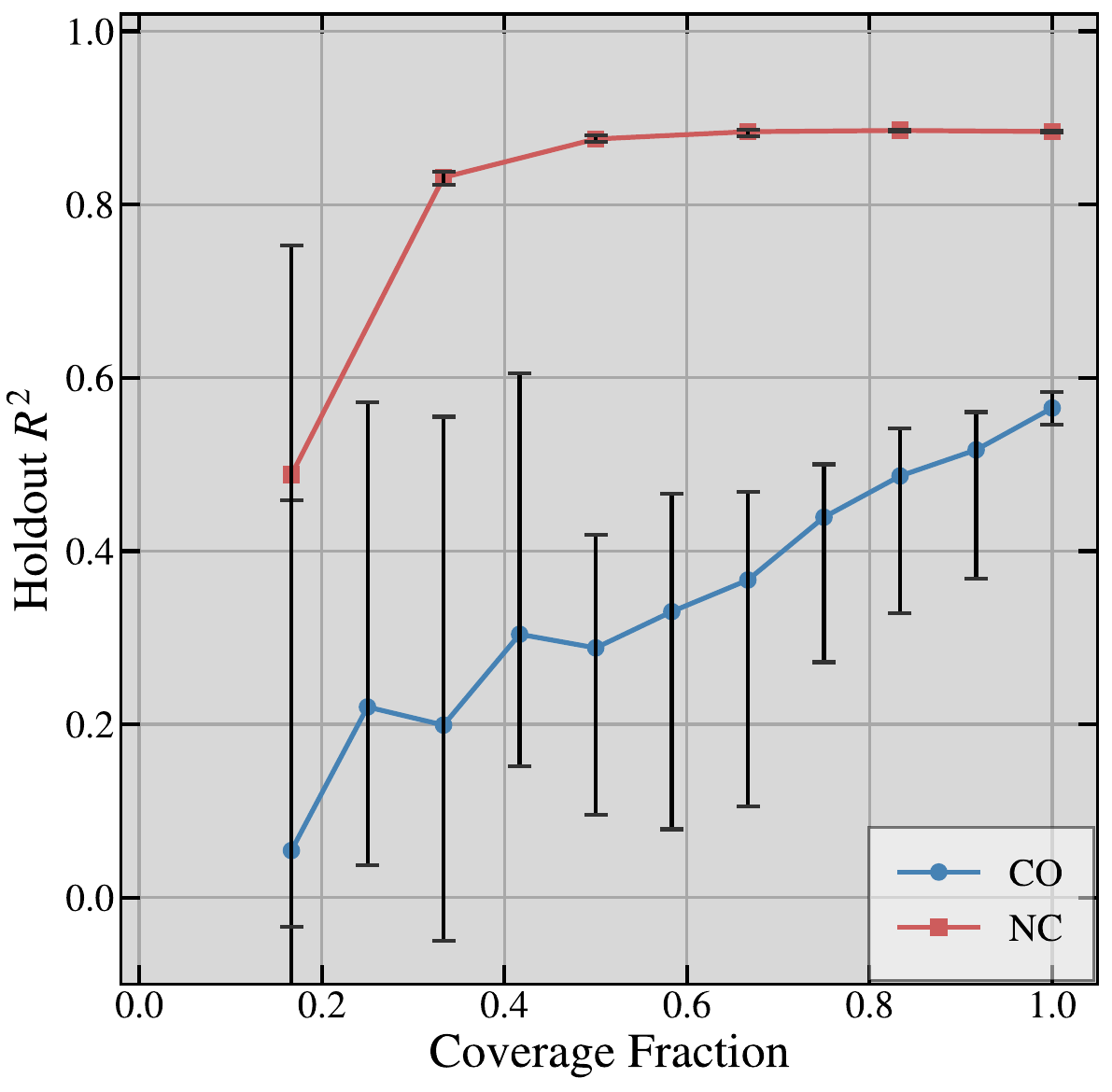}
        \caption{Holdout $R^{2}$ vs.\ TPP coverage.}
        \label{fig:03-convergence}
    \end{subfigure}%
    \hfill
    \begin{subfigure}[b]{0.30\textwidth}
        \centering
        \includegraphics[width=\linewidth]{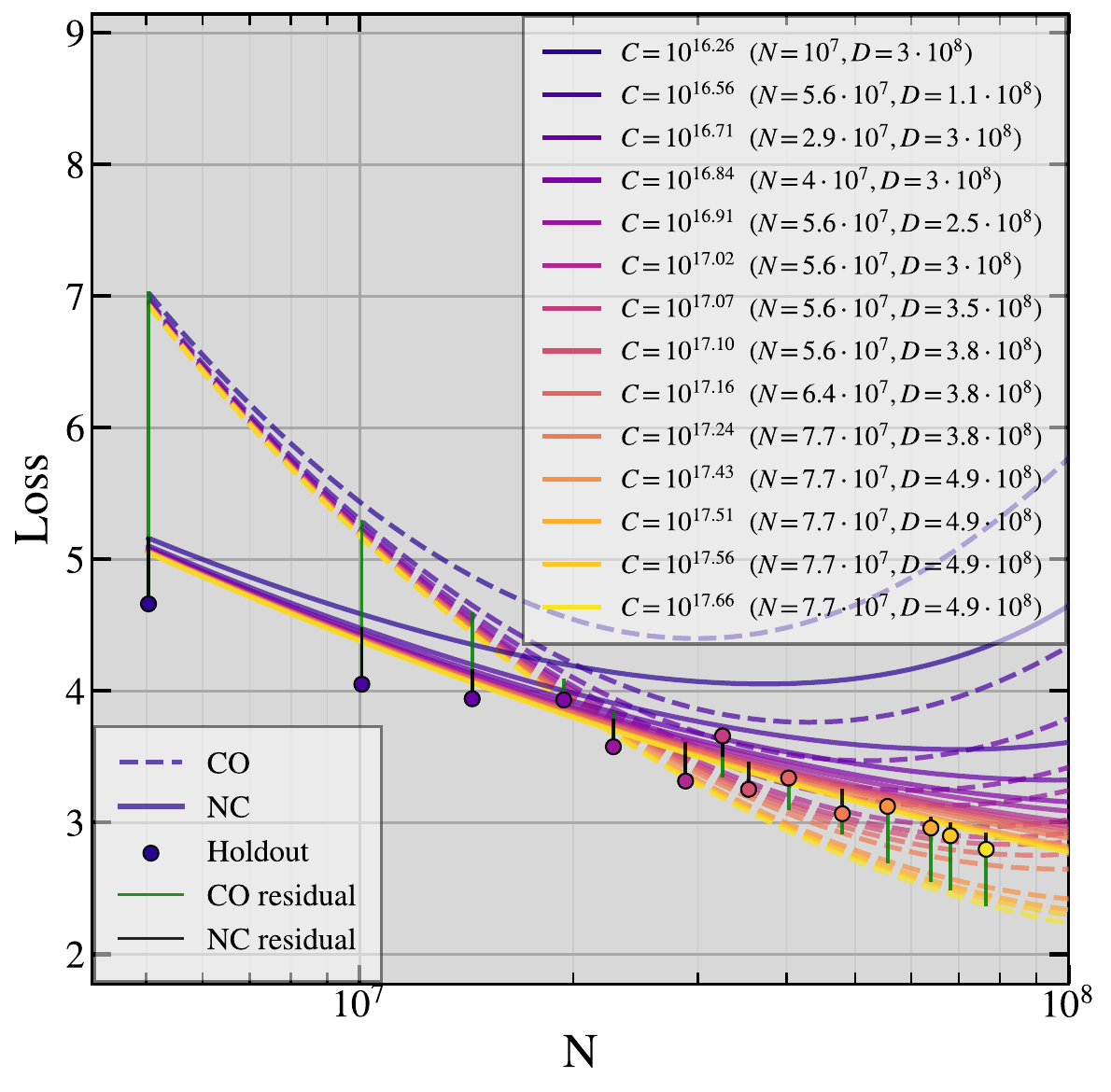}
        \caption{Induced isoFLOP curves.}
        \label{fig:03-isoflop}
    \end{subfigure}
    \caption{Kaplan law on Wikipedia, second epoch: CO and NC fits, with the panels of Figure~\ref{fig:select-plots}.}
    \label{fig:dump-03}
\end{figure*}

\begin{figure*}[!ht]
    \centering
    \begin{subfigure}[b]{0.37\textwidth}
        \centering
        \includegraphics[width=\linewidth]{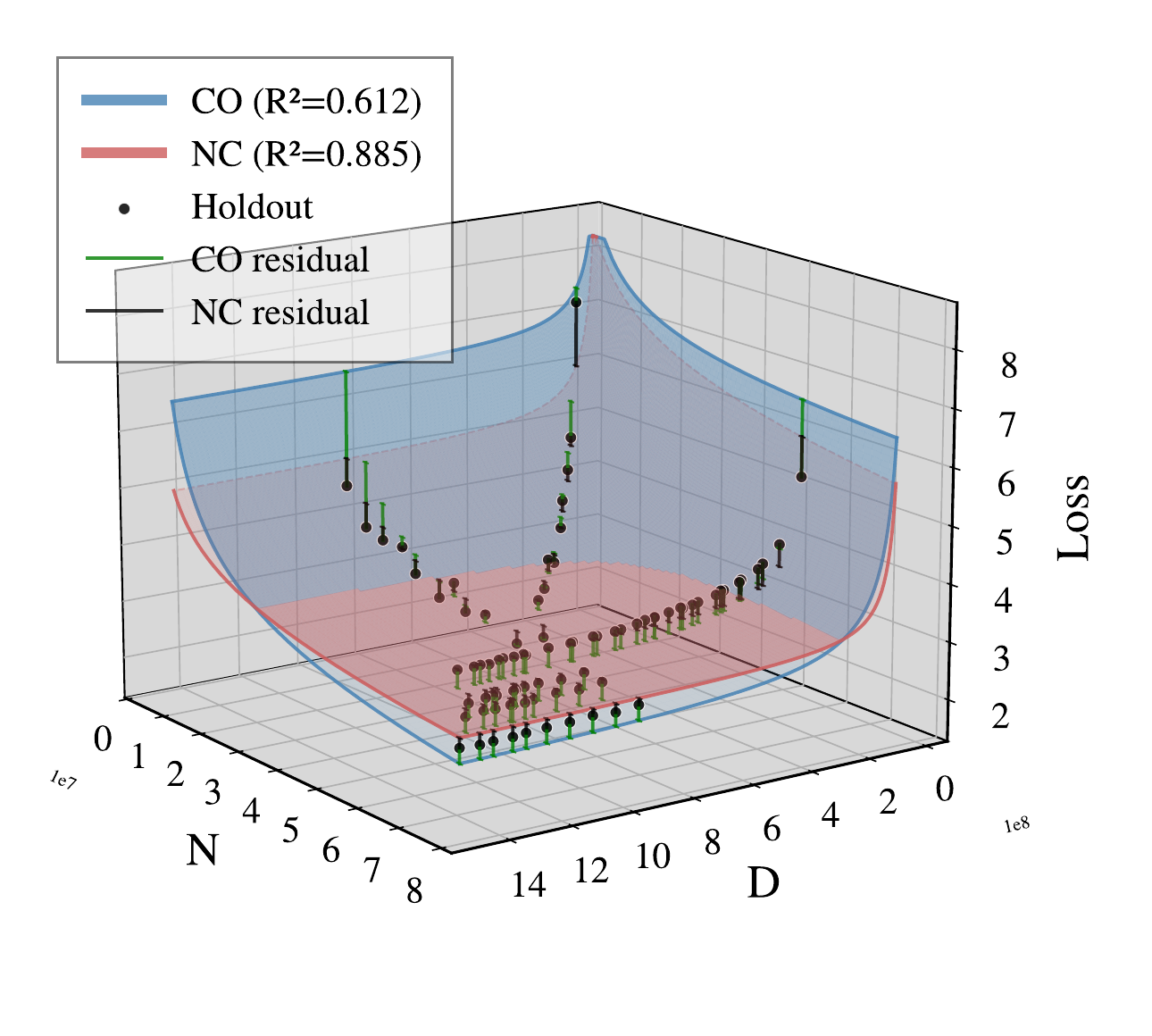}
        \caption{Loss surface.}
        \label{fig:04-surface}
    \end{subfigure}%
    \hfill
    \begin{subfigure}[b]{0.30\textwidth}
        \centering
        \includegraphics[width=\linewidth]{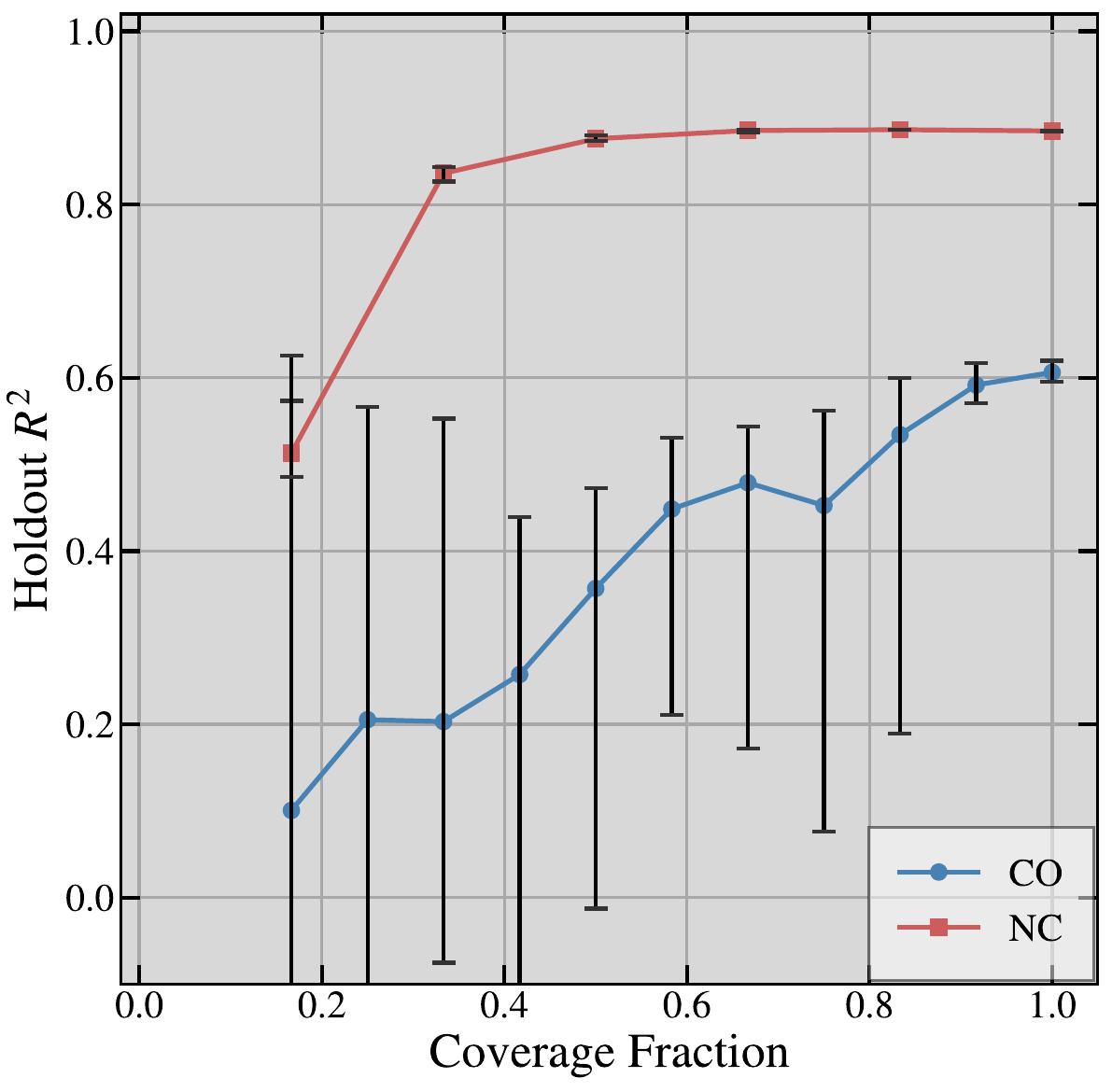}
        \caption{Holdout $R^{2}$ vs.\ TPP coverage.}
        \label{fig:04-convergence}
    \end{subfigure}%
    \hfill
    \begin{subfigure}[b]{0.30\textwidth}
        \centering
        \includegraphics[width=\linewidth]{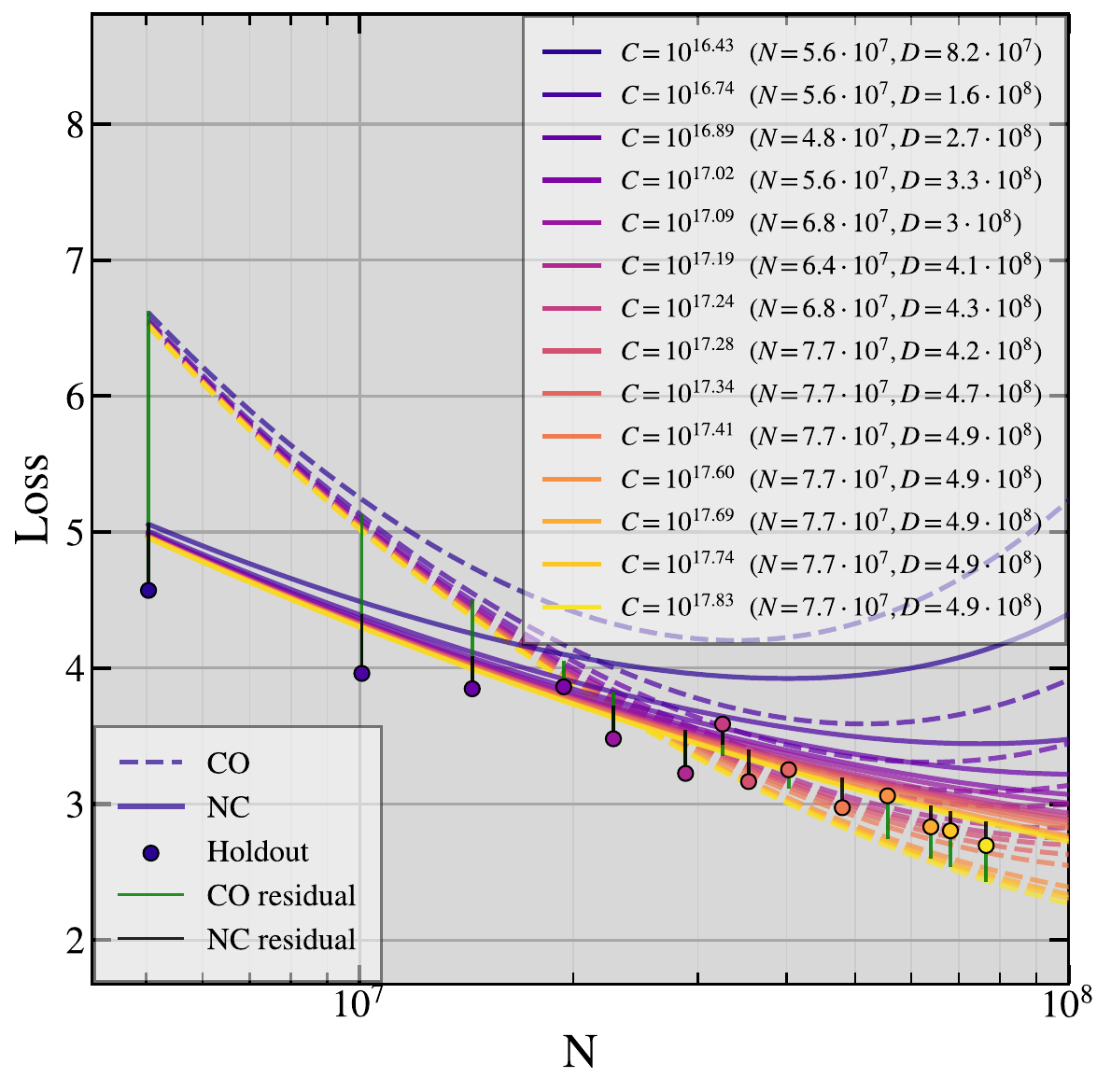}
        \caption{Induced isoFLOP curves.}
        \label{fig:04-isoflop}
    \end{subfigure}
    \caption{Kaplan law on Wikipedia, final epoch: CO and NC fits, with the panels of Figure~\ref{fig:select-plots}.}
    \label{fig:dump-04}
\end{figure*}

\begin{figure*}[!ht]
    \centering
    \begin{subfigure}[b]{0.37\textwidth}
        \centering
        \includegraphics[width=\linewidth]{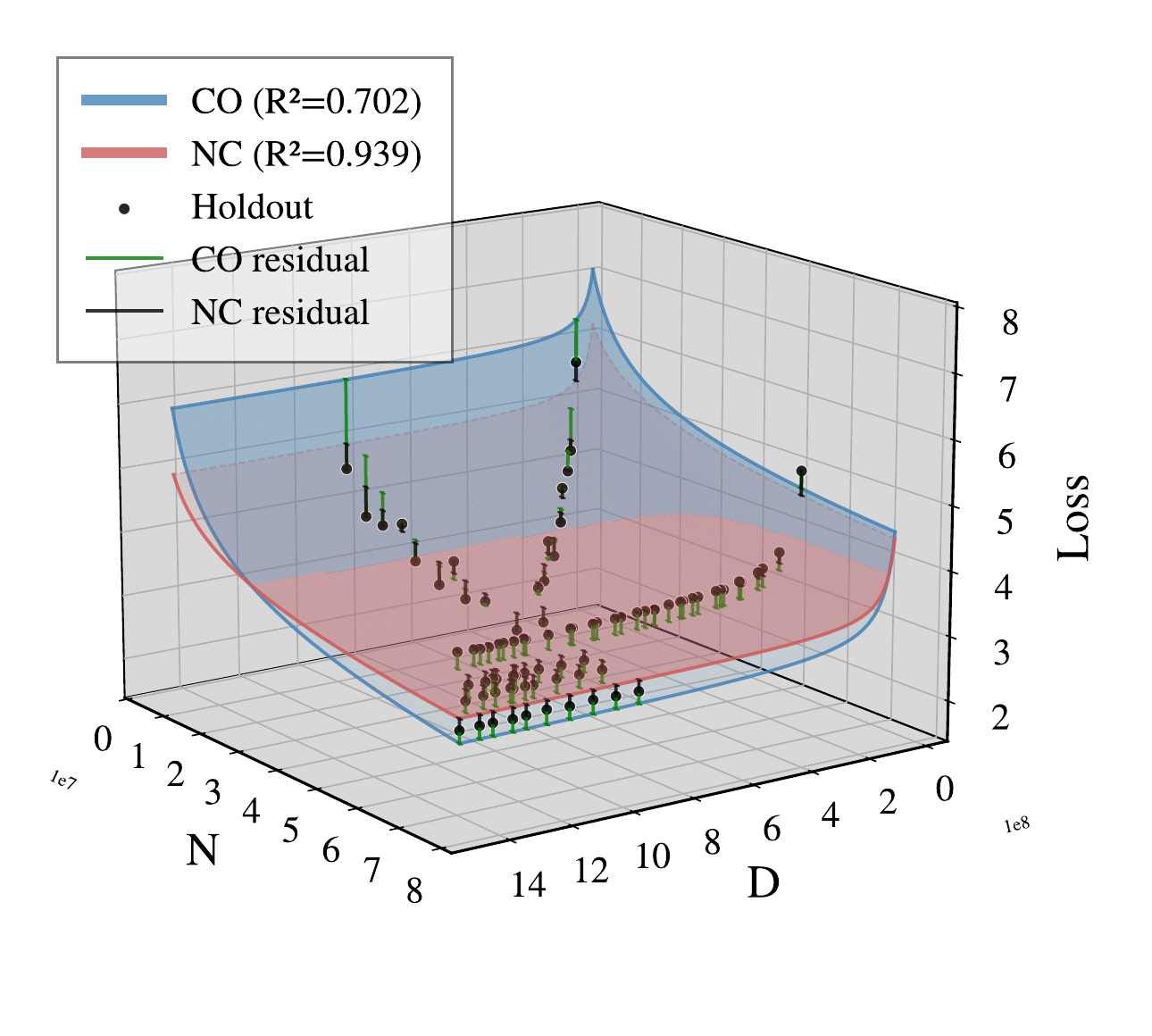}
        \caption{Loss surface.}
        \label{fig:05-surface}
    \end{subfigure}%
    \hfill
    \begin{subfigure}[b]{0.30\textwidth}
        \centering
        \includegraphics[width=\linewidth]{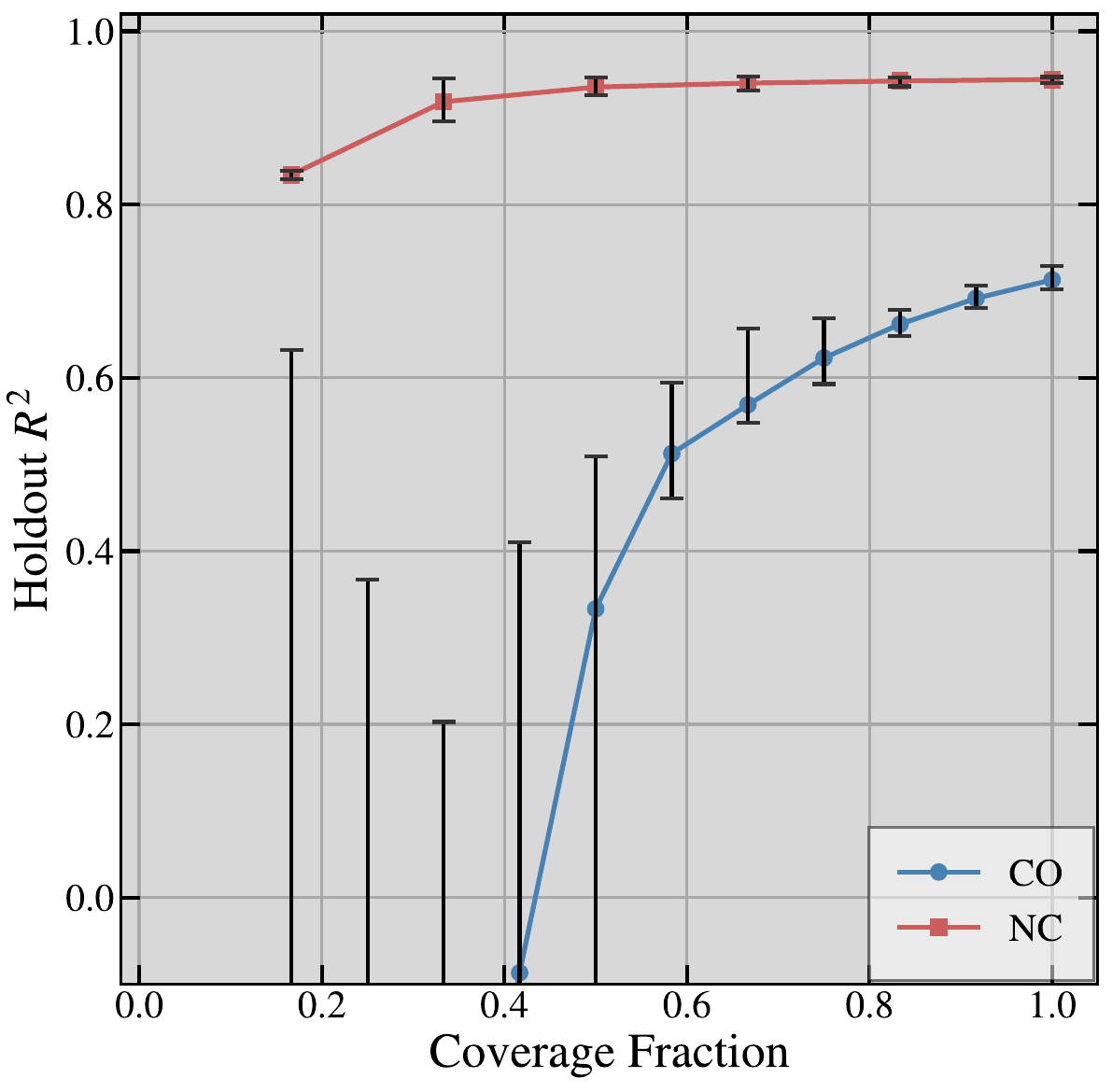}
        \caption{Holdout $R^{2}$ vs.\ TPP coverage.}
        \label{fig:05-convergence}
    \end{subfigure}%
    \hfill
    \begin{subfigure}[b]{0.30\textwidth}
        \centering
        \includegraphics[width=\linewidth]{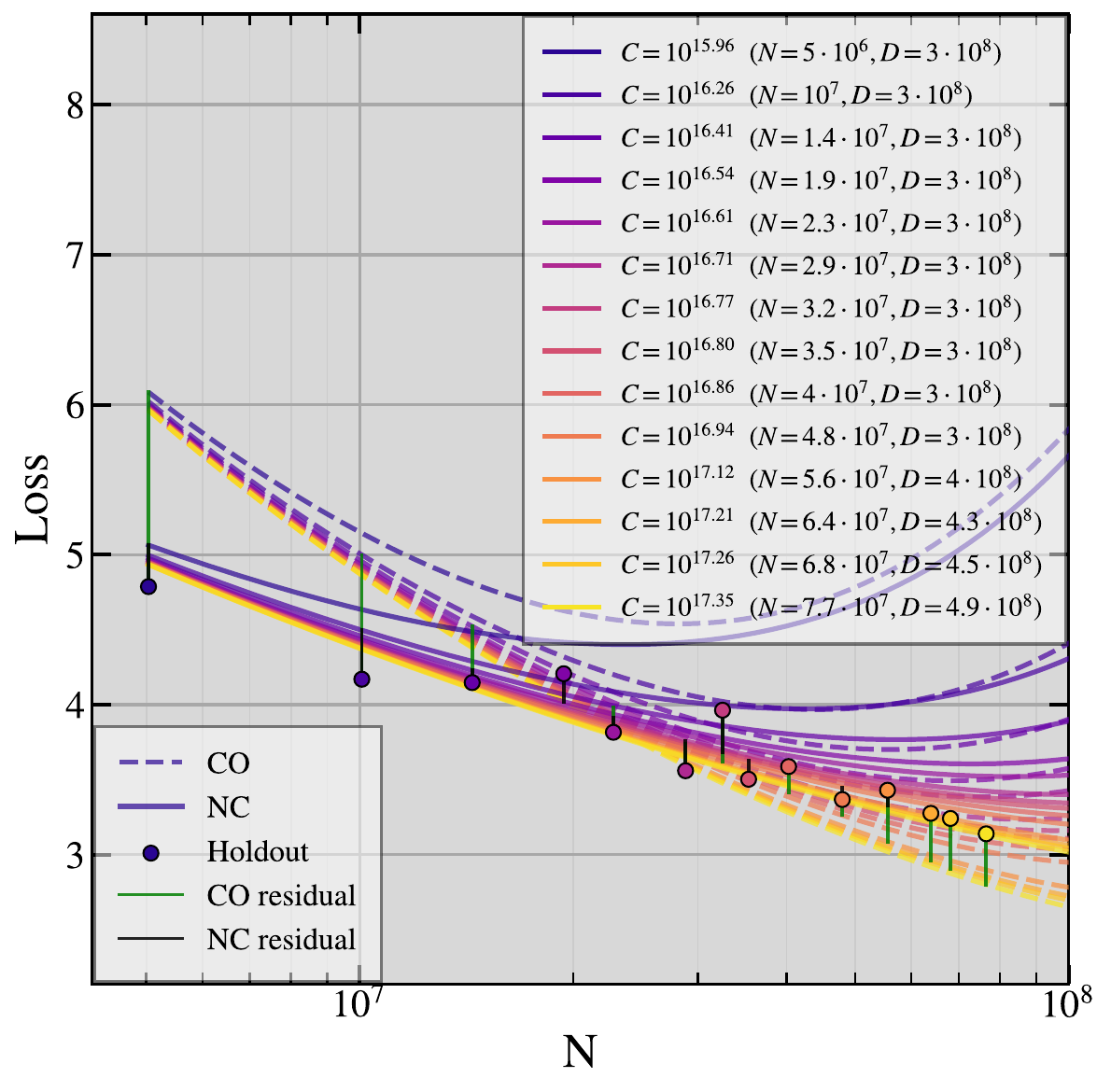}
        \caption{Induced isoFLOP curves.}
        \label{fig:05-isoflop}
    \end{subfigure}
    \caption{Kaplan law on peS2o, first epoch: CO and NC fits, with the panels of Figure~\ref{fig:select-plots}.}
    \label{fig:dump-05}
\end{figure*}

\begin{figure*}[!ht]
    \centering
    \begin{subfigure}[b]{0.37\textwidth}
        \centering
        \includegraphics[width=\linewidth]{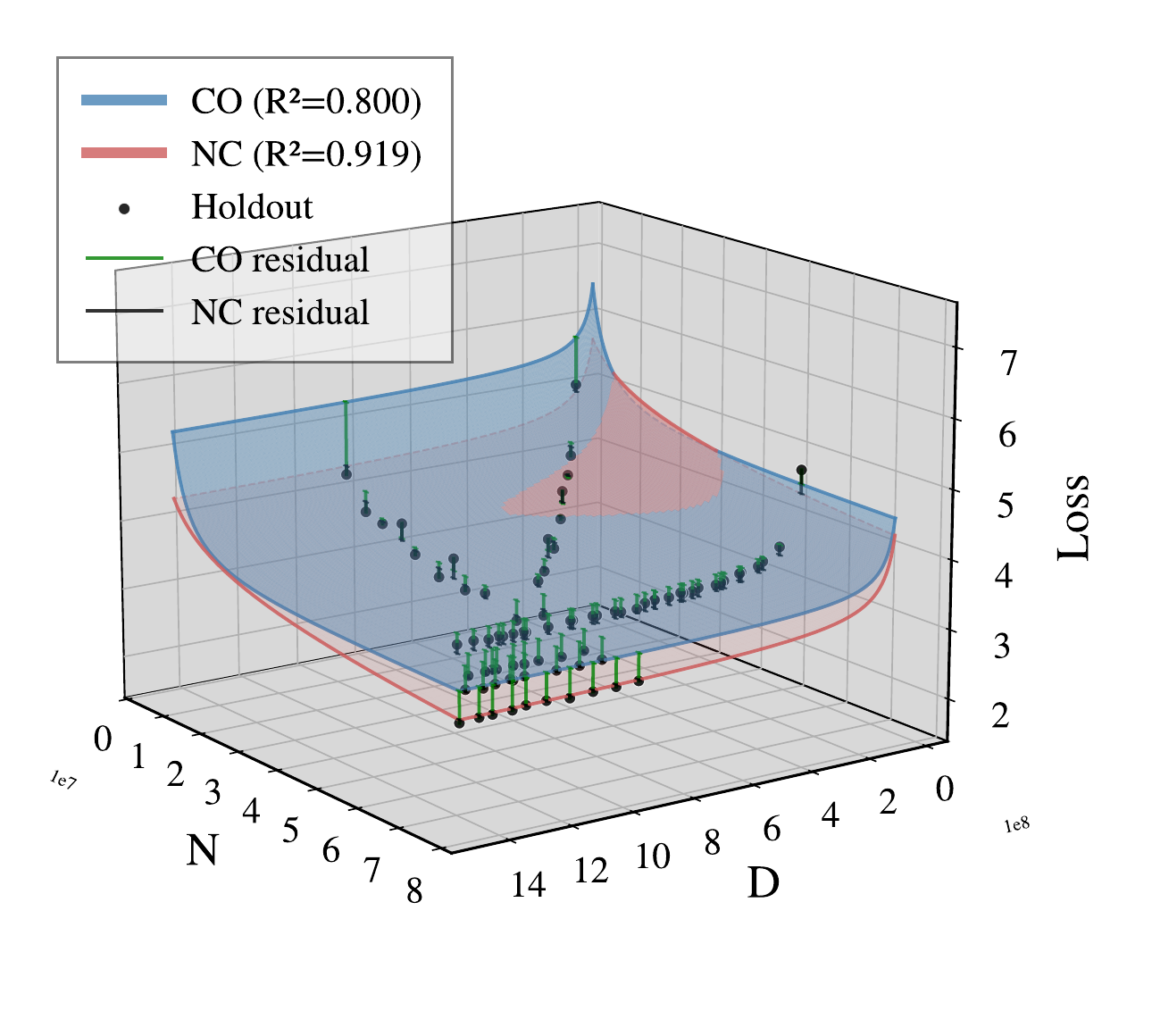}
        \caption{Loss surface.}
        \label{fig:06-surface}
    \end{subfigure}%
    \hfill
    \begin{subfigure}[b]{0.30\textwidth}
        \centering
        \includegraphics[width=\linewidth]{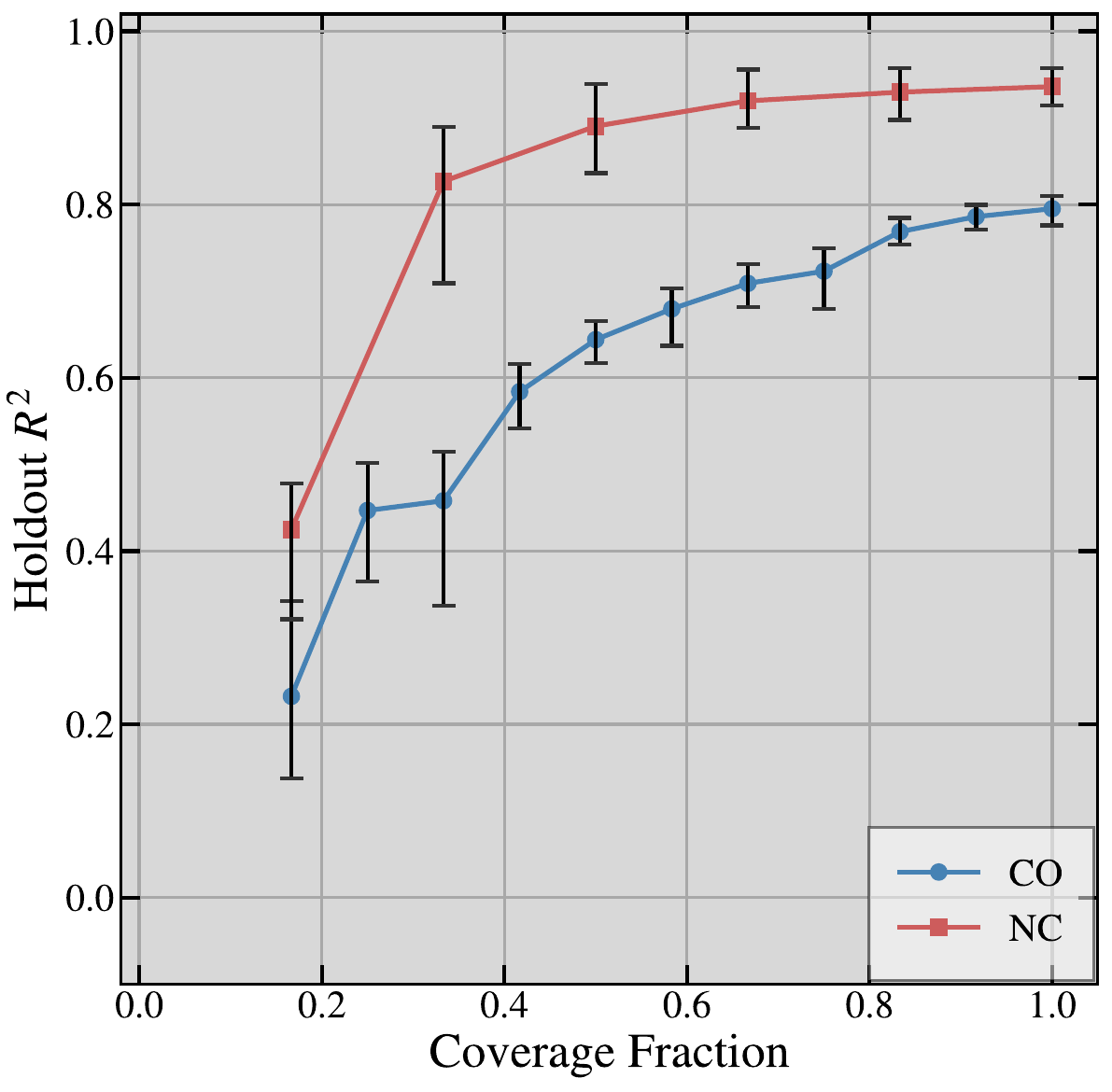}
        \caption{Holdout $R^{2}$ vs.\ TPP coverage.}
        \label{fig:06-convergence}
    \end{subfigure}%
    \hfill
    \begin{subfigure}[b]{0.30\textwidth}
        \centering
        \includegraphics[width=\linewidth]{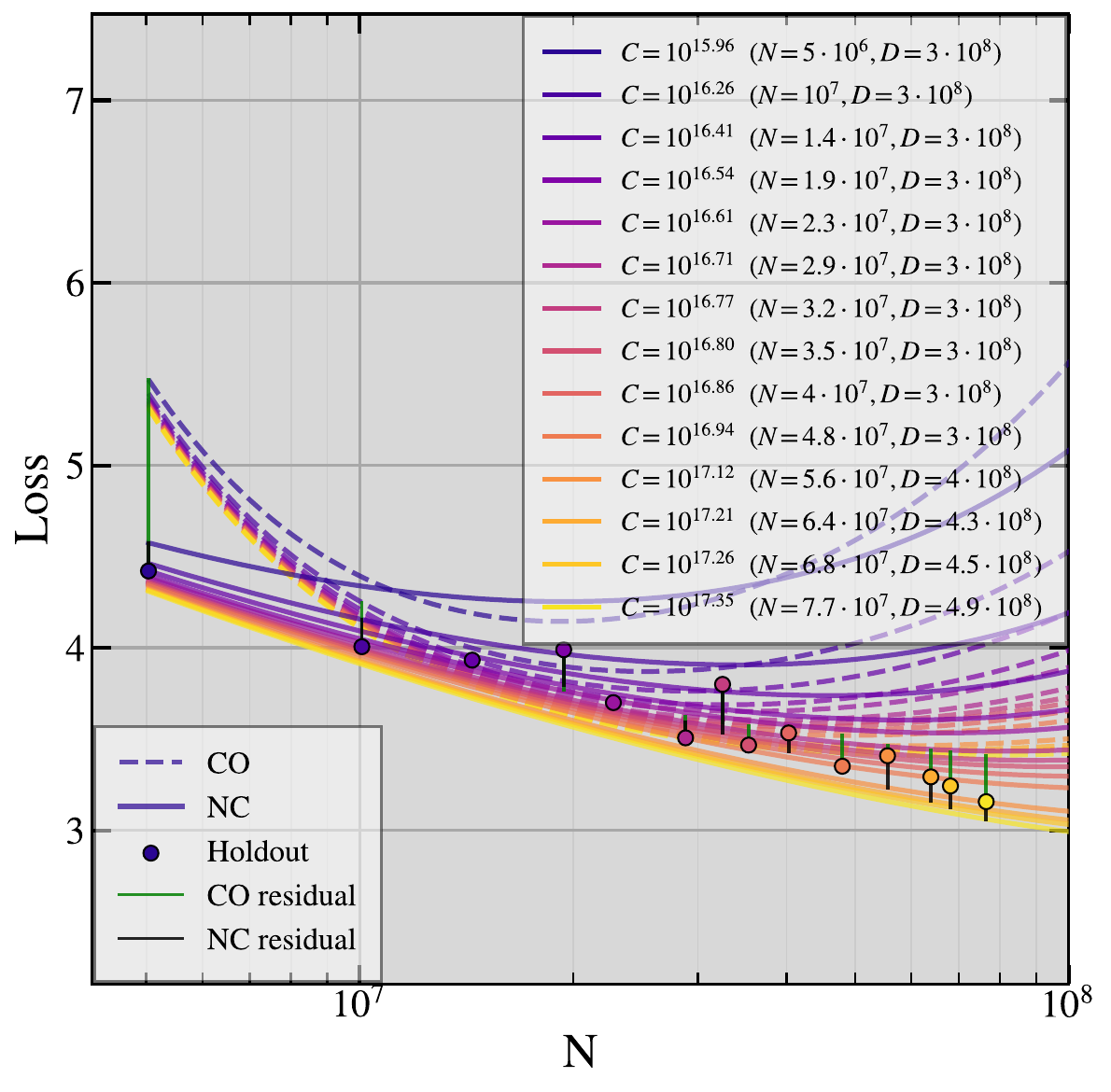}
        \caption{Induced isoFLOP curves.}
        \label{fig:06-isoflop}
    \end{subfigure}
    \caption{Chinchilla law on C4, first epoch: CO and NC fits, with the panels of Figure~\ref{fig:select-plots}.}
    \label{fig:dump-06}
\end{figure*}

\begin{figure*}[!ht]
    \centering
    \begin{subfigure}[b]{0.37\textwidth}
        \centering
        \includegraphics[width=\linewidth]{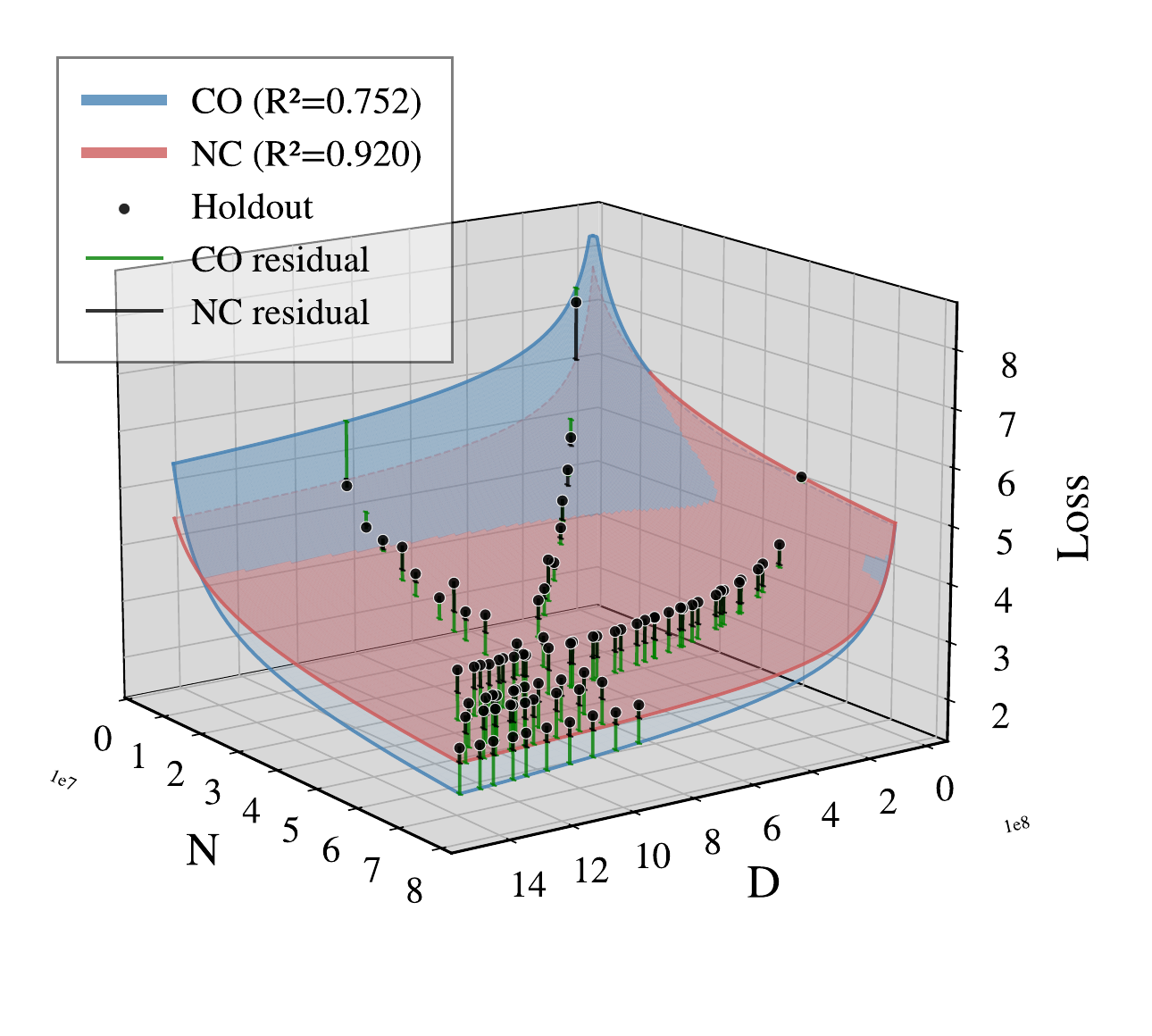}
        \caption{Loss surface.}
        \label{fig:app-select-surface}
    \end{subfigure}%
    \hfill
    \begin{subfigure}[b]{0.30\textwidth}
        \centering
        \includegraphics[width=\linewidth]{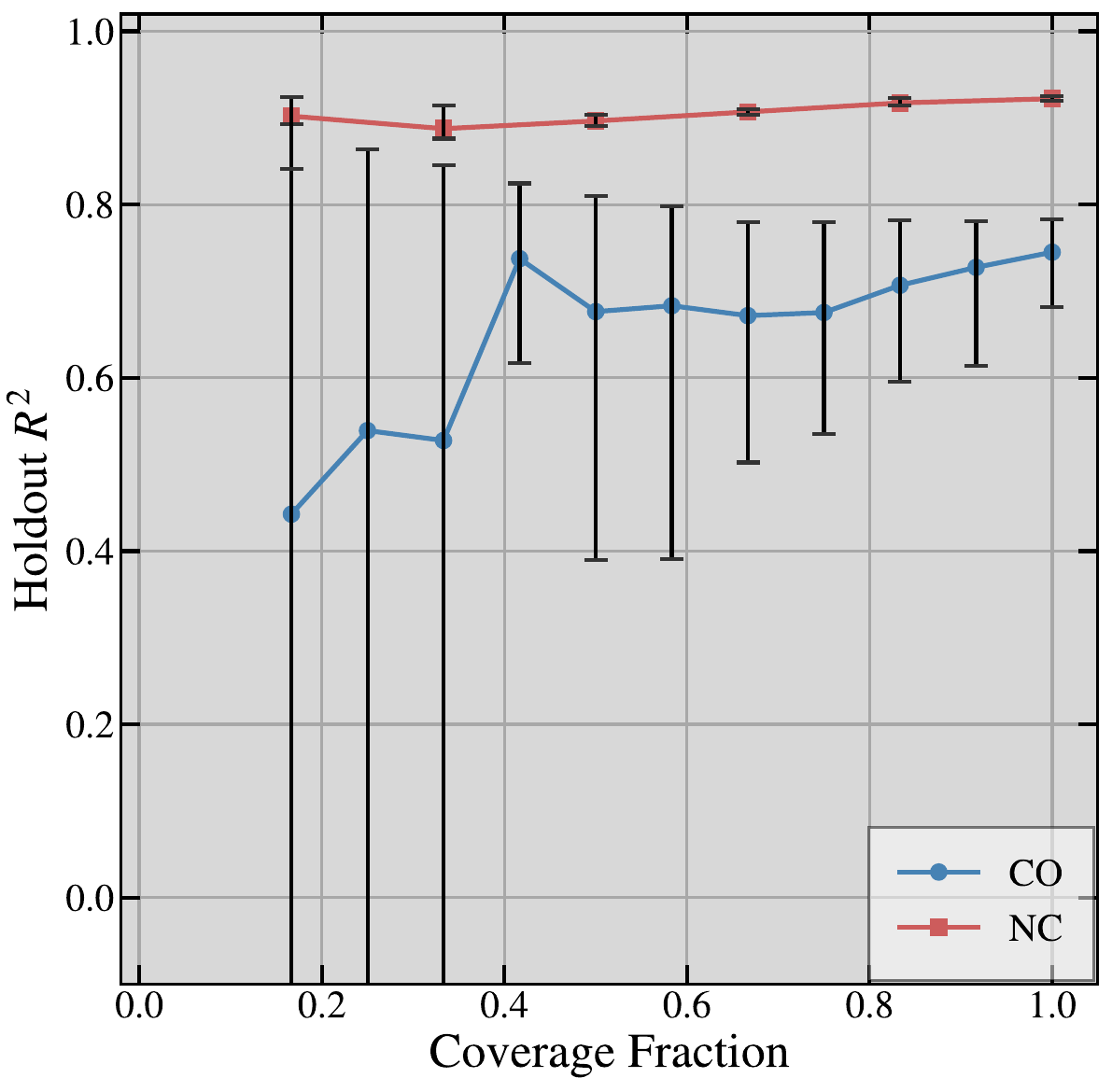}
        \caption{Holdout $R^{2}$ vs.\ TPP coverage.}
        \label{fig:app-select-convergence}
    \end{subfigure}%
    \hfill
    \begin{subfigure}[b]{0.30\textwidth}
        \centering
        \includegraphics[width=\linewidth]{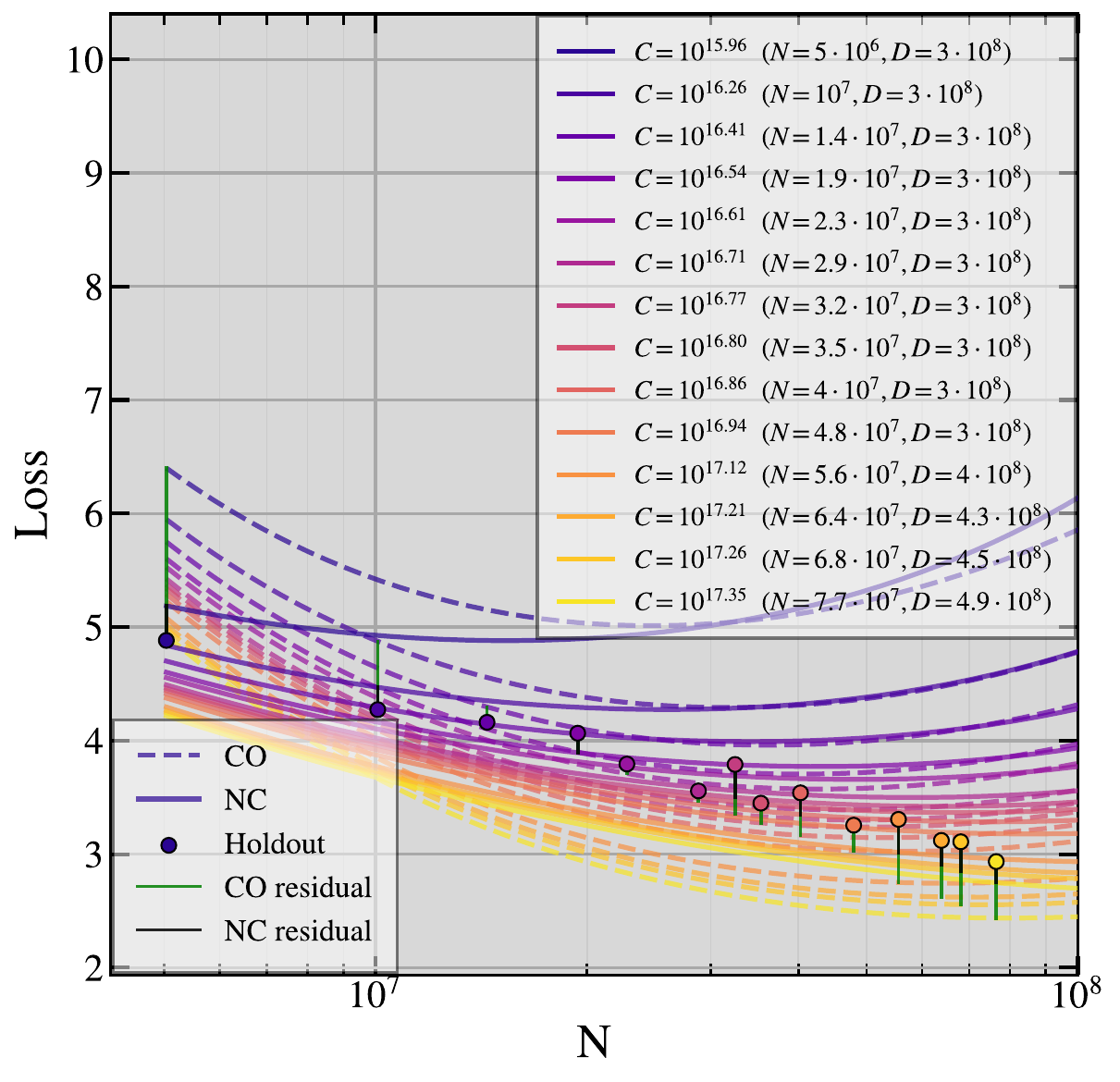}
        \caption{Induced isoFLOP curves.}
        \label{fig:app-select-isoflop}
    \end{subfigure}
    \caption{Droppo-Elibol law on Wikipedia, first epoch: CO and NC fits, with the panels of Figure~\ref{fig:select-plots}.}
    \label{fig:app-select-surfaces}
\end{figure*}

\begin{figure*}[!ht]
    \centering
    \begin{subfigure}[b]{0.37\textwidth}
        \centering
        \includegraphics[width=\linewidth]{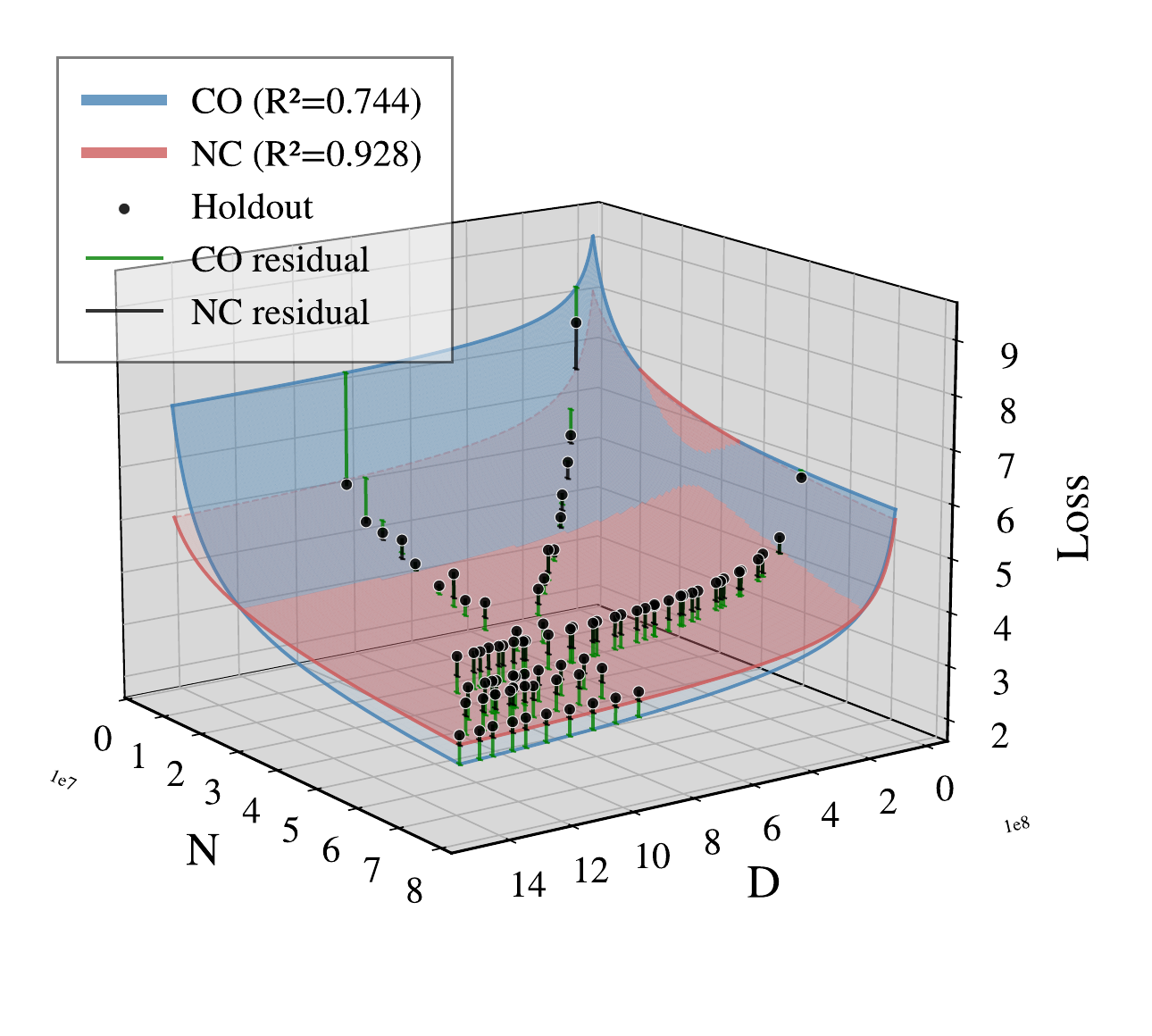}
        \caption{Loss surface.}
        \label{fig:08-surface}
    \end{subfigure}%
    \hfill
    \begin{subfigure}[b]{0.30\textwidth}
        \centering
        \includegraphics[width=\linewidth]{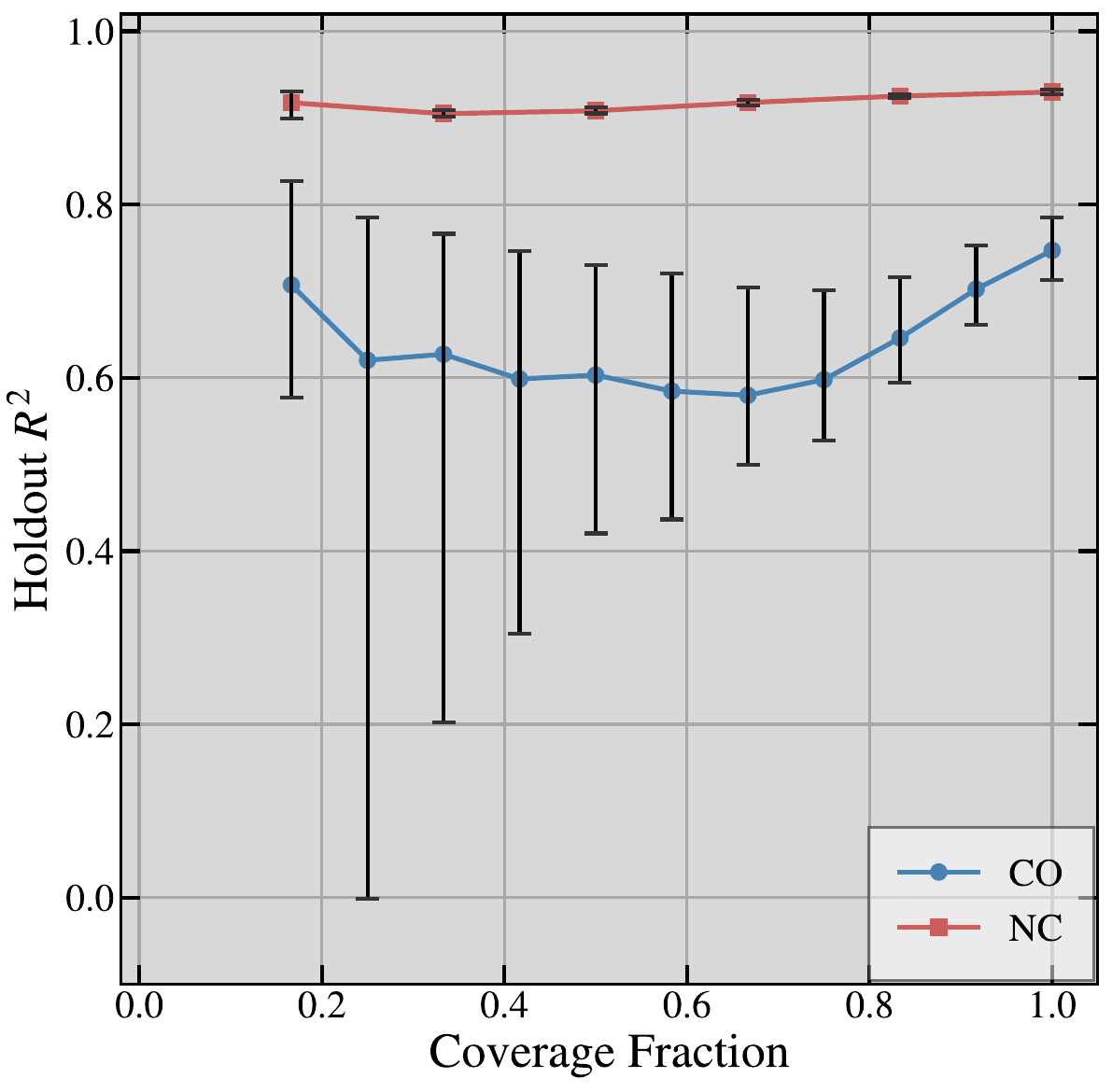}
        \caption{Holdout $R^{2}$ vs.\ TPP coverage.}
        \label{fig:08-convergence}
    \end{subfigure}%
    \hfill
    \begin{subfigure}[b]{0.30\textwidth}
        \centering
        \includegraphics[width=\linewidth]{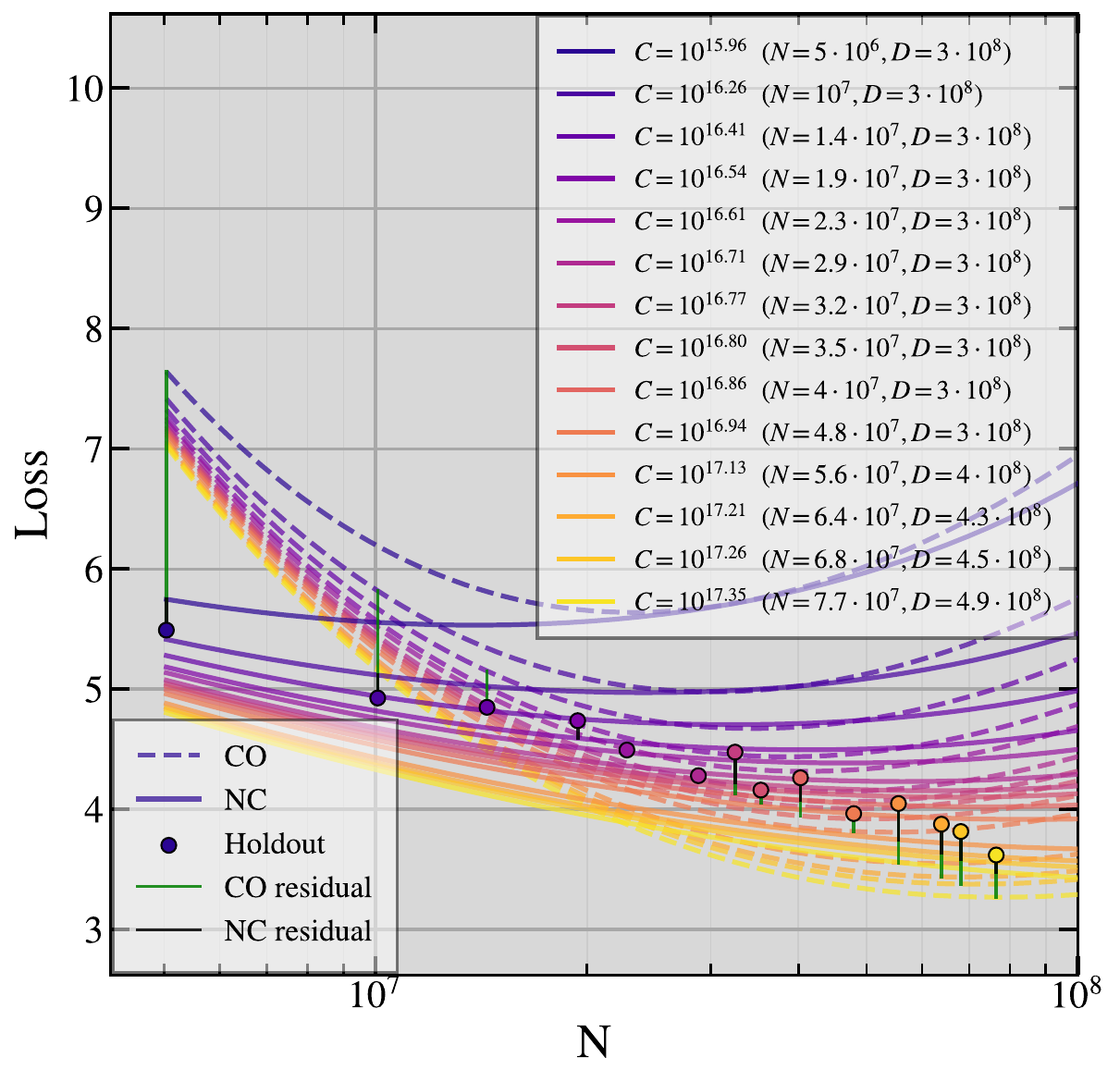}
        \caption{Induced isoFLOP curves.}
        \label{fig:08-isoflop}
    \end{subfigure}
    \caption{Droppo-Elibol law on RedPajama, first epoch: CO and NC fits, with the panels of Figure~\ref{fig:select-plots}.}
    \label{fig:dump-08}
\end{figure*}

\begin{figure*}[!ht]
    \centering
    \begin{subfigure}[b]{0.37\textwidth}
        \centering
        \includegraphics[width=\linewidth]{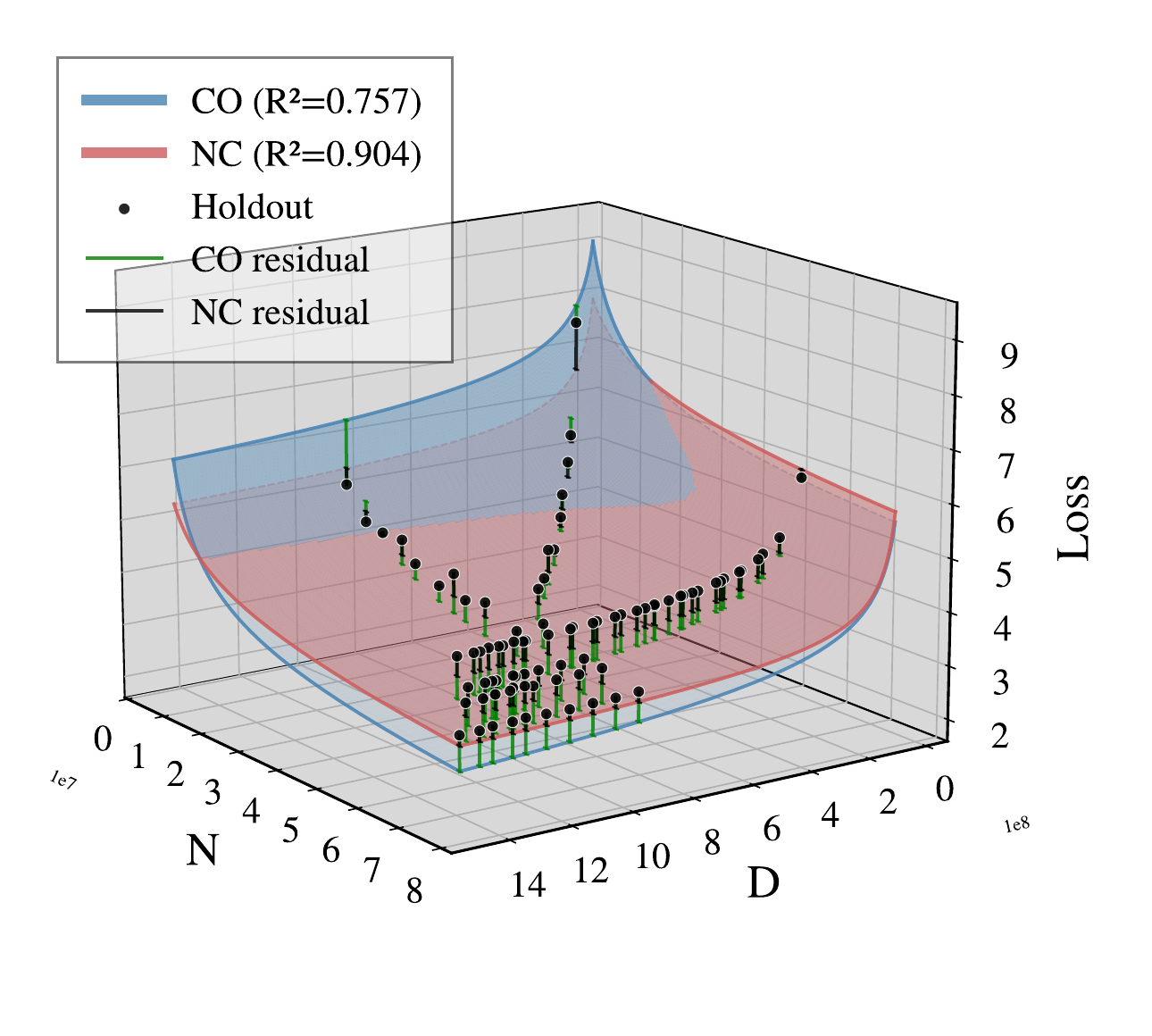}
        \caption{Loss surface.}
        \label{fig:09-surface}
    \end{subfigure}%
    \hfill
    \begin{subfigure}[b]{0.30\textwidth}
        \centering
        \includegraphics[width=\linewidth]{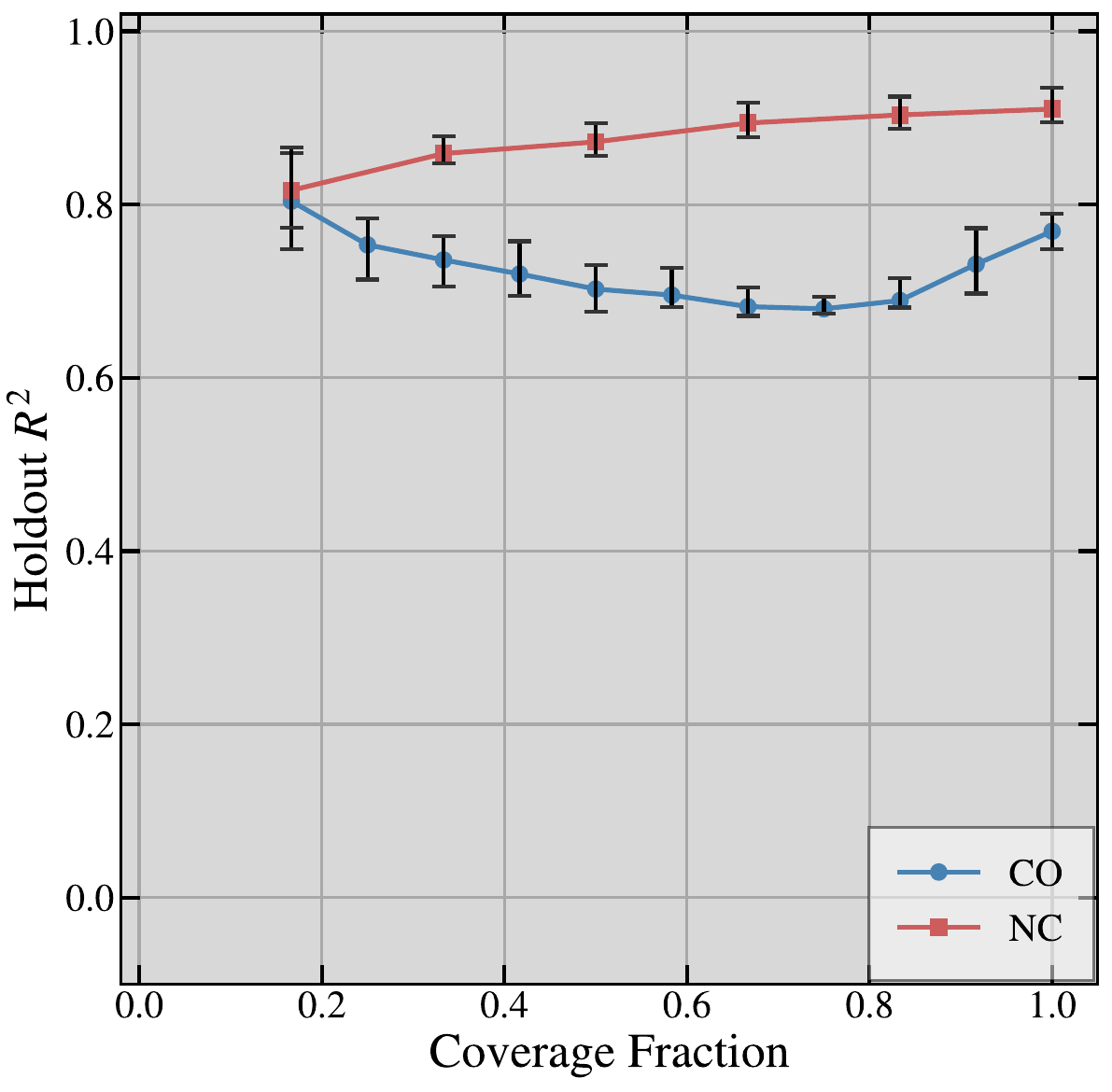}
        \caption{Holdout $R^{2}$ vs.\ TPP coverage.}
        \label{fig:09-convergence}
    \end{subfigure}%
    \hfill
    \begin{subfigure}[b]{0.30\textwidth}
        \centering
        \includegraphics[width=\linewidth]{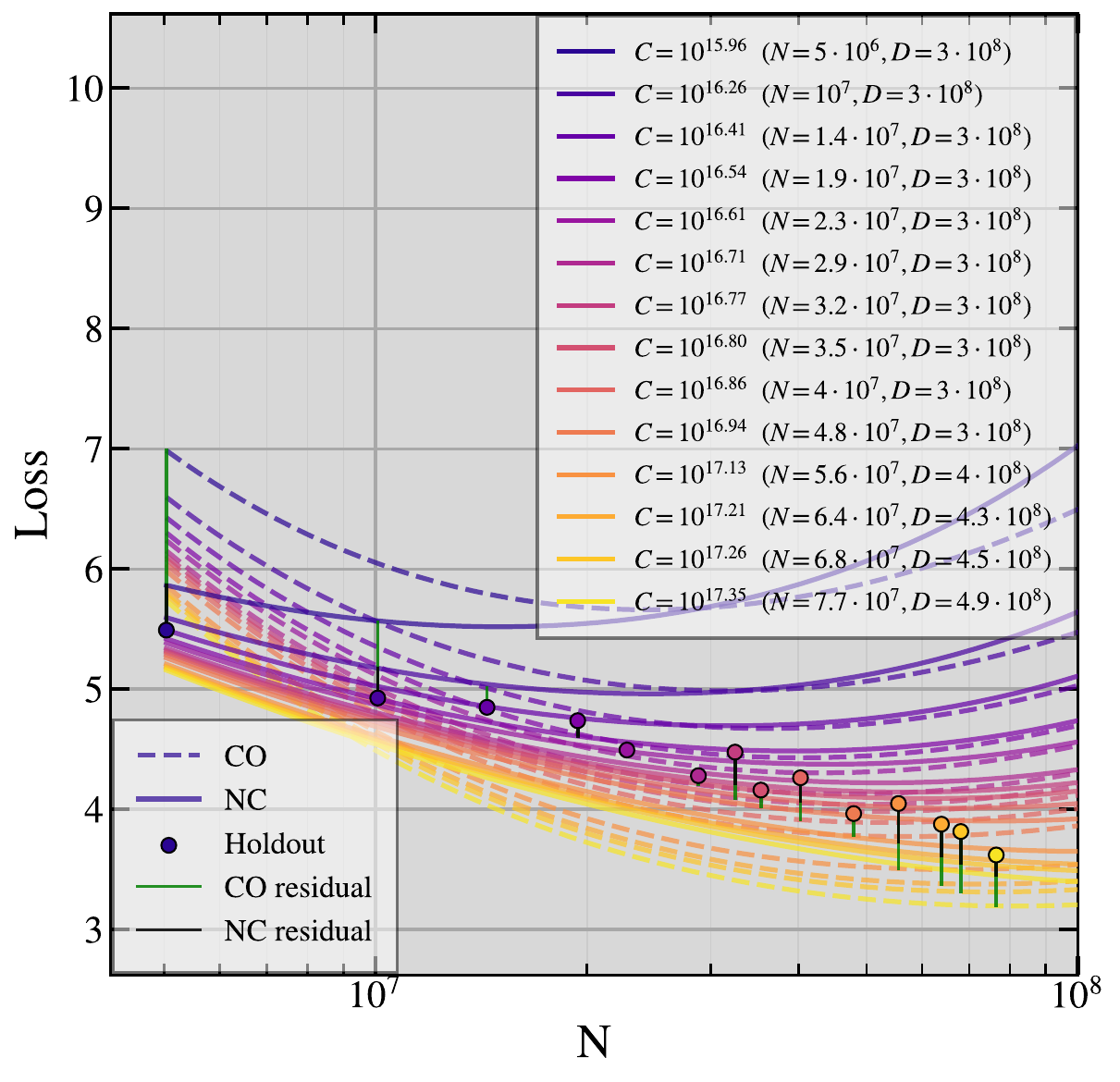}
        \caption{Induced isoFLOP curves.}
        \label{fig:09-isoflop}
    \end{subfigure}
    \caption{Chinchilla law on RedPajama, first epoch: CO and NC fits, with the panels of Figure~\ref{fig:select-plots}.}
    \label{fig:dump-09}
\end{figure*}

\begin{figure*}[!ht]
    \centering
    \begin{subfigure}[b]{0.37\textwidth}
        \centering
        \includegraphics[width=\linewidth]{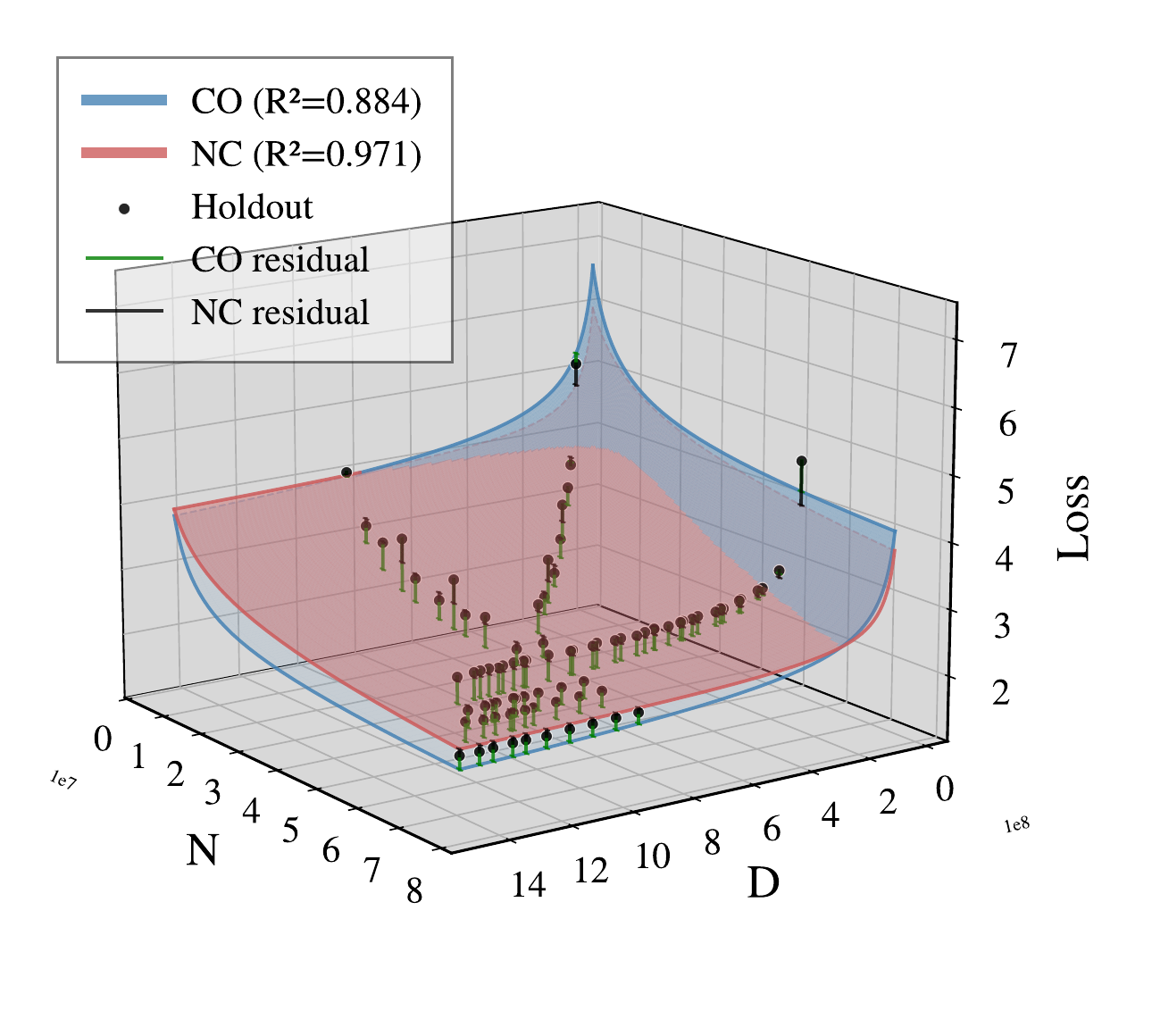}
        \caption{Loss surface.}
        \label{fig:10-surface}
    \end{subfigure}%
    \hfill
    \begin{subfigure}[b]{0.30\textwidth}
        \centering
        \includegraphics[width=\linewidth]{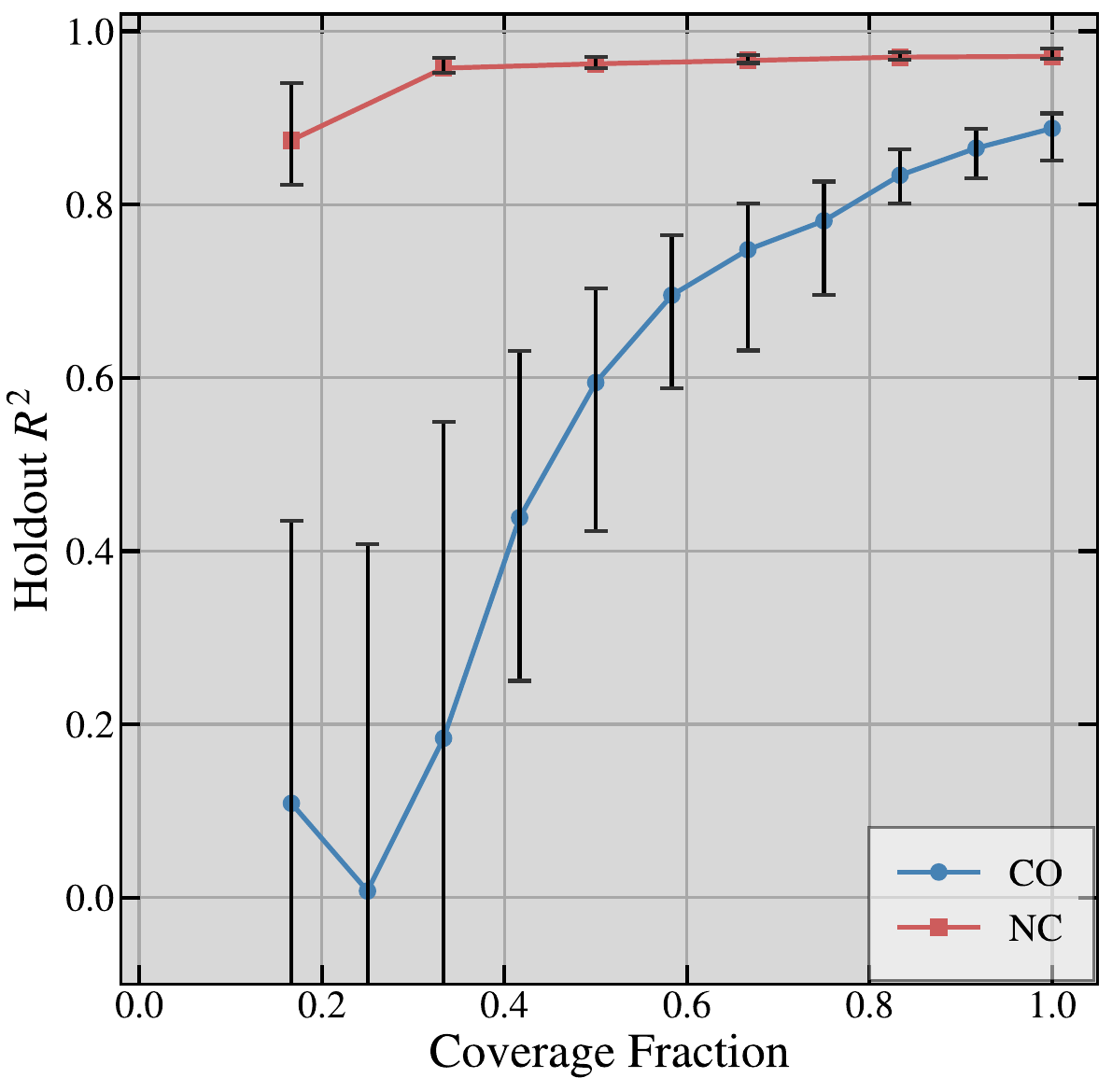}
        \caption{Holdout $R^{2}$ vs.\ TPP coverage.}
        \label{fig:10-convergence}
    \end{subfigure}%
    \hfill
    \begin{subfigure}[b]{0.30\textwidth}
        \centering
        \includegraphics[width=\linewidth]{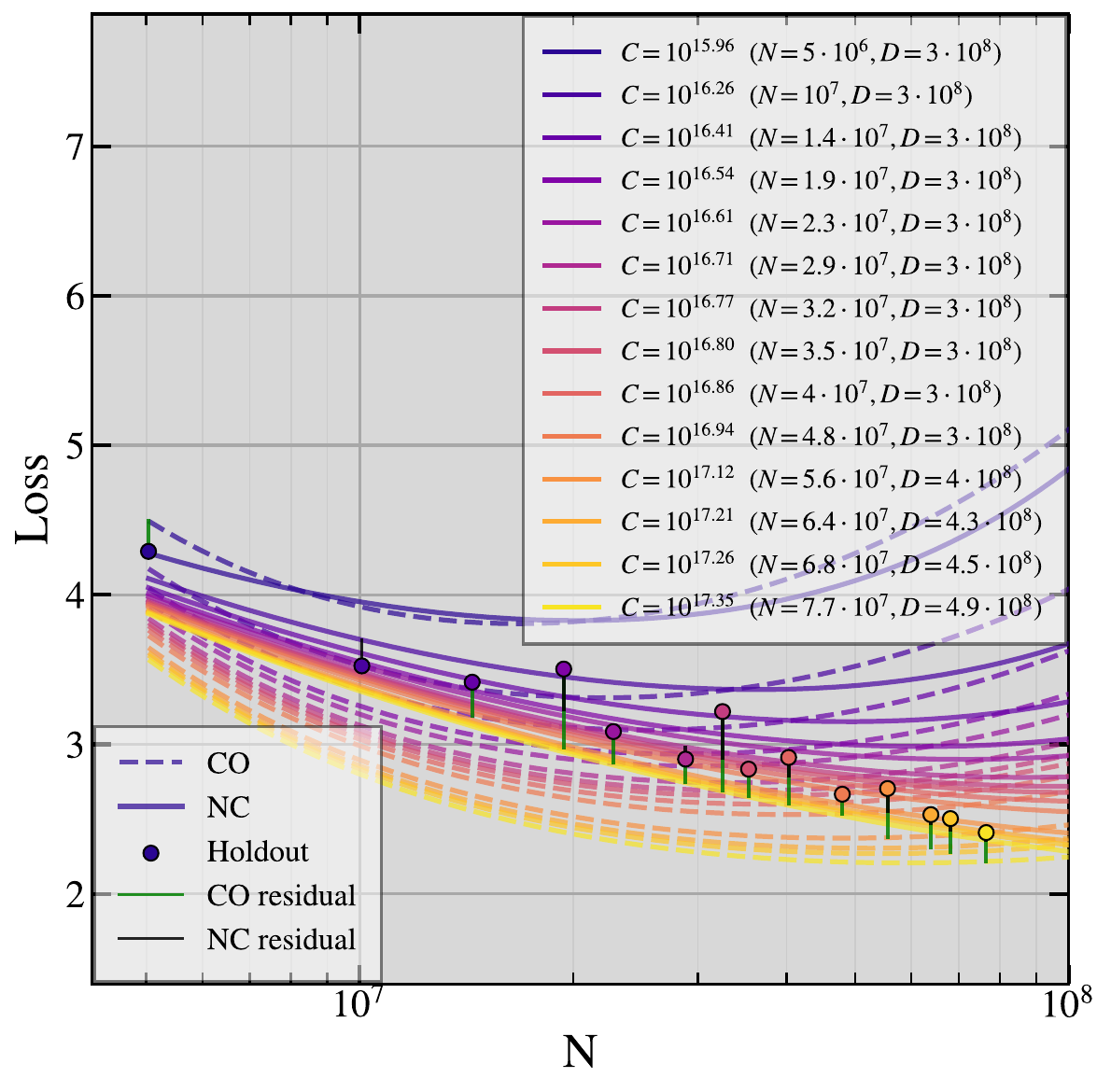}
        \caption{Induced isoFLOP curves.}
        \label{fig:10-isoflop}
    \end{subfigure}
    \caption{Droppo-Elibol law on Cosmopedia, first epoch: CO and NC fits, with the panels of Figure~\ref{fig:select-plots}.}
    \label{fig:dump-10}
\end{figure*}

\begin{figure*}[!ht]
    \centering
    \begin{subfigure}[b]{0.37\textwidth}
        \centering
        \includegraphics[width=\linewidth]{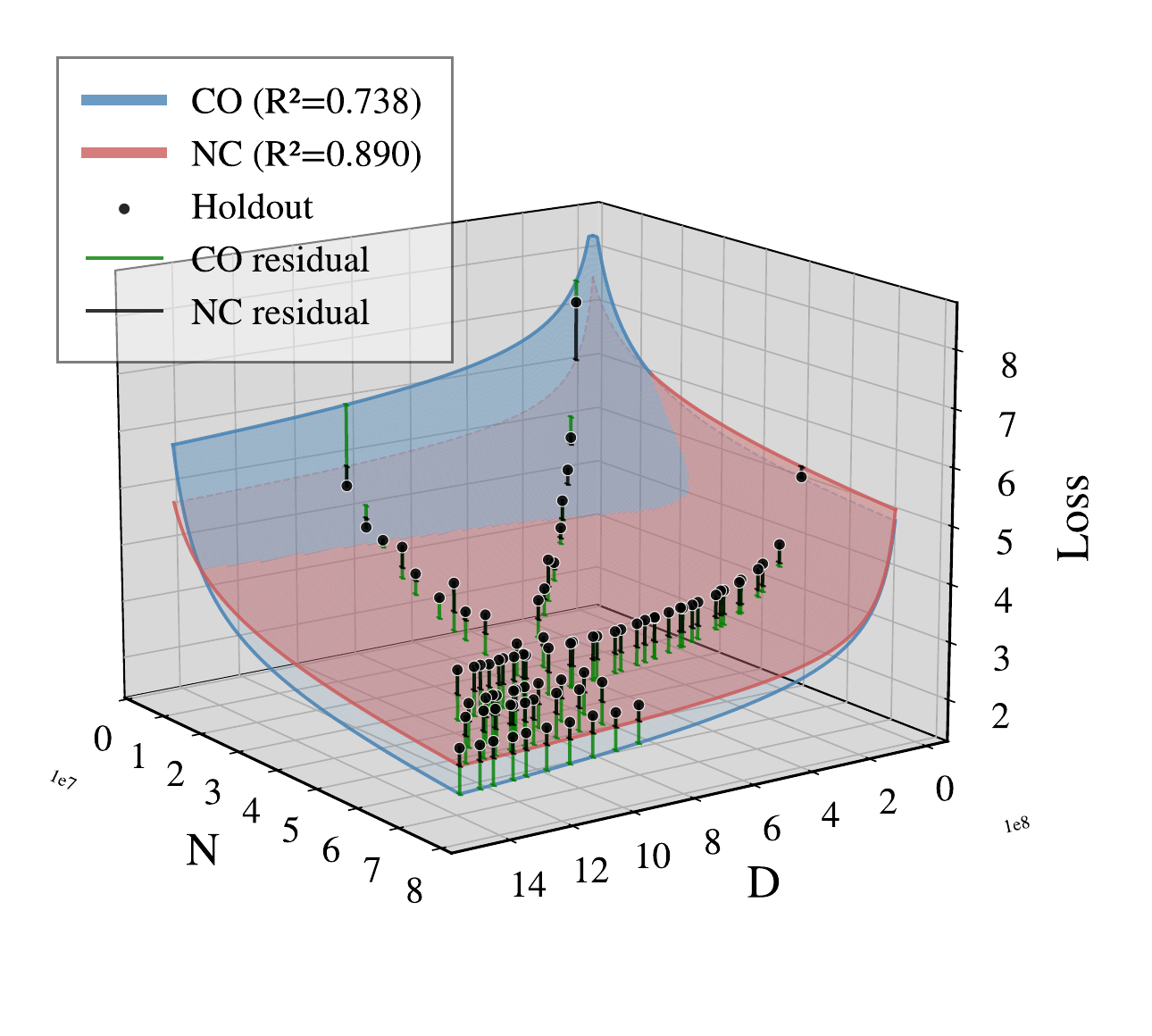}
        \caption{Loss surface.}
        \label{fig:11-surface}
    \end{subfigure}%
    \hfill
    \begin{subfigure}[b]{0.30\textwidth}
        \centering
        \includegraphics[width=\linewidth]{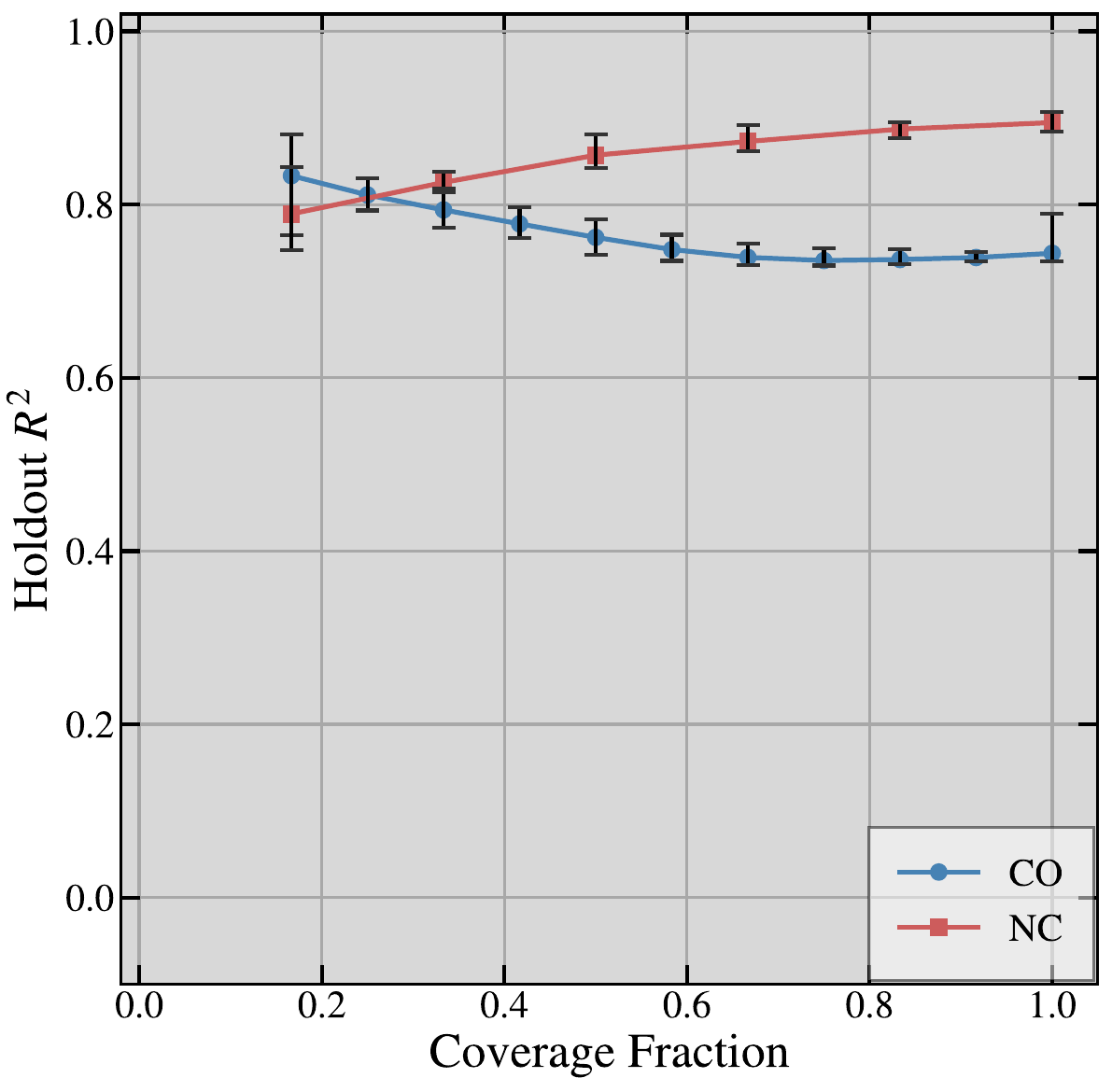}
        \caption{Holdout $R^{2}$ vs.\ TPP coverage.}
        \label{fig:11-convergence}
    \end{subfigure}%
    \hfill
    \begin{subfigure}[b]{0.30\textwidth}
        \centering
        \includegraphics[width=\linewidth]{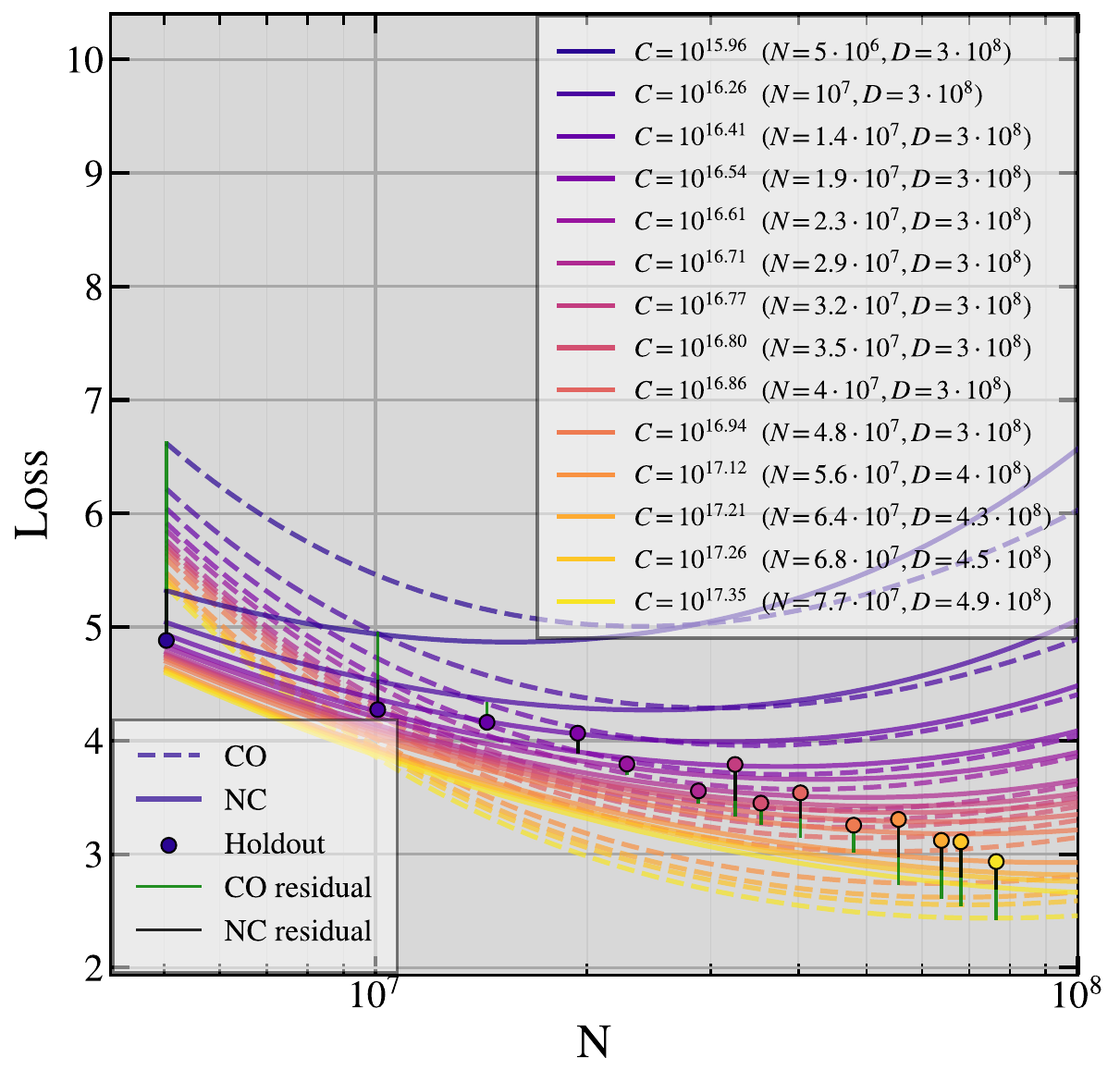}
        \caption{Induced isoFLOP curves.}
        \label{fig:11-isoflop}
    \end{subfigure}
    \caption{Chinchilla law on Wikipedia, first epoch: CO and NC fits, with the panels of Figure~\ref{fig:select-plots}.}
    \label{fig:dump-11}
\end{figure*}

\begin{figure*}[!ht]
    \centering
    \begin{subfigure}[b]{0.37\textwidth}
        \centering
        \includegraphics[width=\linewidth]{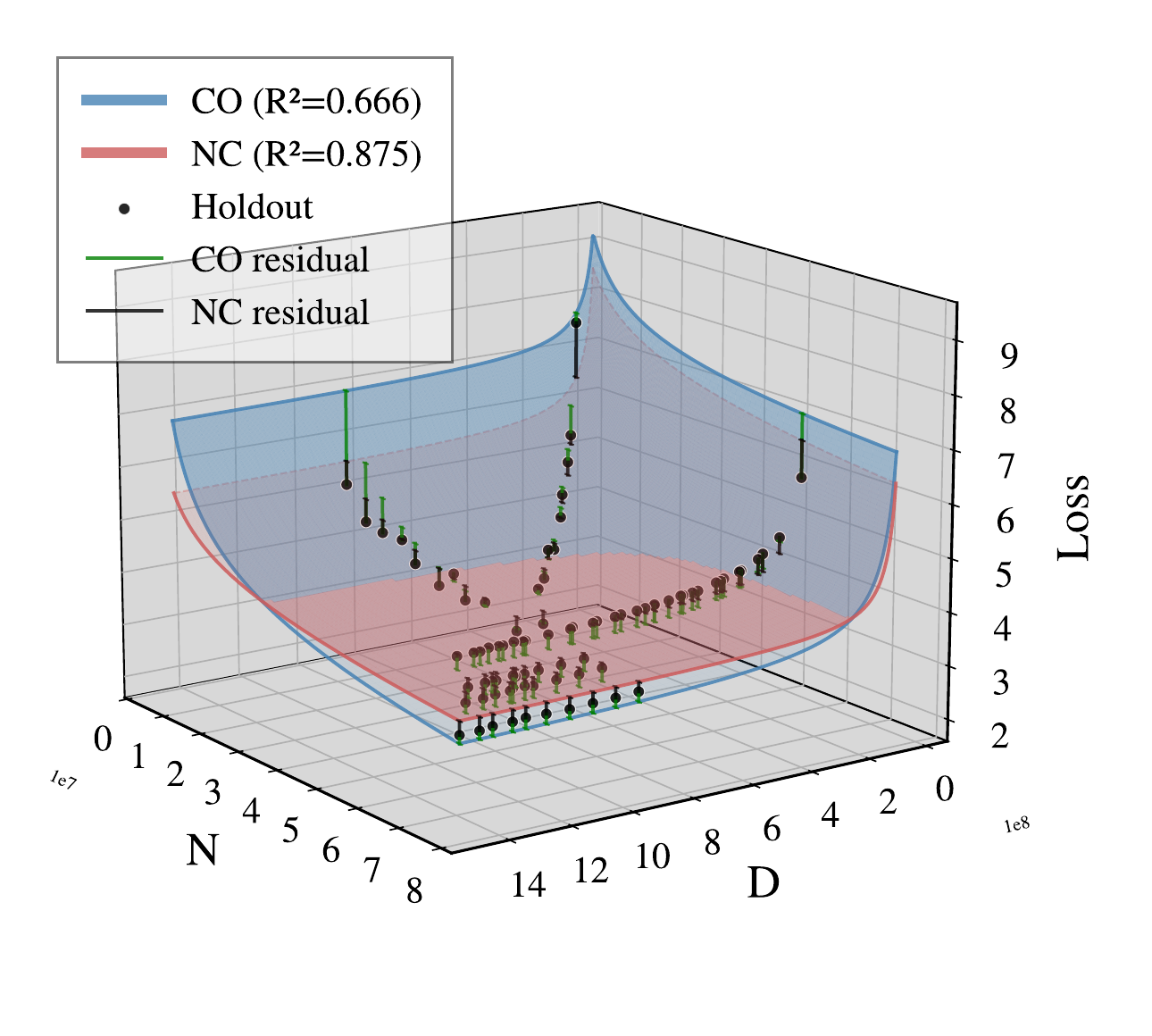}
        \caption{Loss surface.}
        \label{fig:12-surface}
    \end{subfigure}%
    \hfill
    \begin{subfigure}[b]{0.30\textwidth}
        \centering
        \includegraphics[width=\linewidth]{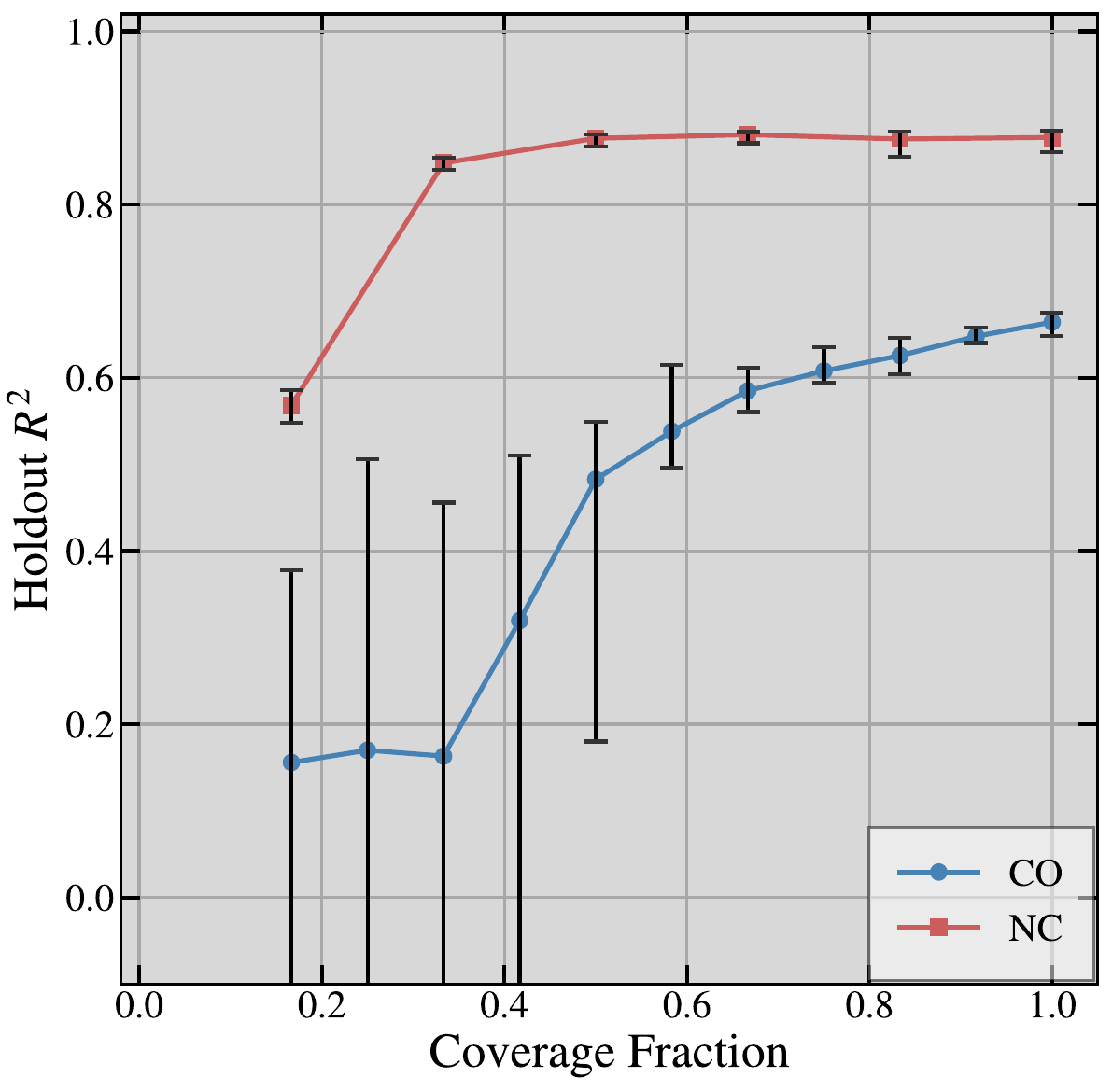}
        \caption{Holdout $R^{2}$ vs.\ TPP coverage.}
        \label{fig:12-convergence}
    \end{subfigure}%
    \hfill
    \begin{subfigure}[b]{0.30\textwidth}
        \centering
        \includegraphics[width=\linewidth]{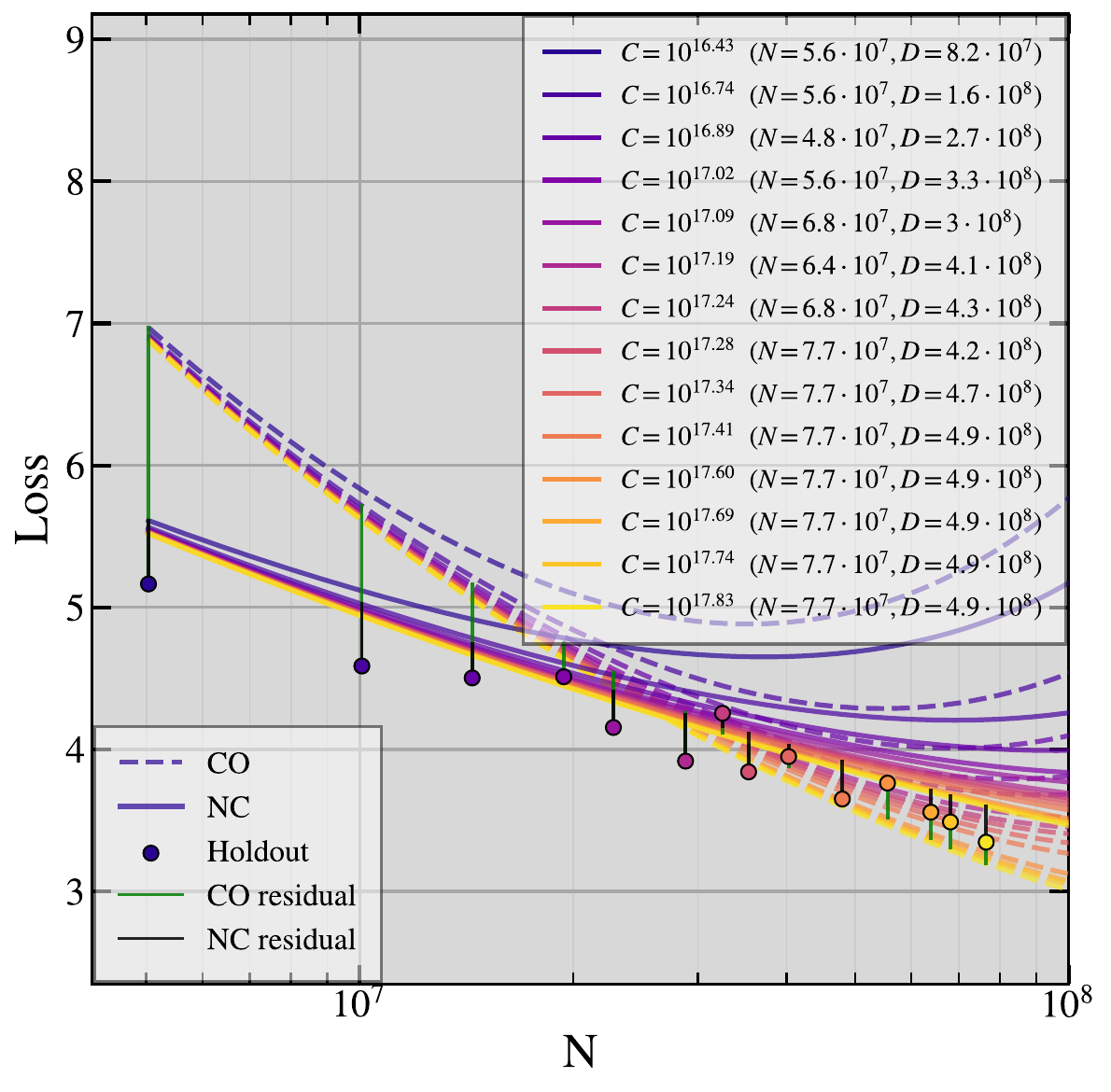}
        \caption{Induced isoFLOP curves.}
        \label{fig:12-isoflop}
    \end{subfigure}
    \caption{Kaplan law on RedPajama, final epoch: CO and NC fits, with the panels of Figure~\ref{fig:select-plots}.}
    \label{fig:dump-12}
\end{figure*}

\begin{figure*}[!ht]
    \centering
    \begin{subfigure}[b]{0.37\textwidth}
        \centering
        \includegraphics[width=\linewidth]{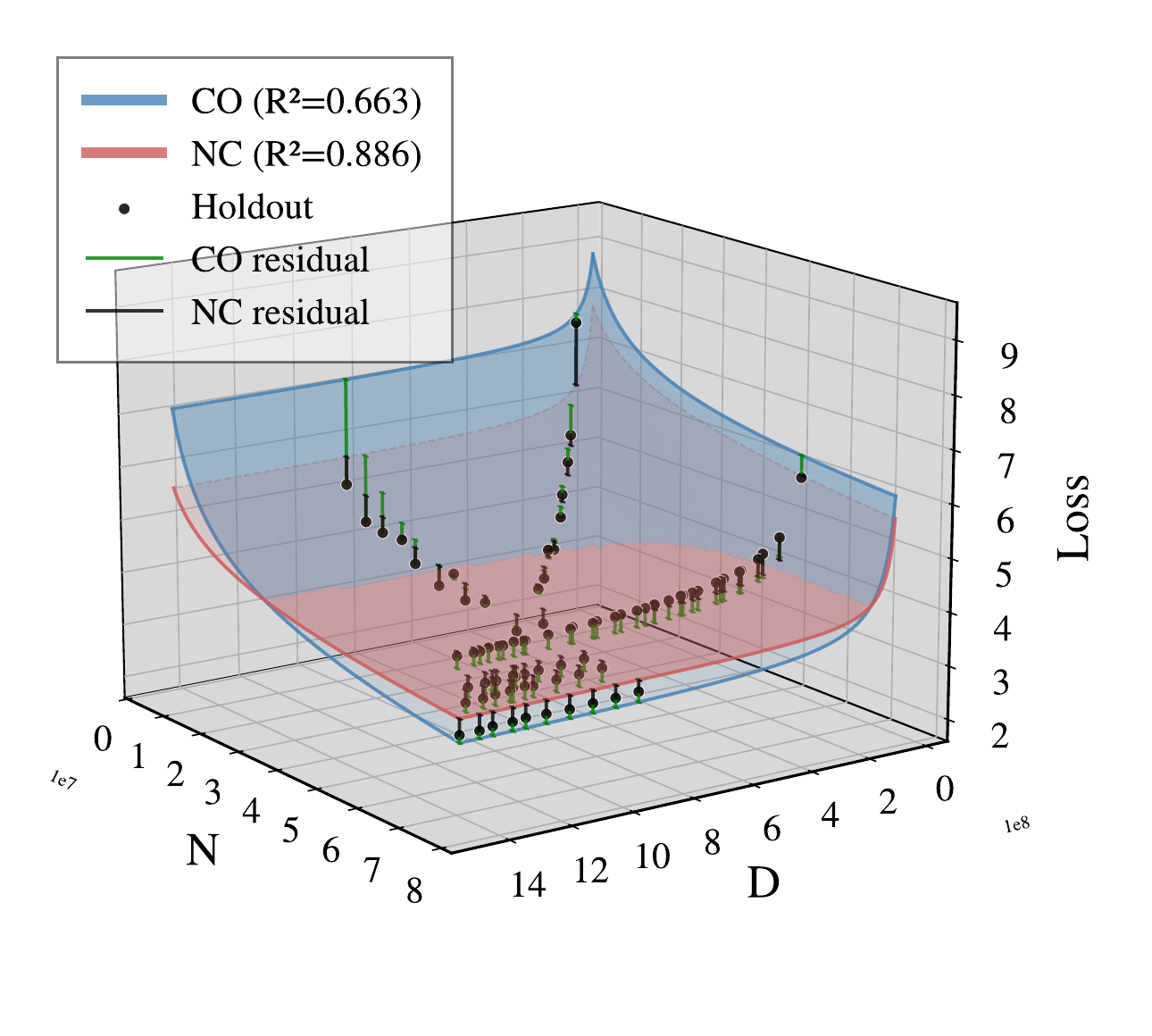}
        \caption{Loss surface.}
        \label{fig:13-surface}
    \end{subfigure}%
    \hfill
    \begin{subfigure}[b]{0.30\textwidth}
        \centering
        \includegraphics[width=\linewidth]{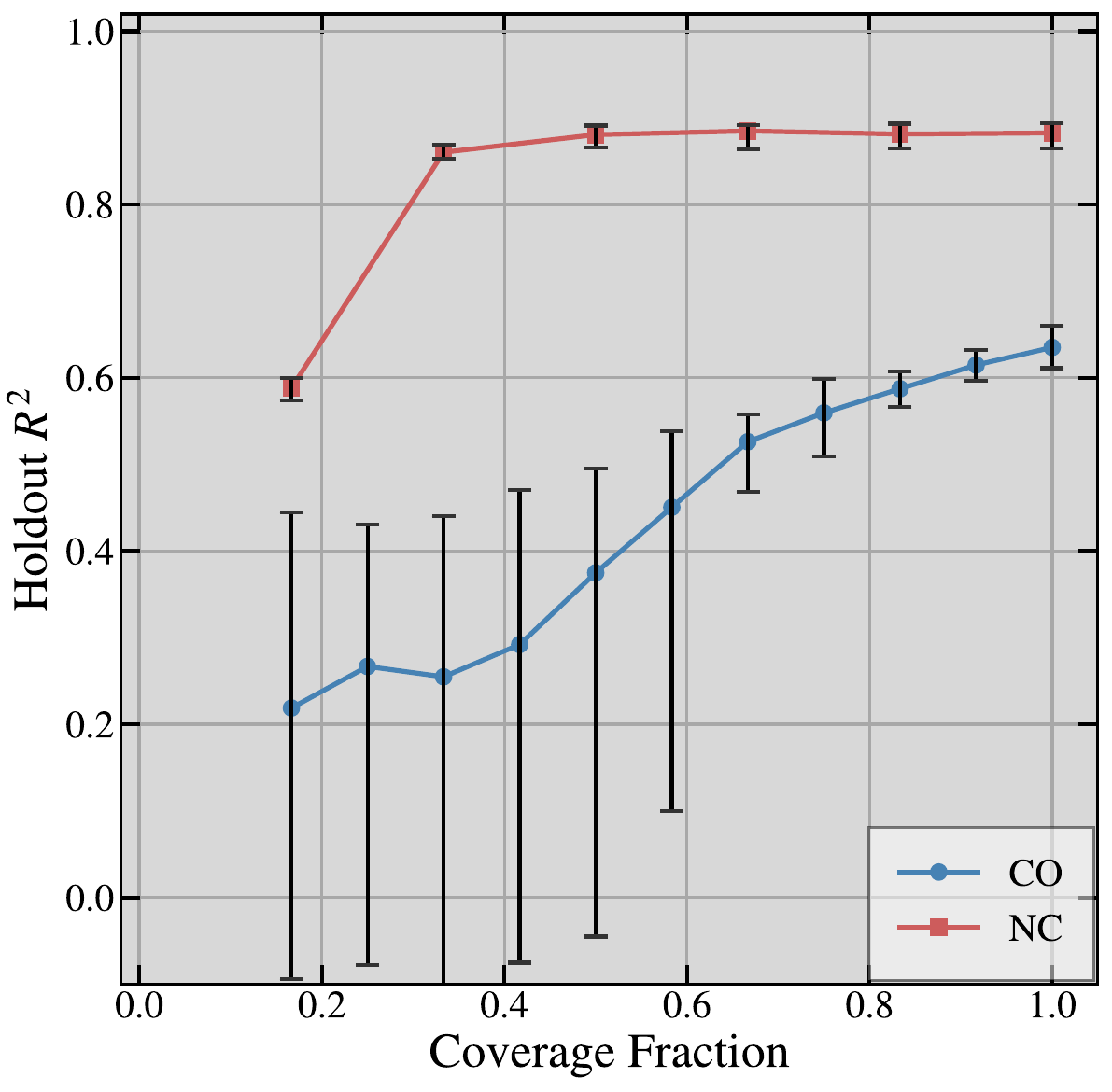}
        \caption{Holdout $R^{2}$ vs.\ TPP coverage.}
        \label{fig:13-convergence}
    \end{subfigure}%
    \hfill
    \begin{subfigure}[b]{0.30\textwidth}
        \centering
        \includegraphics[width=\linewidth]{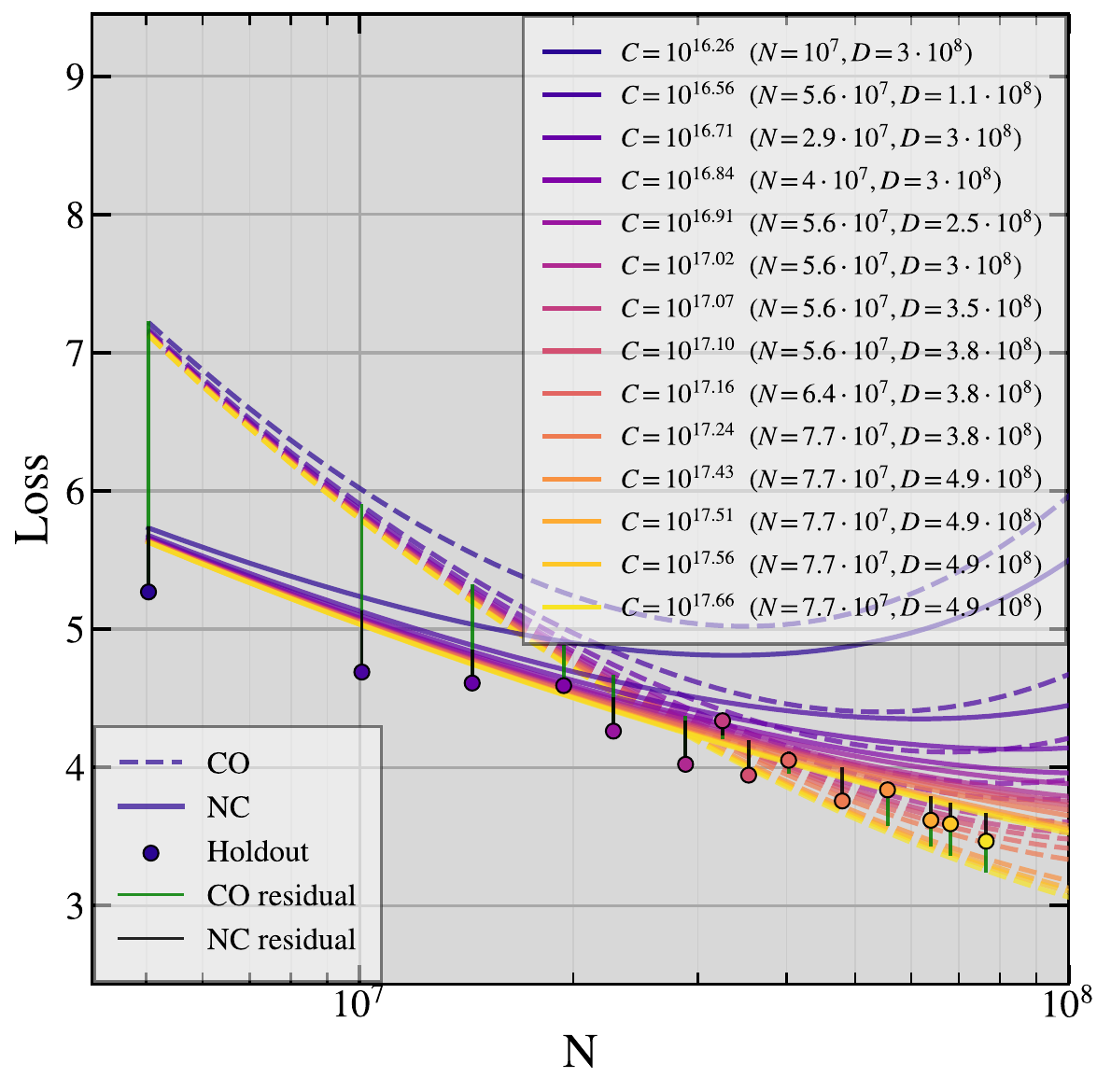}
        \caption{Induced isoFLOP curves.}
        \label{fig:13-isoflop}
    \end{subfigure}
    \caption{Kaplan law on RedPajama, second epoch: CO and NC fits, with the panels of Figure~\ref{fig:select-plots}.}
    \label{fig:dump-13}
\end{figure*}

\begin{figure*}[!ht]
    \centering
    \begin{subfigure}[b]{0.37\textwidth}
        \centering
        \includegraphics[width=\linewidth]{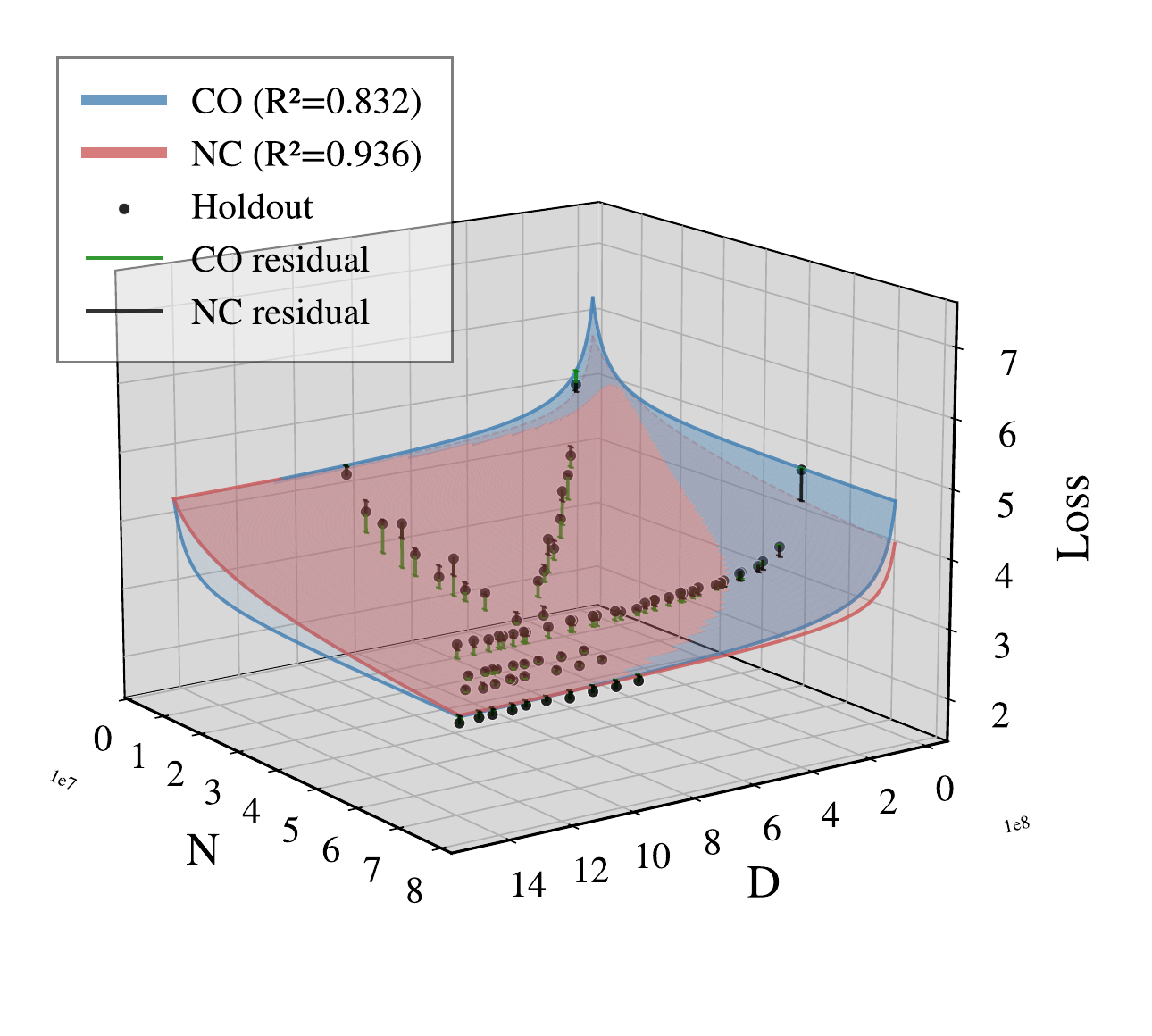}
        \caption{Loss surface.}
        \label{fig:14-surface}
    \end{subfigure}%
    \hfill
    \begin{subfigure}[b]{0.30\textwidth}
        \centering
        \includegraphics[width=\linewidth]{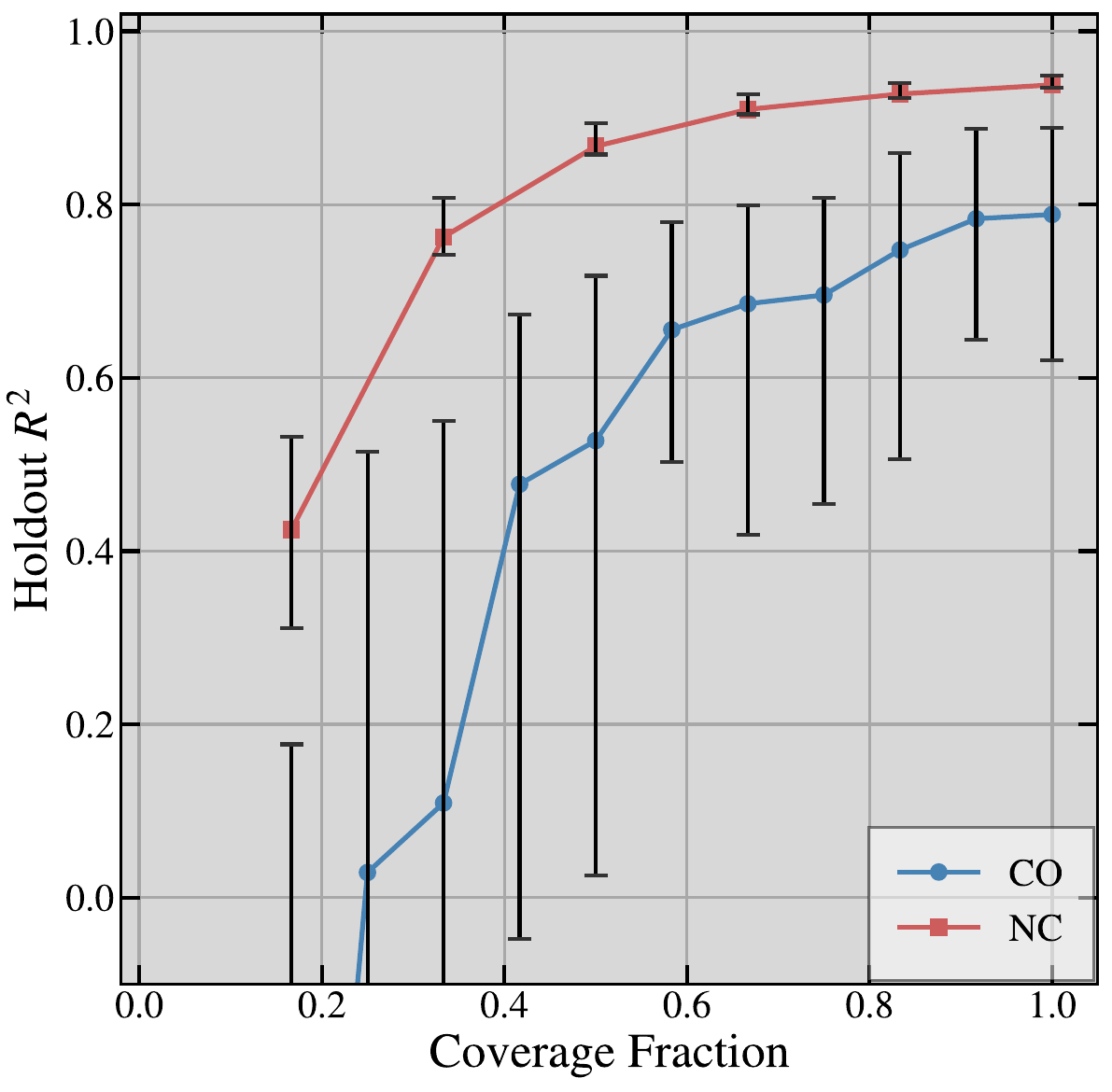}
        \caption{Holdout $R^{2}$ vs.\ TPP coverage.}
        \label{fig:14-convergence}
    \end{subfigure}%
    \hfill
    \begin{subfigure}[b]{0.30\textwidth}
        \centering
        \includegraphics[width=\linewidth]{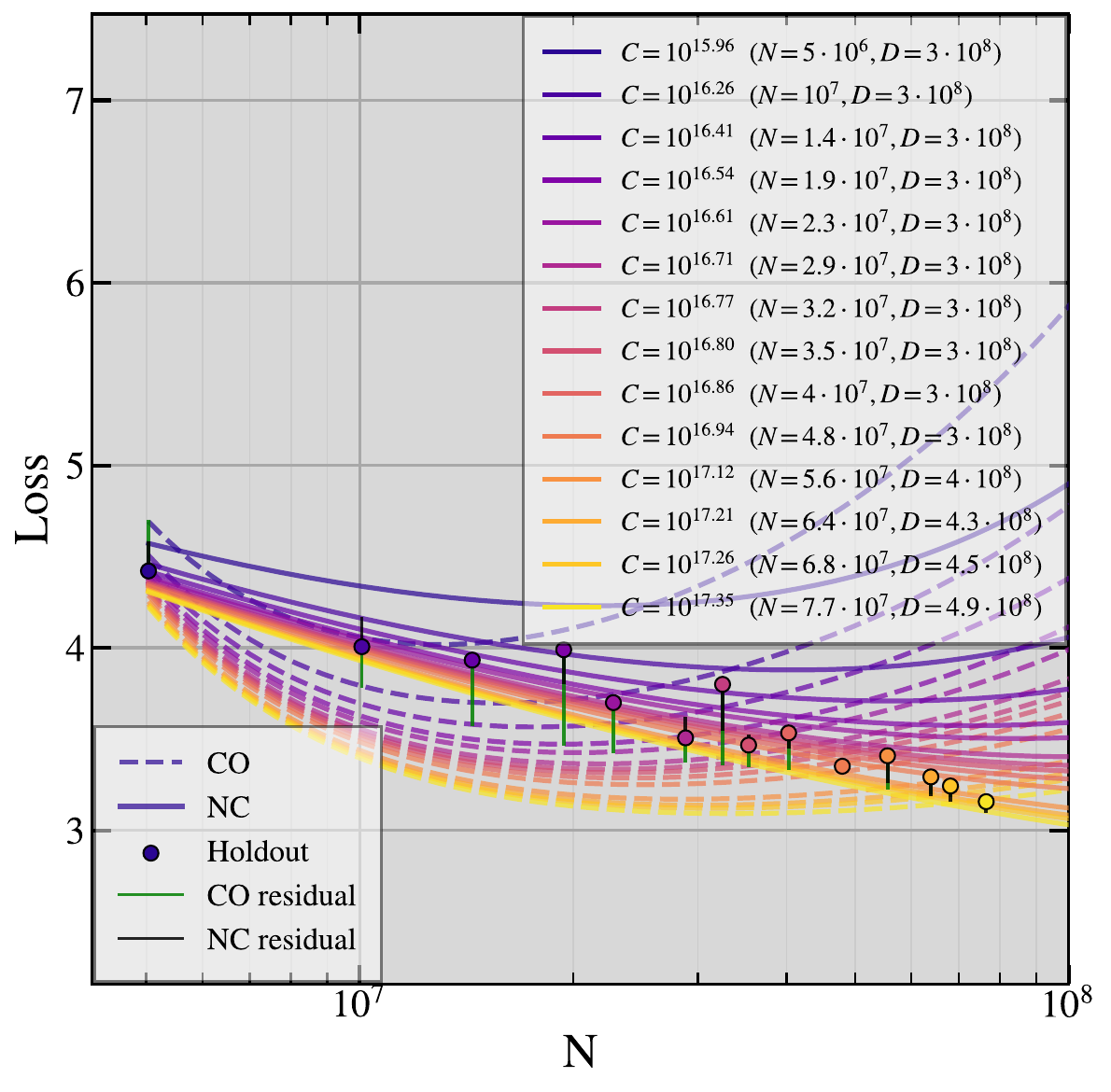}
        \caption{Induced isoFLOP curves.}
        \label{fig:14-isoflop}
    \end{subfigure}
    \caption{Droppo-Elibol law on C4, first epoch: CO and NC fits, with the panels of Figure~\ref{fig:select-plots}.}
    \label{fig:dump-14}
\end{figure*}

\begin{figure*}[!ht]
    \centering
    \begin{subfigure}[b]{0.37\textwidth}
        \centering
        \includegraphics[width=\linewidth]{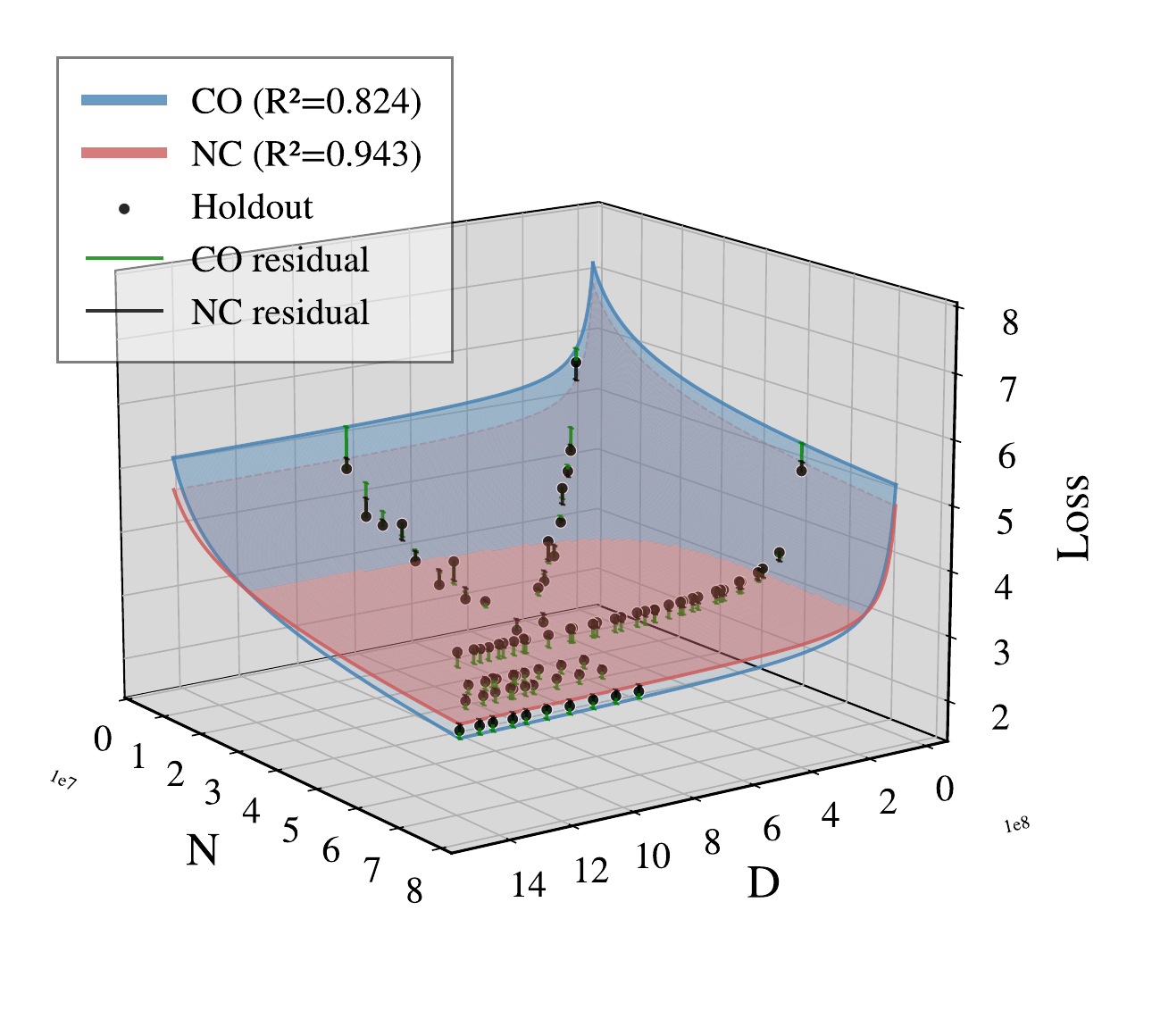}
        \caption{Loss surface.}
        \label{fig:15-surface}
    \end{subfigure}%
    \hfill
    \begin{subfigure}[b]{0.30\textwidth}
        \centering
        \includegraphics[width=\linewidth]{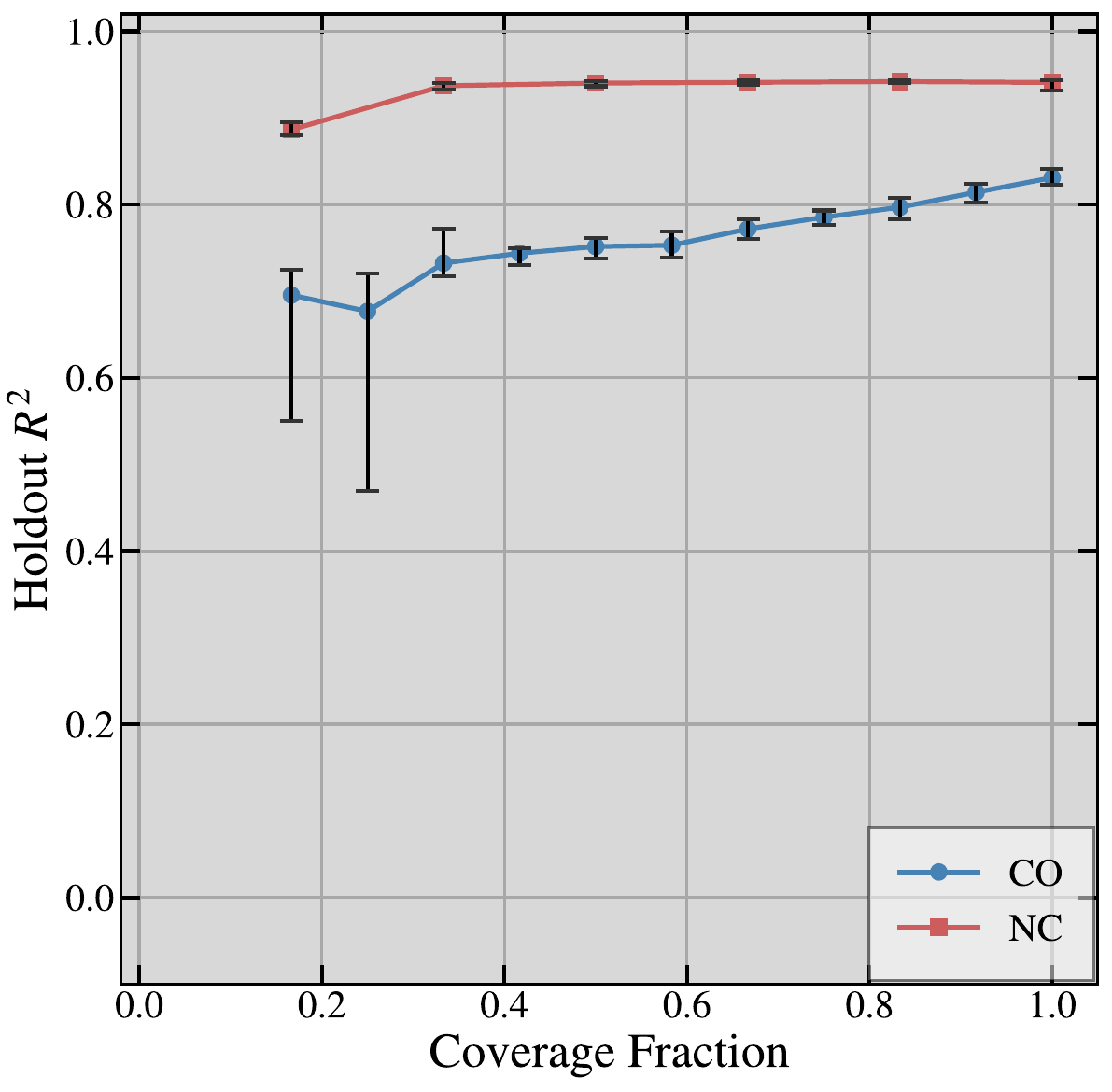}
        \caption{Holdout $R^{2}$ vs.\ TPP coverage.}
        \label{fig:15-convergence}
    \end{subfigure}%
    \hfill
    \begin{subfigure}[b]{0.30\textwidth}
        \centering
        \includegraphics[width=\linewidth]{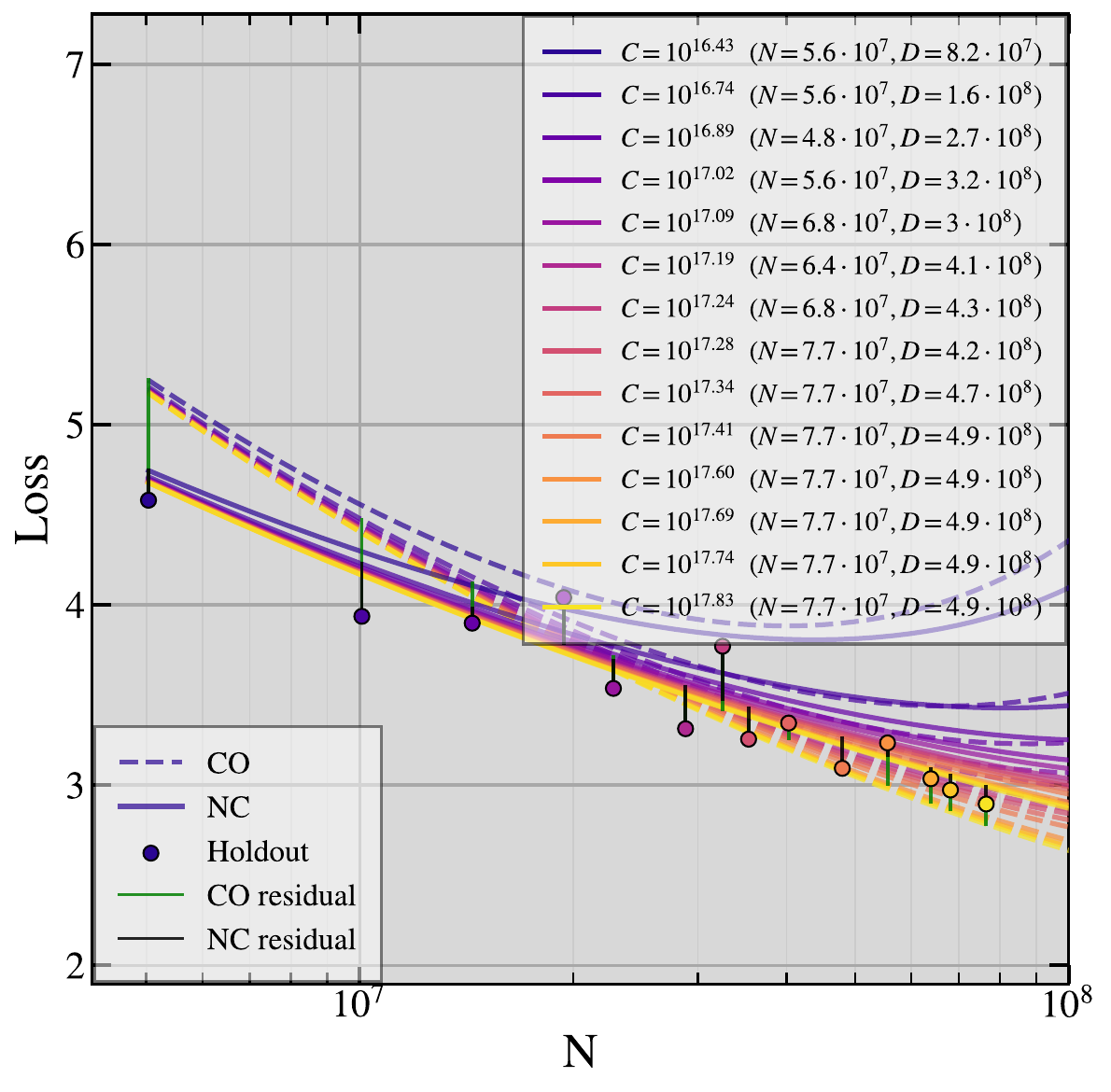}
        \caption{Induced isoFLOP curves.}
        \label{fig:15-isoflop}
    \end{subfigure}
    \caption{Kaplan law on peS2o, final epoch: CO and NC fits, with the panels of Figure~\ref{fig:select-plots}.}
    \label{fig:dump-15}
\end{figure*}

\begin{figure*}[!ht]
    \centering
    \begin{subfigure}[b]{0.37\textwidth}
        \centering
        \includegraphics[width=\linewidth]{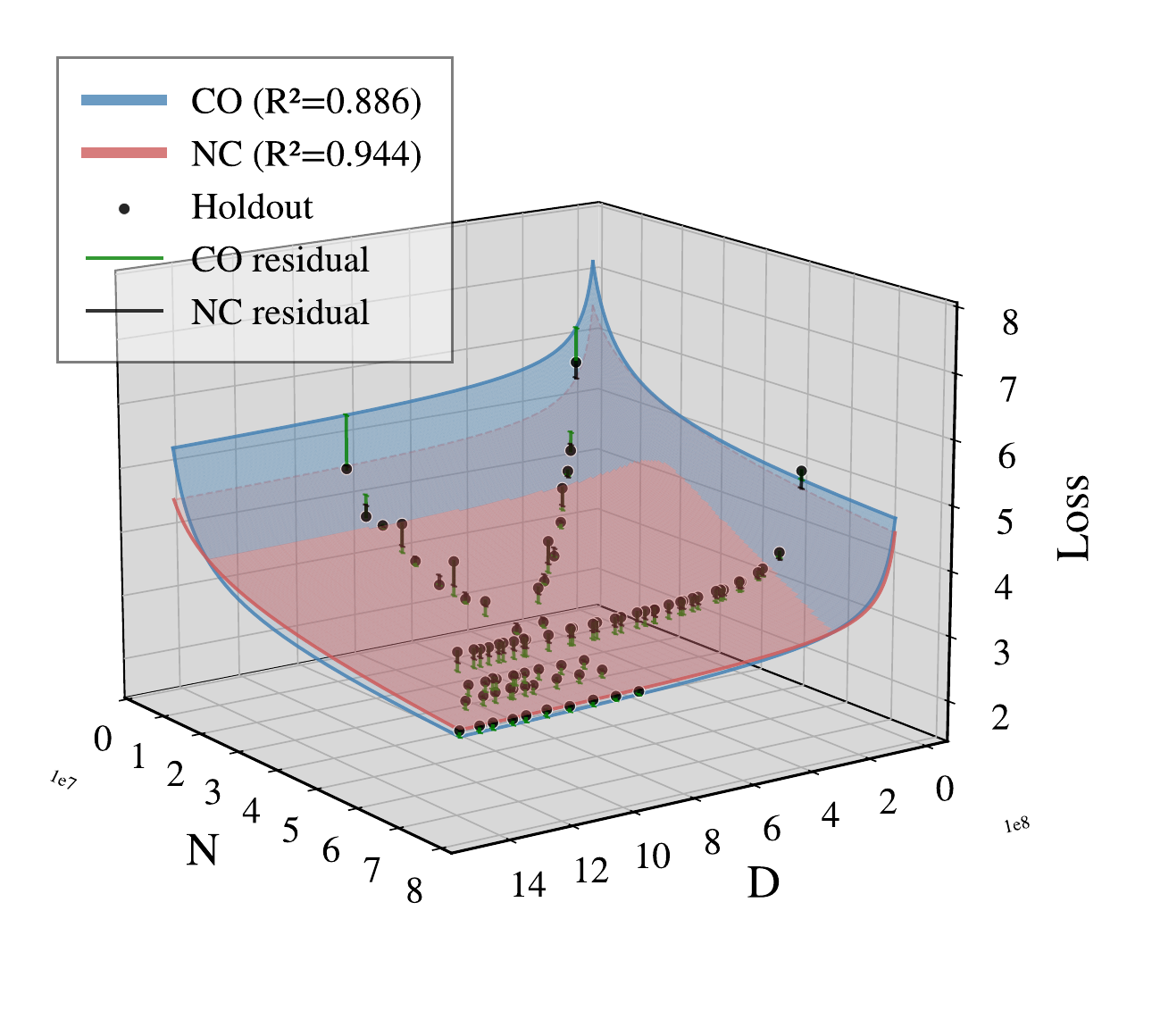}
        \caption{Loss surface.}
        \label{fig:17-surface}
    \end{subfigure}%
    \hfill
    \begin{subfigure}[b]{0.30\textwidth}
        \centering
        \includegraphics[width=\linewidth]{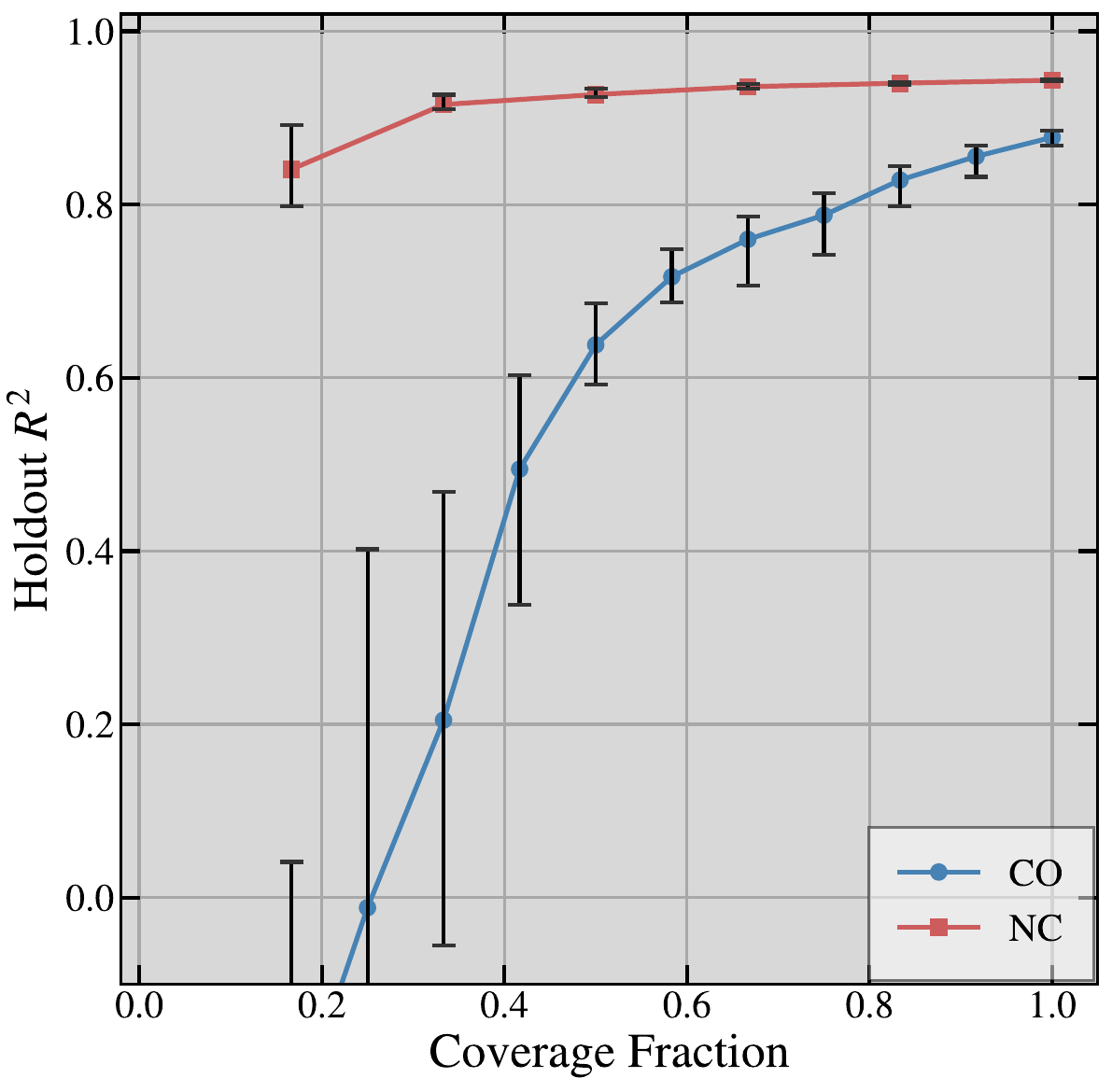}
        \caption{Holdout $R^{2}$ vs.\ TPP coverage.}
        \label{fig:17-convergence}
    \end{subfigure}%
    \hfill
    \begin{subfigure}[b]{0.30\textwidth}
        \centering
        \includegraphics[width=\linewidth]{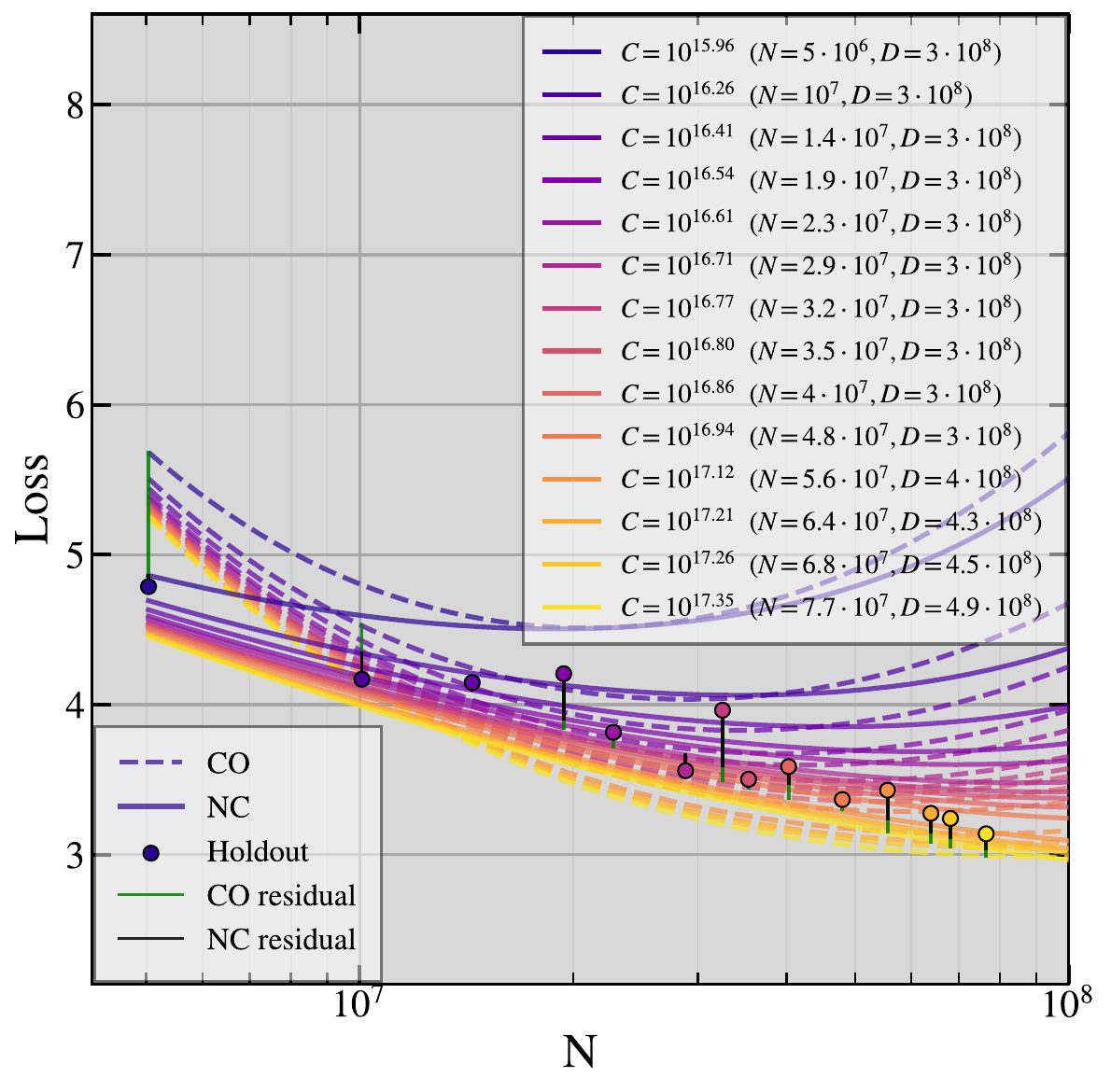}
        \caption{Induced isoFLOP curves.}
        \label{fig:17-isoflop}
    \end{subfigure}
    \caption{Droppo-Elibol law on peS2o, first epoch: CO and NC fits, with the panels of Figure~\ref{fig:select-plots}.}
    \label{fig:dump-17}
\end{figure*}

\begin{figure*}[!ht]
    \centering
    \begin{subfigure}[b]{0.37\textwidth}
        \centering
        \includegraphics[width=\linewidth]{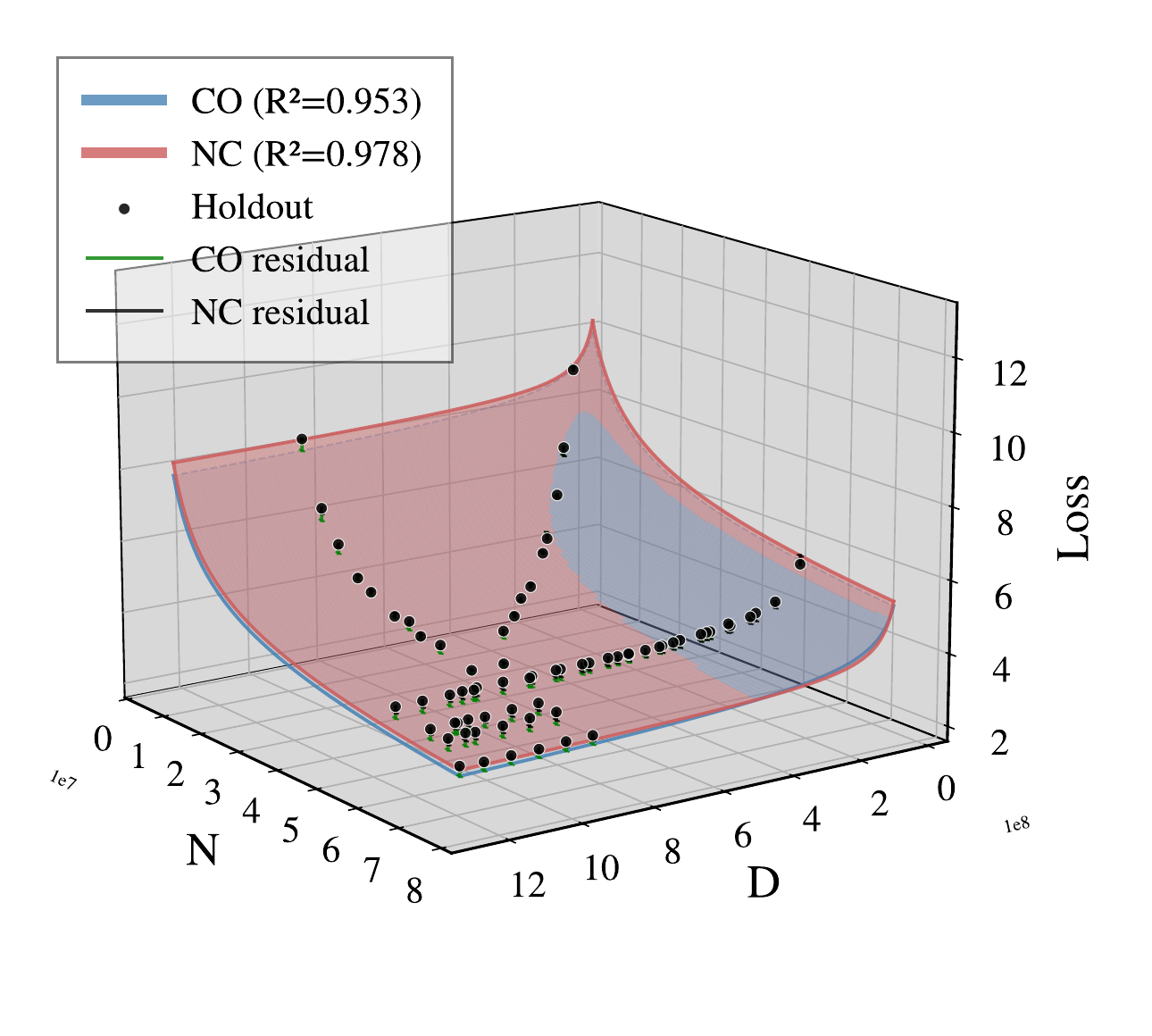}
        \caption{Loss surface.}
        \label{fig:bf16-01-surface}
    \end{subfigure}\hfill
    \begin{subfigure}[b]{0.30\textwidth}
        \centering
        \includegraphics[width=\linewidth]{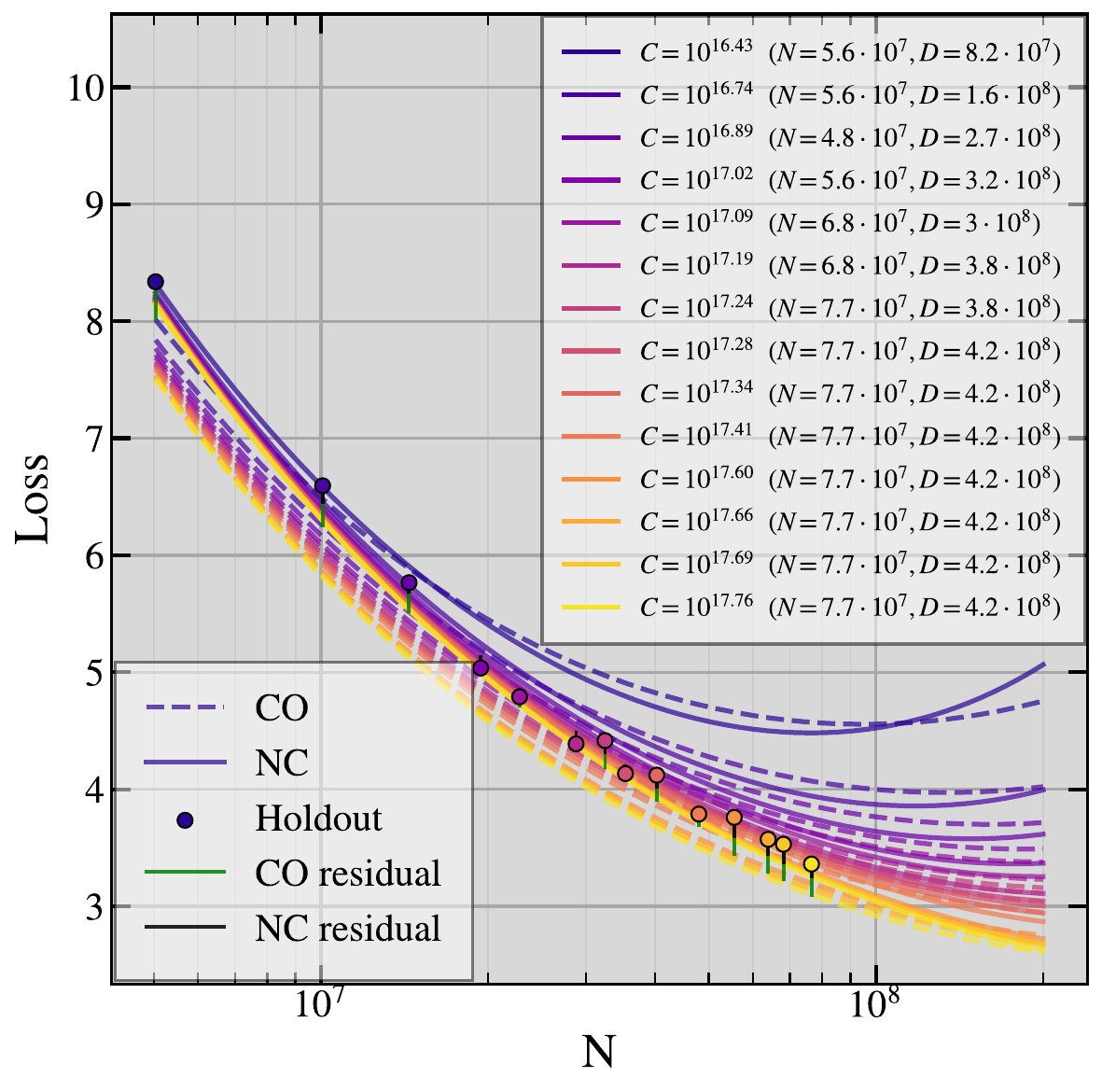}
        \caption{Induced isoFLOP curves.}
        \label{fig:bf16-01-isoflop}
    \end{subfigure}
    \caption{Chinchilla law on Wikipedia BF16, final epoch: CO and NC fits, with panels (a) and (c) of Figure~\ref{fig:select-plots}.}
    \label{fig:dump-bf16-01}
\end{figure*}
\begin{figure*}[!ht]
    \centering
    \begin{subfigure}[b]{0.37\textwidth}
        \centering
        \includegraphics[width=\linewidth]{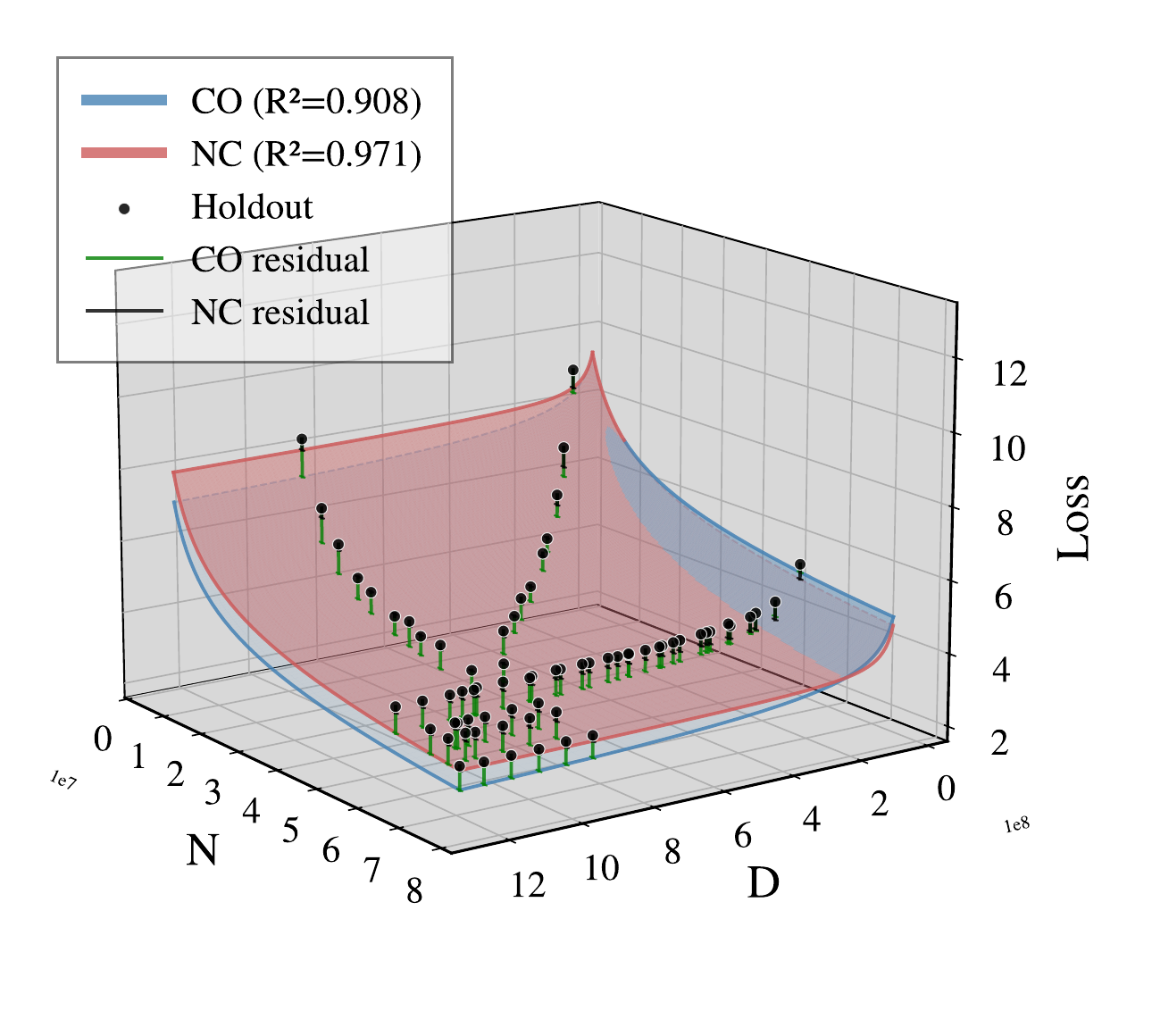}
        \caption{Loss surface.}
        \label{fig:bf16-02-surface}
    \end{subfigure}\hfill
    \begin{subfigure}[b]{0.30\textwidth}
        \centering
        \includegraphics[width=\linewidth]{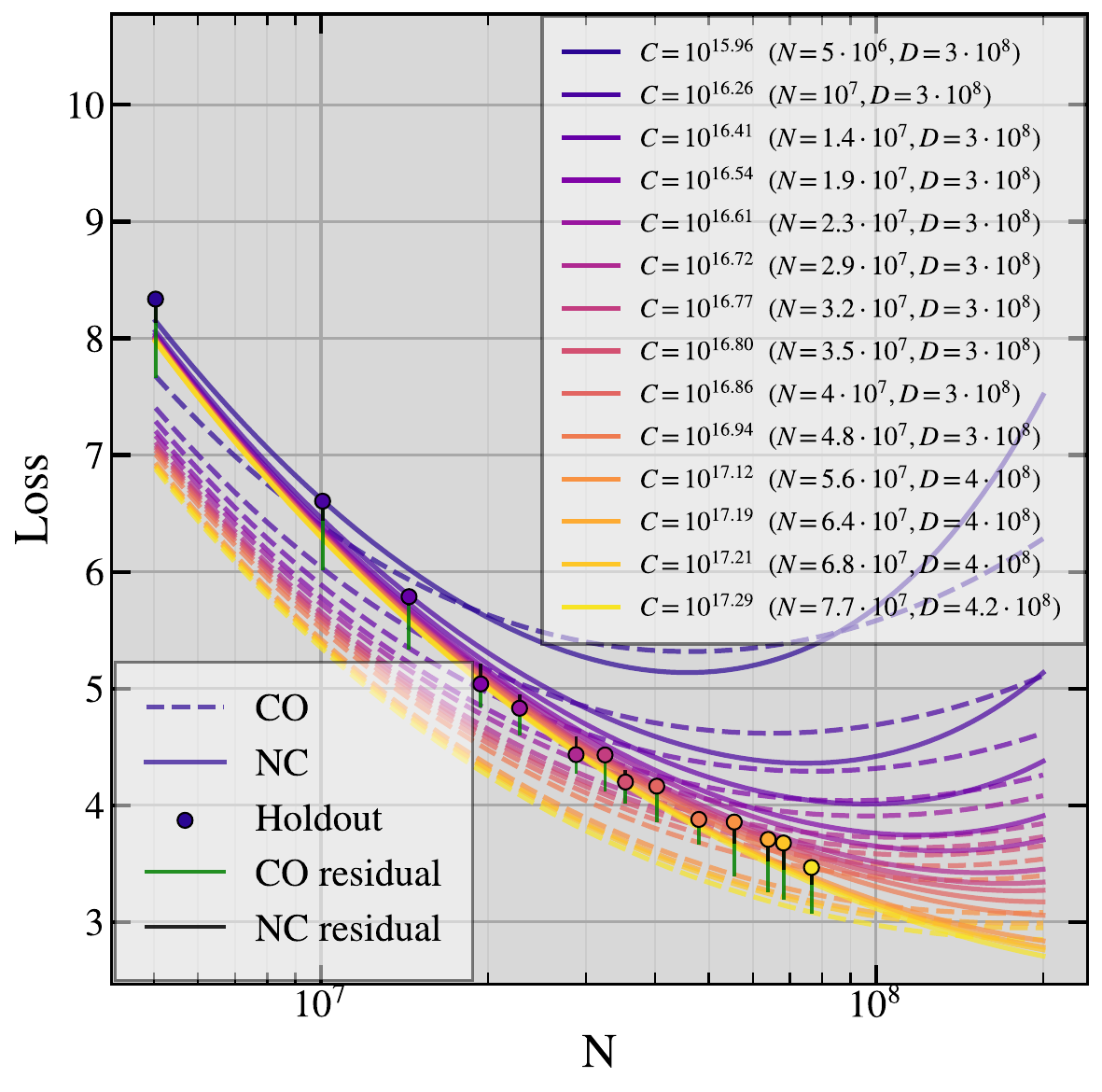}
        \caption{Induced isoFLOP curves.}
        \label{fig:bf16-02-isoflop}
    \end{subfigure}
    \caption{Chinchilla law on Wikipedia BF16, first epoch: CO and NC fits, with panels (a) and (c) of Figure~\ref{fig:select-plots}.}
    \label{fig:dump-bf16-02}
\end{figure*}
\begin{figure*}[!ht]
    \centering
    \begin{subfigure}[b]{0.37\textwidth}
        \centering
        \includegraphics[width=\linewidth]{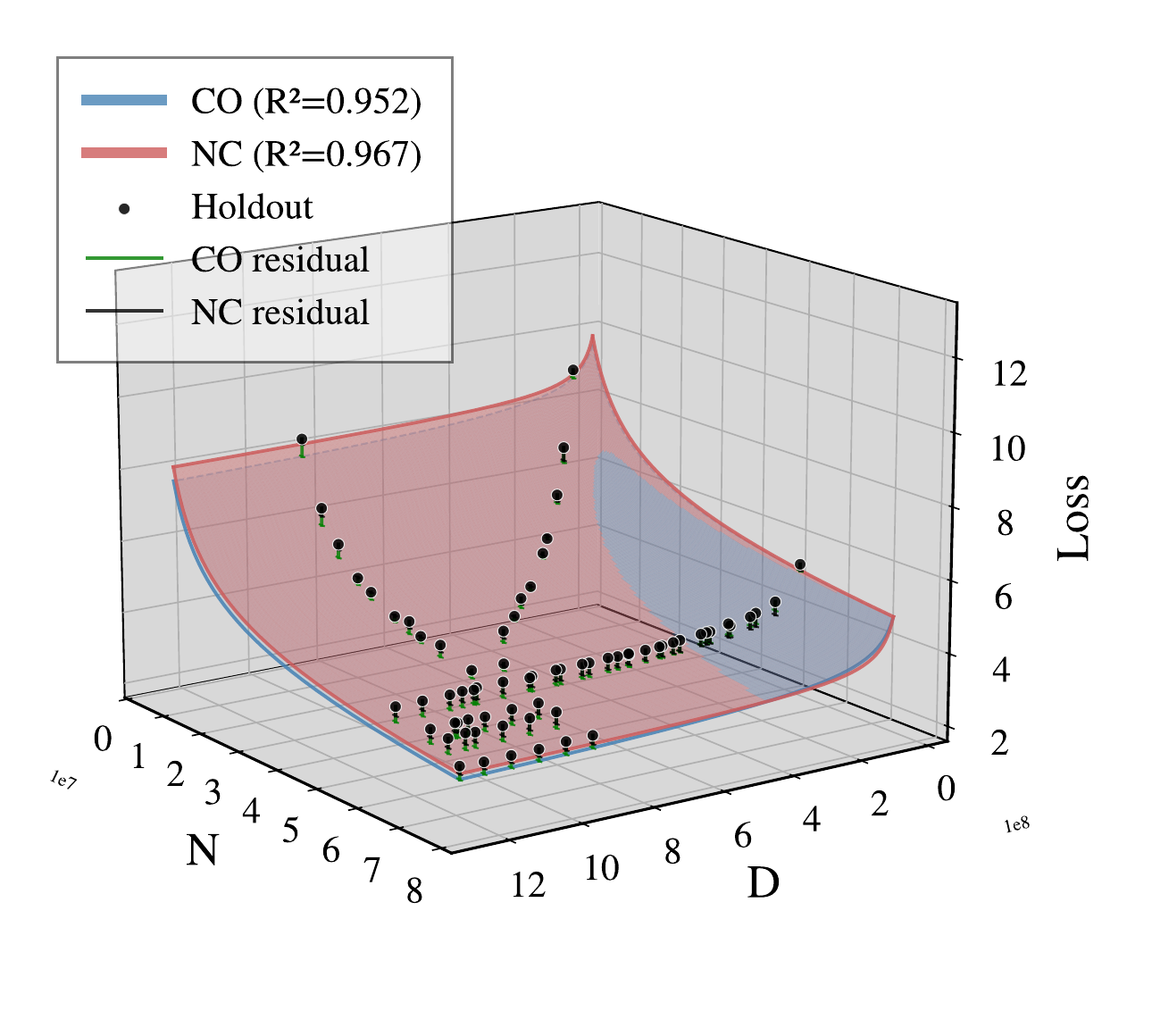}
        \caption{Loss surface.}
        \label{fig:bf16-03-surface}
    \end{subfigure}\hfill
    \begin{subfigure}[b]{0.30\textwidth}
        \centering
        \includegraphics[width=\linewidth]{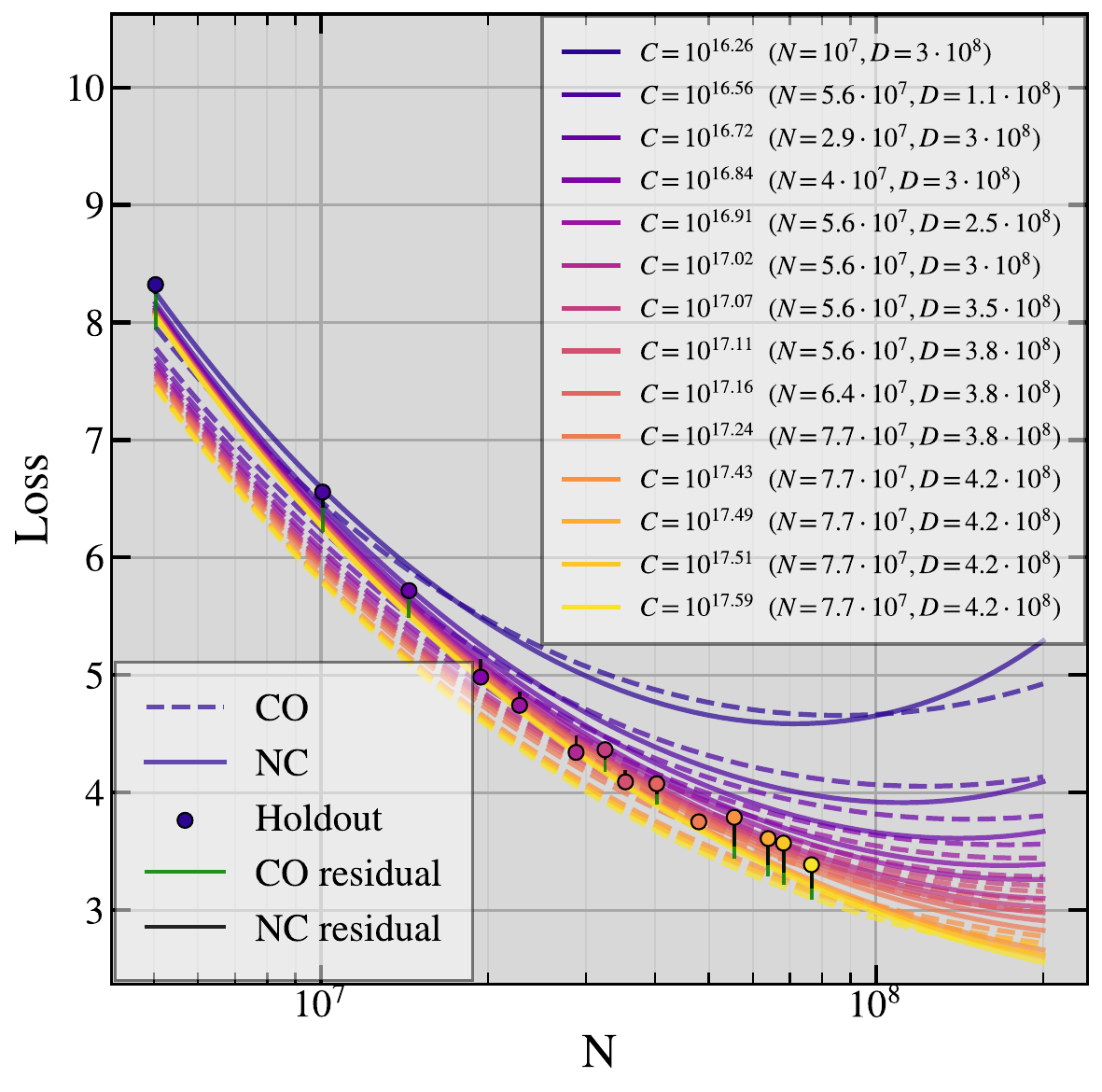}
        \caption{Induced isoFLOP curves.}
        \label{fig:bf16-03-isoflop}
    \end{subfigure}
    \caption{Chinchilla law on Wikipedia BF16, second epoch: CO and NC fits, with panels (a) and (c) of Figure~\ref{fig:select-plots}.}
    \label{fig:dump-bf16-03}
\end{figure*}
\begin{figure*}[!ht]
    \centering
    \begin{subfigure}[b]{0.37\textwidth}
        \centering
        \includegraphics[width=\linewidth]{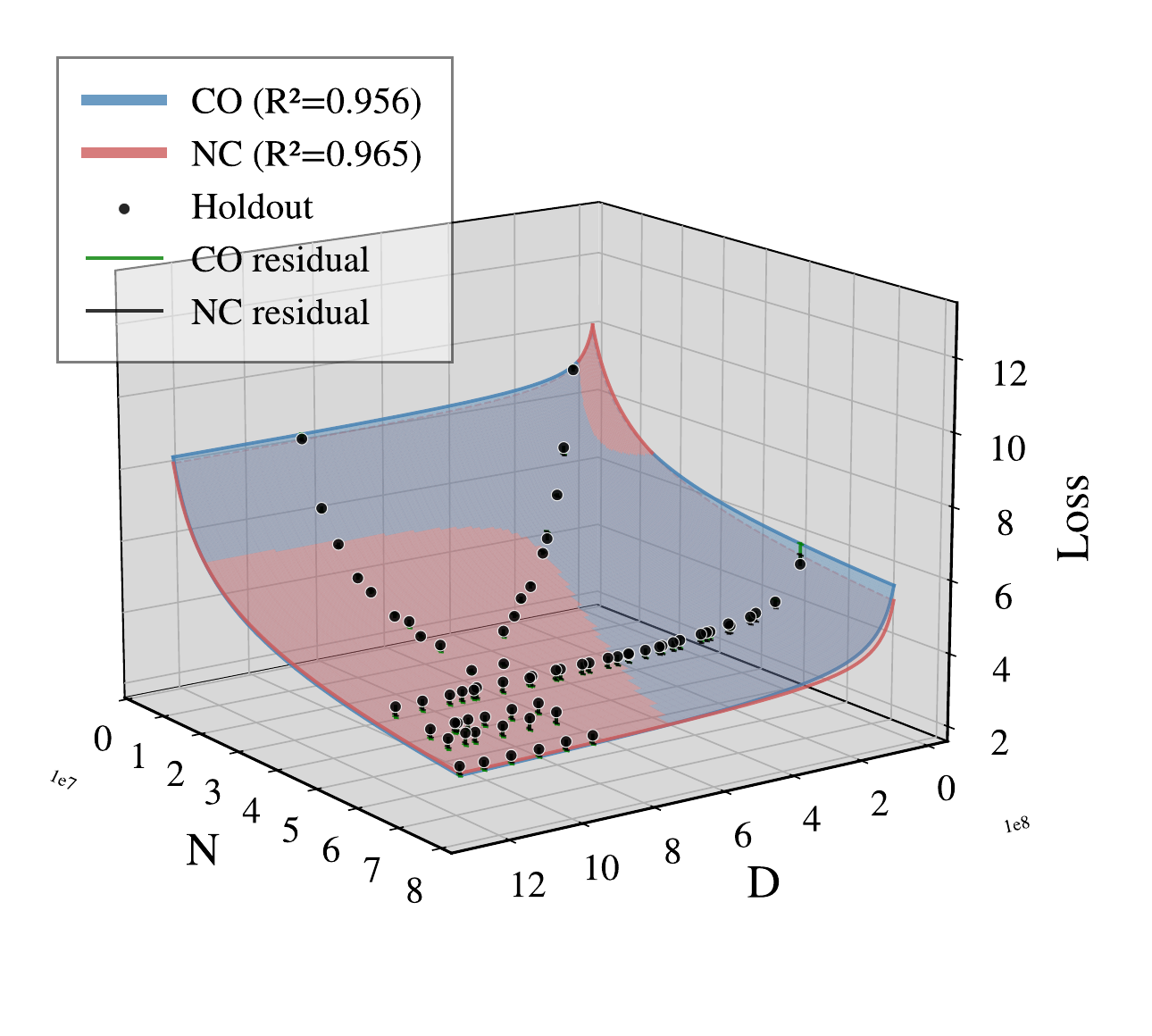}
        \caption{Loss surface.}
        \label{fig:bf16-04-surface}
    \end{subfigure}\hfill
    \begin{subfigure}[b]{0.30\textwidth}
        \centering
        \includegraphics[width=\linewidth]{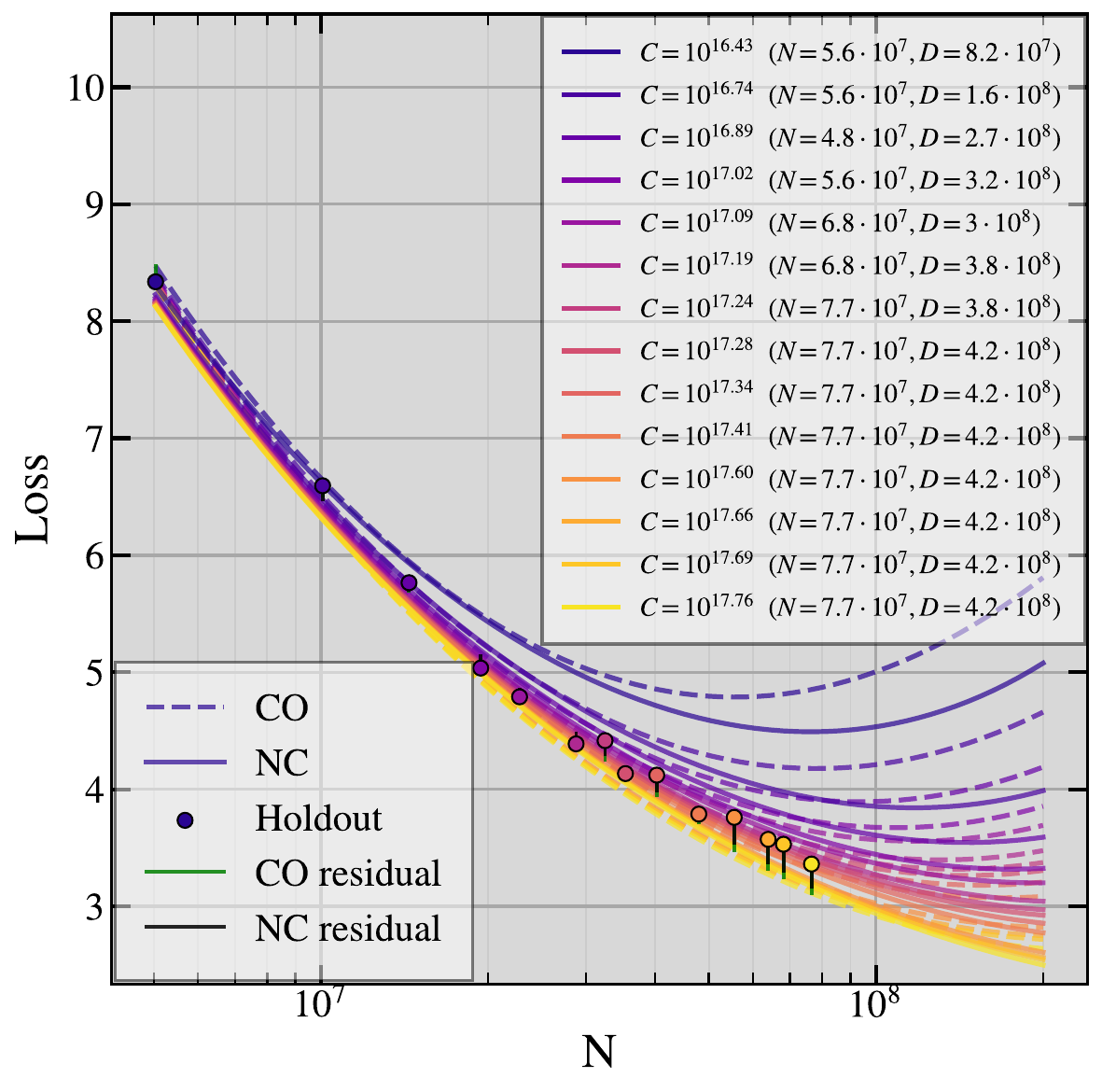}
        \caption{Induced isoFLOP curves.}
        \label{fig:bf16-04-isoflop}
    \end{subfigure}
    \caption{Droppo-Elibol law on Wikipedia BF16, final epoch: CO and NC fits, with panels (a) and (c) of Figure~\ref{fig:select-plots}.}
    \label{fig:dump-bf16-04}
\end{figure*}
\begin{figure*}[!ht]
    \centering
    \begin{subfigure}[b]{0.37\textwidth}
        \centering
        \includegraphics[width=\linewidth]{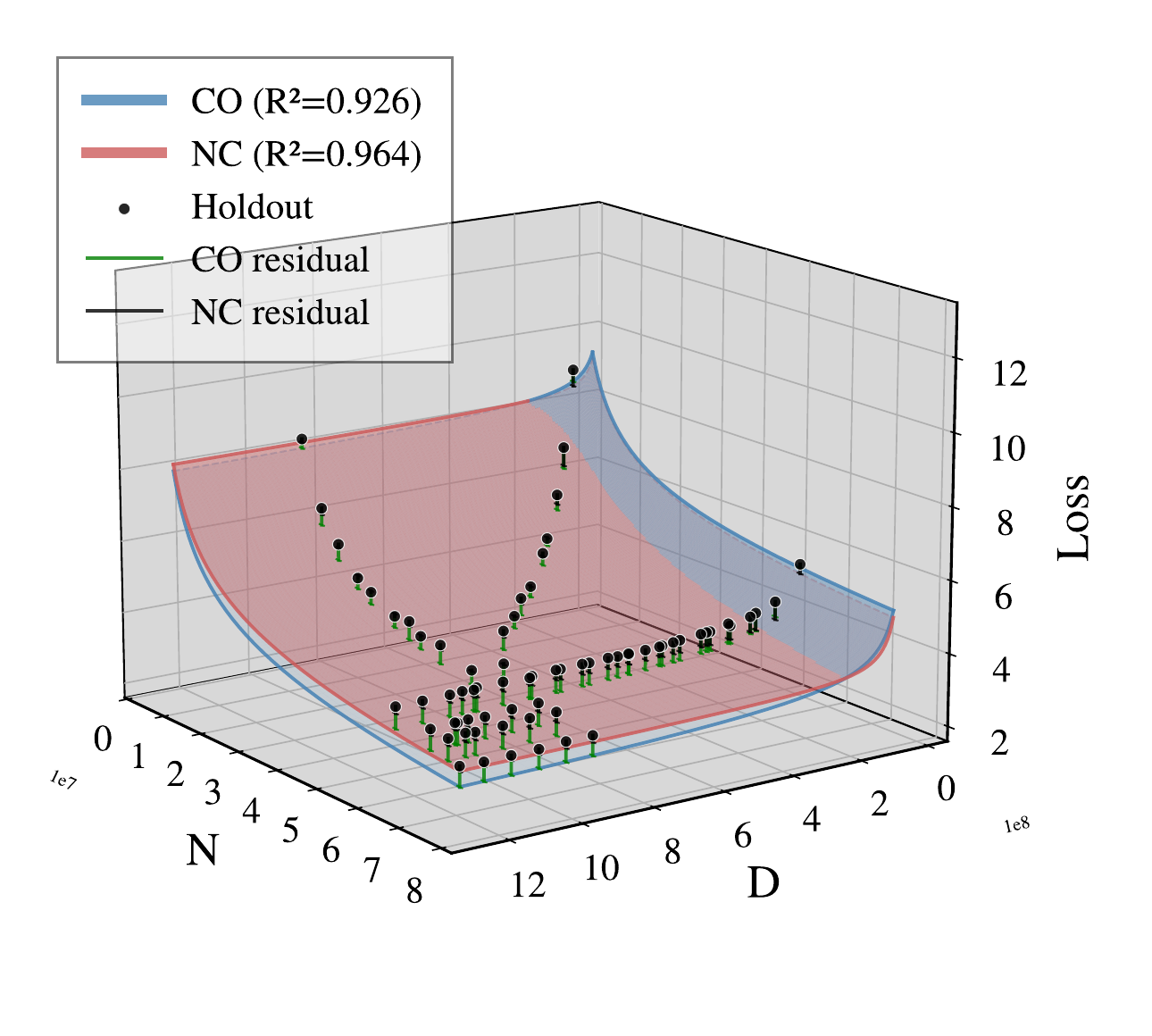}
        \caption{Loss surface.}
        \label{fig:bf16-05-surface}
    \end{subfigure}\hfill
    \begin{subfigure}[b]{0.30\textwidth}
        \centering
        \includegraphics[width=\linewidth]{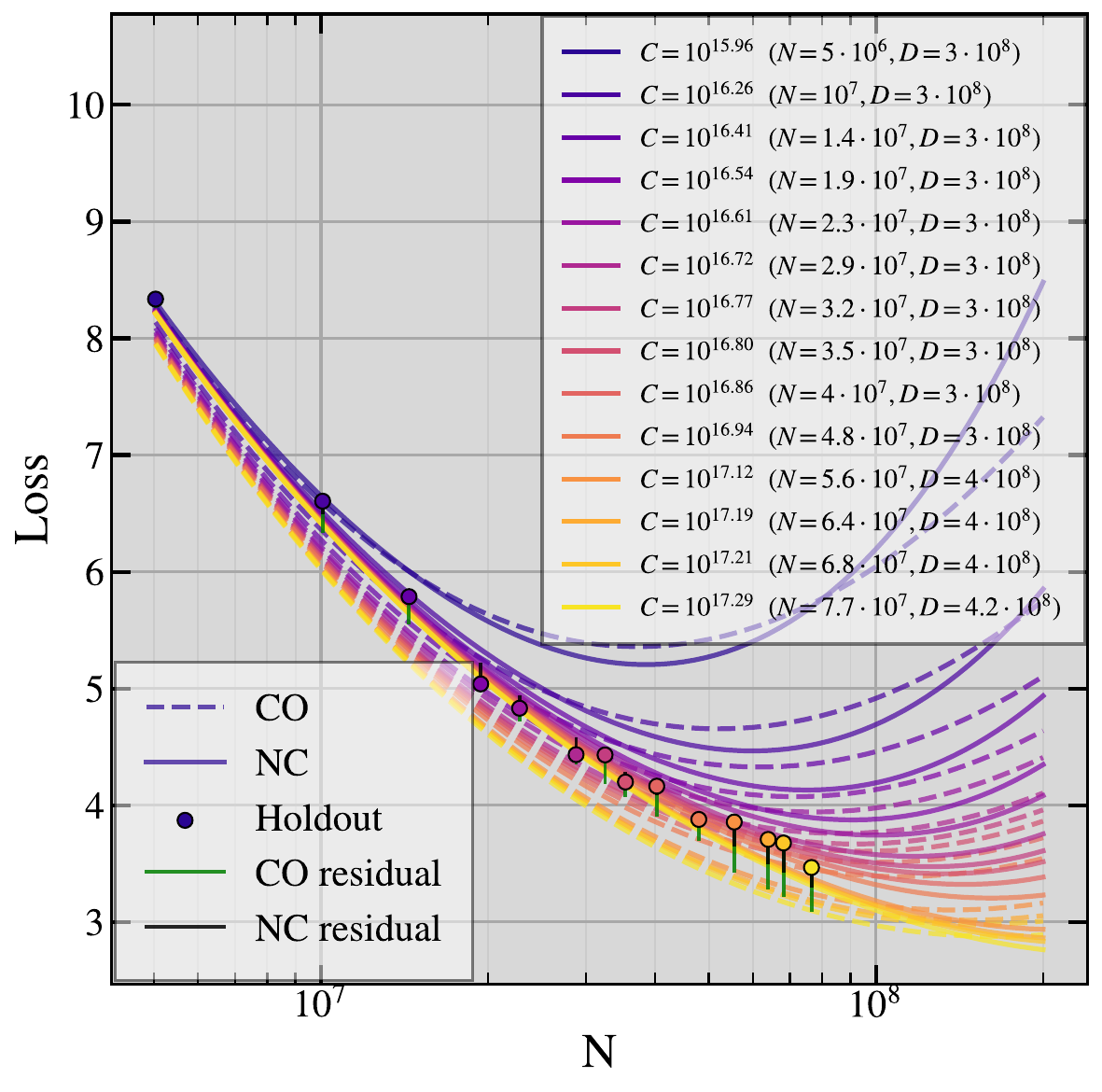}
        \caption{Induced isoFLOP curves.}
        \label{fig:bf16-05-isoflop}
    \end{subfigure}
    \caption{Droppo-Elibol law on Wikipedia BF16, first epoch: CO and NC fits, with panels (a) and (c) of Figure~\ref{fig:select-plots}.}
    \label{fig:dump-bf16-05}
\end{figure*}

\clearpage

\section{Appendix - Dataset, scaling-law, epoch extended breakdown}\label{app:full_all_results}
The following appendix material shows the full breakdown of our experimental fits at full coverage, disaggregating Table~\ref{tab:summary_full} and, in the last three tables, the BF16 Table~\ref{tab:summary_bf16_app}. At full coverage, seed-to-seed variance is small for both designs; the seed spread is widest at lower coverage (convergence panels of Appendix~\ref{app:surfaces}). The CI annotation on each cell is the half-width of the 95\% confidence interval on the mean - i.e., the interval is the reported value $\pm$ CI. 

\begin{table}[H]
\small
\centering
\setlength{\tabcolsep}{2pt}
\caption{C4 - Unified holdout ($\mathcal{H}$). Training and holdout $R^2$ and RMSE per (law, epoch) and design: mean over $30$ optimizer seeds $\pm$ the half-width of its 95\% CI; highlighted: the design with the higher holdout $R^2$; repeated-data (Epoch -) is fit on all epochs at once.}
\renewcommand{\arraystretch}{1.2}
\begin{tabular}{ll l c c c c}
\toprule
\textbf{Epoch} & \textbf{Law} & \textbf{Design} & \textbf{R$^2$ Train} & \textbf{R$^2$ $\mathcal{H}$} & \textbf{RMSE Train} & \textbf{RMSE $\mathcal{H}$} \\
\midrule
\multirow{2}{*}{-} & \multirow{2}{*}{Repeated-data} & Collinear & 0.9781~$\pm$~0.000 & 0.8183~$\pm$~0.000 & 0.1644~$\pm$~0.000 & 0.1844~$\pm$~0.000 \\
 &  & Non-Collinear & 0.9361~$\pm$~0.000 & \win{0.9570~$\pm$~0.000} & 0.1778~$\pm$~0.000 & 0.0897~$\pm$~0.000 \\
\midrule
\multirow{6}{*}{Final} & \multirow{2}{*}{Chinchilla} & Collinear & 0.9788~$\pm$~0.000 & 0.9158~$\pm$~0.014 & 0.1243~$\pm$~0.001 & 0.1091~$\pm$~0.009 \\
 &  & Non-Collinear & 0.9354~$\pm$~0.000 & \win{0.9520~$\pm$~0.001} & 0.1387~$\pm$~0.000 & 0.0842~$\pm$~0.001 \\
\cmidrule{2-7}
 & \multirow{2}{*}{Droppo-Elibol} & Collinear & 0.9805~$\pm$~0.000 & 0.9403~$\pm$~0.003 & 0.1192~$\pm$~0.000 & 0.0939~$\pm$~0.002 \\
 &  & Non-Collinear & 0.9395~$\pm$~0.000 & \win{0.9544~$\pm$~0.001} & 0.1343~$\pm$~0.000 & 0.0822~$\pm$~0.000 \\
\cmidrule{2-7}
 & \multirow{2}{*}{Kaplan} & Collinear & 0.9802~$\pm$~0.000 & \win{0.9412~$\pm$~0.001} & 0.1203~$\pm$~0.000 & 0.0933~$\pm$~0.001 \\
 &  & Non-Collinear & 0.9271~$\pm$~0.001 & 0.9385~$\pm$~0.003 & 0.1473~$\pm$~0.001 & 0.0953~$\pm$~0.002 \\
\midrule
\multirow{6}{*}{First} & \multirow{2}{*}{Chinchilla} & Collinear & 0.9825~$\pm$~0.000 & 0.7953~$\pm$~0.004 & 0.1808~$\pm$~0.000 & 0.2135~$\pm$~0.002 \\
 &  & Non-Collinear & 0.9328~$\pm$~0.000 & \win{0.9362~$\pm$~0.006} & 0.2236~$\pm$~0.000 & 0.1184~$\pm$~0.005 \\
\cmidrule{2-7}
 & \multirow{2}{*}{Droppo-Elibol} & Collinear & 0.9818~$\pm$~0.001 & 0.7886~$\pm$~0.029 & 0.1845~$\pm$~0.003 & 0.2135~$\pm$~0.014 \\
 &  & Non-Collinear & 0.9463~$\pm$~0.000 & \win{0.9382~$\pm$~0.002} & 0.1999~$\pm$~0.000 & 0.1172~$\pm$~0.002 \\
\cmidrule{2-7}
 & \multirow{2}{*}{Kaplan} & Collinear & 0.9674~$\pm$~0.000 & 0.7838~$\pm$~0.004 & 0.2468~$\pm$~0.000 & 0.2194~$\pm$~0.002 \\
 &  & Non-Collinear & 0.9356~$\pm$~0.001 & \win{0.9398~$\pm$~0.001} & 0.2189~$\pm$~0.002 & 0.1158~$\pm$~0.001 \\
\midrule
\multirow{6}{*}{Second} & \multirow{2}{*}{Chinchilla} & Collinear & 0.9799~$\pm$~0.000 & 0.8857~$\pm$~0.013 & 0.1345~$\pm$~0.001 & 0.1303~$\pm$~0.007 \\
 &  & Non-Collinear & 0.9380~$\pm$~0.000 & \win{0.9472~$\pm$~0.002} & 0.1486~$\pm$~0.000 & 0.0893~$\pm$~0.001 \\
\cmidrule{2-7}
 & \multirow{2}{*}{Droppo-Elibol} & Collinear & 0.9808~$\pm$~0.000 & 0.9282~$\pm$~0.004 & 0.1316~$\pm$~0.000 & 0.1040~$\pm$~0.003 \\
 &  & Non-Collinear & 0.9432~$\pm$~0.000 & \win{0.9525~$\pm$~0.001} & 0.1422~$\pm$~0.000 & 0.0848~$\pm$~0.001 \\
\cmidrule{2-7}
 & \multirow{2}{*}{Kaplan} & Collinear & 0.9795~$\pm$~0.000 & 0.9296~$\pm$~0.001 & 0.1361~$\pm$~0.000 & 0.1032~$\pm$~0.001 \\
 &  & Non-Collinear & 0.9301~$\pm$~0.001 & \win{0.9395~$\pm$~0.002} & 0.1579~$\pm$~0.001 & 0.0956~$\pm$~0.002 \\
\bottomrule
\end{tabular}
\end{table}

\begin{table}[H]
\small
\centering
\setlength{\tabcolsep}{2pt}
\caption{C4 - Collinear holdout ($\mathcal{H}_{\mathrm{col}}$). Training and holdout $R^2$ and RMSE per (law, epoch) and design: mean over $30$ optimizer seeds $\pm$ the half-width of its 95\% CI; highlighted: the design with the higher holdout $R^2$; repeated-data (Epoch -) is fit on all epochs at once.}
\renewcommand{\arraystretch}{1.2}
\begin{tabular}{ll l c c c c}
\toprule
\textbf{Epoch} & \textbf{Law} & \textbf{Design} & \textbf{R$^2$ Train} & \textbf{R$^2$ $\mathcal{H}_{\mathrm{col}}$} & \textbf{RMSE Train} & \textbf{RMSE $\mathcal{H}_{\mathrm{col}}$} \\
\midrule
\multirow{2}{*}{-} & \multirow{2}{*}{Repeated-data} & Collinear & 0.9781~$\pm$~0.000 & 0.8571~$\pm$~0.000 & 0.1644~$\pm$~0.000 & 0.1642~$\pm$~0.000 \\
 &  & Non-Collinear & 0.9361~$\pm$~0.000 & \win{0.9590~$\pm$~0.000} & 0.1778~$\pm$~0.000 & 0.0880~$\pm$~0.000 \\
\midrule
\multirow{6}{*}{Final} & \multirow{2}{*}{Chinchilla} & Collinear & 0.9788~$\pm$~0.000 & 0.9334~$\pm$~0.014 & 0.1243~$\pm$~0.001 & 0.0980~$\pm$~0.009 \\
 &  & Non-Collinear & 0.9354~$\pm$~0.000 & \win{0.9595~$\pm$~0.001} & 0.1387~$\pm$~0.000 & 0.0788~$\pm$~0.001 \\
\cmidrule{2-7}
 & \multirow{2}{*}{Droppo-Elibol} & Collinear & 0.9805~$\pm$~0.000 & 0.9510~$\pm$~0.002 & 0.1192~$\pm$~0.000 & 0.0866~$\pm$~0.002 \\
 &  & Non-Collinear & 0.9395~$\pm$~0.000 & \win{0.9631~$\pm$~0.001} & 0.1343~$\pm$~0.000 & 0.0753~$\pm$~0.001 \\
\cmidrule{2-7}
 & \multirow{2}{*}{Kaplan} & Collinear & 0.9802~$\pm$~0.000 & \win{0.9554~$\pm$~0.001} & 0.1203~$\pm$~0.000 & 0.0827~$\pm$~0.001 \\
 &  & Non-Collinear & 0.9271~$\pm$~0.001 & 0.9456~$\pm$~0.003 & 0.1473~$\pm$~0.001 & 0.0912~$\pm$~0.002 \\
\midrule
\multirow{6}{*}{First} & \multirow{2}{*}{Chinchilla} & Collinear & 0.9825~$\pm$~0.000 & 0.8575~$\pm$~0.003 & 0.1808~$\pm$~0.000 & 0.1783~$\pm$~0.002 \\
 &  & Non-Collinear & 0.9328~$\pm$~0.000 & \win{0.9402~$\pm$~0.006} & 0.2236~$\pm$~0.000 & 0.1142~$\pm$~0.006 \\
\cmidrule{2-7}
 & \multirow{2}{*}{Droppo-Elibol} & Collinear & 0.9818~$\pm$~0.001 & 0.8680~$\pm$~0.013 & 0.1845~$\pm$~0.003 & 0.1701~$\pm$~0.008 \\
 &  & Non-Collinear & 0.9463~$\pm$~0.000 & \win{0.9398~$\pm$~0.002} & 0.1999~$\pm$~0.000 & 0.1158~$\pm$~0.002 \\
\cmidrule{2-7}
 & \multirow{2}{*}{Kaplan} & Collinear & 0.9674~$\pm$~0.000 & 0.8397~$\pm$~0.004 & 0.2468~$\pm$~0.000 & 0.1891~$\pm$~0.002 \\
 &  & Non-Collinear & 0.9356~$\pm$~0.001 & \win{0.9493~$\pm$~0.001} & 0.2189~$\pm$~0.002 & 0.1064~$\pm$~0.001 \\
\midrule
\multirow{6}{*}{Second} & \multirow{2}{*}{Chinchilla} & Collinear & 0.9799~$\pm$~0.000 & 0.9023~$\pm$~0.012 & 0.1345~$\pm$~0.001 & 0.1193~$\pm$~0.007 \\
 &  & Non-Collinear & 0.9380~$\pm$~0.000 & \win{0.9480~$\pm$~0.002} & 0.1486~$\pm$~0.000 & 0.0881~$\pm$~0.001 \\
\cmidrule{2-7}
 & \multirow{2}{*}{Droppo-Elibol} & Collinear & 0.9808~$\pm$~0.000 & 0.9343~$\pm$~0.003 & 0.1316~$\pm$~0.000 & 0.0990~$\pm$~0.002 \\
 &  & Non-Collinear & 0.9432~$\pm$~0.000 & \win{0.9542~$\pm$~0.001} & 0.1422~$\pm$~0.000 & 0.0828~$\pm$~0.001 \\
\cmidrule{2-7}
 & \multirow{2}{*}{Kaplan} & Collinear & 0.9795~$\pm$~0.000 & 0.9398~$\pm$~0.001 & 0.1361~$\pm$~0.000 & 0.0949~$\pm$~0.001 \\
 &  & Non-Collinear & 0.9301~$\pm$~0.001 & \win{0.9417~$\pm$~0.002} & 0.1579~$\pm$~0.001 & 0.0933~$\pm$~0.002 \\
\bottomrule
\end{tabular}
\end{table}

\begin{table}[H]
\small
\centering
\setlength{\tabcolsep}{2pt}
\caption{C4 - Non-collinear holdout ($\mathcal{H}_{\mathrm{nc}}$). Training and holdout $R^2$ and RMSE per (law, epoch) and design: mean over $30$ optimizer seeds $\pm$ the half-width of its 95\% CI; highlighted: the design with the higher holdout $R^2$; repeated-data (Epoch -) is fit on all epochs at once.}
\renewcommand{\arraystretch}{1.2}
\begin{tabular}{ll l c c c c}
\toprule
\textbf{Epoch} & \textbf{Law} & \textbf{Design} & \textbf{R$^2$ Train} & \textbf{R$^2$ $\mathcal{H}_{\mathrm{nc}}$} & \textbf{RMSE Train} & \textbf{RMSE $\mathcal{H}_{\mathrm{nc}}$} \\
\midrule
\multirow{2}{*}{-} & \multirow{2}{*}{Repeated-data} & Collinear & 0.9781~$\pm$~0.000 & 0.7753~$\pm$~0.000 & 0.1644~$\pm$~0.000 & 0.2039~$\pm$~0.000 \\
 &  & Non-Collinear & 0.9361~$\pm$~0.000 & \win{0.9547~$\pm$~0.000} & 0.1778~$\pm$~0.000 & 0.0915~$\pm$~0.000 \\
\midrule
\multirow{6}{*}{Final} & \multirow{2}{*}{Chinchilla} & Collinear & 0.9788~$\pm$~0.000 & 0.8952~$\pm$~0.016 & 0.1243~$\pm$~0.001 & 0.1198~$\pm$~0.008 \\
 &  & Non-Collinear & 0.9354~$\pm$~0.000 & \win{0.9433~$\pm$~0.001} & 0.1387~$\pm$~0.000 & 0.0897~$\pm$~0.001 \\
\cmidrule{2-7}
 & \multirow{2}{*}{Droppo-Elibol} & Collinear & 0.9805~$\pm$~0.000 & 0.9277~$\pm$~0.003 & 0.1192~$\pm$~0.000 & 0.1011~$\pm$~0.002 \\
 &  & Non-Collinear & 0.9395~$\pm$~0.000 & \win{0.9441~$\pm$~0.001} & 0.1343~$\pm$~0.000 & 0.0890~$\pm$~0.000 \\
\cmidrule{2-7}
 & \multirow{2}{*}{Kaplan} & Collinear & 0.9802~$\pm$~0.000 & 0.9244~$\pm$~0.002 & 0.1203~$\pm$~0.000 & 0.1035~$\pm$~0.001 \\
 &  & Non-Collinear & 0.9271~$\pm$~0.001 & \win{0.9301~$\pm$~0.003} & 0.1473~$\pm$~0.001 & 0.0994~$\pm$~0.002 \\
\midrule
\multirow{6}{*}{First} & \multirow{2}{*}{Chinchilla} & Collinear & 0.9825~$\pm$~0.000 & 0.7276~$\pm$~0.004 & 0.1808~$\pm$~0.000 & 0.2459~$\pm$~0.002 \\
 &  & Non-Collinear & 0.9328~$\pm$~0.000 & \win{0.9317~$\pm$~0.005} & 0.2236~$\pm$~0.000 & 0.1226~$\pm$~0.004 \\
\cmidrule{2-7}
 & \multirow{2}{*}{Droppo-Elibol} & Collinear & 0.9818~$\pm$~0.001 & 0.7023~$\pm$~0.048 & 0.1845~$\pm$~0.003 & 0.2516~$\pm$~0.019 \\
 &  & Non-Collinear & 0.9463~$\pm$~0.000 & \win{0.9364~$\pm$~0.002} & 0.1999~$\pm$~0.000 & 0.1187~$\pm$~0.002 \\
\cmidrule{2-7}
 & \multirow{2}{*}{Kaplan} & Collinear & 0.9674~$\pm$~0.000 & 0.7231~$\pm$~0.005 & 0.2468~$\pm$~0.000 & 0.2480~$\pm$~0.002 \\
 &  & Non-Collinear & 0.9356~$\pm$~0.001 & \win{0.9294~$\pm$~0.001} & 0.2189~$\pm$~0.002 & 0.1252~$\pm$~0.001 \\
\midrule
\multirow{6}{*}{Second} & \multirow{2}{*}{Chinchilla} & Collinear & 0.9799~$\pm$~0.000 & 0.8677~$\pm$~0.013 & 0.1345~$\pm$~0.001 & 0.1411~$\pm$~0.006 \\
 &  & Non-Collinear & 0.9380~$\pm$~0.000 & \win{0.9462~$\pm$~0.001} & 0.1486~$\pm$~0.000 & 0.0906~$\pm$~0.001 \\
\cmidrule{2-7}
 & \multirow{2}{*}{Droppo-Elibol} & Collinear & 0.9808~$\pm$~0.000 & 0.9216~$\pm$~0.005 & 0.1316~$\pm$~0.000 & 0.1092~$\pm$~0.003 \\
 &  & Non-Collinear & 0.9432~$\pm$~0.000 & \win{0.9505~$\pm$~0.001} & 0.1422~$\pm$~0.000 & 0.0870~$\pm$~0.001 \\
\cmidrule{2-7}
 & \multirow{2}{*}{Kaplan} & Collinear & 0.9795~$\pm$~0.000 & 0.9185~$\pm$~0.002 & 0.1361~$\pm$~0.000 & 0.1116~$\pm$~0.001 \\
 &  & Non-Collinear & 0.9301~$\pm$~0.001 & \win{0.9368~$\pm$~0.003} & 0.1579~$\pm$~0.001 & 0.0981~$\pm$~0.002 \\
\bottomrule
\end{tabular}
\end{table}

\begin{table}[H]
\small
\centering
\setlength{\tabcolsep}{2pt}
\caption{Cosmopedia - Unified holdout ($\mathcal{H}$). Training and holdout $R^2$ and RMSE per (law, epoch) and design: mean over $30$ optimizer seeds $\pm$ the half-width of its 95\% CI; highlighted: the design with the higher holdout $R^2$; repeated-data (Epoch -) is fit on all epochs at once.}
\renewcommand{\arraystretch}{1.2}
\begin{tabular}{ll l c c c c}
\toprule
\textbf{Epoch} & \textbf{Law} & \textbf{Design} & \textbf{R$^2$ Train} & \textbf{R$^2$ $\mathcal{H}$} & \textbf{RMSE Train} & \textbf{RMSE $\mathcal{H}$} \\
\midrule
\multirow{2}{*}{-} & \multirow{2}{*}{Repeated-data} & Collinear & 0.9841~$\pm$~0.000 & 0.9496~$\pm$~0.000 & 0.1805~$\pm$~0.000 & 0.1333~$\pm$~0.000 \\
 &  & Non-Collinear & 0.9656~$\pm$~0.000 & \win{0.9684~$\pm$~0.000} & 0.1862~$\pm$~0.000 & 0.1056~$\pm$~0.000 \\
\midrule
\multirow{6}{*}{Final} & \multirow{2}{*}{Chinchilla} & Collinear & 0.9857~$\pm$~0.000 & 0.9450~$\pm$~0.005 & 0.1401~$\pm$~0.001 & 0.1225~$\pm$~0.005 \\
 &  & Non-Collinear & 0.9502~$\pm$~0.000 & \win{0.9580~$\pm$~0.003} & 0.1807~$\pm$~0.000 & 0.1073~$\pm$~0.004 \\
\cmidrule{2-7}
 & \multirow{2}{*}{Droppo-Elibol} & Collinear & 0.9861~$\pm$~0.000 & 0.9581~$\pm$~0.001 & 0.1382~$\pm$~0.000 & 0.1076~$\pm$~0.002 \\
 &  & Non-Collinear & 0.9545~$\pm$~0.000 & \win{0.9624~$\pm$~0.001} & 0.1728~$\pm$~0.000 & 0.1020~$\pm$~0.002 \\
\cmidrule{2-7}
 & \multirow{2}{*}{Kaplan} & Collinear & 0.9847~$\pm$~0.000 & 0.9428~$\pm$~0.000 & 0.1446~$\pm$~0.000 & 0.1259~$\pm$~0.000 \\
 &  & Non-Collinear & 0.9500~$\pm$~0.000 & \win{0.9558~$\pm$~0.001} & 0.1811~$\pm$~0.000 & 0.1105~$\pm$~0.001 \\
\midrule
\multirow{6}{*}{First} & \multirow{2}{*}{Chinchilla} & Collinear & 0.9858~$\pm$~0.000 & 0.8959~$\pm$~0.003 & 0.2020~$\pm$~0.000 & 0.2115~$\pm$~0.003 \\
 &  & Non-Collinear & 0.9741~$\pm$~0.000 & \win{0.9626~$\pm$~0.003} & 0.1947~$\pm$~0.001 & 0.1260~$\pm$~0.005 \\
\cmidrule{2-7}
 & \multirow{2}{*}{Droppo-Elibol} & Collinear & 0.9858~$\pm$~0.000 & 0.8880~$\pm$~0.009 & 0.2023~$\pm$~0.001 & 0.2185~$\pm$~0.008 \\
 &  & Non-Collinear & 0.9826~$\pm$~0.000 & \win{0.9711~$\pm$~0.001} & 0.1598~$\pm$~0.000 & 0.1114~$\pm$~0.003 \\
\cmidrule{2-7}
 & \multirow{2}{*}{Kaplan} & Collinear & 0.9814~$\pm$~0.000 & 0.7448~$\pm$~0.015 & 0.2313~$\pm$~0.001 & 0.3306~$\pm$~0.009 \\
 &  & Non-Collinear & 0.9704~$\pm$~0.000 & \win{0.9661~$\pm$~0.001} & 0.2083~$\pm$~0.000 & 0.1207~$\pm$~0.001 \\
\midrule
\multirow{6}{*}{Second} & \multirow{2}{*}{Chinchilla} & Collinear & 0.9867~$\pm$~0.000 & 0.9504~$\pm$~0.002 & 0.1464~$\pm$~0.001 & 0.1214~$\pm$~0.003 \\
 &  & Non-Collinear & 0.9574~$\pm$~0.000 & \win{0.9573~$\pm$~0.006} & 0.1812~$\pm$~0.001 & 0.1114~$\pm$~0.006 \\
\cmidrule{2-7}
 & \multirow{2}{*}{Droppo-Elibol} & Collinear & 0.9869~$\pm$~0.000 & 0.9562~$\pm$~0.003 & 0.1451~$\pm$~0.001 & 0.1138~$\pm$~0.003 \\
 &  & Non-Collinear & 0.9615~$\pm$~0.000 & \win{0.9676~$\pm$~0.001} & 0.1722~$\pm$~0.000 & 0.0981~$\pm$~0.002 \\
\cmidrule{2-7}
 & \multirow{2}{*}{Kaplan} & Collinear & 0.9846~$\pm$~0.000 & 0.9284~$\pm$~0.001 & 0.1575~$\pm$~0.000 & 0.1461~$\pm$~0.001 \\
 &  & Non-Collinear & 0.9569~$\pm$~0.000 & \win{0.9618~$\pm$~0.001} & 0.1822~$\pm$~0.000 & 0.1066~$\pm$~0.001 \\
\bottomrule
\end{tabular}
\end{table}

\begin{table}[H]
\small
\centering
\setlength{\tabcolsep}{2pt}
\caption{Cosmopedia - Collinear holdout ($\mathcal{H}_{\mathrm{col}}$). Training and holdout $R^2$ and RMSE per (law, epoch) and design: mean over $30$ optimizer seeds $\pm$ the half-width of its 95\% CI; highlighted: the design with the higher holdout $R^2$; repeated-data (Epoch -) is fit on all epochs at once.}
\renewcommand{\arraystretch}{1.2}
\begin{tabular}{ll l c c c c}
\toprule
\textbf{Epoch} & \textbf{Law} & \textbf{Design} & \textbf{R$^2$ Train} & \textbf{R$^2$ $\mathcal{H}_{\mathrm{col}}$} & \textbf{RMSE Train} & \textbf{RMSE $\mathcal{H}_{\mathrm{col}}$} \\
\midrule
\multirow{2}{*}{-} & \multirow{2}{*}{Repeated-data} & Collinear & 0.9841~$\pm$~0.000 & 0.9646~$\pm$~0.000 & 0.1805~$\pm$~0.000 & 0.1137~$\pm$~0.000 \\
 &  & Non-Collinear & 0.9656~$\pm$~0.000 & \win{0.9719~$\pm$~0.000} & 0.1862~$\pm$~0.000 & 0.1014~$\pm$~0.000 \\
\midrule
\multirow{6}{*}{Final} & \multirow{2}{*}{Chinchilla} & Collinear & 0.9857~$\pm$~0.000 & 0.9573~$\pm$~0.004 & 0.1401~$\pm$~0.001 & 0.1119~$\pm$~0.005 \\
 &  & Non-Collinear & 0.9502~$\pm$~0.000 & \win{0.9618~$\pm$~0.003} & 0.1807~$\pm$~0.000 & 0.1060~$\pm$~0.004 \\
\cmidrule{2-7}
 & \multirow{2}{*}{Droppo-Elibol} & Collinear & 0.9861~$\pm$~0.000 & 0.9651~$\pm$~0.001 & 0.1382~$\pm$~0.000 & 0.1019~$\pm$~0.001 \\
 &  & Non-Collinear & 0.9545~$\pm$~0.000 & \win{0.9658~$\pm$~0.001} & 0.1728~$\pm$~0.000 & 0.1007~$\pm$~0.002 \\
\cmidrule{2-7}
 & \multirow{2}{*}{Kaplan} & Collinear & 0.9847~$\pm$~0.000 & \win{0.9593~$\pm$~0.000} & 0.1446~$\pm$~0.000 & 0.1100~$\pm$~0.000 \\
 &  & Non-Collinear & 0.9500~$\pm$~0.000 & 0.9578~$\pm$~0.001 & 0.1811~$\pm$~0.000 & 0.1121~$\pm$~0.001 \\
\midrule
\multirow{6}{*}{First} & \multirow{2}{*}{Chinchilla} & Collinear & 0.9858~$\pm$~0.000 & 0.9100~$\pm$~0.003 & 0.2020~$\pm$~0.000 & 0.1980~$\pm$~0.004 \\
 &  & Non-Collinear & 0.9741~$\pm$~0.000 & \win{0.9693~$\pm$~0.003} & 0.1947~$\pm$~0.001 & 0.1145~$\pm$~0.006 \\
\cmidrule{2-7}
 & \multirow{2}{*}{Droppo-Elibol} & Collinear & 0.9858~$\pm$~0.000 & 0.9041~$\pm$~0.008 & 0.2023~$\pm$~0.001 & 0.2034~$\pm$~0.008 \\
 &  & Non-Collinear & 0.9826~$\pm$~0.000 & \win{0.9736~$\pm$~0.001} & 0.1598~$\pm$~0.000 & 0.1070~$\pm$~0.003 \\
\cmidrule{2-7}
 & \multirow{2}{*}{Kaplan} & Collinear & 0.9814~$\pm$~0.000 & 0.8048~$\pm$~0.014 & 0.2313~$\pm$~0.001 & 0.2909~$\pm$~0.009 \\
 &  & Non-Collinear & 0.9704~$\pm$~0.000 & \win{0.9681~$\pm$~0.001} & 0.2083~$\pm$~0.000 & 0.1179~$\pm$~0.001 \\
\midrule
\multirow{6}{*}{Second} & \multirow{2}{*}{Chinchilla} & Collinear & 0.9867~$\pm$~0.000 & \win{0.9612~$\pm$~0.002} & 0.1464~$\pm$~0.001 & 0.1093~$\pm$~0.003 \\
 &  & Non-Collinear & 0.9574~$\pm$~0.000 & 0.9590~$\pm$~0.006 & 0.1812~$\pm$~0.001 & 0.1110~$\pm$~0.007 \\
\cmidrule{2-7}
 & \multirow{2}{*}{Droppo-Elibol} & Collinear & 0.9869~$\pm$~0.000 & 0.9650~$\pm$~0.001 & 0.1451~$\pm$~0.001 & 0.1039~$\pm$~0.002 \\
 &  & Non-Collinear & 0.9615~$\pm$~0.000 & \win{0.9695~$\pm$~0.001} & 0.1722~$\pm$~0.000 & 0.0968~$\pm$~0.002 \\
\cmidrule{2-7}
 & \multirow{2}{*}{Kaplan} & Collinear & 0.9846~$\pm$~0.000 & 0.9483~$\pm$~0.001 & 0.1575~$\pm$~0.000 & 0.1264~$\pm$~0.001 \\
 &  & Non-Collinear & 0.9569~$\pm$~0.000 & \win{0.9621~$\pm$~0.001} & 0.1822~$\pm$~0.000 & 0.1081~$\pm$~0.001 \\
\bottomrule
\end{tabular}
\end{table}

\begin{table}[H]
\small
\centering
\setlength{\tabcolsep}{2pt}
\caption{Cosmopedia - Non-collinear holdout ($\mathcal{H}_{\mathrm{nc}}$). Training and holdout $R^2$ and RMSE per (law, epoch) and design: mean over $30$ optimizer seeds $\pm$ the half-width of its 95\% CI; highlighted: the design with the higher holdout $R^2$; repeated-data (Epoch -) is fit on all epochs at once.}
\renewcommand{\arraystretch}{1.2}
\begin{tabular}{ll l c c c c}
\toprule
\textbf{Epoch} & \textbf{Law} & \textbf{Design} & \textbf{R$^2$ Train} & \textbf{R$^2$ $\mathcal{H}_{\mathrm{nc}}$} & \textbf{RMSE Train} & \textbf{RMSE $\mathcal{H}_{\mathrm{nc}}$} \\
\midrule
\multirow{2}{*}{-} & \multirow{2}{*}{Repeated-data} & Collinear & 0.9841~$\pm$~0.000 & 0.9335~$\pm$~0.000 & 0.1805~$\pm$~0.000 & 0.1504~$\pm$~0.000 \\
 &  & Non-Collinear & 0.9656~$\pm$~0.000 & \win{0.9646~$\pm$~0.000} & 0.1862~$\pm$~0.000 & 0.1097~$\pm$~0.000 \\
\midrule
\multirow{6}{*}{Final} & \multirow{2}{*}{Chinchilla} & Collinear & 0.9857~$\pm$~0.000 & 0.9308~$\pm$~0.006 & 0.1401~$\pm$~0.001 & 0.1323~$\pm$~0.005 \\
 &  & Non-Collinear & 0.9502~$\pm$~0.000 & \win{0.9536~$\pm$~0.004} & 0.1807~$\pm$~0.000 & 0.1085~$\pm$~0.004 \\
\cmidrule{2-7}
 & \multirow{2}{*}{Droppo-Elibol} & Collinear & 0.9861~$\pm$~0.000 & 0.9500~$\pm$~0.002 & 0.1382~$\pm$~0.000 & 0.1131~$\pm$~0.002 \\
 &  & Non-Collinear & 0.9545~$\pm$~0.000 & \win{0.9583~$\pm$~0.001} & 0.1728~$\pm$~0.000 & 0.1033~$\pm$~0.002 \\
\cmidrule{2-7}
 & \multirow{2}{*}{Kaplan} & Collinear & 0.9847~$\pm$~0.000 & 0.9235~$\pm$~0.000 & 0.1446~$\pm$~0.000 & 0.1400~$\pm$~0.000 \\
 &  & Non-Collinear & 0.9500~$\pm$~0.000 & \win{0.9536~$\pm$~0.001} & 0.1811~$\pm$~0.000 & 0.1090~$\pm$~0.001 \\
\midrule
\multirow{6}{*}{First} & \multirow{2}{*}{Chinchilla} & Collinear & 0.9858~$\pm$~0.000 & 0.8814~$\pm$~0.002 & 0.2020~$\pm$~0.000 & 0.2241~$\pm$~0.002 \\
 &  & Non-Collinear & 0.9741~$\pm$~0.000 & \win{0.9556~$\pm$~0.003} & 0.1947~$\pm$~0.001 & 0.1365~$\pm$~0.005 \\
\cmidrule{2-7}
 & \multirow{2}{*}{Droppo-Elibol} & Collinear & 0.9858~$\pm$~0.000 & 0.8713~$\pm$~0.013 & 0.2023~$\pm$~0.001 & 0.2317~$\pm$~0.011 \\
 &  & Non-Collinear & 0.9826~$\pm$~0.000 & \win{0.9684~$\pm$~0.001} & 0.1598~$\pm$~0.000 & 0.1155~$\pm$~0.002 \\
\cmidrule{2-7}
 & \multirow{2}{*}{Kaplan} & Collinear & 0.9814~$\pm$~0.000 & 0.6829~$\pm$~0.016 & 0.2313~$\pm$~0.001 & 0.3660~$\pm$~0.008 \\
 &  & Non-Collinear & 0.9704~$\pm$~0.000 & \win{0.9641~$\pm$~0.001} & 0.2083~$\pm$~0.000 & 0.1234~$\pm$~0.001 \\
\midrule
\multirow{6}{*}{Second} & \multirow{2}{*}{Chinchilla} & Collinear & 0.9867~$\pm$~0.000 & 0.9388~$\pm$~0.003 & 0.1464~$\pm$~0.001 & 0.1323~$\pm$~0.003 \\
 &  & Non-Collinear & 0.9574~$\pm$~0.000 & \win{0.9555~$\pm$~0.006} & 0.1812~$\pm$~0.001 & 0.1117~$\pm$~0.006 \\
\cmidrule{2-7}
 & \multirow{2}{*}{Droppo-Elibol} & Collinear & 0.9869~$\pm$~0.000 & 0.9468~$\pm$~0.005 & 0.1451~$\pm$~0.001 & 0.1228~$\pm$~0.005 \\
 &  & Non-Collinear & 0.9615~$\pm$~0.000 & \win{0.9656~$\pm$~0.001} & 0.1722~$\pm$~0.000 & 0.0993~$\pm$~0.002 \\
\cmidrule{2-7}
 & \multirow{2}{*}{Kaplan} & Collinear & 0.9846~$\pm$~0.000 & 0.9070~$\pm$~0.001 & 0.1575~$\pm$~0.000 & 0.1634~$\pm$~0.001 \\
 &  & Non-Collinear & 0.9569~$\pm$~0.000 & \win{0.9615~$\pm$~0.001} & 0.1822~$\pm$~0.000 & 0.1051~$\pm$~0.001 \\
\bottomrule
\end{tabular}
\end{table}

\begin{table}[H]
\small
\centering
\setlength{\tabcolsep}{2pt}
\caption{peS2o - Unified holdout ($\mathcal{H}$). Training and holdout $R^2$ and RMSE per (law, epoch) and design: mean over $30$ optimizer seeds $\pm$ the half-width of its 95\% CI; highlighted: the design with the higher holdout $R^2$; repeated-data (Epoch -) is fit on all epochs at once.}
\renewcommand{\arraystretch}{1.2}
\begin{tabular}{ll l c c c c}
\toprule
\textbf{Epoch} & \textbf{Law} & \textbf{Design} & \textbf{R$^2$ Train} & \textbf{R$^2$ $\mathcal{H}$} & \textbf{RMSE Train} & \textbf{RMSE $\mathcal{H}$} \\
\midrule
\multirow{2}{*}{-} & \multirow{2}{*}{Repeated-data} & Collinear & 0.9862~$\pm$~0.000 & 0.9232~$\pm$~0.000 & 0.1588~$\pm$~0.000 & 0.1545~$\pm$~0.000 \\
 &  & Non-Collinear & 0.9616~$\pm$~0.000 & \win{0.9464~$\pm$~0.000} & 0.1836~$\pm$~0.000 & 0.1290~$\pm$~0.000 \\
\midrule
\multirow{6}{*}{Final} & \multirow{2}{*}{Chinchilla} & Collinear & 0.9829~$\pm$~0.000 & 0.9055~$\pm$~0.002 & 0.1463~$\pm$~0.000 & 0.1446~$\pm$~0.002 \\
 &  & Non-Collinear & 0.9367~$\pm$~0.000 & \win{0.9421~$\pm$~0.002} & 0.1833~$\pm$~0.000 & 0.1132~$\pm$~0.002 \\
\cmidrule{2-7}
 & \multirow{2}{*}{Droppo-Elibol} & Collinear & 0.9833~$\pm$~0.000 & 0.9081~$\pm$~0.004 & 0.1447~$\pm$~0.000 & 0.1425~$\pm$~0.003 \\
 &  & Non-Collinear & 0.9426~$\pm$~0.000 & \win{0.9520~$\pm$~0.001} & 0.1744~$\pm$~0.001 & 0.1031~$\pm$~0.001 \\
\cmidrule{2-7}
 & \multirow{2}{*}{Kaplan} & Collinear & 0.9810~$\pm$~0.000 & 0.8310~$\pm$~0.002 & 0.1543~$\pm$~0.000 & 0.1935~$\pm$~0.001 \\
 &  & Non-Collinear & 0.9354~$\pm$~0.000 & \win{0.9409~$\pm$~0.002} & 0.1851~$\pm$~0.001 & 0.1143~$\pm$~0.002 \\
\midrule
\multirow{6}{*}{First} & \multirow{2}{*}{Chinchilla} & Collinear & 0.9887~$\pm$~0.000 & 0.8821~$\pm$~0.004 & 0.1674~$\pm$~0.000 & 0.2192~$\pm$~0.003 \\
 &  & Non-Collinear & 0.9734~$\pm$~0.000 & \win{0.9272~$\pm$~0.006} & 0.1873~$\pm$~0.001 & 0.1711~$\pm$~0.007 \\
\cmidrule{2-7}
 & \multirow{2}{*}{Droppo-Elibol} & Collinear & 0.9887~$\pm$~0.000 & 0.8776~$\pm$~0.002 & 0.1675~$\pm$~0.000 & 0.2234~$\pm$~0.002 \\
 &  & Non-Collinear & 0.9810~$\pm$~0.000 & \win{0.9435~$\pm$~0.000} & 0.1586~$\pm$~0.000 & 0.1517~$\pm$~0.001 \\
\cmidrule{2-7}
 & \multirow{2}{*}{Kaplan} & Collinear & 0.9845~$\pm$~0.000 & 0.7131~$\pm$~0.004 & 0.1963~$\pm$~0.000 & 0.3420~$\pm$~0.002 \\
 &  & Non-Collinear & 0.9676~$\pm$~0.000 & \win{0.9444~$\pm$~0.001} & 0.2068~$\pm$~0.001 & 0.1506~$\pm$~0.001 \\
\midrule
\multirow{6}{*}{Second} & \multirow{2}{*}{Chinchilla} & Collinear & 0.9863~$\pm$~0.000 & 0.9108~$\pm$~0.003 & 0.1408~$\pm$~0.000 & 0.1483~$\pm$~0.002 \\
 &  & Non-Collinear & 0.9484~$\pm$~0.000 & \win{0.9389~$\pm$~0.004} & 0.1822~$\pm$~0.001 & 0.1223~$\pm$~0.004 \\
\cmidrule{2-7}
 & \multirow{2}{*}{Droppo-Elibol} & Collinear & 0.9864~$\pm$~0.000 & 0.9105~$\pm$~0.004 & 0.1404~$\pm$~0.000 & 0.1484~$\pm$~0.003 \\
 &  & Non-Collinear & 0.9546~$\pm$~0.000 & \win{0.9534~$\pm$~0.001} & 0.1709~$\pm$~0.000 & 0.1072~$\pm$~0.001 \\
\cmidrule{2-7}
 & \multirow{2}{*}{Kaplan} & Collinear & 0.9834~$\pm$~0.000 & 0.8129~$\pm$~0.002 & 0.1553~$\pm$~0.000 & 0.2150~$\pm$~0.001 \\
 &  & Non-Collinear & 0.9455~$\pm$~0.000 & \win{0.9457~$\pm$~0.001} & 0.1872~$\pm$~0.001 & 0.1158~$\pm$~0.001 \\
\bottomrule
\end{tabular}
\end{table}

\begin{table}[H]
\small
\centering
\setlength{\tabcolsep}{2pt}
\caption{peS2o - Collinear holdout ($\mathcal{H}_{\mathrm{col}}$). Training and holdout $R^2$ and RMSE per (law, epoch) and design: mean over $30$ optimizer seeds $\pm$ the half-width of its 95\% CI; highlighted: the design with the higher holdout $R^2$; repeated-data (Epoch -) is fit on all epochs at once.}
\renewcommand{\arraystretch}{1.2}
\begin{tabular}{ll l c c c c}
\toprule
\textbf{Epoch} & \textbf{Law} & \textbf{Design} & \textbf{R$^2$ Train} & \textbf{R$^2$ $\mathcal{H}_{\mathrm{col}}$} & \textbf{RMSE Train} & \textbf{RMSE $\mathcal{H}_{\mathrm{col}}$} \\
\midrule
\multirow{2}{*}{-} & \multirow{2}{*}{Repeated-data} & Collinear & 0.9862~$\pm$~0.000 & 0.9495~$\pm$~0.000 & 0.1588~$\pm$~0.000 & 0.1312~$\pm$~0.000 \\
 &  & Non-Collinear & 0.9616~$\pm$~0.000 & \win{0.9513~$\pm$~0.000} & 0.1836~$\pm$~0.000 & 0.1290~$\pm$~0.000 \\
\midrule
\multirow{6}{*}{Final} & \multirow{2}{*}{Chinchilla} & Collinear & 0.9829~$\pm$~0.000 & 0.9370~$\pm$~0.003 & 0.1463~$\pm$~0.000 & 0.1263~$\pm$~0.003 \\
 &  & Non-Collinear & 0.9367~$\pm$~0.000 & \win{0.9538~$\pm$~0.002} & 0.1833~$\pm$~0.000 & 0.1082~$\pm$~0.002 \\
\cmidrule{2-7}
 & \multirow{2}{*}{Droppo-Elibol} & Collinear & 0.9833~$\pm$~0.000 & 0.9302~$\pm$~0.003 & 0.1447~$\pm$~0.000 & 0.1330~$\pm$~0.003 \\
 &  & Non-Collinear & 0.9426~$\pm$~0.000 & \win{0.9620~$\pm$~0.001} & 0.1744~$\pm$~0.001 & 0.0983~$\pm$~0.001 \\
\cmidrule{2-7}
 & \multirow{2}{*}{Kaplan} & Collinear & 0.9810~$\pm$~0.000 & 0.8872~$\pm$~0.002 & 0.1543~$\pm$~0.000 & 0.1694~$\pm$~0.001 \\
 &  & Non-Collinear & 0.9354~$\pm$~0.000 & \win{0.9526~$\pm$~0.002} & 0.1851~$\pm$~0.001 & 0.1097~$\pm$~0.002 \\
\midrule
\multirow{6}{*}{First} & \multirow{2}{*}{Chinchilla} & Collinear & 0.9887~$\pm$~0.000 & 0.9107~$\pm$~0.004 & 0.1674~$\pm$~0.000 & 0.1979~$\pm$~0.004 \\
 &  & Non-Collinear & 0.9734~$\pm$~0.000 & \win{0.9314~$\pm$~0.006} & 0.1873~$\pm$~0.001 & 0.1726~$\pm$~0.007 \\
\cmidrule{2-7}
 & \multirow{2}{*}{Droppo-Elibol} & Collinear & 0.9887~$\pm$~0.000 & 0.9077~$\pm$~0.006 & 0.1675~$\pm$~0.000 & 0.2009~$\pm$~0.006 \\
 &  & Non-Collinear & 0.9810~$\pm$~0.000 & \win{0.9451~$\pm$~0.000} & 0.1586~$\pm$~0.000 & 0.1555~$\pm$~0.001 \\
\cmidrule{2-7}
 & \multirow{2}{*}{Kaplan} & Collinear & 0.9845~$\pm$~0.000 & 0.7955~$\pm$~0.003 & 0.1963~$\pm$~0.000 & 0.3000~$\pm$~0.002 \\
 &  & Non-Collinear & 0.9676~$\pm$~0.000 & \win{0.9417~$\pm$~0.001} & 0.2068~$\pm$~0.001 & 0.1602~$\pm$~0.001 \\
\midrule
\multirow{6}{*}{Second} & \multirow{2}{*}{Chinchilla} & Collinear & 0.9863~$\pm$~0.000 & 0.9369~$\pm$~0.003 & 0.1408~$\pm$~0.000 & 0.1311~$\pm$~0.003 \\
 &  & Non-Collinear & 0.9484~$\pm$~0.000 & \win{0.9482~$\pm$~0.004} & 0.1822~$\pm$~0.001 & 0.1185~$\pm$~0.004 \\
\cmidrule{2-7}
 & \multirow{2}{*}{Droppo-Elibol} & Collinear & 0.9864~$\pm$~0.000 & 0.9338~$\pm$~0.004 & 0.1404~$\pm$~0.000 & 0.1342~$\pm$~0.004 \\
 &  & Non-Collinear & 0.9546~$\pm$~0.000 & \win{0.9607~$\pm$~0.001} & 0.1709~$\pm$~0.000 & 0.1037~$\pm$~0.001 \\
\cmidrule{2-7}
 & \multirow{2}{*}{Kaplan} & Collinear & 0.9834~$\pm$~0.000 & 0.8692~$\pm$~0.002 & 0.1553~$\pm$~0.000 & 0.1893~$\pm$~0.001 \\
 &  & Non-Collinear & 0.9455~$\pm$~0.000 & \win{0.9517~$\pm$~0.001} & 0.1872~$\pm$~0.001 & 0.1149~$\pm$~0.002 \\
\bottomrule
\end{tabular}
\end{table}

\begin{table}[H]
\small
\centering
\setlength{\tabcolsep}{2pt}
\caption{peS2o - Non-collinear holdout ($\mathcal{H}_{\mathrm{nc}}$). Training and holdout $R^2$ and RMSE per (law, epoch) and design: mean over $30$ optimizer seeds $\pm$ the half-width of its 95\% CI; highlighted: the design with the higher holdout $R^2$; repeated-data (Epoch -) is fit on all epochs at once.}
\renewcommand{\arraystretch}{1.2}
\begin{tabular}{ll l c c c c}
\toprule
\textbf{Epoch} & \textbf{Law} & \textbf{Design} & \textbf{R$^2$ Train} & \textbf{R$^2$ $\mathcal{H}_{\mathrm{nc}}$} & \textbf{RMSE Train} & \textbf{RMSE $\mathcal{H}_{\mathrm{nc}}$} \\
\midrule
\multirow{2}{*}{-} & \multirow{2}{*}{Repeated-data} & Collinear & 0.9862~$\pm$~0.000 & 0.8911~$\pm$~0.000 & 0.1588~$\pm$~0.000 & 0.1747~$\pm$~0.000 \\
 &  & Non-Collinear & 0.9616~$\pm$~0.000 & \win{0.9405~$\pm$~0.000} & 0.1836~$\pm$~0.000 & 0.1291~$\pm$~0.000 \\
\midrule
\multirow{6}{*}{Final} & \multirow{2}{*}{Chinchilla} & Collinear & 0.9829~$\pm$~0.000 & 0.8631~$\pm$~0.002 & 0.1463~$\pm$~0.000 & 0.1607~$\pm$~0.001 \\
 &  & Non-Collinear & 0.9367~$\pm$~0.000 & \win{0.9263~$\pm$~0.002} & 0.1833~$\pm$~0.000 & 0.1179~$\pm$~0.002 \\
\cmidrule{2-7}
 & \multirow{2}{*}{Droppo-Elibol} & Collinear & 0.9833~$\pm$~0.000 & 0.8783~$\pm$~0.006 & 0.1447~$\pm$~0.000 & 0.1513~$\pm$~0.003 \\
 &  & Non-Collinear & 0.9426~$\pm$~0.000 & \win{0.9385~$\pm$~0.001} & 0.1744~$\pm$~0.001 & 0.1077~$\pm$~0.001 \\
\cmidrule{2-7}
 & \multirow{2}{*}{Kaplan} & Collinear & 0.9810~$\pm$~0.000 & 0.7552~$\pm$~0.002 & 0.1543~$\pm$~0.000 & 0.2149~$\pm$~0.001 \\
 &  & Non-Collinear & 0.9354~$\pm$~0.000 & \win{0.9252~$\pm$~0.002} & 0.1851~$\pm$~0.001 & 0.1187~$\pm$~0.001 \\
\midrule
\multirow{6}{*}{First} & \multirow{2}{*}{Chinchilla} & Collinear & 0.9887~$\pm$~0.000 & 0.8484~$\pm$~0.003 & 0.1674~$\pm$~0.000 & 0.2385~$\pm$~0.003 \\
 &  & Non-Collinear & 0.9734~$\pm$~0.000 & \win{0.9223~$\pm$~0.007} & 0.1873~$\pm$~0.001 & 0.1696~$\pm$~0.007 \\
\cmidrule{2-7}
 & \multirow{2}{*}{Droppo-Elibol} & Collinear & 0.9887~$\pm$~0.000 & 0.8423~$\pm$~0.005 & 0.1675~$\pm$~0.000 & 0.2431~$\pm$~0.004 \\
 &  & Non-Collinear & 0.9810~$\pm$~0.000 & \win{0.9417~$\pm$~0.000} & 0.1586~$\pm$~0.000 & 0.1479~$\pm$~0.001 \\
\cmidrule{2-7}
 & \multirow{2}{*}{Kaplan} & Collinear & 0.9845~$\pm$~0.000 & 0.6165~$\pm$~0.004 & 0.1963~$\pm$~0.000 & 0.3794~$\pm$~0.002 \\
 &  & Non-Collinear & 0.9676~$\pm$~0.000 & \win{0.9476~$\pm$~0.002} & 0.2068~$\pm$~0.001 & 0.1401~$\pm$~0.002 \\
\midrule
\multirow{6}{*}{Second} & \multirow{2}{*}{Chinchilla} & Collinear & 0.9863~$\pm$~0.000 & 0.8782~$\pm$~0.003 & 0.1408~$\pm$~0.000 & 0.1637~$\pm$~0.002 \\
 &  & Non-Collinear & 0.9484~$\pm$~0.000 & \win{0.9274~$\pm$~0.005} & 0.1822~$\pm$~0.001 & 0.1260~$\pm$~0.004 \\
\cmidrule{2-7}
 & \multirow{2}{*}{Droppo-Elibol} & Collinear & 0.9864~$\pm$~0.000 & 0.8814~$\pm$~0.007 & 0.1404~$\pm$~0.000 & 0.1610~$\pm$~0.005 \\
 &  & Non-Collinear & 0.9546~$\pm$~0.000 & \win{0.9443~$\pm$~0.001} & 0.1709~$\pm$~0.000 & 0.1106~$\pm$~0.001 \\
\cmidrule{2-7}
 & \multirow{2}{*}{Kaplan} & Collinear & 0.9834~$\pm$~0.000 & 0.7428~$\pm$~0.003 & 0.1553~$\pm$~0.000 & 0.2379~$\pm$~0.001 \\
 &  & Non-Collinear & 0.9455~$\pm$~0.000 & \win{0.9381~$\pm$~0.001} & 0.1872~$\pm$~0.001 & 0.1167~$\pm$~0.001 \\
\bottomrule
\end{tabular}
\end{table}

\begin{table}[H]
\small
\centering
\setlength{\tabcolsep}{2pt}
\caption{RedPajama - Unified holdout ($\mathcal{H}$). Training and holdout $R^2$ and RMSE per (law, epoch) and design: mean over $30$ optimizer seeds $\pm$ the half-width of its 95\% CI; highlighted: the design with the higher holdout $R^2$; repeated-data (Epoch -) is fit on all epochs at once.}
\renewcommand{\arraystretch}{1.2}
\begin{tabular}{ll l c c c c}
\toprule
\textbf{Epoch} & \textbf{Law} & \textbf{Design} & \textbf{R$^2$ Train} & \textbf{R$^2$ $\mathcal{H}$} & \textbf{RMSE Train} & \textbf{RMSE $\mathcal{H}$} \\
\midrule
\multirow{2}{*}{-} & \multirow{2}{*}{Repeated-data} & Collinear & 0.9879~$\pm$~0.000 & 0.8905~$\pm$~0.000 & 0.1842~$\pm$~0.000 & 0.2391~$\pm$~0.000 \\
 &  & Non-Collinear & 0.9644~$\pm$~0.000 & \win{0.9290~$\pm$~0.000} & 0.2206~$\pm$~0.000 & 0.1926~$\pm$~0.000 \\
\midrule
\multirow{6}{*}{Final} & \multirow{2}{*}{Chinchilla} & Collinear & 0.9931~$\pm$~0.000 & 0.8802~$\pm$~0.001 & 0.1254~$\pm$~0.000 & 0.2176~$\pm$~0.001 \\
 &  & Non-Collinear & 0.9444~$\pm$~0.000 & \win{0.9079~$\pm$~0.002} & 0.2272~$\pm$~0.000 & 0.1908~$\pm$~0.002 \\
\cmidrule{2-7}
 & \multirow{2}{*}{Droppo-Elibol} & Collinear & 0.9932~$\pm$~0.000 & 0.8932~$\pm$~0.008 & 0.1247~$\pm$~0.000 & 0.2046~$\pm$~0.007 \\
 &  & Non-Collinear & 0.9594~$\pm$~0.000 & \win{0.9177~$\pm$~0.001} & 0.1941~$\pm$~0.000 & 0.1803~$\pm$~0.001 \\
\cmidrule{2-7}
 & \multirow{2}{*}{Kaplan} & Collinear & 0.9866~$\pm$~0.000 & 0.6643~$\pm$~0.003 & 0.1750~$\pm$~0.000 & 0.3642~$\pm$~0.002 \\
 &  & Non-Collinear & 0.9189~$\pm$~0.001 & \win{0.8776~$\pm$~0.003} & 0.2741~$\pm$~0.002 & 0.2199~$\pm$~0.002 \\
\midrule
\multirow{6}{*}{First} & \multirow{2}{*}{Chinchilla} & Collinear & 0.9851~$\pm$~0.000 & 0.7692~$\pm$~0.006 & 0.2236~$\pm$~0.000 & 0.3902~$\pm$~0.005 \\
 &  & Non-Collinear & 0.9790~$\pm$~0.000 & \win{0.9102~$\pm$~0.004} & 0.1992~$\pm$~0.001 & 0.2430~$\pm$~0.006 \\
\cmidrule{2-7}
 & \multirow{2}{*}{Droppo-Elibol} & Collinear & 0.9852~$\pm$~0.000 & 0.7471~$\pm$~0.009 & 0.2226~$\pm$~0.000 & 0.4082~$\pm$~0.007 \\
 &  & Non-Collinear & 0.9865~$\pm$~0.000 & \win{0.9299~$\pm$~0.001} & 0.1595~$\pm$~0.000 & 0.2152~$\pm$~0.001 \\
\cmidrule{2-7}
 & \multirow{2}{*}{Kaplan} & Collinear & 0.9838~$\pm$~0.000 & 0.5222~$\pm$~0.005 & 0.2332~$\pm$~0.000 & 0.5618~$\pm$~0.003 \\
 &  & Non-Collinear & 0.9571~$\pm$~0.000 & \win{0.9035~$\pm$~0.001} & 0.2844~$\pm$~0.000 & 0.2525~$\pm$~0.001 \\
\midrule
\multirow{6}{*}{Second} & \multirow{2}{*}{Chinchilla} & Collinear & 0.9923~$\pm$~0.000 & 0.8823~$\pm$~0.002 & 0.1388~$\pm$~0.000 & 0.2252~$\pm$~0.002 \\
 &  & Non-Collinear & 0.9538~$\pm$~0.000 & \win{0.9118~$\pm$~0.004} & 0.2235~$\pm$~0.001 & 0.1945~$\pm$~0.004 \\
\cmidrule{2-7}
 & \multirow{2}{*}{Droppo-Elibol} & Collinear & 0.9923~$\pm$~0.000 & 0.8877~$\pm$~0.008 & 0.1385~$\pm$~0.000 & 0.2190~$\pm$~0.007 \\
 &  & Non-Collinear & 0.9678~$\pm$~0.000 & \win{0.9217~$\pm$~0.001} & 0.1866~$\pm$~0.000 & 0.1836~$\pm$~0.001 \\
\cmidrule{2-7}
 & \multirow{2}{*}{Kaplan} & Collinear & 0.9856~$\pm$~0.000 & 0.6351~$\pm$~0.006 & 0.1898~$\pm$~0.000 & 0.3964~$\pm$~0.003 \\
 &  & Non-Collinear & 0.9265~$\pm$~0.001 & \win{0.8827~$\pm$~0.004} & 0.2819~$\pm$~0.002 & 0.2246~$\pm$~0.004 \\
\bottomrule
\end{tabular}
\end{table}

\begin{table}[H]
\small
\centering
\setlength{\tabcolsep}{2pt}
\caption{RedPajama - Collinear holdout ($\mathcal{H}_{\mathrm{col}}$). Training and holdout $R^2$ and RMSE per (law, epoch) and design: mean over $30$ optimizer seeds $\pm$ the half-width of its 95\% CI; highlighted: the design with the higher holdout $R^2$; repeated-data (Epoch -) is fit on all epochs at once.}
\renewcommand{\arraystretch}{1.2}
\begin{tabular}{ll l c c c c}
\toprule
\textbf{Epoch} & \textbf{Law} & \textbf{Design} & \textbf{R$^2$ Train} & \textbf{R$^2$ $\mathcal{H}_{\mathrm{col}}$} & \textbf{RMSE Train} & \textbf{RMSE $\mathcal{H}_{\mathrm{col}}$} \\
\midrule
\multirow{2}{*}{-} & \multirow{2}{*}{Repeated-data} & Collinear & 0.9879~$\pm$~0.000 & \win{0.9604~$\pm$~0.000} & 0.1842~$\pm$~0.000 & 0.1648~$\pm$~0.000 \\
 &  & Non-Collinear & 0.9644~$\pm$~0.000 & 0.9427~$\pm$~0.000 & 0.2206~$\pm$~0.000 & 0.1981~$\pm$~0.000 \\
\midrule
\multirow{6}{*}{Final} & \multirow{2}{*}{Chinchilla} & Collinear & 0.9931~$\pm$~0.000 & \win{0.9672~$\pm$~0.001} & 0.1254~$\pm$~0.000 & 0.1329~$\pm$~0.002 \\
 &  & Non-Collinear & 0.9444~$\pm$~0.000 & 0.9324~$\pm$~0.002 & 0.2272~$\pm$~0.000 & 0.1908~$\pm$~0.002 \\
\cmidrule{2-7}
 & \multirow{2}{*}{Droppo-Elibol} & Collinear & 0.9932~$\pm$~0.000 & \win{0.9635~$\pm$~0.002} & 0.1247~$\pm$~0.000 & 0.1401~$\pm$~0.003 \\
 &  & Non-Collinear & 0.9594~$\pm$~0.000 & 0.9379~$\pm$~0.001 & 0.1941~$\pm$~0.000 & 0.1830~$\pm$~0.001 \\
\cmidrule{2-7}
 & \multirow{2}{*}{Kaplan} & Collinear & 0.9866~$\pm$~0.000 & 0.8875~$\pm$~0.002 & 0.1750~$\pm$~0.000 & 0.2463~$\pm$~0.002 \\
 &  & Non-Collinear & 0.9189~$\pm$~0.001 & \win{0.9080~$\pm$~0.003} & 0.2741~$\pm$~0.002 & 0.2226~$\pm$~0.004 \\
\midrule
\multirow{6}{*}{First} & \multirow{2}{*}{Chinchilla} & Collinear & 0.9851~$\pm$~0.000 & 0.8506~$\pm$~0.005 & 0.2236~$\pm$~0.000 & 0.3570~$\pm$~0.006 \\
 &  & Non-Collinear & 0.9790~$\pm$~0.000 & \win{0.9292~$\pm$~0.004} & 0.1992~$\pm$~0.001 & 0.2454~$\pm$~0.007 \\
\cmidrule{2-7}
 & \multirow{2}{*}{Droppo-Elibol} & Collinear & 0.9852~$\pm$~0.000 & 0.8981~$\pm$~0.007 & 0.2226~$\pm$~0.000 & 0.2941~$\pm$~0.009 \\
 &  & Non-Collinear & 0.9865~$\pm$~0.000 & \win{0.9416~$\pm$~0.001} & 0.1595~$\pm$~0.000 & 0.2234~$\pm$~0.001 \\
\cmidrule{2-7}
 & \multirow{2}{*}{Kaplan} & Collinear & 0.9838~$\pm$~0.000 & 0.8069~$\pm$~0.003 & 0.2332~$\pm$~0.000 & 0.4062~$\pm$~0.004 \\
 &  & Non-Collinear & 0.9571~$\pm$~0.000 & \win{0.9228~$\pm$~0.000} & 0.2844~$\pm$~0.000 & 0.2569~$\pm$~0.001 \\
\midrule
\multirow{6}{*}{Second} & \multirow{2}{*}{Chinchilla} & Collinear & 0.9923~$\pm$~0.000 & \win{0.9625~$\pm$~0.002} & 0.1388~$\pm$~0.000 & 0.1471~$\pm$~0.003 \\
 &  & Non-Collinear & 0.9538~$\pm$~0.000 & 0.9374~$\pm$~0.003 & 0.2235~$\pm$~0.001 & 0.1901~$\pm$~0.005 \\
\cmidrule{2-7}
 & \multirow{2}{*}{Droppo-Elibol} & Collinear & 0.9923~$\pm$~0.000 & \win{0.9577~$\pm$~0.003} & 0.1385~$\pm$~0.000 & 0.1560~$\pm$~0.005 \\
 &  & Non-Collinear & 0.9678~$\pm$~0.000 & 0.9421~$\pm$~0.001 & 0.1866~$\pm$~0.000 & 0.1832~$\pm$~0.001 \\
\cmidrule{2-7}
 & \multirow{2}{*}{Kaplan} & Collinear & 0.9856~$\pm$~0.000 & 0.8666~$\pm$~0.003 & 0.1898~$\pm$~0.000 & 0.2780~$\pm$~0.004 \\
 &  & Non-Collinear & 0.9265~$\pm$~0.001 & \win{0.9099~$\pm$~0.005} & 0.2819~$\pm$~0.002 & 0.2280~$\pm$~0.006 \\
\bottomrule
\end{tabular}
\end{table}

\begin{table}[H]
\small
\centering
\setlength{\tabcolsep}{2pt}
\caption{RedPajama - Non-collinear holdout ($\mathcal{H}_{\mathrm{nc}}$). Training and holdout $R^2$ and RMSE per (law, epoch) and design: mean over $30$ optimizer seeds $\pm$ the half-width of its 95\% CI; highlighted: the design with the higher holdout $R^2$; repeated-data (Epoch -) is fit on all epochs at once.}
\renewcommand{\arraystretch}{1.2}
\begin{tabular}{ll l c c c c}
\toprule
\textbf{Epoch} & \textbf{Law} & \textbf{Design} & \textbf{R$^2$ Train} & \textbf{R$^2$ $\mathcal{H}_{\mathrm{nc}}$} & \textbf{RMSE Train} & \textbf{RMSE $\mathcal{H}_{\mathrm{nc}}$} \\
\midrule
\multirow{2}{*}{-} & \multirow{2}{*}{Repeated-data} & Collinear & 0.9879~$\pm$~0.000 & 0.7560~$\pm$~0.000 & 0.1842~$\pm$~0.000 & 0.2953~$\pm$~0.000 \\
 &  & Non-Collinear & 0.9644~$\pm$~0.000 & \win{0.9023~$\pm$~0.000} & 0.2206~$\pm$~0.000 & 0.1868~$\pm$~0.000 \\
\midrule
\multirow{6}{*}{Final} & \multirow{2}{*}{Chinchilla} & Collinear & 0.9931~$\pm$~0.000 & 0.6910~$\pm$~0.001 & 0.1254~$\pm$~0.000 & 0.2775~$\pm$~0.001 \\
 &  & Non-Collinear & 0.9444~$\pm$~0.000 & \win{0.8540~$\pm$~0.003} & 0.2272~$\pm$~0.000 & 0.1907~$\pm$~0.002 \\
\cmidrule{2-7}
 & \multirow{2}{*}{Droppo-Elibol} & Collinear & 0.9932~$\pm$~0.000 & 0.7403~$\pm$~0.026 & 0.1247~$\pm$~0.000 & 0.2523~$\pm$~0.012 \\
 &  & Non-Collinear & 0.9594~$\pm$~0.000 & \win{0.8734~$\pm$~0.002} & 0.1941~$\pm$~0.000 & 0.1776~$\pm$~0.001 \\
\cmidrule{2-7}
 & \multirow{2}{*}{Kaplan} & Collinear & 0.9866~$\pm$~0.000 & 0.1789~$\pm$~0.006 & 0.1750~$\pm$~0.000 & 0.4524~$\pm$~0.002 \\
 &  & Non-Collinear & 0.9189~$\pm$~0.001 & \win{0.8108~$\pm$~0.002} & 0.2741~$\pm$~0.002 & 0.2171~$\pm$~0.001 \\
\midrule
\multirow{6}{*}{First} & \multirow{2}{*}{Chinchilla} & Collinear & 0.9851~$\pm$~0.000 & 0.6182~$\pm$~0.007 & 0.2236~$\pm$~0.000 & 0.4208~$\pm$~0.004 \\
 &  & Non-Collinear & 0.9790~$\pm$~0.000 & \win{0.8750~$\pm$~0.005} & 0.1992~$\pm$~0.001 & 0.2405~$\pm$~0.005 \\
\cmidrule{2-7}
 & \multirow{2}{*}{Droppo-Elibol} & Collinear & 0.9852~$\pm$~0.000 & 0.4679~$\pm$~0.035 & 0.2226~$\pm$~0.000 & 0.4948~$\pm$~0.017 \\
 &  & Non-Collinear & 0.9865~$\pm$~0.000 & \win{0.9081~$\pm$~0.001} & 0.1595~$\pm$~0.000 & 0.2065~$\pm$~0.001 \\
\cmidrule{2-7}
 & \multirow{2}{*}{Kaplan} & Collinear & 0.9838~$\pm$~0.000 & -0.0044~$\pm$~0.009 & 0.2332~$\pm$~0.000 & 0.6827~$\pm$~0.003 \\
 &  & Non-Collinear & 0.9571~$\pm$~0.000 & \win{0.8675~$\pm$~0.003} & 0.2844~$\pm$~0.000 & 0.2479~$\pm$~0.002 \\
\midrule
\multirow{6}{*}{Second} & \multirow{2}{*}{Chinchilla} & Collinear & 0.9923~$\pm$~0.000 & 0.7156~$\pm$~0.002 & 0.1388~$\pm$~0.000 & 0.2823~$\pm$~0.001 \\
 &  & Non-Collinear & 0.9538~$\pm$~0.000 & \win{0.8586~$\pm$~0.005} & 0.2235~$\pm$~0.001 & 0.1988~$\pm$~0.004 \\
\cmidrule{2-7}
 & \multirow{2}{*}{Droppo-Elibol} & Collinear & 0.9923~$\pm$~0.000 & 0.7423~$\pm$~0.028 & 0.1385~$\pm$~0.000 & 0.2661~$\pm$~0.014 \\
 &  & Non-Collinear & 0.9678~$\pm$~0.000 & \win{0.8792~$\pm$~0.001} & 0.1866~$\pm$~0.000 & 0.1840~$\pm$~0.001 \\
\cmidrule{2-7}
 & \multirow{2}{*}{Kaplan} & Collinear & 0.9856~$\pm$~0.000 & 0.1543~$\pm$~0.010 & 0.1898~$\pm$~0.000 & 0.4868~$\pm$~0.003 \\
 &  & Non-Collinear & 0.9265~$\pm$~0.001 & \win{0.8258~$\pm$~0.003} & 0.2819~$\pm$~0.002 & 0.2209~$\pm$~0.002 \\
\bottomrule
\end{tabular}
\end{table}

\begin{table}[H]
\small
\centering
\setlength{\tabcolsep}{2pt}
\caption{Wikipedia - Unified holdout ($\mathcal{H}$). Training and holdout $R^2$ and RMSE per (law, epoch) and design: mean over $30$ optimizer seeds $\pm$ the half-width of its 95\% CI; highlighted: the design with the higher holdout $R^2$; repeated-data (Epoch -) is fit on all epochs at once.}
\renewcommand{\arraystretch}{1.2}
\begin{tabular}{ll l c c c c}
\toprule
\textbf{Epoch} & \textbf{Law} & \textbf{Design} & \textbf{R$^2$ Train} & \textbf{R$^2$ $\mathcal{H}$} & \textbf{RMSE Train} & \textbf{RMSE $\mathcal{H}$} \\
\midrule
\multirow{2}{*}{-} & \multirow{2}{*}{Repeated-data} & Collinear & 0.9853~$\pm$~0.000 & 0.8666~$\pm$~0.000 & 0.2247~$\pm$~0.000 & 0.2779~$\pm$~0.000 \\
 &  & Non-Collinear & 0.9585~$\pm$~0.000 & \win{0.9155~$\pm$~0.000} & 0.2594~$\pm$~0.000 & 0.2212~$\pm$~0.000 \\
\midrule
\multirow{6}{*}{Final} & \multirow{2}{*}{Chinchilla} & Collinear & 0.9918~$\pm$~0.000 & 0.8609~$\pm$~0.002 & 0.1507~$\pm$~0.000 & 0.2462~$\pm$~0.002 \\
 &  & Non-Collinear & 0.9372~$\pm$~0.000 & \win{0.8974~$\pm$~0.002} & 0.2626~$\pm$~0.000 & 0.2115~$\pm$~0.002 \\
\cmidrule{2-7}
 & \multirow{2}{*}{Droppo-Elibol} & Collinear & 0.9919~$\pm$~0.000 & 0.8701~$\pm$~0.011 & 0.1500~$\pm$~0.000 & 0.2366~$\pm$~0.010 \\
 &  & Non-Collinear & 0.9567~$\pm$~0.000 & \win{0.9136~$\pm$~0.001} & 0.2181~$\pm$~0.000 & 0.1941~$\pm$~0.001 \\
\cmidrule{2-7}
 & \multirow{2}{*}{Kaplan} & Collinear & 0.9849~$\pm$~0.000 & 0.6063~$\pm$~0.003 & 0.2052~$\pm$~0.000 & 0.4143~$\pm$~0.002 \\
 &  & Non-Collinear & 0.9271~$\pm$~0.000 & \win{0.8850~$\pm$~0.000} & 0.2830~$\pm$~0.000 & 0.2240~$\pm$~0.000 \\
\midrule
\multirow{6}{*}{First} & \multirow{2}{*}{Chinchilla} & Collinear & 0.9823~$\pm$~0.000 & 0.7438~$\pm$~0.006 & 0.2704~$\pm$~0.001 & 0.4360~$\pm$~0.006 \\
 &  & Non-Collinear & 0.9750~$\pm$~0.000 & \win{0.8947~$\pm$~0.003} & 0.2367~$\pm$~0.001 & 0.2796~$\pm$~0.004 \\
\cmidrule{2-7}
 & \multirow{2}{*}{Droppo-Elibol} & Collinear & 0.9823~$\pm$~0.000 & 0.7450~$\pm$~0.013 & 0.2704~$\pm$~0.000 & 0.4342~$\pm$~0.011 \\
 &  & Non-Collinear & 0.9847~$\pm$~0.000 & \win{0.9222~$\pm$~0.001} & 0.1848~$\pm$~0.000 & 0.2404~$\pm$~0.001 \\
\cmidrule{2-7}
 & \multirow{2}{*}{Kaplan} & Collinear & 0.9786~$\pm$~0.001 & 0.5540~$\pm$~0.025 & 0.2963~$\pm$~0.009 & 0.5737~$\pm$~0.017 \\
 &  & Non-Collinear & 0.9550~$\pm$~0.000 & \win{0.8903~$\pm$~0.001} & 0.3173~$\pm$~0.000 & 0.2855~$\pm$~0.001 \\
\midrule
\multirow{6}{*}{Second} & \multirow{2}{*}{Chinchilla} & Collinear & 0.9906~$\pm$~0.000 & 0.8575~$\pm$~0.004 & 0.1691~$\pm$~0.000 & 0.2623~$\pm$~0.003 \\
 &  & Non-Collinear & 0.9461~$\pm$~0.000 & \win{0.8804~$\pm$~0.003} & 0.2640~$\pm$~0.000 & 0.2404~$\pm$~0.003 \\
\cmidrule{2-7}
 & \multirow{2}{*}{Droppo-Elibol} & Collinear & 0.9906~$\pm$~0.000 & 0.8467~$\pm$~0.012 & 0.1692~$\pm$~0.000 & 0.2708~$\pm$~0.010 \\
 &  & Non-Collinear & 0.9643~$\pm$~0.000 & \win{0.8976~$\pm$~0.001} & 0.2148~$\pm$~0.000 & 0.2225~$\pm$~0.001 \\
\cmidrule{2-7}
 & \multirow{2}{*}{Kaplan} & Collinear & 0.9843~$\pm$~0.000 & 0.5652~$\pm$~0.006 & 0.2193~$\pm$~0.000 & 0.4584~$\pm$~0.003 \\
 &  & Non-Collinear & 0.9314~$\pm$~0.001 & \win{0.8843~$\pm$~0.000} & 0.2977~$\pm$~0.002 & 0.2365~$\pm$~0.000 \\
\bottomrule
\end{tabular}
\end{table}

\begin{table}[H]
\small
\centering
\setlength{\tabcolsep}{2pt}
\caption{Wikipedia - Collinear holdout ($\mathcal{H}_{\mathrm{col}}$). Training and holdout $R^2$ and RMSE per (law, epoch) and design: mean over $30$ optimizer seeds $\pm$ the half-width of its 95\% CI; highlighted: the design with the higher holdout $R^2$; repeated-data (Epoch -) is fit on all epochs at once.}
\renewcommand{\arraystretch}{1.2}
\begin{tabular}{ll l c c c c}
\toprule
\textbf{Epoch} & \textbf{Law} & \textbf{Design} & \textbf{R$^2$ Train} & \textbf{R$^2$ $\mathcal{H}_{\mathrm{col}}$} & \textbf{RMSE Train} & \textbf{RMSE $\mathcal{H}_{\mathrm{col}}$} \\
\midrule
\multirow{2}{*}{-} & \multirow{2}{*}{Repeated-data} & Collinear & 0.9853~$\pm$~0.000 & \win{0.9521~$\pm$~0.000} & 0.2247~$\pm$~0.000 & 0.1930~$\pm$~0.000 \\
 &  & Non-Collinear & 0.9585~$\pm$~0.000 & 0.9317~$\pm$~0.000 & 0.2594~$\pm$~0.000 & 0.2304~$\pm$~0.000 \\
\midrule
\multirow{6}{*}{Final} & \multirow{2}{*}{Chinchilla} & Collinear & 0.9918~$\pm$~0.000 & \win{0.9632~$\pm$~0.002} & 0.1507~$\pm$~0.000 & 0.1485~$\pm$~0.004 \\
 &  & Non-Collinear & 0.9372~$\pm$~0.000 & 0.9280~$\pm$~0.001 & 0.2626~$\pm$~0.000 & 0.2085~$\pm$~0.002 \\
\cmidrule{2-7}
 & \multirow{2}{*}{Droppo-Elibol} & Collinear & 0.9919~$\pm$~0.000 & \win{0.9611~$\pm$~0.002} & 0.1500~$\pm$~0.000 & 0.1529~$\pm$~0.004 \\
 &  & Non-Collinear & 0.9567~$\pm$~0.000 & 0.9377~$\pm$~0.001 & 0.2181~$\pm$~0.000 & 0.1939~$\pm$~0.001 \\
\cmidrule{2-7}
 & \multirow{2}{*}{Kaplan} & Collinear & 0.9849~$\pm$~0.000 & 0.8583~$\pm$~0.002 & 0.2052~$\pm$~0.000 & 0.2925~$\pm$~0.002 \\
 &  & Non-Collinear & 0.9271~$\pm$~0.000 & \win{0.9099~$\pm$~0.000} & 0.2830~$\pm$~0.000 & 0.2333~$\pm$~0.000 \\
\midrule
\multirow{6}{*}{First} & \multirow{2}{*}{Chinchilla} & Collinear & 0.9823~$\pm$~0.000 & 0.8384~$\pm$~0.005 & 0.2704~$\pm$~0.001 & 0.3984~$\pm$~0.007 \\
 &  & Non-Collinear & 0.9750~$\pm$~0.000 & \win{0.9201~$\pm$~0.002} & 0.2367~$\pm$~0.001 & 0.2803~$\pm$~0.004 \\
\cmidrule{2-7}
 & \multirow{2}{*}{Droppo-Elibol} & Collinear & 0.9823~$\pm$~0.000 & 0.8361~$\pm$~0.002 & 0.2704~$\pm$~0.000 & 0.4016~$\pm$~0.002 \\
 &  & Non-Collinear & 0.9847~$\pm$~0.000 & \win{0.9371~$\pm$~0.001} & 0.1848~$\pm$~0.000 & 0.2489~$\pm$~0.001 \\
\cmidrule{2-7}
 & \multirow{2}{*}{Kaplan} & Collinear & 0.9786~$\pm$~0.001 & 0.8566~$\pm$~0.018 & 0.2963~$\pm$~0.009 & 0.3691~$\pm$~0.025 \\
 &  & Non-Collinear & 0.9550~$\pm$~0.000 & \win{0.9162~$\pm$~0.000} & 0.3173~$\pm$~0.000 & 0.2872~$\pm$~0.000 \\
\midrule
\multirow{6}{*}{Second} & \multirow{2}{*}{Chinchilla} & Collinear & 0.9906~$\pm$~0.000 & \win{0.9544~$\pm$~0.003} & 0.1691~$\pm$~0.000 & 0.1733~$\pm$~0.006 \\
 &  & Non-Collinear & 0.9461~$\pm$~0.000 & 0.9138~$\pm$~0.002 & 0.2640~$\pm$~0.000 & 0.2394~$\pm$~0.003 \\
\cmidrule{2-7}
 & \multirow{2}{*}{Droppo-Elibol} & Collinear & 0.9906~$\pm$~0.000 & \win{0.9525~$\pm$~0.005} & 0.1692~$\pm$~0.000 & 0.1762~$\pm$~0.008 \\
 &  & Non-Collinear & 0.9643~$\pm$~0.000 & 0.9240~$\pm$~0.001 & 0.2148~$\pm$~0.000 & 0.2248~$\pm$~0.001 \\
\cmidrule{2-7}
 & \multirow{2}{*}{Kaplan} & Collinear & 0.9843~$\pm$~0.000 & 0.8332~$\pm$~0.004 & 0.2193~$\pm$~0.000 & 0.3330~$\pm$~0.003 \\
 &  & Non-Collinear & 0.9314~$\pm$~0.001 & \win{0.9059~$\pm$~0.001} & 0.2977~$\pm$~0.002 & 0.2501~$\pm$~0.001 \\
\bottomrule
\end{tabular}
\end{table}

\begin{table}[H]
\small
\centering
\setlength{\tabcolsep}{2pt}
\caption{Wikipedia - Non-collinear holdout ($\mathcal{H}_{\mathrm{nc}}$). Training and holdout $R^2$ and RMSE per (law, epoch) and design: mean over $30$ optimizer seeds $\pm$ the half-width of its 95\% CI; highlighted: the design with the higher holdout $R^2$; repeated-data (Epoch -) is fit on all epochs at once.}
\renewcommand{\arraystretch}{1.2}
\begin{tabular}{ll l c c c c}
\toprule
\textbf{Epoch} & \textbf{Law} & \textbf{Design} & \textbf{R$^2$ Train} & \textbf{R$^2$ $\mathcal{H}_{\mathrm{nc}}$} & \textbf{RMSE Train} & \textbf{RMSE $\mathcal{H}_{\mathrm{nc}}$} \\
\midrule
\multirow{2}{*}{-} & \multirow{2}{*}{Repeated-data} & Collinear & 0.9853~$\pm$~0.000 & 0.6905~$\pm$~0.000 & 0.2247~$\pm$~0.000 & 0.3423~$\pm$~0.000 \\
 &  & Non-Collinear & 0.9585~$\pm$~0.000 & \win{0.8817~$\pm$~0.000} & 0.2594~$\pm$~0.000 & 0.2116~$\pm$~0.000 \\
\midrule
\multirow{6}{*}{Final} & \multirow{2}{*}{Chinchilla} & Collinear & 0.9918~$\pm$~0.000 & 0.6288~$\pm$~0.002 & 0.1507~$\pm$~0.000 & 0.3148~$\pm$~0.001 \\
 &  & Non-Collinear & 0.9372~$\pm$~0.000 & \win{0.8276~$\pm$~0.003} & 0.2626~$\pm$~0.000 & 0.2145~$\pm$~0.002 \\
\cmidrule{2-7}
 & \multirow{2}{*}{Droppo-Elibol} & Collinear & 0.9919~$\pm$~0.000 & 0.6637~$\pm$~0.037 & 0.1500~$\pm$~0.000 & 0.2963~$\pm$~0.016 \\
 &  & Non-Collinear & 0.9567~$\pm$~0.000 & \win{0.8586~$\pm$~0.002} & 0.2181~$\pm$~0.000 & 0.1943~$\pm$~0.001 \\
\cmidrule{2-7}
 & \multirow{2}{*}{Kaplan} & Collinear & 0.9849~$\pm$~0.000 & 0.0343~$\pm$~0.006 & 0.2052~$\pm$~0.000 & 0.5078~$\pm$~0.002 \\
 &  & Non-Collinear & 0.9271~$\pm$~0.000 & \win{0.8281~$\pm$~0.000} & 0.2830~$\pm$~0.000 & 0.2142~$\pm$~0.000 \\
\midrule
\multirow{6}{*}{First} & \multirow{2}{*}{Chinchilla} & Collinear & 0.9823~$\pm$~0.000 & 0.5555~$\pm$~0.008 & 0.2704~$\pm$~0.001 & 0.4707~$\pm$~0.005 \\
 &  & Non-Collinear & 0.9750~$\pm$~0.000 & \win{0.8439~$\pm$~0.004} & 0.2367~$\pm$~0.001 & 0.2788~$\pm$~0.004 \\
\cmidrule{2-7}
 & \multirow{2}{*}{Droppo-Elibol} & Collinear & 0.9823~$\pm$~0.000 & 0.5638~$\pm$~0.037 & 0.2704~$\pm$~0.000 & 0.4637~$\pm$~0.018 \\
 &  & Non-Collinear & 0.9847~$\pm$~0.000 & \win{0.8925~$\pm$~0.001} & 0.1848~$\pm$~0.000 & 0.2315~$\pm$~0.001 \\
\cmidrule{2-7}
 & \multirow{2}{*}{Kaplan} & Collinear & 0.9786~$\pm$~0.001 & -0.0459~$\pm$~0.045 & 0.2963~$\pm$~0.009 & 0.7208~$\pm$~0.016 \\
 &  & Non-Collinear & 0.9550~$\pm$~0.000 & \win{0.8385~$\pm$~0.002} & 0.3173~$\pm$~0.000 & 0.2838~$\pm$~0.002 \\
\midrule
\multirow{6}{*}{Second} & \multirow{2}{*}{Chinchilla} & Collinear & 0.9906~$\pm$~0.000 & 0.6424~$\pm$~0.004 & 0.1691~$\pm$~0.000 & 0.3277~$\pm$~0.002 \\
 &  & Non-Collinear & 0.9461~$\pm$~0.000 & \win{0.8058~$\pm$~0.004} & 0.2640~$\pm$~0.000 & 0.2414~$\pm$~0.003 \\
\cmidrule{2-7}
 & \multirow{2}{*}{Droppo-Elibol} & Collinear & 0.9906~$\pm$~0.000 & 0.6117~$\pm$~0.044 & 0.1692~$\pm$~0.000 & 0.3375~$\pm$~0.019 \\
 &  & Non-Collinear & 0.9643~$\pm$~0.000 & \win{0.8386~$\pm$~0.002} & 0.2148~$\pm$~0.000 & 0.2202~$\pm$~0.001 \\
\cmidrule{2-7}
 & \multirow{2}{*}{Kaplan} & Collinear & 0.9843~$\pm$~0.000 & -0.0303~$\pm$~0.011 & 0.2193~$\pm$~0.000 & 0.5563~$\pm$~0.003 \\
 &  & Non-Collinear & 0.9314~$\pm$~0.001 & \win{0.8359~$\pm$~0.002} & 0.2977~$\pm$~0.002 & 0.2220~$\pm$~0.001 \\
\bottomrule
\end{tabular}
\end{table}

\begin{table}[H]
\small
\centering
\setlength{\tabcolsep}{2pt}
\caption{Wikipedia (BF16) - Unified holdout ($\mathcal{H}$). Training and holdout $R^2$ and RMSE per (law, epoch) and design: mean over $30$ optimizer seeds $\pm$ the half-width of its 95\% CI; highlighted: the design with the higher holdout $R^2$; repeated-data (Epoch -) is fit on all epochs at once.}
\renewcommand{\arraystretch}{1.2}
\begin{tabular}{ll l c c c c}
\toprule
\textbf{Epoch} & \textbf{Law} & \textbf{Design} & \textbf{R$^2$ Train} & \textbf{R$^2$ $\mathcal{H}$} & \textbf{RMSE Train} & \textbf{RMSE $\mathcal{H}$} \\
\midrule
\multirow{2}{*}{-} & \multirow{2}{*}{Repeated-data} & Collinear & 0.9882~$\pm$~0.000 & 0.9410~$\pm$~0.000 & 0.1990~$\pm$~0.000 & 0.2574~$\pm$~0.000 \\
 &  & Non-Collinear & 0.9856~$\pm$~0.000 & \win{0.9729~$\pm$~0.000} & 0.1824~$\pm$~0.000 & 0.1744~$\pm$~0.000 \\
\midrule
\multirow{6}{*}{Final} & \multirow{2}{*}{Chinchilla} & Collinear & 0.9916~$\pm$~0.000 & 0.9528~$\pm$~0.001 & 0.1601~$\pm$~0.000 & 0.2273~$\pm$~0.003 \\
 &  & Non-Collinear & 0.9915~$\pm$~0.000 & \win{0.9675~$\pm$~0.002} & 0.1334~$\pm$~0.000 & 0.1882~$\pm$~0.005 \\
\cmidrule{2-7}
 & \multirow{2}{*}{Droppo-Elibol} & Collinear & 0.9917~$\pm$~0.000 & 0.9559~$\pm$~0.001 & 0.1596~$\pm$~0.000 & 0.2199~$\pm$~0.002 \\
 &  & Non-Collinear & 0.9916~$\pm$~0.000 & \win{0.9664~$\pm$~0.001} & 0.1328~$\pm$~0.000 & 0.1918~$\pm$~0.002 \\
\cmidrule{2-7}
 & \multirow{2}{*}{Kaplan} & Collinear & 0.9915~$\pm$~0.000 & 0.9453~$\pm$~0.001 & 0.1612~$\pm$~0.000 & 0.2450~$\pm$~0.001 \\
 &  & Non-Collinear & 0.9898~$\pm$~0.000 & \win{0.9663~$\pm$~0.000} & 0.1458~$\pm$~0.000 & 0.1923~$\pm$~0.001 \\
\midrule
\multirow{6}{*}{First} & \multirow{2}{*}{Chinchilla} & Collinear & 0.9857~$\pm$~0.000 & 0.9061~$\pm$~0.003 & 0.2331~$\pm$~0.000 & 0.3303~$\pm$~0.006 \\
 &  & Non-Collinear & 0.9893~$\pm$~0.000 & \win{0.9765~$\pm$~0.002} & 0.1689~$\pm$~0.001 & 0.1647~$\pm$~0.006 \\
\cmidrule{2-7}
 & \multirow{2}{*}{Droppo-Elibol} & Collinear & 0.9858~$\pm$~0.000 & 0.9243~$\pm$~0.002 & 0.2324~$\pm$~0.000 & 0.2967~$\pm$~0.004 \\
 &  & Non-Collinear & 0.9923~$\pm$~0.000 & \win{0.9614~$\pm$~0.001} & 0.1438~$\pm$~0.000 & 0.2118~$\pm$~0.003 \\
\cmidrule{2-7}
 & \multirow{2}{*}{Kaplan} & Collinear & 0.9856~$\pm$~0.000 & 0.9249~$\pm$~0.001 & 0.2342~$\pm$~0.000 & 0.2957~$\pm$~0.002 \\
 &  & Non-Collinear & 0.9922~$\pm$~0.000 & \win{0.9525~$\pm$~0.000} & 0.1446~$\pm$~0.000 & 0.2351~$\pm$~0.001 \\
\midrule
\multirow{6}{*}{Second} & \multirow{2}{*}{Chinchilla} & Collinear & 0.9920~$\pm$~0.000 & 0.9550~$\pm$~0.001 & 0.1578~$\pm$~0.000 & 0.2217~$\pm$~0.003 \\
 &  & Non-Collinear & 0.9918~$\pm$~0.000 & \win{0.9698~$\pm$~0.002} & 0.1329~$\pm$~0.001 & 0.1813~$\pm$~0.006 \\
\cmidrule{2-7}
 & \multirow{2}{*}{Droppo-Elibol} & Collinear & 0.9920~$\pm$~0.000 & 0.9571~$\pm$~0.002 & 0.1574~$\pm$~0.000 & 0.2165~$\pm$~0.004 \\
 &  & Non-Collinear & 0.9921~$\pm$~0.000 & \win{0.9635~$\pm$~0.001} & 0.1301~$\pm$~0.000 & 0.1997~$\pm$~0.003 \\
\cmidrule{2-7}
 & \multirow{2}{*}{Kaplan} & Collinear & 0.9919~$\pm$~0.000 & 0.9494~$\pm$~0.001 & 0.1588~$\pm$~0.000 & 0.2354~$\pm$~0.001 \\
 &  & Non-Collinear & 0.9906~$\pm$~0.000 & \win{0.9600~$\pm$~0.000} & 0.1419~$\pm$~0.000 & 0.2091~$\pm$~0.001 \\
\bottomrule
\end{tabular}
\end{table}

\begin{table}[H]
\small
\centering
\setlength{\tabcolsep}{2pt}
\caption{Wikipedia (BF16) - Collinear holdout ($\mathcal{H}_{\mathrm{col}}$). Training and holdout $R^2$ and RMSE per (law, epoch) and design: mean over $30$ optimizer seeds $\pm$ the half-width of its 95\% CI; highlighted: the design with the higher holdout $R^2$; repeated-data (Epoch -) is fit on all epochs at once.}
\renewcommand{\arraystretch}{1.2}
\begin{tabular}{ll l c c c c}
\toprule
\textbf{Epoch} & \textbf{Law} & \textbf{Design} & \textbf{R$^2$ Train} & \textbf{R$^2$ $\mathcal{H}_{\mathrm{col}}$} & \textbf{RMSE Train} & \textbf{RMSE $\mathcal{H}_{\mathrm{col}}$} \\
\midrule
\multirow{2}{*}{-} & \multirow{2}{*}{Repeated-data} & Collinear & 0.9882~$\pm$~0.000 & 0.9639~$\pm$~0.000 & 0.1990~$\pm$~0.000 & 0.2244~$\pm$~0.000 \\
 &  & Non-Collinear & 0.9856~$\pm$~0.000 & \win{0.9752~$\pm$~0.000} & 0.1824~$\pm$~0.000 & 0.1861~$\pm$~0.000 \\
\midrule
\multirow{6}{*}{Final} & \multirow{2}{*}{Chinchilla} & Collinear & 0.9916~$\pm$~0.000 & 0.9699~$\pm$~0.001 & 0.1601~$\pm$~0.000 & 0.2046~$\pm$~0.002 \\
 &  & Non-Collinear & 0.9915~$\pm$~0.000 & \win{0.9743~$\pm$~0.001} & 0.1334~$\pm$~0.000 & 0.1887~$\pm$~0.004 \\
\cmidrule{2-7}
 & \multirow{2}{*}{Droppo-Elibol} & Collinear & 0.9917~$\pm$~0.000 & 0.9716~$\pm$~0.000 & 0.1596~$\pm$~0.000 & 0.1987~$\pm$~0.001 \\
 &  & Non-Collinear & 0.9916~$\pm$~0.000 & \win{0.9738~$\pm$~0.001} & 0.1328~$\pm$~0.000 & 0.1906~$\pm$~0.002 \\
\cmidrule{2-7}
 & \multirow{2}{*}{Kaplan} & Collinear & 0.9915~$\pm$~0.000 & \win{0.9735~$\pm$~0.000} & 0.1612~$\pm$~0.000 & 0.1918~$\pm$~0.001 \\
 &  & Non-Collinear & 0.9898~$\pm$~0.000 & 0.9719~$\pm$~0.000 & 0.1458~$\pm$~0.000 & 0.1978~$\pm$~0.001 \\
\midrule
\multirow{6}{*}{First} & \multirow{2}{*}{Chinchilla} & Collinear & 0.9857~$\pm$~0.000 & 0.9497~$\pm$~0.002 & 0.2331~$\pm$~0.000 & 0.2655~$\pm$~0.004 \\
 &  & Non-Collinear & 0.9893~$\pm$~0.000 & \win{0.9809~$\pm$~0.001} & 0.1689~$\pm$~0.001 & 0.1633~$\pm$~0.005 \\
\cmidrule{2-7}
 & \multirow{2}{*}{Droppo-Elibol} & Collinear & 0.9858~$\pm$~0.000 & 0.9526~$\pm$~0.001 & 0.2324~$\pm$~0.000 & 0.2579~$\pm$~0.002 \\
 &  & Non-Collinear & 0.9923~$\pm$~0.000 & \win{0.9750~$\pm$~0.001} & 0.1438~$\pm$~0.000 & 0.1871~$\pm$~0.003 \\
\cmidrule{2-7}
 & \multirow{2}{*}{Kaplan} & Collinear & 0.9856~$\pm$~0.000 & 0.9560~$\pm$~0.001 & 0.2342~$\pm$~0.000 & 0.2484~$\pm$~0.002 \\
 &  & Non-Collinear & 0.9922~$\pm$~0.000 & \win{0.9688~$\pm$~0.000} & 0.1446~$\pm$~0.000 & 0.2094~$\pm$~0.001 \\
\midrule
\multirow{6}{*}{Second} & \multirow{2}{*}{Chinchilla} & Collinear & 0.9920~$\pm$~0.000 & 0.9687~$\pm$~0.001 & 0.1578~$\pm$~0.000 & 0.2079~$\pm$~0.002 \\
 &  & Non-Collinear & 0.9918~$\pm$~0.000 & \win{0.9748~$\pm$~0.001} & 0.1329~$\pm$~0.001 & 0.1862~$\pm$~0.005 \\
\cmidrule{2-7}
 & \multirow{2}{*}{Droppo-Elibol} & Collinear & 0.9920~$\pm$~0.000 & 0.9704~$\pm$~0.001 & 0.1574~$\pm$~0.000 & 0.2023~$\pm$~0.002 \\
 &  & Non-Collinear & 0.9921~$\pm$~0.000 & \win{0.9712~$\pm$~0.001} & 0.1301~$\pm$~0.000 & 0.1996~$\pm$~0.002 \\
\cmidrule{2-7}
 & \multirow{2}{*}{Kaplan} & Collinear & 0.9919~$\pm$~0.000 & \win{0.9731~$\pm$~0.000} & 0.1588~$\pm$~0.000 & 0.1929~$\pm$~0.001 \\
 &  & Non-Collinear & 0.9906~$\pm$~0.000 & 0.9679~$\pm$~0.000 & 0.1419~$\pm$~0.000 & 0.2106~$\pm$~0.001 \\
\bottomrule
\end{tabular}
\end{table}

\begin{table}[H]
\small
\centering
\setlength{\tabcolsep}{2pt}
\caption{Wikipedia (BF16) - Non-collinear holdout ($\mathcal{H}_{\mathrm{nc}}$). Training and holdout $R^2$ and RMSE per (law, epoch) and design: mean over $30$ optimizer seeds $\pm$ the half-width of its 95\% CI; highlighted: the design with the higher holdout $R^2$; repeated-data (Epoch -) is fit on all epochs at once.}
\renewcommand{\arraystretch}{1.2}
\begin{tabular}{ll l c c c c}
\toprule
\textbf{Epoch} & \textbf{Law} & \textbf{Design} & \textbf{R$^2$ Train} & \textbf{R$^2$ $\mathcal{H}_{\mathrm{nc}}$} & \textbf{RMSE Train} & \textbf{RMSE $\mathcal{H}_{\mathrm{nc}}$} \\
\midrule
\multirow{2}{*}{-} & \multirow{2}{*}{Repeated-data} & Collinear & 0.9882~$\pm$~0.000 & 0.9157~$\pm$~0.000 & 0.1990~$\pm$~0.000 & 0.2772~$\pm$~0.000 \\
 &  & Non-Collinear & 0.9856~$\pm$~0.000 & \win{0.9697~$\pm$~0.000} & 0.1824~$\pm$~0.000 & 0.1662~$\pm$~0.000 \\
\midrule
\multirow{6}{*}{Final} & \multirow{2}{*}{Chinchilla} & Collinear & 0.9916~$\pm$~0.000 & 0.9330~$\pm$~0.002 & 0.1601~$\pm$~0.000 & 0.2413~$\pm$~0.004 \\
 &  & Non-Collinear & 0.9915~$\pm$~0.000 & \win{0.9592~$\pm$~0.002} & 0.1334~$\pm$~0.000 & 0.1879~$\pm$~0.006 \\
\cmidrule{2-7}
 & \multirow{2}{*}{Droppo-Elibol} & Collinear & 0.9917~$\pm$~0.000 & 0.9376~$\pm$~0.001 & 0.1596~$\pm$~0.000 & 0.2330~$\pm$~0.002 \\
 &  & Non-Collinear & 0.9916~$\pm$~0.000 & \win{0.9573~$\pm$~0.001} & 0.1328~$\pm$~0.000 & 0.1926~$\pm$~0.003 \\
\cmidrule{2-7}
 & \multirow{2}{*}{Kaplan} & Collinear & 0.9915~$\pm$~0.000 & 0.9133~$\pm$~0.001 & 0.1612~$\pm$~0.000 & 0.2747~$\pm$~0.002 \\
 &  & Non-Collinear & 0.9898~$\pm$~0.000 & \win{0.9591~$\pm$~0.001} & 0.1458~$\pm$~0.000 & 0.1886~$\pm$~0.001 \\
\midrule
\multirow{6}{*}{First} & \multirow{2}{*}{Chinchilla} & Collinear & 0.9857~$\pm$~0.000 & 0.8621~$\pm$~0.005 & 0.2331~$\pm$~0.000 & 0.3672~$\pm$~0.007 \\
 &  & Non-Collinear & 0.9893~$\pm$~0.000 & \win{0.9717~$\pm$~0.002} & 0.1689~$\pm$~0.001 & 0.1656~$\pm$~0.006 \\
\cmidrule{2-7}
 & \multirow{2}{*}{Droppo-Elibol} & Collinear & 0.9858~$\pm$~0.000 & 0.8954~$\pm$~0.004 & 0.2324~$\pm$~0.000 & 0.3200~$\pm$~0.005 \\
 &  & Non-Collinear & 0.9923~$\pm$~0.000 & \win{0.9475~$\pm$~0.001} & 0.1438~$\pm$~0.000 & 0.2267~$\pm$~0.003 \\
\cmidrule{2-7}
 & \multirow{2}{*}{Kaplan} & Collinear & 0.9856~$\pm$~0.000 & 0.8934~$\pm$~0.002 & 0.2342~$\pm$~0.000 & 0.3233~$\pm$~0.003 \\
 &  & Non-Collinear & 0.9922~$\pm$~0.000 & \win{0.9358~$\pm$~0.000} & 0.1446~$\pm$~0.000 & 0.2508~$\pm$~0.001 \\
\midrule
\multirow{6}{*}{Second} & \multirow{2}{*}{Chinchilla} & Collinear & 0.9920~$\pm$~0.000 & 0.9387~$\pm$~0.002 & 0.1578~$\pm$~0.000 & 0.2305~$\pm$~0.004 \\
 &  & Non-Collinear & 0.9918~$\pm$~0.000 & \win{0.9632~$\pm$~0.002} & 0.1329~$\pm$~0.001 & 0.1779~$\pm$~0.006 \\
\cmidrule{2-7}
 & \multirow{2}{*}{Droppo-Elibol} & Collinear & 0.9920~$\pm$~0.000 & 0.9412~$\pm$~0.003 & 0.1574~$\pm$~0.000 & 0.2254~$\pm$~0.006 \\
 &  & Non-Collinear & 0.9921~$\pm$~0.000 & \win{0.9539~$\pm$~0.001} & 0.1301~$\pm$~0.000 & 0.1998~$\pm$~0.003 \\
\cmidrule{2-7}
 & \multirow{2}{*}{Kaplan} & Collinear & 0.9919~$\pm$~0.000 & 0.9222~$\pm$~0.001 & 0.1588~$\pm$~0.000 & 0.2599~$\pm$~0.001 \\
 &  & Non-Collinear & 0.9906~$\pm$~0.000 & \win{0.9501~$\pm$~0.001} & 0.1419~$\pm$~0.000 & 0.2081~$\pm$~0.001 \\
\bottomrule
\end{tabular}
\end{table}

\clearpage

\section{Appendix - Fitted scaling law coefficients}\label{app:coefficients}

Tables~\ref{tab:chinchilla-params}-\ref{tab:epoch-params-r2} report the
best-fit parameters obtained by multi-start L-BFGS-B with differential
evolution polish, on seed 0, and
Tables~\ref{tab:chinchilla-params-bf16}-\ref{tab:epoch-params-bf16-r2} report them for the BF16 runs on Wikipedia.
$R^2$ is reported on four splits: training ($R^2_{\text{train}}$), unified holdout ($R^2_{\mathcal{H}}$),
collinear holdout ($R^2_{\mathcal{H}_{\mathrm{col}}}$), and non-collinear holdout ($R^2_{\mathcal{H}_{\mathrm{nc}}}$). Please refer to Appendix~\ref{app:optimizer_bounds} for the optimizer bounds we used. Both designs were fit with the same bounds.

\begin{table}[H]
\small
\centering
\setlength{\tabcolsep}{2pt}
\caption{Chinchilla fitted parameters: $L(N,D)= E + A\,N^{-\alpha} + B\,D^{-\beta}$. Final epoch.}
\label{tab:chinchilla-params}
\begin{tabular}{llccccccccc}
\toprule
\textbf{Dataset} & \textbf{Design} & $E$ & $A$ & $\alpha$ & $B$ & $\beta$ & $R^2_{\text{train}}$ & $R^2_{\mathcal{H}}$ & $R^2_{\mathcal{H}_{\mathrm{col}}}$ & $R^2_{\mathcal{H}_{\mathrm{nc}}}$ \\
\midrule
Wikipedia      & CO & $1.857$ & $4.26{\rm e}6$ & $0.908$ & $2.99{\rm e}4$ & $0.518$ & 0.992 & 0.864 & 0.966 & 0.633 \\
               & NC & $0$ & $239.156$ & $0.260$ & $5.02{\rm e}4$ & $0.552$ & 0.937 & 0.898 & 0.929 & 0.829 \\
\midrule
Cosmopedia     & CO & $1.367$ & $4249.535$ & $0.471$ & $8.53{\rm e}6$ & $0.874$ & 0.986 & 0.950 & 0.961 & 0.937 \\
               & NC & $1.174$ & $871.957$ & $0.370$ & $1.37{\rm e}7$ & $0.902$ & 0.950 & 0.952 & 0.956 & 0.947 \\
\midrule
RedPajama     & CO & $2.400$ & $1.87{\rm e}5$ & $0.706$ & $2.68{\rm e}4$ & $0.517$ & 0.993 & 0.880 & 0.967 & 0.692 \\
               & NC & $0.242$ & $118.949$ & $0.209$ & $5.02{\rm e}4$ & $0.557$ & 0.944 & 0.916 & 0.939 & 0.865 \\
\midrule
peS2o          & CO & $2.117$ & $1.66{\rm e}4$ & $0.566$ & $2.45{\rm e}5$ & $0.667$ & 0.983 & 0.896 & 0.924 & 0.857 \\
               & NC & $0$ & $68.799$ & $0.174$ & $2.73{\rm e}7$ & $0.951$ & 0.935 & 0.945 & 0.957 & 0.929 \\
\midrule
C4             & CO & $0.410$ & $47.594$ & $0.161$ & $3.59{\rm e}7$ & $0.971$ & 0.979 & 0.945 & 0.960 & 0.927 \\
               & NC & $1.131$ & $82.645$ & $0.216$ & $3.82{\rm e}4$ & $0.580$ & 0.934 & 0.954 & 0.963 & 0.944 \\
\bottomrule
\end{tabular}
\end{table}

\begin{table}[H]
\small
\centering
\setlength{\tabcolsep}{2pt}
\caption{Kaplan fitted parameters: $L(N,D)=\bigl[(N_c/N)^{\alpha_N/\alpha_D} + D_c/D\bigr]^{\alpha_D}$. Final epoch. $^{\dagger}$: $(N_c, D_c)$ is a start draw of the optimizer left unmoved; only the exponents moved (Appendix~\ref{app:kaplan_details}).}
\label{tab:kaplan-params}
\begin{tabular}{llccccccccc}
\toprule
\textbf{Dataset} & \textbf{Design} & $N_c$ & $D_c$ & $\alpha_N$ & $\alpha_D$ & $\alpha_N/\alpha_D$ & $R^2_{\text{train}}$ & $R^2_{\mathcal{H}}$ & $R^2_{\mathcal{H}_{\mathrm{col}}}$ & $R^2_{\mathcal{H}_{\mathrm{nc}}}$ \\
\midrule
Wikipedia      & CO & $6.78{\rm e}8$ & $6.90{\rm e}8$ & $0.382$ & $0.563$ & $0.679$ & 0.985 & 0.612 & 0.862 & 0.045 \\
               & NC & $1.11{\rm e}10$ & $2.11{\rm e}7$ & $0.208$ & $2.000$ & $0.104$ & 0.927 & 0.884 & 0.909 & 0.829 \\
\midrule
Cosmopedia     & CO & $1.17{\rm e}9$ & $1.15{\rm e}8$ & $0.274$ & $0.841$ & $0.325$ & 0.985 & 0.942 & 0.959 & 0.923 \\
               & NC$^{\dagger}$ & $2.53{\rm e}9$ & $6.15{\rm e}7$ & $0.227$ & $1.075$ & $0.211$ & 0.950 & 0.959 & 0.961 & 0.956 \\
\midrule
RedPajama     & CO & $3.41{\rm e}9$ & $10^{9}$ & $0.296$ & $0.512$ & $0.577$ & 0.987 & 0.666 & 0.888 & 0.181 \\
               & NC$^{\dagger}$ & $1.89{\rm e}11$ & $3.74{\rm e}7$ & $0.162$ & $1.411$ & $0.115$ & 0.920 & 0.875 & 0.904 & 0.811 \\
\midrule
peS2o          & CO & $5.32{\rm e}9$ & $2.10{\rm e}8$ & $0.236$ & $0.685$ & $0.345$ & 0.981 & 0.824 & 0.881 & 0.747 \\
               & NC & $4.67{\rm e}10$ & $2.76{\rm e}7$ & $0.169$ & $1.435$ & $0.118$ & 0.936 & 0.943 & 0.954 & 0.928 \\
\midrule
C4             & CO$^{\dagger}$ & $1.89{\rm e}11$ & $3.74{\rm e}7$ & $0.141$ & $1.192$ & $0.118$ & 0.980 & 0.943 & 0.956 & 0.927 \\
               & NC$^{\dagger}$ & $1.89{\rm e}11$ & $3.74{\rm e}7$ & $0.140$ & $1.100$ & $0.128$ & 0.928 & 0.940 & 0.950 & 0.929 \\
\bottomrule
\end{tabular}
\end{table}

\begin{table}[H]
\small
\centering
\setlength{\tabcolsep}{2pt}
\caption{Droppo-Elibol fitted parameters (law in Appendix~\ref{app:elibol_details}): irreducible loss $L_\infty$, scale parameters $(N_C, D_C)$ and exponents $\alpha_N, \alpha_D, \alpha$. Final epoch. $^{\dagger}$: $(N_C, D_C)$ is a start draw of the optimizer left unmoved (Appendix~\ref{app:optimizer_bounds}).}
\label{tab:droppo-params}
\begin{tabular}{llcccccccccc}
\toprule
\textbf{Dataset} & \textbf{Design} & $L_\infty$ & $N_C$ & $D_C$ & $\alpha_N$ & $\alpha_D$ & $\alpha$ & $R^2_{\text{train}}$ & $R^2_{\mathcal{H}}$ & $R^2_{\mathcal{H}_{\mathrm{col}}}$ & $R^2_{\mathcal{H}_{\mathrm{nc}}}$ \\
\midrule
Wikipedia      & CO & $2.094$ & $4.64{\rm e}7$ & $7.72{\rm e}8$ & $0.735$ & $0.505$ & $0.720$ & 0.992 & 0.810 & 0.967 & 0.456 \\
               & NC & $0.114$ & $1.36{\rm e}8$ & $2.57{\rm e}7$ & $0.298$ & $0.853$ & $2.000$ & 0.957 & 0.916 & 0.940 & 0.861 \\
\midrule
Cosmopedia     & CO$^{\dagger}$ & $0.509$ & $8.02{\rm e}7$ & $3.34{\rm e}7$ & $0.358$ & $1.095$ & $1.476$ & 0.986 & 0.960 & 0.966 & 0.952 \\
               & NC$^{\dagger}$ & $0.556$ & $8.02{\rm e}7$ & $3.34{\rm e}7$ & $0.329$ & $1.158$ & $1.481$ & 0.954 & 0.963 & 0.966 & 0.959 \\
\midrule
RedPajama     & CO$^{\dagger}$ & $1.995$ & $3.69{\rm e}6$ & $2.61{\rm e}7$ & $1.016$ & $0.675$ & $1.918$ & 0.993 & 0.915 & 0.963 & 0.811 \\
               & NC & $0.054$ & $1.61{\rm e}9$ & $1.97{\rm e}7$ & $0.215$ & $0.776$ & $2.000$ & 0.959 & 0.918 & 0.939 & 0.875 \\
\midrule
peS2o          & CO$^{\dagger}$ & $1.047$ & $8.02{\rm e}7$ & $3.34{\rm e}7$ & $0.363$ & $0.983$ & $1.393$ & 0.983 & 0.894 & 0.916 & 0.863 \\
               & NC$^{\dagger}$ & $0.534$ & $1.10{\rm e}8$ & $1.52{\rm e}7$ & $0.264$ & $1.312$ & $1.833$ & 0.943 & 0.954 & 0.964 & 0.940 \\
\midrule
C4             & CO$^{\dagger}$ & $0.468$ & $3.48{\rm e}8$ & $1.49{\rm e}7$ & $0.205$ & $1.638$ & $1.794$ & 0.981 & 0.945 & 0.955 & 0.933 \\
               & NC$^{\dagger}$ & $0.050$ & $1.23{\rm e}10$ & $7.80{\rm e}6$ & $0.153$ & $1.266$ & $1.984$ & 0.940 & 0.954 & 0.963 & 0.944 \\
\bottomrule
\end{tabular}
\end{table}

\begin{table}[H]
\small
\centering
\caption{Repeated-data fitted parameters $A, \alpha, B, \beta, E$ (law in Appendix~\ref{app:dataconstrained_details}). Final epoch.}
\label{tab:epoch-params-fit}
\begin{tabular}{llccccc}
\toprule
\textbf{Dataset} & \textbf{Design} & $A$ & $\alpha$ & $B$ & $\beta$ & $E$ \\
\midrule
Wikipedia  & CO & $4.22{\rm e}5$ & $0.746$ & $2.19{\rm e}4$ & $0.515$ & $1.538$ \\
           & NC & $194.407$ & $0.244$ & $9.09{\rm e}4$ & $0.601$ & $1.86{\rm e}{-}11$ \\
\midrule
Cosmopedia & CO & $7.22{\rm e}4$ & $0.656$ & $1.45{\rm e}6$ & $0.787$ & $1.752$ \\
           & NC & $303.693$ & $0.290$ & $1.46{\rm e}6$ & $0.790$ & $0.618$ \\
\midrule
RedPajama & CO & $1.15{\rm e}5$ & $0.667$ & $1.26{\rm e}4$ & $0.485$ & $2.164$ \\
           & NC & $97.241$ & $0.192$ & $4.86{\rm e}4$ & $0.566$ & $5.23{\rm e}{-}12$ \\
\midrule
peS2o      & CO & $3.07{\rm e}4$ & $0.600$ & $4.75{\rm e}5$ & $0.723$ & $2.265$ \\
           & NC & $70.386$ & $0.177$ & $8.92{\rm e}5$ & $0.762$ & $2.25{\rm e}{-}11$ \\
\midrule
C4         & CO & $1.68{\rm e}6$ & $0.879$ & $1.84{\rm e}7$ & $0.963$ & $3.088$ \\
           & NC & $34.408$ & $0.136$ & $1.82{\rm e}5$ & $0.684$ & $2.94{\rm e}{-}11$ \\
\bottomrule
\end{tabular}
\end{table}

\begin{table}[H]
\small
\centering
\caption{Repeated-data fit quality and decay constants $R^*_D, R^*_N$. Final epoch. Values reported as $\RhalfLow$ for $R^*_N$ indicate the L-BFGS-B optimizer terminating at the documented lower bound $R^*_N \in [\RhalfLow,\,\RhalfHigh]$; see Appendix~\ref{app:optimizer_bounds} for the full bound list and an explanation of why $R^*_N$ is not informatively constrained by our grid (epochs vary $D$ but not $N$).}
\label{tab:epoch-params-r2}
\begin{tabular}{llcccccc}
\toprule
\textbf{Dataset} & \textbf{Design} & $R^*_D$ & $R^*_N$ & $R^2_{\text{train}}$ & $R^2_{\mathcal{H}}$ & $R^2_{\mathcal{H}_{\mathrm{col}}}$ & $R^2_{\mathcal{H}_{\mathrm{nc}}}$ \\
\midrule
Wikipedia  & CO & $1.074$ & $0.100$ & 0.985 & 0.867 & 0.952 & 0.691 \\
           & NC & $1.469$ & $0.100$ & 0.959 & 0.915 & 0.932 & 0.882 \\
\midrule
Cosmopedia & CO & $1.500$ & $0.153$ & 0.984 & 0.950 & 0.965 & 0.933 \\
           & NC & $1.286$ & $0.100$ & 0.966 & 0.968 & 0.972 & 0.965 \\
\midrule
RedPajama & CO & $1.162$ & $0.100$ & 0.988 & 0.890 & 0.960 & 0.756 \\
           & NC & $1.624$ & $0.100$ & 0.964 & 0.929 & 0.943 & 0.902 \\
\midrule
peS2o      & CO & $1.897$ & $0.119$ & 0.986 & 0.923 & 0.950 & 0.891 \\
           & NC & $2.188$ & $0.100$ & 0.962 & 0.946 & 0.951 & 0.941 \\
\midrule
C4         & CO & $0.801$ & $0.564$ & 0.978 & 0.818 & 0.857 & 0.775 \\
           & NC & $2.352$ & $0.100$ & 0.936 & 0.957 & 0.959 & 0.955 \\
\bottomrule
\end{tabular}
\end{table}

\begin{table}[H]
\small
\centering
\setlength{\tabcolsep}{2pt}
\caption{BF16 Chinchilla fitted parameters: $L(N,D)= E + A\,N^{-\alpha} + B\,D^{-\beta}$. Final epoch.}
\label{tab:chinchilla-params-bf16}
\begin{tabular}{llccccccccc}
\toprule
\textbf{Dataset} & \textbf{Design} & $E$ & $A$ & $\alpha$ & $B$ & $\beta$ & $R^2_{\text{train}}$ & $R^2_{\mathcal{H}}$ & $R^2_{\mathcal{H}_{\mathrm{col}}}$ & $R^2_{\mathcal{H}_{\mathrm{nc}}}$ \\
\midrule
Wikipedia (BF16) & CO & $0$ & $3428.447$ & $0.399$ & $2366.171$ & $0.392$ & 0.992 & 0.953 & 0.970 & 0.933 \\
               & NC & $0.758$ & $4904.551$ & $0.421$ & $4.82{\rm e}5$ & $0.713$ & 0.991 & 0.978 & 0.981 & 0.973 \\
\bottomrule
\end{tabular}
\end{table}

\begin{table}[H]
\small
\centering
\setlength{\tabcolsep}{2pt}
\caption{BF16 Kaplan fitted parameters: $L(N,D)=\bigl[(N_c/N)^{\alpha_N/\alpha_D} + D_c/D\bigr]^{\alpha_D}$. Final epoch.}
\label{tab:kaplan-params-bf16}
\begin{tabular}{llccccccccc}
\toprule
\textbf{Dataset} & \textbf{Design} & $N_c$ & $D_c$ & $\alpha_N$ & $\alpha_D$ & $\alpha_N/\alpha_D$ & $R^2_{\text{train}}$ & $R^2_{\mathcal{H}}$ & $R^2_{\mathcal{H}_{\mathrm{col}}}$ & $R^2_{\mathcal{H}_{\mathrm{nc}}}$ \\
\midrule
Wikipedia (BF16) & CO & $1.19{\rm e}9$ & $2.84{\rm e}9$ & $0.397$ & $0.394$ & $1.008$ & 0.992 & 0.946 & 0.974 & 0.914 \\
               & NC & $1.98{\rm e}9$ & $1.12{\rm e}8$ & $0.354$ & $0.823$ & $0.430$ & 0.990 & 0.966 & 0.972 & 0.959 \\
\bottomrule
\end{tabular}
\end{table}

\begin{table}[H]
\small
\centering
\setlength{\tabcolsep}{2pt}
\caption{BF16 Droppo-Elibol fitted parameters (law in Appendix~\ref{app:elibol_details}): irreducible loss $L_\infty$, scale parameters $(N_C, D_C)$ and exponents $\alpha_N, \alpha_D, \alpha$. Final epoch.}
\label{tab:droppo-params-bf16}
\begin{tabular}{llcccccccccc}
\toprule
\textbf{Dataset} & \textbf{Design} & $L_\infty$ & $N_C$ & $D_C$ & $\alpha_N$ & $\alpha_D$ & $\alpha$ & $R^2_{\text{train}}$ & $R^2_{\mathcal{H}}$ & $R^2_{\mathcal{H}_{\mathrm{col}}}$ & $R^2_{\mathcal{H}_{\mathrm{nc}}}$ \\
\midrule
Wikipedia (BF16) & CO & $10^{-6}$ & $9.97{\rm e}8$ & $1.67{\rm e}9$ & $0.400$ & $0.388$ & $0.583$ & 0.992 & 0.956 & 0.972 & 0.938 \\
               & NC & $0.224$ & $1.21{\rm e}9$ & $1.65{\rm e}8$ & $0.379$ & $0.612$ & $0.882$ & 0.992 & 0.965 & 0.973 & 0.955 \\
\bottomrule
\end{tabular}
\end{table}

\begin{table}[H]
\small
\centering
\caption{BF16 repeated-data fitted parameters $A, \alpha, B, \beta, E$ (law in Appendix~\ref{app:dataconstrained_details}). Final epoch.}
\label{tab:epoch-params-bf16-fit}
\begin{tabular}{llccccc}
\toprule
\textbf{Dataset} & \textbf{Design} & $A$ & $\alpha$ & $B$ & $\beta$ & $E$ \\
\midrule
Wikipedia (BF16) & CO & $3125.835$ & $0.392$ & $3300.851$ & $0.427$ & $1.29{\rm e}{-}12$ \\
               & NC & $3045.878$ & $0.387$ & $3.49{\rm e}7$ & $0.995$ & $0.612$ \\
\bottomrule
\end{tabular}
\end{table}

\begin{table}[H]
\small
\centering
\caption{BF16 repeated-data fit quality and decay constants $R^*_D, R^*_N$. Final epoch. Values reported as $\RhalfLow$ for $R^*_N$ indicate the L-BFGS-B optimizer terminating at the documented lower bound $R^*_N \in [\RhalfLow,\,\RhalfHigh]$; see Appendix~\ref{app:optimizer_bounds}.}
\label{tab:epoch-params-bf16-r2}
\begin{tabular}{llcccccc}
\toprule
\textbf{Dataset} & \textbf{Design} & $R^*_D$ & $R^*_N$ & $R^2_{\text{train}}$ & $R^2_{\mathcal{H}}$ & $R^2_{\mathcal{H}_{\mathrm{col}}}$ & $R^2_{\mathcal{H}_{\mathrm{nc}}}$ \\
\midrule
Wikipedia (BF16) & CO & $0.283$ & $0.100$ & 0.988 & 0.941 & 0.964 & 0.916 \\
               & NC & $0.723$ & $0.100$ & 0.986 & 0.973 & 0.975 & 0.970 \\
\bottomrule
\end{tabular}
\end{table}

\clearpage

\clearpage
\section*{NeurIPS Paper Checklist}

\begin{enumerate}

\item {\bf Claims}
    \item[] Question: Do the main claims made in the abstract and introduction accurately reflect the paper's contributions and scope?
    \item[] Answer: \answerYes{}
    \item[] Justification: The abstract and introduction state five claims: a TPP-diversity threshold for multi-ray designs (Proposition~\ref{prop:holdout_r2}), ill-conditioning in its one-ray limit (Proposition~\ref{prop:full_cond}), CI inflation (Corollary~\ref{prop:ci_inflation}), NC-vs-CO holdout ordering (Theorem~\ref{thm:holdout_regimes} for one ray, Lemma~\ref{lem:regime_a_margin} a sufficient margin for several; Tables~\ref{tab:summary_full}-\ref{tab:regime_a_winrate} for multi-ray designs), and Jacobian-geometry generality (Appendix~\ref{app:proof_full_cond}) - each backed by a formal result or empirical table; they are the main contributions of the paper. The introduction states the evidence scale ($5$M-$77$M parameters), and the criterion is also checked on published designs (Appendix~\ref{app:published_designs}) and public training-loss releases (Appendix~\ref{app:public_checkpoints}).

\item {\bf Limitations}
    \item[] Question: Does the paper discuss the limitations of the work performed by the authors?
     \item[] Answer: \answerYes{}
    \item[] Justification: The Limitations paragraph of Section~\ref{sec:discussion} gives each limit with its reason: Theorem~\ref{thm:holdout_regimes} orders one ray and gives no gap size, so the multi-ray ordering is measured; the models are small ($5.04$-$76.5$M), and while the rates and the leading-order threshold have no size term, the exact $\kappa_{A,B}$ of a design and the measured gap depend on size, and public training-loss releases do not establish the gap at 1B and above; Corollary~\ref{prop:ci_inflation} assumes i.i.d.\ Gaussian residuals, and correlated noise with covariance $\Sigma$ moves its ratio by at most the square root of $\kappa(\Sigma)$, not its rate (Appendix~\ref{app:proof_ci_inflation}); one architecture, English text, one fixed recipe and pre-training loss bound the breadth, with downstream and cross-domain transfer not validated; the budget-matched advantage softens on Cosmopedia and peS2o, not the full-fan one; no isoFLOP-layout arm and no design at the recipe's spread were run. The text after Theorem~\ref{thm:holdout_regimes} and Appendix~\ref{app:proof_holdout_r2} state that its Regime~B is a limit at a fixed design: at the true gap a design with $K \geq 2$ rays is protected only by its weighted spread, which collinear designs keep at least $V_K \min_i N_i^{2(\alpha - \beta)}$ but unequal sizes per ray can bring to $0$. Section~\ref{sec:gnsetup} states that the leading-order threshold is a screen, the exact $\kappa_{A,B}$ of a planned design deciding. The introduction's contributions state what is proved, with size-free rates, and what is measured at 5M-77M.

\item {\bf Theory assumptions and proofs}
    \item[] Question: For each theoretical result, does the paper provide the full set of assumptions and a complete (and correct) proof?
    \item[] Answer: \answerYes{}
    \item[] Justification: All assumptions are stated in each theorem environment. Complete proofs are in Appendix~\ref{app:proofs}: Proposition~\ref{prop:holdout_r2} and Corollary~\ref{cor:ratio_cases} in~\ref{app:proof_prop_holdout_r2}, Proposition~\ref{prop:full_cond} in~\ref{app:proof_full_cond}, Corollary~\ref{prop:ci_inflation} in~\ref{app:proof_ci_inflation}, Theorem~\ref{thm:holdout_regimes} in~\ref{app:proof_holdout_r2} (with Lemmas~\ref{lem:weighted_spread}-\ref{lem:ray_spread} in~\ref{app:proof_prop_holdout_r2}, and Lemmas~\ref{lem:exp_poly_zeros}-\ref{lem:full_rank_rays} (full rank of the Chinchilla Jacobian from the design) and~\ref{lem:regime_a_margin} in~\ref{app:proof_holdout_r2}). Proposition~\ref{prop:rmse_isoflop_invariance} is proved in~\ref{app:proof_rmse_invariance}, and Proposition~\ref{prop:reduced_id} in~\ref{app:submodel_proofs}.

    \item {\bf Experimental result reproducibility}
    \item[] Question: Does the paper fully disclose all the information needed to reproduce the main experimental results of the paper to the extent that it affects the main claims and/or conclusions of the paper (regardless of whether the code and data are provided or not)?
   \item[] Answer: \answerYes{}
    \item[] Justification: All architectures, hyperparameters, data splits, fitting protocols, and seed strategies are specified in Section~\ref{sec:empirical} and Appendices~\ref{app:experimental_setup}-\ref{app:seed_protocol}, with fitted coefficients in Appendix~\ref{app:coefficients}. The released code (\url{https://github.com/lab-v2/tpp-coverage}) has instructions to reproduce all tables either from pre-computed metrics (CPU only) or from scratch (${\sim}3{,}624$ GPU-hours).
    

\item {\bf Open access to data and code}
    \item[] Question: Does the paper provide open access to the data and code, with sufficient instructions to faithfully reproduce the main experimental results, as described in supplemental material?
    \item[] Answer: \answerYes{}
    \item[] Justification: Code: \url{https://github.com/lab-v2/tpp-coverage}. Trained LLMs checkpoints and metrics: \url{https://huggingface.co/datasets/leibnitz-lab/colinear_scaling_models}. The run tables of the public-release reanalysis: \url{https://huggingface.co/datasets/leibnitz-lab/public_checkpoint_reanalysis}. All five training corpora are publicly available under open licenses or terms of use (Appendix~\ref{app:experimental_setup}).

\item {\bf Experimental setting/details}
    \item[] Question: Does the paper specify all the training and test details (e.g., data splits, hyperparameters, how they were chosen, type of optimizer) necessary to understand the results?
    \item[] Answer: \answerYes{}
        \item[] Justification: All $14$ architectures (Table~\ref{tab:model-architectures}), training hyperparameters (Table~\ref{tab:hyperparams}), per-dataset learning rates, data-split procedures, CO/NC grid definitions, holdout partitions, fitting protocol (restart counts, seed strategy), and compute infrastructure are specified in Section~\ref{sec:empirical} and Appendix~\ref{app:experimental_setup}.

\item {\bf Experiment statistical significance}
    \item[] Question: Does the paper report error bars suitably and correctly defined or other appropriate information about the statistical significance of the experiments?
    \item[] Answer: \answerYes{}
    \item[] Justification: All summary results report across-combination standard deviations and $95\%$ optimizer-noise intervals over $30$ optimizer seeds (Tables~\ref{tab:summary_full} and~\ref{tab:summary_bf16}; Appendix~\ref{app:CI}), the overall win rate includes a $95\%$ Wilson CI over \winrateDenom seed-paired comparisons, which share one training run per $(N, D)$ and so measure fitting variability only (NC has the lower mean holdout RMSE in $49$ of the $50$ cells), and the budget-matched subset enumeration (Table~\ref{tab:regime_a_winrate}) reports bootstrap CIs across $22$ seeds. A pure-BF16 replication on Wikipedia (Table~\ref{tab:summary_bf16} in the main text and Tables~\ref{tab:winrate_bf16}-\ref{tab:winrate_appendix_bf16} in Appendix~\ref{app:bf16_results}) independently confirms the same patterns with seed-paired NC win rate \winratebf (\winratebfNumer/\winratebfDenom). The public-release reanalysis reports $95\%$ bootstrap intervals clustered by size window over $512$ seeds (Appendix~\ref{app:public_checkpoints}).

\item {\bf Experiments compute resources}
    \item[] Question: For each experiment, does the paper provide sufficient information on the computer resources (type of compute workers, memory, time of execution) needed to reproduce the experiments?
    \item[] Answer: \answerYes{}
    \item[] Justification: Appendix~\ref{hardware} reports
    worker type, memory, runtime, and storage for each
    experiment family: ${\sim}3{,}624$ GPU-hours for
    pre-training, ${\sim}5$ wall-clock hours (230 workers)
    for fitting/visualization, and ${\sim}24$ hours
    (480 workers) for subset enumeration, plus
    ${\sim}2.2$\,TB total storage. We also disclose that
    full-project compute exceeded the final reported
    experiment runs due to pilot and failed sweeps, which
    these totals do not include (Appendix~\ref{hardware}).
    
\item {\bf Code of ethics}
    \item[] Question: Does the research conducted in the paper conform, in every respect, with the NeurIPS Code of Ethics \url{https://neurips.cc/public/EthicsGuidelines}?
    \item[] Answer: \answerYes{}
    \item[] Justification: The work uses only publicly available corpora, under open licenses or terms of use (Appendix~\ref{app:experimental_setup}), involves no human subjects, and releases small-scale checkpoints ($5.04$-$76.5$M parameters) that pose no dual-use risk. See Section~\ref{sec:safeguards} for a detailed assessment.

\item {\bf Broader impacts}
    \item[] Question: Does the paper discuss both potential positive societal impacts and negative societal impacts of the work performed?
    \item[] Answer: \answerYes{}
    \item[] Justification: Section~\ref{sec:broader_impact} discusses the corrective intent of the work (deflating overconfident scaling-law extrapolations) and Section~\ref{sec:safeguards} addresses the negligible misuse risk of the released checkpoints and the marginal indirect risk of accelerating capability development.
    
\item {\bf Safeguards}
    \item[] Question: Does the paper describe safeguards that have been put in place for responsible release of data or models that have a high risk for misuse (e.g., pre-trained language models, image generators, or scraped datasets)?
    \item[] Answer: \answerYes{}
    \item[] Justification: Section~\ref{sec:safeguards} describes the safeguards: all corpora are publicly available and pre-filtered, released checkpoints are too small for dual-use capabilities, and the accompanying code targets fitting diagnostics rather than model serving.

\item {\bf Licenses for existing assets}
    \item[] Question: Are the creators or original owners of assets (e.g., code, data, models), used in the paper, properly credited and are the license and terms of use explicitly mentioned and properly respected?
    \item[] Answer: \answerYes{}
    \item[] Justification: Every asset is credited and its license or terms are stated. The card of our reanalysis dataset gives each source's license and ships the Apache-2.0, MIT and CC-BY-4.0 texts; its history is squashed, and no revision holds Farseer's raw run table, whose repository states no license. The five corpora are cited in Appendix~\ref{app:dataset_selection} and the four scaling-law formalisms in Table~\ref{tab:jacobian_summary}. The LLaMA architecture is credited to~\citet{touvron2023llama}. Appendix~\ref{app:experimental_setup} gives each training corpus's license or terms of use, read from its dataset card: Cosmopedia Apache-2.0; peS2o ODC-By; C4 ODC-By, with the Common Crawl Terms of Use for its content; Wikipedia CC BY-SA 3.0 and GFDL; RedPajama-V2 no license of its own, the Common Crawl Foundation Terms of Use. The ten public training-loss releases of Appendix~\ref{app:public_checkpoints} are cited in Table~\ref{tab:public_checkpoints_studies}, with their sources and licenses in that appendix:
    \begin{itemize}[nosep, leftmargin=*]
        \item the Gemstones releases: MIT;
        \item the over-training suite of \citet{gadre2024language} and MAD: MIT, both read through ColPret \citep{choshen2024hitchhiker}, also MIT;
        \item the OLMo ladder and DataDecide: Apache-2.0;
        \item the Chinchilla isoFLOP points: the digitization of \citet{besiroglu2024chinchilla}, CC-BY-4.0, read through ColPret;
        \item the GPT-3 curves: ColPret's digitization of a figure of \citet{brown2020language}, MIT;
        \item the OPT trajectories: ColPret (MIT), which takes the losses from \citet{xia2023training}, whose repository states no license;
        \item Farseer: its repository states no license; our dataset's current version restricts its 404 runs to $N$, $D$ and three loss columns, cited to \citet{li2025farseer}, and the code downloads the original; the dataset's history is squashed, so no revision holds the raw table;
        \item the 1{,}462 single-epoch runs of \citet{muennighoff2023scaling} in our run tables, which the screen rejects: Apache-2.0, read through ColPret (MIT);
        \item DataDecide's per-group perplexity table, which our dataset copies unchanged: Apache-2.0.
    \end{itemize}
    The classification of published designs (Appendix~\ref{app:published_designs}) reads only training ratios; where it reads them from a released file, that file comes from Farseer (no license stated), Epoch AI (CC-BY-4.0), Muennighoff et al.\ (Apache-2.0) or ColPret (MIT; its Pythia schedule comes from Pythia, under Apache-2.0).

\item {\bf New assets}
    \item[] Question: Are new assets introduced in the paper well documented and is the documentation provided alongside the assets?
    \item[] Answer: \answerYes{}
    \item[] Justification: We release trained model checkpoints on HuggingFace (\url{https://huggingface.co/datasets/leibnitz-lab/colinear_scaling_models}) with accompanying documentation describing model configurations, training corpora, and hyperparameters. Experimental code for reproducing scaling law fits and analyses is released at \url{https://github.com/lab-v2/tpp-coverage}, including the published-design classification and the public-release reanalysis (\texttt{camera\_ready/public\_checkpoints/}), whose run tables are at \url{https://huggingface.co/datasets/leibnitz-lab/public_checkpoint_reanalysis}. The training data of our models (RedPajama-V2, Cosmopedia, Wikipedia, peS2o, C4) are publicly available under open licenses or terms of use (Appendix~\ref{app:experimental_setup}). The reanalysis dataset holds the rows of the ten public releases and the 1{,}462 single-epoch runs of \citet{muennighoff2023scaling} that the screen rejects, each with $N$, $D$, the released loss columns and $k$, and an unchanged copy of DataDecide's per-group perplexity table; Appendix~\ref{app:public_checkpoints} lists each source and its license. Of Farseer's 404 runs, whose repository states no license, the dataset keeps only $N$, $D$ and three loss columns, cited to \citet{li2025farseer}. The dataset card gives each source's license and releases our normalisation under CC-BY-4.0; the code and the checkpoints are under GPL-2.0.

\item {\bf Crowdsourcing and research with human subjects}
    \item[] Question: For crowdsourcing experiments and research with human subjects, does the paper include the full text of instructions given to participants and screenshots, if applicable, as well as details about compensation (if any)? 
    \item[] Answer: \answerNA{}
    \item[] Justification: This work does not involve crowdsourcing or human subjects.

\item {\bf Institutional review board (IRB) approvals or equivalent for research with human subjects}
    \item[] Question: Does the paper describe potential risks incurred by study participants, whether such risks were disclosed to the subjects, and whether Institutional Review Board (IRB) approvals (or an equivalent approval/review based on the requirements of your country or institution) were obtained?
    \item[] Answer: \answerNA{}
    \item[] Justification: This paper does not involve crowdsourcing nor research with human subjects.

\item {\bf Declaration of LLM usage}
    \item[] Question: Does the paper describe the usage of LLMs if it is an important, original, or non-standard component of the core methods in this research? Note that if the LLM is used only for writing, editing, or formatting purposes and does \emph{not} impact the core methodology, scientific rigor, or originality of the research, declaration is not required.
    \item[] Answer: \answerYes{}
    \item[] Justification:
    LLM usage is declared in
    the LLM usage paragraph of
    Section~\ref{sec:discussion}. In
    short: all research questions,
    theoretical results, experimental
    design, and empirical analyses
    originated with the authors. The
    primary non-trivial use of LLMs was
    as a sounding board for candidate
    proof steps - the authors drafted
    the structure of each derivation,
    asked an LLM to propose intermediate
    algebraic manipulations, and
    independently verified, rewrote, or
    discarded each suggestion before
    incorporation. Auxiliary uses
    included exposition, \LaTeX{}
    typesetting, boilerplate code for
    data loading, training, analysis,
    and visualization, and prose
    editing - all reviewed by the
    authors. We used Anthropic's Claude
    Opus 4.7, 5 and 5.5 and the Cursor
    coding agent.

\end{enumerate}

\end{document}